\documentclass[twoside,11pt]{article}
\usepackage{journal}
\usepackage{natbib}
\usepackage{epsfig,ifthen}
\usepackage{latexsym}
\usepackage{amsfonts}
\usepackage{eepic}
\usepackage{amsopn}
\usepackage{amsmath,amssymb,rotating,graphicx,rotating,float}
\usepackage{accents}
\usepackage{stmaryrd}
\usepackage[backref,pageanchor=true,plainpages=false, 
pdfpagelabels,bookmarks,bookmarksnumbered,breaklinks,
pdfborder={0 0 0},  
]{hyperref}
\usepackage{cite}
\usepackage[mathscr]{eucal}
\usepackage{url} 

\usepackage{ulem}
\usepackage{xcolor} 
\usepackage{todonotes}

\usepackage{mathtools}
\usepackage[font=small,margin=0.25in,labelfont={sc},labelsep={colon}]{caption}
\usepackage{tikz}
\usepackage{pgfplots}
\pgfplotsset{compat=newest}
\usepackage{graphicx}
\usepackage{xspace}
\usepackage[shortlabels]{enumitem}

\renewcommand{\cite}{\citep*}

\usepackage{bm} 

\colorlet{sgreen}{black!45!green}
\newcommand{\stilltodo}[1]{}

\newsavebox{\savepar}
\newenvironment{bigboxit}{\begin{center}\begin{lrbox}{\savepar}
\begin{minipage}[h]{5.5in}
\normalfont
\begin{flushleft}}
{\end{flushleft}\end{minipage}\end{lrbox}\fbox{\usebox{\savepar}}
\end{center}}

\newcommand{\argmin}{\mathop{\rm argmin}}
\renewcommand{\limsup}{\mathop{\rm limsup}}

\newcommand{\bpi}{\boldsymbol{\pi}}
\newcommand{\Sym}{\mathrm{Sym}}
\newcommand{\bybranch}{\mathbf{y}^{\star}}
\newcommand{\ybranch}{y^{\star}}
\newcommand{\SOA}{\mathrm{SOA}}
\newcommand{\VCL}{\mathrm{VCL}}
\newcommand{\VCLSOA}{\mathrm{SOA}^{\VCL}}
\newcommand{\algSOA}{\alg_{\SOA}}
\newcommand{\algVCLSOA}{\alg_{\VCL}}
\newcommand{\TERM}{\mathrm{TERM}} 

\newcommand{\HasOPT}{\mathrm{HasOPT}}
\newcommand{\OPTY}{\mathrm{OPTY}}
\newcommand{\OPTI}{\mathrm{OPTI}}
\newcommand{\GoodI}{\mathrm{GoodI}}
\newcommand{\PrbGoodB}{\mathrm{PrbGoodB}}
\newcommand{\Gap}{\mathrm{Gap}}
\newcommand{\PrbBadB}{\mathrm{PrbBadB}}
\newcommand{\FarFromOPTI}{\mathrm{FarFromOPTI}}
\newcommand{\F}{\mathcal{F}}
\newcommand{\G}{\mathbb{G}}
\newcommand{\Px}{P_{X}}
\newcommand{\vc}{\mathsf{VC}}
\newcommand{\VC}{\mathsf{VC}}
\newcommand{\Alg}{\mathbb{A}}
\newcommand{\alg}{\Alg}

\newcommand{\target}{h^{\star}}
\newcommand{\er}{\mathrm{er}}

\newcommand{\E}{\mathbf{E}}

\renewcommand{\P}{\mathbf{P}}

\newcommand{\bemph}[1]{\textbf{#1}}

\newcommand{\X}{\mathcal{X}}
\newcommand{\cH}{\mathcal{H}}
\renewcommand{\H}{\cH}
\DeclareSymbolFont{bbold}{U}{bbold}{m}{n}
\DeclareSymbolFontAlphabet{\mathbbold}{bbold}
\newcommand{\ind}{\mathbbold{1}}
\newcommand{\reals}{\mathbb{R}}
\newcommand{\PXY}{P}
\newcommand{\nats}{\mathbb{N}}

\newcommand{\ber}{\overline{\er}}
\newcommand{\Czero}{c_{\scriptscriptstyle 0}}
\newcommand{\Cone}{c_{\scriptscriptstyle 1}}
\newcommand{\Ctwo}{c_{\scriptscriptstyle 2}}
\newcommand{\Cthree}{c_{\scriptscriptstyle 3}}
\newcommand{\Cfour}{c_{\scriptscriptstyle 4}}
\newcommand{\Cfive}{c_{\scriptscriptstyle 5}}
\newcommand{\Csix}{c_{\scriptscriptstyle 6}}
\newcommand{\Cseven}{c_{\scriptscriptstyle 7}}
\newcommand{\HKCzero}{C}

\renewenvironment{proof}[1][]{\par\noindent{\bf Proof #1\ }}{\hfill\BlackBox\\[2mm]}

\jmlrheading{}{}{}{}{}{}

\ShortHeadings{}{}
\firstpageno{1}

\begin{document}

\title{A Theory of Universal Agnostic Learning}


\author{%
\name Steve Hanneke \email steve.hanneke@gmail.com\\
\addr Purdue University\\
\name Shay Moran \email smoran@technion.ac.il\\
\addr Technion and Google Research
}

\editor{}

\maketitle

\begin{abstract}
We provide a complete theory of optimal universal rates for binary classification in the agnostic setting.
This extends the realizable-case theory of 
Bousquet, Hanneke, Moran, van Handel, and Yehudayoff (2021)
by removing the realizability assumption on the distribution.
We identify a fundamental tetrachotomy of optimal rates: 
for every concept class, the optimal universal rate of convergence of the excess error rate 
is one of $e^{-n}$, $e^{-o(n)}$, $o(n^{-1/2})$, or arbitrarily slow.
We further identify simple combinatorial structures which determine which of these categories any given concept class falls into.
\end{abstract}


\section{Introduction}
\label{sec:intro}


The past fifty years of research in statistical learning theory has produced a rich and beautifully general theory 
of learning, building on the foundational works of 
\citet*{vapnik:71,vapnik:74} establishing 
necessary and sufficient conditions for distribution-free learnability of any given concept class.
In particular, one of the most fundamental 
and seminal results in this literature is the 
identification of \emph{minimax optimal rates}
of convergence of excess risk for binary classification 
relative to any given function class.
Specifically, for any set $\H$ (called a \emph{concept class}) 
of functions $\X \to \{0,1\}$,
a learning algorithm is tasked with producing a 
\emph{predictor} $\hat{h}_n : \X \to \{0,1\}$ 
learned from $n$ i.i.d.\ samples from an arbitrary 
unknown distribution $\PXY$ on $\X \times \{0,1\}$.
The objective is to obtain a small 
\emph{error rate} $\er_{\PXY}(\hat{h}_n) = \PXY( (x,y) : h(x) \neq y )$
relative to the best achievable error rate 
$\inf_{h \in \H} \er_{\PXY}(h)$ by the class $\H$
(known as \emph{agnostic learning}, following terminology introduced by \citealp*{kearns:94a}).
A seminal result of \citet*{vapnik:74} establishes 
that it is possible to design such a learning algorithm 
achieving a convergence guarantee 
holding \emph{uniformly} in $\PXY$
--- i.e., 
$\sup_{\PXY} \E[ \er_{\PXY}(\hat{h}_n) ] - \inf_{h \in \H} \er_{\PXY}(h) \to 0$ ---
if and only if $\VC(\H) < \infty$,
where $\VC(\H)$ is a combinatorial dimension of $\H$
known as the \emph{VC dimension} \citep*{vapnik:71}.
Moreover, results of \citet*{vapnik:74,talagrand:94} 
provide a further quantitative guarantee: if $\VC(\H) < \infty$, 
the \emph{optimal} rate of convergence
of $\sup_{\PXY} \E[ \er_{\PXY}(\hat{h}_n) ] - \inf_{h \in \H} \er_{\PXY}(h)$
is of order $n^{-1/2}$.
This fact identifies a fundamental \emph{dichotomy} 
of optimal uniform rates: 
the optimal rate is $n^{-1/2}$ if $\VC(\H) < \infty$, 
and otherwise it is $\Omega(1)$. 

While this result forms a cornerstone of statistical learning theory, 
it does not tell the whole story of what makes a learning 
algorithm successful.
In reality, any given learning scenario represents just a \emph{single} distribution $\PXY$, and therefore 
the \emph{uniform} rate of convergence of 
$\sup_{\PXY} \E[ \er_{\PXY}(\hat{h}_n) ] - \inf_{h \in \H} \er_{\PXY}(h)$
may represent an overly \emph{pessimistic} 
bound on the actual rate of convergence of 
$\E[ \er_{\PXY}(\hat{h}_n) ] - \inf_{h \in \H} \er_{\PXY}(h)$
under $\PXY$.
We may therefore be interested in obtaining a \emph{refinement}
of the theory of uniform rates, 
which better describes the rates of convergence of 
$\E[ \er_{\PXY}(\hat{h}_n) ] - \inf_{h \in \H} \er_{\PXY}(h)$
in the sample size $n$, holding for every $\PXY$, 
though without the requirement of uniformity in $\PXY$.

In the present work, we develop such a theory.
Following prior work of \citet*{bousquet:21} (which studied a more-restrictive setting, discussed below), 
we are interested in understanding which rates $R(n) \to 0$
are achievable \emph{universally} in $\PXY$: that is, 
there is a learning algorithm which guarantees that, 
for \emph{every} distribution $\PXY$, 
$\E[ \er_{\PXY}(\hat{h}_n) ] - \inf_{h \in \H} \er_{\PXY}(h)$
converges at an asymptotic rate $R(n)$, 
where the asymptotic nature of the guarantee allows for
$\PXY$-dependent constant factors (and sometimes other relevant $\PXY$ dependences).
For instance, in sharp contrast to the $n^{-1/2}$ optimal \emph{uniform} rate, 
we will find that some concept classes 
admit rates which are \emph{exponentially fast}, i.e., 
$\forall \PXY$, $\E[ \er_{\PXY}(\hat{h}_n) ] - \inf_{h \in \H} \er_{\PXY}(h) \leq e^{-c_{\PXY} n}$, 
for a $\PXY$-dependent constant $\PXY$.
This theory of universal rates notably retains desirable 
properties of the classical uniform analysis. 
It admits a notion of an \emph{optimal} rate
for any given class $\H$, which can therefore 
serve to motivate algorithm design (i.e., the theory can 
distinguish between learning algorithms guaranteeing 
optimal rates vs suboptimal rates).
As we will show, it also admits simple combinatorial 
structures which completely characterize the optimal 
universal rates achievable for any given concept class $\H$.

In contrast to the fundamental \emph{dichotomy} of 
optimal uniform rates established by the classical theory,
our main contribution in this work is to identify a 
fundamental \emph{tetrachotomy} of 
optimal \emph{universal} rates.
We show that every concept class $\H$ exhibits an optimal 
rate that is one of $e^{-n}$, $e^{-o(n)}$, $o(n^{-1/2})$, or arbitrarily slow (precise definitions are given below).
We further provide combinatorial structures which 
completely characterize which concept classes $\H$ 
exhibit which of these as its optimal universal rate.
Specifically, we show that the optimal rate is $e^{-n}$ 
iff $|\H|< \infty$, 
while the optimal rate is $e^{-o(n)}$ iff 
$|\H|=\infty$ and $\H$ does not shatter an infinite 
\emph{Littlestone tree} (Definition~\ref{defn:littlestone-tree}), 
and the optimal rate is $o(n^{-1/2})$ iff 
$\H$ shatters an infinite Littlestone tree 
but does not shatter an infinite \emph{VCL tree}
(Definition~\ref{defn:vcl-tree}),
and otherwise $\H$ requires arbitrarily slow rates.

\paragraph{Background:}
The study of universal rates in learning theory has a long 
history, with numerous works establishing universal rate
upper and lower bounds in specialized settings, 
under various names, such as \emph{individual rates}, \emph{strong minimax rates}, \emph{true sample complexity}, etc.
\citep*[e.g.,][]{bendavid:95,castelli:96,schuurmans:97,antos:98,hanneke:10a,hanneke:12a,hanneke:13,kpotufe:13},
and of course, notions of universal asymptotic optimality have a 
rich history in theoretical statistics, 
particularly for certain families of parametric tests and estimators
\citep*[e.g.,][]{fisher1924distribution,wilks:38,lecam:52}.

The first \emph{general} theory of optimal universal rates 
for binary classification were presented in the 
work of \citet*{bousquet:21}.
The main difference from the present work 
is that \citet*{bousquet:21} study a
restricted setting, known as the \emph{realizable case}, 
which restricts to distributions $\PXY$ satisfying $\inf_{h \in \H} \er_{\PXY}(h) = 0$: 
i.e., they assume there are near-perfect functions in $\H$.
The present work represents an extension 
of their theory to the agnostic setting.
Under the realizability assumption, they identify a 
fundamental \emph{trichotomy} of optimal universal rates, 
where every concept class $\H$ exhibits an optimal rate
that is either $e^{-n}$, $\frac{1}{n}$, or arbitrarily slow.
They further identify simple combinatorial structures which 
completely determine which concept classes $\H$ fall into 
which of these categories. 
Concept classes $\H$ which do not shatter an infinite Littlestone tree (Definition~\ref{defn:littlestone-tree})
exhibit optimal rate $e^{-n}$, 
while classes which shatter an infinite Littlestone tree 
but do not shatter an infinite VCL tree (Definition~\ref{defn:vcl-tree}) 
exhibit optimal rate $\frac{1}{n}$, 
and otherwise $\H$ requires arbitrarily slow rates.
They provide numerous examples of concept classes 
illustrating these rates, and distinguishing 
optimal universal rates from optimal uniform rates 
and other notions from the literature, 
such as \emph{non-uniform} rates \citep*{benedek:94} 
and the \emph{universal Glivenko-Cantelli} property \citep*{dudley:91,van-handel:13}.

While the results of \citet*{bousquet:21} 
have since been extended to numerous 
variations on the original binary classification setting 
\citep*{kalavasis:22,hanneke:22b,hanneke:23b,bousquet:23,attias:24,hanneke:24b,hanneke:25b}
and analysis of specific algorithms \citep*{hanneke:24c,hanneke:25a},
the extension of the general theory of optimal rates beyond the realizable case has remained open.
Here we completely resolve this issue, providing a complete
extension of the theory of optimal rates to the agnostic setting
which makes no restrictions on the distribution $\PXY$.

\paragraph{Technical contributions and challenges:}
The key technical tool introduced in the work of \citet*{bousquet:21} 
is an expression of measurable winning strategies 
for certain Gale-Stewart games: 
sequential games defined based on the class $\H$ 
which are guaranteed to have a winning strategy 
for a particular player when $\H$ does not shatter an 
appropriate tree structure (see Definitions~\ref{defn:littlestone-tree}, \ref{defn:vcl-tree}, and Lemmas~\ref{lem:ordinal-SOA}, \ref{lem:ordinal-VCL-SOA}).
They use these winning strategies, in combination with 
careful statistical techniques for aggregation and parameter tuning, and learning methods for prediction with 
certain data-dependent partial concept classes induced by 
the winning strategies, to define the optimal learning methods.
In our present work, we again make use of these winning 
strategies as the core of our learning methods.
However, the techniques used to produce the 
desired learning rates are significantly more-involved, 
due to the fact that the data sequence may be non-realizable.
For many of these aspects, we make use of appropriate concentration inequalities 
to replace the simpler statistical aggregation and parameter tuning
techniques of \citet*{bousquet:21},
together with a data-relabeling technique 
(similar to a technique of \citealp*{hopkins:22}).
However, we face a particularly challenging point, 
specific to the agnostic setting, in the portion of the 
method making use of data-dependent partial concept classes induced by the winning strategies mentioned above.
This requires a novel analysis of universal learning 
rates for partial concept classes of finite VC dimension,
critical to obtaining the claimed $o(n^{-1/2})$ rates.
This analysis differs significantly from the analogous 
problem for total concept classes,
and this contribution should itself be of independent interest
(Section~\ref{sec:subsection-partial-concepts-super-root-upper-bound}).


\section{Summary of the Main Results}
\label{sec:agnostic-main}

We adopt the same basic definitions from the original work of \citet*{bousquet:21} on universal learning in the realizable case.  Let $\X$ be a non-empty set (called the \emph{instance space}).
A \emph{concept} is a function $h : \X \to \{0,1\}$, 
and a \emph{concept class} $\H$ is a non-empty set of concepts. 
We adopt the same (standard) measure-theoretic requirements for $\X$ and $\H$ 
from the original work of
\citet*{bousquet:21} (Definition 3.3 therein).\footnote{Specifically, we require the \emph{image admissible Suslin} property \citep*[][Section 5.3]{Dud14}: namely, we suppose $\X$ is a Polish space, and there exists a Polish space $\Theta$ and a Borel-measurable map $h : \Theta \times \X \to \{0,1\}$ such that $\H = \{ h(\theta,\cdot) : \theta \in \Theta \}$.  This is a standard 
condition relied on by much of the literature on empirical process theory, and holds for most concept classes of practical interest. 
 Henceforth, any concept class $\H$ in this article will be assumed to satisfy this condition; for brevity we omit this qualification from the statements of all results.\label{footnote:measurable-class}} 
 \stilltodo{TODO: Say something about all of the results also hold for 
partial concept classes $\H$, 
though for simplicity we will not discuss this fact 
further below, and will merely present the 
definitions and proofs for total concepts.
Verifying that everything extends to partial concept classes 
is an easy exercise.}%
A \emph{learning algorithm} is a sequence of universally measurable functions $f_n : (\X \times \{0,1\})^n \times \X \to \{0,1\}$, $n \in \nats$.
In the context of learning, 
we will consider a distribution $\PXY$ on $\X \times \{0,1\}$, 
and independent $\PXY$-distributed samples 
$(X_1,Y_1),\ldots,(X_n,Y_n)$, 
and when $\PXY$ is clear from the context we denote by
$\hat{h}_n$ the corresponding \emph{trained classifier},
namely,  
$\hat{h}_n(x) := f_n(X_1,Y_1,\ldots,X_n,Y_n,x)$.
Moreover, to simplify notation, 
we will simply refer to $\hat{h}_n$ itself as a \emph{learning algorithm}, 
leaving implicit the dependence of $\hat{h}_n$ 
on $(X_1,Y_1),\ldots,(X_n,Y_n)$ via the $f_n$ function.
We also adopt the colloquialism of referring to $\hat{h}_n$ 
as the classifier \emph{returned} by the learning algorithm, 
and we describe learning algorithms $f_n$ below as procedures which construct such a classifier $\hat{h}_n$.
For any measurable $h : \X \to \{0,1\}$, 
define the \emph{error rate} (or \emph{risk}) 
$\er_{\PXY}(h) := \PXY( (x,y) : h(x) \neq y )$.

The work of \citet*{bousquet:21} focused on the 
case of \emph{realizable} distributions $\PXY$: 
namely, those satisfying $\inf_{h \in \H} \er_{\PXY}(h) = 0$, 
studying the rate of convergence of the expected error rate
$\E\!\left[ \er_{\PXY}\!\left(\hat{h}_n\right) \right]$
holding universally over all realizable distributions.
In contrast, in the context of \emph{agnostic learning}, 
our interest is in characterizing the optimal rates of 
convergence of the expected \emph{excess} error rate 
$\E\!\left[ \er_{\PXY}\!\left(\hat{h}_n\right) \right] - \inf_{h \in \H} \er_{\PXY}(h)$
holding universally over \emph{all} distribution $\PXY$ (without realizability restrictions).
We extend the definition of achievable rates from 
the work of \citet*{bousquet:21}, both to apply 
in the context of agnostic learning, 
and to allow for 
asymptotic behaviors not captured by that work
(namely, rates that are \emph{arbitrarily close} to a given rate).
Importantly, in the following definitions, 
note that the order of quantifiers allows for the constants $C,c$ to be $\PXY$-dependent (which is what distinguishes universal rates from uniform rates).

\begin{itemize}[$\bullet$] 
\item $\H$ is \bemph{agnostically learnable at rate $\bm R$} if there is a learning 
algorithm $\hat{h}_n$ such that for every distribution~$\PXY$, 
there exist $C,c > 0$ for which 
$\E[\er_{\PXY}(\hat{h}_n)] - \inf_{h \in \H} \er_{\PXY}(h) \leq C R(c n)$ for all $n$.
\item $\H$ is \bemph{not agnostically learnable faster than $\bm R$} if there exist $C,c > 0$ such that,
for every learning algorithm $\hat{h}_n$, there exists a 
distribution $\PXY$ for which 
$\E[\er_{\PXY}(\hat{h}_n)] - \inf_{h \in \H} \er_{\PXY}(h) \geq C R(c n)$ for infinitely many~$n$.
\item $\H$ is \bemph{agnostically learnable with optimal rate $\bm R$} if 
$\H$ is agnostically learnable at rate $R$ and 
$\H$ is not agnostically learnable faster than $R$.
\item $\H$ requires \bemph{arbitrarily slow rates} for agnostic learning if there is a learning algorithm $\hat{h}_n$ such that, for every distribution $\PXY$, $\limsup_{n \to \infty} \E[\er_{\PXY}(\hat{h}_n)] - \inf_{h \in \H} \er_{\PXY}(h) \leq 0$, 
but for every $R(n) \to 0$, $\H$ is not agnostically learnable faster than $R$.
\end{itemize}

In addition to the definition of rates above, we will need another type of rate definition, capturing \bemph{super-root}, and \bemph{near-exponential}, as follows.

\begin{itemize}[$\bullet$]
\item (\bemph{super-root rate})~$\H$ is \bemph{agnostically learnable with optimal rate exactly}
$\boldsymbol{o(n^{-1/2})}$
if there is a learning algorithm $\hat{h}_n$ such that, for every distribution $\PXY$,
$\E[ \er_{\PXY}(\hat{h}_n) ] - \inf_{h \in \H} \er_{\PXY}(h) = o(n^{-1/2})$, 
but for every $R(n) = o(n^{-1/2})$, 
$\H$ is not agnostically learnable at rate $R$.
\item 
(\bemph{near-exponential rate})~$\H$ is \bemph{agnostically learnable with optimal rate exactly} 
$\boldsymbol{e^{-o(n)}}$
if for every $\psi(n) = o(n)$, 
$\H$ is agnostically learnable at rate $R(n) = e^{-\psi(n)}$,
and for every learning algorithm $\hat{h}_n$, 
$\exists \psi(n) = o(n)$
and a distribution $\PXY$
such that 
$\E[ \er_{\PXY}(\hat{h}_n) ] - \inf_{h \in \H} \er_{\PXY}(h) \geq e^{-\psi(n)}$ 
for infinitely many $n$. 
\end{itemize}

The main result of this work is a fundamental \emph{trichotomy} of possible optimal rates, applicable to any \emph{infinite} concept class (finite classes will be addressed below, by a vastly simpler analysis).

\begin{theorem}
\label{thm:agnostic-trichotomy}
For every infinite concept class $\H$,
exactly one of the following holds.
\begin{itemize}[$\bullet$]
\item $\H$ is agnostically learnable with optimal rate exactly $e^{-o(n)}$. 
\item $\H$ is agnostically learnable with optimal rate exactly $o(n^{-1/2})$.
\item $\H$ requires arbitrarily slow rates for agnostic learning.
\end{itemize}
\end{theorem}

The fact that both of the rate guarantees above are \emph{near} rates is an interesting distinction, both from the realizable case (where the corresponding rates are $e^{-n}$ and $\frac{1}{n}$ exactly) and from uniform analysis (where there is a dichotomy between $n^{-1/2}$ rate and non-learnable classes).

Interestingly, finite concept classes do not fit into this trichotomy.
For completeness, we also prove the following (rather obvious) observation: finite classes always have optimal rate $e^{-n}$ (proven in Section~\ref{sec:finite} via Theorems~\ref{thm:finite-upper-agnostic} and \ref{thm:finite-lower-agnostic}).
Technically, this means we have in fact established a \emph{tetrachotomy} of possible optimal rates 
(excluding the trivial case $|\H|=1$, which clearly has optimal rate $R(n)=0$).

\begin{theorem}
\label{thm:finite-agnostic-exponential}
For any \emph{finite} concept class $\H$ with $|\H| \geq 2$,
$\H$ is agnostically learnable with optimal rate $e^{-n}$.
\end{theorem}

For infinite concept classes, 
we also completely characterize which classes fall into which categories.  We note that these criteria are precisely the same as those establishing the trichotomy of possible rates for the realizable case \citep*{bousquet:21}.
Formal definitions of the relevant structures --- Littlestone trees and VCL trees --- are given in the next section.

\begin{theorem}
\label{thm:agnostic-char}
For every infinite concept class $\H$, 
the following hold:
\begin{itemize}[$\bullet$]
\item If $\H$ does not shatter an infinite Littlestone tree, 
then $\H$ is agnostically learnable with optimal rate exactly $e^{-o(n)}$. 
\item If $\H$ shatters an infinite Littlestone tree but does not shatter an infinite 
VCL tree, then $\H$ is agnostically learnable with optimal rate exactly $o(n^{-1/2})$.
\item If $\H$ shatters an infinite VCL tree, then $\H$ requires arbitrarily slow rates for agnostic learning.
\end{itemize}
\end{theorem}

Two comparisons are in order.  
First, we should compare to the \emph{realizable case} 
analysis of \citet*{bousquet:21}.
In the realizable case, classes with no infinite Littlestone 
tree have optimal rate $e^{-n}$ (compare to $e^{-o(n)}$ in Thereom~\ref{thm:agnostic-char}), 
while classes with infinite Littlestone tree and no infinite VCL tree have optimal rate $\frac{1}{n}$ (compare to $o(n^{-1/2})$ in Theorem~\ref{thm:agnostic-char}).
Second, we should compare to the \emph{uniform} rates
of agnostic learning.  Results of \citet*{vapnik:74} and \citet*{talagrand:94} reveal the optimal uniform rates 
are $n^{-1/2}$ for classes of finite \emph{VC dimension}, 
and do not converge to $0$ for classes of infinite VC dimension.
\citet*{bousquet:21} provide a number of examples of 
concept classes with infinite VC dimension, but for which 
there is no infinite Littlestone tree, or having an infinite Littlestone tree but no infinite VCL tree. 
We repeat some such examples in Section~\ref{sec:notation}
for completeness.

The proof of Theorem~\ref{thm:agnostic-char} is presented in several parts below.
The last item in Theorem~\ref{thm:agnostic-char} 
follows from known results.
Specifically, the fact that every class $\H$ 
is learnable with (at most) arbitrarily slow rates 
follows from the existence of 
universally Bayes-consistent learning algorithms 
(see Lemma~\ref{lem:universally-bayes-consistent} below).
Moreover, the fact that classes with an infinite VCL tree 
require arbitrarily slow rates follows immediately 
from the fact that this is true even in the realizable case, 
as shown by \citet*{bousquet:21}.
Thus, our proofs below will focus on establishing the 
other items in Theorem~\ref{thm:agnostic-char}.
These claims will be established in the sections below by 
Theorems~\ref{thm:near-exponential-lower-agnostic} and \ref{thm:agnostic-near-exponential-upper-bound} 
(establishing the first item, concerning $e^{-o(n)}$ optimal rates)
and 
Theorems~\ref{thm:agnostic-super-root-lower-bound} and \ref{thm:agnostic-super-sqrt-upper-bound}
(establishing the second item, concerning $o(n^{-1/2})$ optimal rates).


\stilltodo{TODO: add a paragraph discussing technical challenges and innovations, now in more detail.}

\section{Additional Definitions and Notation}
\label{sec:notation}

For a distribution $\PXY$ on $\X \times \{0,1\}$, 
we denote by $\Px$ the marginal distribution of $\PXY$ on $\X$.
For a sequence $x_1,x_2,\ldots$, 
we denote by $x_{\leq n} = \{x_1,\ldots,x_n\}$ 
and $x_{< n} = \{x_1,\ldots,x_{n-1}\}$.
For another sequence $y_1,y_2,\ldots$, 
we may write 
$(x_{\leq n},y_{\leq n}) = \{(x_1,y_1),\ldots,(x_n,y_n)\}$ 
and $(x_{< n},y_{< n}) = \{(x_1,y_1),\ldots,(x_{n-1},y_{n-1})\}$.
For any set $A \subseteq \X$, denote by $\hat{P}_{x_{\leq n}}(A) = \frac{1}{n} \sum_{i=1}^{n} \ind[ x_i \in A ]$ the empirical measure of $A$.
Also, for functions $h : \X \to \{0,1\}$, 
we define $h(x_{\leq n}) = (h(x_1),\ldots,h(x_n))$.
In some cases we express a sequence of sequences 
(e.g., Definition~\ref{defn:vcl-tree}) 
$x_{\mathbf{u}} \in \X^{k}$ for an index $\mathbf{u}$, 
in which case we index the elements of $x_{\mathbf{u}}$
by superscripts, i.e., $x_{\mathbf{u}} = \{x_{\mathbf{u}}^1,\ldots,x_{\mathbf{u}}^{k}\}$.
For $x \in \reals$, 
define $\log(x) = \ln(\max\{x,e\})$.
Also, for $x = 0$, define $x \log(1/x) = 0$
and $\left( \frac{c}{x} \right)^x = 1$ for any $c > 0$.

Many of the lemmas and proofs in this article 
concern conditional probabilities and expectations.
In all such contexts, to account for such quantities 
being defined only up to measure-zero differences, 
any claims concerning these 
(e.g., inequalities of the type $\P(E|X) \leq \delta$)
are formally interpreted as stating that 
there exists a version of the conditional probability 
(or expectation) for which the claim holds.

For any $n \in \nats$ and $S = \{(x_1,y_1),\ldots,(x_n,y_n)\} \in (\X \times \{0,1\})^n$, for any $h : \X \to \{0,1\}$,
define the \emph{empirical risk} 
\begin{equation*}
  \hat{\er}_S(h) = \frac{1}{n} \sum_{i=1}^{n} \ind[ h(x_i) \neq y_i ].
\end{equation*}
Any such sequence $S$ is said to be 
\emph{realizable} with respect to $\H$ if 
$\min_{h \in \H} \hat{\er}_{S}(h) = 0$.
Moreover, an infinite sequence $\{(x_t,y_t)\}_{t \in \nats}$ 
is said to be realizable with respect to $\H$ if
every finite \emph{prefix} $(x_{\leq n},y_{\leq n})$, $n \in \nats$, is realizable with respect to $\H$.
A distribution $P$ on $\X \times \{0,1\}$ is 
said to be realizable with respect to $\H$ if 
$\inf_{h \in \H} \er_{\PXY}(h) = 0$. 

The following definitions from \citet*{bousquet:21} 
supply the combinatorial 
structures critical to characterizing the optimal 
rates in Theorem~\ref{thm:agnostic-char}. 
The first can also be viewed as an infinite variant 
of a structure well known to the literature on adversarial 
online learning, introduced in the seminal work of 
\citet*{littlestone:88}.

\begin{definition}[\citealp*{bousquet:21}]
\label{defn:littlestone-tree}
A \bemph{Littlestone tree} shattered by $\H$ is a perfect binary tree of depth $d\leq\infty$ whose internal nodes are labeled by elements of $\X$, and whose two edges connecting a node to its children are labelled $0$ and~$1$, such that every finite path emanating from the root is consistent with a concept $h\in\H$.
More precisely, a Littlestone tree of depth $d$ shattered by $\H$ is a collection 
\begin{equation*} 
\{x_{\mathbf{u}}:0\leq k<d, 
\mathbf{u}\in\{0,1\}^k\}\subseteq\mathcal{X}
\end{equation*}
such that for every $\mathbf{y}\in\{0,1\}^d$ and $n<d$, 
there exists
$h\in\H$ such that $h(x_{\mathbf{y}_{\le k}})=y_{k+1}$ for  
all $0\le k\le n$.
We say $\H$ shatters an \bemph{infinite Littlestone tree} if there is 
a Littlestone tree of depth $d=\infty$ shattered by $\H$.
\end{definition}

The second structure can be viewed as a combination 
of the notion of a Littlestone tree with the notion of \emph{shattering} from the definition of VC dimension \citep*{vapnik:71}
appearing in each node (and is therefore referred to as a VC-Littlestone tree, or VCL tree for short).

\begin{definition}[\citealp*{bousquet:21}]
\label{defn:vcl-tree}
A \bemph{VCL tree} shattered by $\H$ of depth $d\le\infty$ is a collection
\begin{equation*}
	\{x_{\mathbf{u}}\in\mathcal{X}^{k+1}:0\le k<d, 
	\mathbf{u}\in\{0,1\}^1\times\{0,1\}^2\times\cdots\times\{0,1\}^k\}
\end{equation*}
such that for every $n<d$ and
$\mathbf{y}\in\{0,1\}^1\times\cdots\times\{0,1\}^{n+1}$, there exists a 
concept $h\in\H$ such that $h(x_{\mathbf{y}_{\le k}}^i)=y_{k+1}^i$
for all $0\le i\le k$ and $0\le k\le n$, where we denote
\begin{equation*}
	\mathbf{y}_{\le k}=(y_1^0,(y_2^0,y_2^1),\ldots,(y_k^0,\ldots,y_k^{k-1})),
	\qquad
	x_{\mathbf{y}_{\le k}}=(x_{\mathbf{y}_{\le k}}^0,\ldots,
	x_{\mathbf{y}_{\le k}}^k).
\end{equation*}
We say that $\H$ shatters an \bemph{infinite VCL tree} if there is
a VCL tree of depth $d=\infty$ shattered by $\H$.
\end{definition}

We refer the reader to the original work 
of \citet*{bousquet:21}
for numerous examples of concept classes illustrating these definitions, and comparing to previously-studied combinatorial dimensions such as VC dimension.
A simple example which does not shatter an infinite Littlestone tree is \emph{threshold} classifiers $\{ \ind_{[a,\infty)} : a \in \X \}$ on $\X = \nats$, 
whereas threshold classifiers on $\X = \reals$ \emph{does} shatter an infinite Littlestone tree (but not an infinite VCL tree).
We can also relate these structures to the familiar \emph{VC dimension} \citep*{vapnik:71}
by noting that any concept class of finite VC dimension cannot shatter an infinite VCL tree.
On the other hand, any class which VC-shatters some \emph{infinite} set $\X' \subset \X$ necessarily shatters an infinite VCL tree.
However, there is an intermediate category, namely classes which have infinite VC dimension (witnessed by arbitrarily large, but not infinite, shattered sets), 
yet which do not shatter an infinite VCL tree.
Indeed, consider the class $\H$ of all (coordinate-wise) \emph{non-decreasing} 
binary functions $\nats^d \to \{0,1\}$ 
on $\X = \nats^d$.
While this class $\H$ has infinite VC dimension, 
\citet*{bousquet:21} show it does not even shatter an infinite \emph{Littlestone} tree.



\section{Finite Classes}
\label{sec:finite}

As a warm-up, we present a very simple proof of Theorem~\ref{thm:finite-agnostic-exponential}, establishing that finite classes with $|\H| \geq 2$ have optimal rate $e^{-n}$.
This will be established by the following two theorems, providing the upper and lower bounds, respectively.
The first was also presented in the work of \citet*{hanneke:25a}, repeated here for completeness.

\begin{theorem}
\label{thm:finite-upper-agnostic}
Any finite concept class $\H$ is agnostically learnable at rate $e^{-n}$.
Moreover, this is achieved by \emph{empirical risk minimization}: 
that is, any learning algorithm $f_n$ which, for any $S \in (\X \times \{0,1\})^n$, 
has the function $x \mapsto f_n(S,x)$ 
equal to an element $h_{S}$ of $\H$ satisfying $\hat{\er}_{S}(h_{S}) = \min_{h \in \H} \hat{\er}_{S}(h)$.
\end{theorem}
\begin{proof}
Let $\PXY$ be any distribution,
let $(X_1,Y_1),\ldots,(X_n,Y_n)$ be i.i.d.\ $\PXY$ random variables,
and let us abbreviate 
$\hat{\er}_{n}(\cdot) = \hat{\er}_{(X_{\leq n},Y_{\leq n})}(\cdot)$.
Let $f_n$ be a learning algorithm such that $\hat{h}_n \in \H$ 
and $\hat{\er}_n(\hat{h}_n) = \min_{h \in \H} \hat{\er}_n(h)$: 
that is, $\hat{h}_n$ is an empirical risk minimizer.
Since $\H$ is a finite set, 
the value $\inf_{h \in \H} \er_{\PXY}(h)$ 
is achieved; 
let $\target$ be any element of $\H$ with $\er_{\PXY}(\target) = \min_{h \in \H} \er_{\PXY}(h)$.

Define the set
\begin{equation*}
\H_{\mathrm{bad}} = \left\{ h \in \H : \er_{\PXY}(h) > \er_{\PXY}(\target) \right\}.
\end{equation*}
If $\H_{\mathrm{bad}} = \emptyset$, 
then since $\hat{h}_n \in \H$ we trivially have 
$\er_{\PXY}(\hat{h}_n) - \er_{\PXY}(\target) = 0$ for all $n$.
To handle the remaining case,
suppose $\H_{\mathrm{bad}} \neq \emptyset$, 
and define
\begin{equation*}
\epsilon = \min_{h \in \H_{\mathrm{bad}}} \er_{\PXY}(h) - \er_{\PXY}(\target).
\end{equation*}
Since $\er_{\PXY}(\hat{h}_n) - \er_{\PXY}(\target) = 0$ when $\hat{h}_n \notin \H_{\mathrm{bad}}$, 
and since $\er_{\PXY}(\hat{h}_n) \leq 1$ always, 
we have
\begin{align*}
\E\!\left[ \er_{\PXY}(\hat{h}_n) \right] - \er_{\PXY}(\target)
& = \E\!\left[ \left( \er_{\PXY}(\hat{h}_n) - \er_{\PXY}(\target) \right) \ind[ \hat{h}_n \in \H_{\mathrm{bad}} ] \right]
\\ & \leq \P\!\left( \hat{h}_n \in \H_{\mathrm{bad}} \right).
\end{align*}
Since $\hat{h}_n$ is an empirical risk minimizer, 
we always have that 
$\hat{\er}_n(\hat{h}_n) \leq \hat{\er}_n(\target)$.
Therefore, if $\hat{h}_n \in \H_{\mathrm{bad}}$, 
it must be the case that 
either 
$\left| \hat{\er}_n(\target) - \er_{\PXY}(\target) \right| \geq \frac{\epsilon}{2}$
or 
$\left| \hat{\er}_n(\hat{h}_n) - \er_{\PXY}(\hat{h}_n) \right| \geq \frac{\epsilon}{2}$.
Thus, 
\begin{align*}
\P\!\left( \hat{h}_n \in \H_{\mathrm{bad}} \right)
& \leq \P\!\left( \max_{h \in \H} \left| \hat{\er}_n(h) - \er_{\PXY}(h) \right| \geq \frac{\epsilon}{2} \right) 
\\ & \leq \sum_{h \in \H} \P\!\left( \left| \hat{\er}_n(h) - \er_{\PXY}(h) \right| \geq \frac{\epsilon}{2} \right)
\leq 2 |\H| e^{-(\epsilon^2/2)n},
\end{align*}
where the final two inequalities follow from the union bound and Hoeffding's inequality, respectively.
\end{proof}

\begin{theorem}
\label{thm:finite-lower-agnostic}
Any 
concept class $\H$ with $|\H| \!\geq\! 2$ is not agnostically learnable faster than $e^{-n}$.
\end{theorem}

We remark that this theorem is only needed because our 
condition ``$|\H| \geq 2$'' is slightly weaker than the 
condition 
assumed by \citet*{bousquet:21} 
to establish a lower bound $e^{-n}$ for the realizable case 
(explicitly, the classes $\H$ for which this result establishes an $e^{-n}$ lower bound for agnostic learning, and which do not admit an $e^{-n}$ lower bound in the realizable case, are precisely those classes of the form $\H = \{h,1-h\}$).
The proof of Theorem~\ref{thm:finite-lower-agnostic} will rely 
on the following basic lower bound on the sample complexity 
of determining whether a coin is slightly biased toward heads or tails.
See Lemma 5.1 of \citet*{anthony:99} (from which this lemma immediately follows).

\begin{lemma}
\label{lem:coin-testing-lower-bound}
Let $\gamma \in \left(0,\frac{1}{5}\right)$, $\delta \in \left(0,\frac{1}{8e}\right]$,
and $n \in \nats \cup \{0\}$ satisfy
\begin{equation*}
n < \frac{1}{8 \gamma^2} \ln\!\left(\frac{1}{8\delta}\right).
\end{equation*}
Let $p_0 = \frac{1}{2} - \gamma$ and $p_1 = \frac{1}{2} + \gamma$.
Let $t^* \sim \mathrm{Bernoulli}(\frac{1}{2})$ 
and, conditioned on $t^*$, 
let $B_1,\ldots,B_n$ be conditionally i.i.d.\ 
$\mathrm{Bernoulli}(p_{t^*})$.
Fix any function $\hat{t}_n : \{0,1\}^n \to \{0,1\}$
(possibly randomized, independent of $t^*$).
Then
\begin{equation*}
\P\!\left( \hat{t}_n(B_1,\ldots,B_n) \neq t^* \right) > \delta.
\end{equation*}
\end{lemma}

We are now ready for the proof of Theorem~\ref{thm:finite-lower-agnostic}.

\begin{proof}[of Theorem~\ref{thm:finite-lower-agnostic}]
Let $h_0,h_1$ be any two distinct elements of $\H$, 
and fix any $x \in \X$ with $h_0(x) \neq h_1(x)$
(and without loss of generality, suppose $h_0(x) = 0 = 1-h_1(x)$).
Define two distributions, $P_0$ and $P_1$, 
such that $P_0(\{(x,0)\}) = P_1(\{(x,1)\}) = \frac{2}{3}$ 
and $P_0(\{(x,1)\}) = P_1(\{(x,0)\}) = \frac{1}{3}$: 
that is, $P_0$ and $P_1$ both have marginal distribution 
on $\X$ with support exactly $\{x\}$ (i.e., a single point-mass),
and they differ in the conditional distribution of $Y$ given $X$, 
in that each $t,t' \in \{0,1\}$ have 
\begin{equation*} 
P_t(Y=t'|X=x) = 
\begin{cases} 
\frac{2}{3}, & \text{ if } t'=t\\
\frac{1}{3}, & \text{ if } t' \neq t
\end{cases}.
\end{equation*}
In particular, note that for either $t \in \{0,1\}$, 
we have that 
$\inf_{h \in \H} \er_{P_t}(h) = \er_{P_t}(h_t) = \frac{1}{3}$,
and any $h : \X \to \{0,1\}$ 
with $h(x) \neq t = h_t(x)$ 
has $\er_{P_t}(h) - \inf_{h' \in \H} \er_{P_t}(h') 
= \er_{P_t}(h) - \er_{P_t}(h_t)
= \frac{2}{3} - \frac{1}{3} = \frac{1}{3}$.

Fix any learning algorithm $f_n$.
Define a sequence $X_i = x$, for $i \in \nats$,
let $Y_1^{(0)},Y_2^{(0)},\ldots$ be independent $\mathrm{Bernoulli}(\frac{1}{3})$ random variables,
and 
$Y_1^{(1)},Y_2^{(1)},\ldots$ be independent $\mathrm{Bernoulli}(\frac{2}{3})$ random variables.
Let $t^* \sim \mathrm{Uniform}(\{0,1\})$ 
be independent of all $Y_i^{(t)}$, $i \in \nats$, $t \in \{0,1\}$.
Note that for $t \in \{0,1\}$, 
$(X_1,Y_1^{(t)}),(X_2,Y_2^{(t)}),\ldots$
are independent $P_t$-distributed random variables,
and since $t^*$ is independent of these, 
for any $n \in \nats$ we have that 
$(Y_1^{(t^*)},\ldots,Y_n^{(t^*)})$
is conditionally $P_{t^*}^n$-distributed
given $t^*$.

We now apply Lemma~\ref{lem:coin-testing-lower-bound} 
with $\gamma = \frac{1}{6}$.
Note that defining 
$(B_1,\ldots,B_n) := (Y_1^{(t^*)},\ldots,Y_n^{(t^*)})$
satisfies the definition of the $B_i$ variables 
in Lemma~\ref{lem:coin-testing-lower-bound}.
Also note that the $X_1,\ldots,X_n$ variables 
(being all equal $x$) are invariant to $t^*$.
In particular, 
defining $\hat{t}_n(B_1,\ldots,B_n) := f_n(X_{\leq n},Y_{\leq n}^{(t^*)},x)$ in the context of Lemma~\ref{lem:coin-testing-lower-bound}, 
the facts about $P_{t^*}$ argued above yield that 
\begin{align*}
\hat{t}_n(B_1,\ldots,B_n) \!\neq\! t^*
\Leftrightarrow f_n\!\left(X_{\leq n},Y^{(t^*)}_{\leq n}\!,x\right) \!\neq\! h_{t^*}(x)
\Leftrightarrow \er_{P_{t^*}}\!\!\left( f_n\!\left(X_{\leq n},Y^{(t^*)}_{\leq n}\right) \right) \!-\! \er_{P_{t^*}}\!(h_{t^*}) = \frac{1}{3}.
\end{align*}

Note that any $n \in \nats$ with $n \geq 5$ 
satisfies $e^{-n} < \frac{1}{8e}$
and $\frac{1}{8} e^{n} > e^{n / 2}$,
and moreover this latter inequality implies 
\begin{equation*}
\frac{1}{8 \gamma^2} \ln\!\left( \frac{1}{8 e^{-n}} \right)
> \frac{36 n}{16}
> n.
\end{equation*} 
Thus, applying Lemma~\ref{lem:coin-testing-lower-bound}
with $\delta = e^{- n}$ implies
\begin{equation*}
\P\!\left( \er_{P_{t^*}}\!\left(f_n\!\left(X_{\leq n},Y_{\leq n}^{(t^*)}\right)\right) - \er_{P_{t^*}}(h_{t^*}) = \frac{1}{3} \right)
= \P\!\left( \hat{t}_n(B_1,\ldots,B_n) \neq t^* \right)
> e^{-n}.
\end{equation*}
By Markov's inequality, 
\begin{align*}
& \E\!\left[ \er_{P_{t^*}}\!\left(f_n\!\left(X_{\leq n},Y_{\leq n}^{(t^*)}\right)\right) - \er_{P_{t^*}}(h_{t^*}) \right] 
\\ & \geq \frac{1}{3} \cdot \P\!\left( \er_{P_{t^*}}\!\left(f_n\!\left(X_{\leq n},Y_{\leq n}^{(t^*)}\right)\right) - \er_{P_{t^*}}(h_{t^*}) = \frac{1}{3} \right) > \frac{1}{3} \cdot e^{-n}.
\end{align*}
Finally, by the law of total expectation, 
the above implies that, for any $n \geq 5$, 
there exists $t_n \in \{0,1\}$ such that 
\begin{equation*}
\E\!\left[ \er_{P_{t_n}}\!\left(f_n\!\left(X_{\leq n},Y_{\leq n}^{(t_n)}\right)\right) - \er_{P_{t_n}}(h_{t_n}) \right] 
> \frac{1}{3} \cdot e^{-n}.
\end{equation*}
In particular, by the Pigeonhole principle,
there exists $t \in \{0,1\}$ 
with $t_n = t$ for infinitely many $n$.
For this $t$, 
letting $N \subset \nats$ denote the set of all 
$n \geq 5$ with $t_n = t$, 
we conclude that when applying the learning algorithm $f_n$ under the distribution $P_t$, 
every $n \in N$ satisfies 
\begin{equation*}
\E\!\left[ \er_{P_{t}}\!\left(f_n\!\left(X_{\leq n},Y_{\leq n}^{(t)}\right)\right) - \er_{P_{t}}(h_{t}) \right] 
> \frac{1}{3} \cdot e^{-n}.
\end{equation*}
Since $|N| = \infty$, this completes the proof.
\end{proof}

Theorem~\ref{thm:finite-agnostic-exponential} now follows immediately from Theorems~\ref{thm:finite-upper-agnostic} and \ref{thm:finite-lower-agnostic}.

\section{Supporting Lemmas}
\label{sec:lemmas}

Turning now to the more-challenging case of infinite classes, 
we begin by introducing a few supporting lemmas.
The first is a key result from the original work of 
\citet*{bousquet:21}, which extends Littlestone's 
\emph{Standard Optimal Algorithm} (SOA) \citep*{littlestone:88}
to classes which do not shatter an infinite Littlestone 
tree, establishing the existence of an \emph{online} 
learning algorithm which guarantees a \emph{finite} 
(not necessarily bounded) number of mistakes 
on any sequence $(x_t,y_t)$ realizable with respect to $\H$.
The result stems from an extension of the definition 
of the \emph{Littlestone dimension} \citep*{littlestone:88}
to allow for \emph{infinite ordinal} values of the 
Littlestone dimension, leading to measurable 
solutions to a corresponding \emph{Gale-Stewart game} 
(see \citealp*{bousquet:21} for the detailed definitions and a formal proof of this lemma).

\begin{lemma}[\citealp*{bousquet:21}]
\label{lem:ordinal-SOA}
For any concept class $\H$ which does not 
shatter an infinite Littlestone tree, 
there exists a sequence $\SOA_t : (\X \times \{0,1\})^{t-1} \times \X \to \{0,1\}$
of universally measurable functions such that, 
there \emph{does not exist} an infinite sequence 
$\{(x_t,y_t)\}_{t \in \nats}$ in $\X \times \{0,1\}$
realizable with respect to $\H$ such that 
\begin{equation*}
\forall t \in \nats, \SOA_t(x_{< t},y_{< t},x_t) \neq y_t.
\end{equation*}
\end{lemma}

In the context of learning from i.i.d.\ samples,
\citet*{bousquet:21} apply the $\SOA_t$ predictor as a core component 
in a learning algorithm achieving exponential rates $e^{-n}$ 
in the realizable case.
Specifically, the utility of the $\SOA_t$ predictor 
in that context is based on its property (guaranteed by the above lemma) of being an \emph{eventually perfect} predictor.
In particular, by applying the $\SOA_t$ predictor as a 
\emph{conservative online learning algorithm} with a sequence of i.i.d.\ samples from a realizable distribution, 
there exists some finite $b^*$ 
such that, with a given constant probability, the predictor 
achieves error rate \emph{zero} within the first $b^*$ examples.
Formally, this yields the following result, 
established by \citet*{bousquet:21}.

\begin{lemma}[\citealp*{bousquet:21}]
\label{lem:SOA-zero-error-rate-bstar}
Let\\ $\H$ be any concept class which does not shatter an infinite Littlestone tree.
There exists a universally measurable function 
$\algSOA : \bigcup_{b \in \nats \cup \{0\}} (\X \times \{0,1\})^b \times \X \to \{0,1\}$ such that, 
for any distribution $P$ on $\X \times \{0,1\}$ 
which is realizable with respect to $\H$, 
for any $\gamma \in (0,1)$, 
there exists $b^*_{\gamma} \in \nats$ such that, 
$\forall b \geq b^*_{\gamma}$, 
for $S_b \sim P^b$, 
\begin{equation*}
\P\Big( \er_P(\algSOA(S_b)) = 0 \Big) > \gamma.
\end{equation*}
\end{lemma}

\citet*{bousquet:21} also extend these same types of results 
to classes which do not shatter an infinite \emph{VCL tree}.
This forms a core component in their algorithm achieving rate $\frac{1}{n}$ in the realizable case,
and is likewise a core component in our algorithm achieving 
rate $o(n^{-1/2})$ in the agnostic case.
In this case, the predictor $\SOA_t$ is replaced by 
a \emph{pattern avoidance} function 
$\VCLSOA_t$, which predicts label patterns which 
\emph{do not} match the true labels on 
given $k$-tuples of $x$'s.
In the learning algorithm below 
(used to establish Theorem~\ref{thm:agnostic-super-sqrt-upper-bound} below),
the role of such pattern avoidance functions is to perform 
a kind of sequential data-dependent \emph{model selection},
constructing a \emph{partial concept class}
for which favorable learning rates are achievable.
The existence of this pattern avoidance function is again based on 
reasoning about ordinal values of a VCL-type extension 
of the Littlestone dimension, leading to measurable 
winning strategies for Gale-Stewart games corresponding 
to VCL game trees (see \citealp*{bousquet:21} for the 
detailed definitions and a formal proof of this lemma).

\begin{lemma}[\citealp*{bousquet:21}]
\label{lem:ordinal-VCL-SOA}
For any concept class $\H$ which does not 
shatter an infinite VCL tree, 
there exists a sequence $\VCLSOA_t : \times_{k=1}^{t-1}(\X \times \{0,1\})^{k} \times \X^t \to \{0,1\}^t$
of universally measurable functions such that, 
there \emph{does not exist} an infinite sequence $\{S_t\}_{t \in \nats}$
of finite sequences 
$S_t = \{(x^t_1,y^t_1),\ldots,(x^t_t,y^t_t)\}$ $\in (\X \times \{0,1\})^t$
such that the sequence $S_1 \cup S_2 \cup \cdots$ is 
realizable with respect to $\H$ and 
\begin{equation*}
\forall t \in \nats, \VCLSOA_t(S_{< t},(x^t_1,\ldots,x^t_t)) = (y^t_1,\ldots,y^t_t).
\end{equation*}
\end{lemma}

Based on this $\VCLSOA_t$ pattern avoidance function, 
\citet*{bousquet:21} prove a variant of
Lemma~\ref{lem:SOA-zero-error-rate-bstar} for the case 
of classes which do not shatter an infinite VCL tree,
stated formally as follows. 

\begin{lemma}[\citealp*{bousquet:21}]
\label{lem:VCL-SOA-zero-error-rate-bstar}
Let $\H$ be any concept class which does not shatter an infinite VCL tree.
For every $b \in \nats$, 
there exists a universally measurable function
$k^b_{\VCL} : (\X \times \{0,1\})^b \to \{1,\ldots,b\}$,
and for every $k \in \{1,\ldots,b\}$ there exists 
a universally measurable function 
$\algVCLSOA^{b,k} : (\X \times \{0,1\})^b \times \X^k \to \{0,1\}^k$
such that, 
for any distribution $P$ on $\X \times \{0,1\}$ which is realizable with respect to $\H$, 
$\forall \gamma \in (0,1)$, 
there exists $b^*_{\gamma} \in \nats$ such that, 
$\forall b \geq b^*_{\gamma}$, for $S_b \sim P^b$,
for $\mathbf{k} = k^{b}_{\VCL}(S_b)$,
\begin{equation*}
\P\!\left( P^{\mathbf{k}}\!\left( (x_1,y_1),\ldots,(x_{\mathbf{k}},y_{\mathbf{k}}) : \algVCLSOA^{b,\mathbf{k}}(S_b,x_1,\ldots,x_{\mathbf{k}}) = (y_1,\ldots,y_{\mathbf{k}}) \right) = 0 \right) > \gamma.
\end{equation*}
\end{lemma}

\begin{remark}[A Remark on Measurability]
\label{rem:measurability}
The learning algorithms we construct below to establish the upper bounds 
on achievable rates in Theorem~\ref{thm:agnostic-char} (namely, in the proofs of Theorems~\ref{thm:agnostic-near-exponential-upper-bound} and \ref{thm:agnostic-super-sqrt-upper-bound}) 
apply the universally measurable functions 
$\algSOA$, $k^b_{\VCL}$, and $\algVCLSOA^{b,k}$
from Lemmas~\ref{lem:SOA-zero-error-rate-bstar} and \ref{lem:VCL-SOA-zero-error-rate-bstar}
in a simple manner, to a finite number of 
subsets of the data sequence (plus the test point).
As such, it is a straightforward (though rather tedious) 
exercise to verify that these learning algorithms 
can be described as a sequence of 
universally measurable functions $f_n : (\X \times \{0,1\})^n \times \X \to \{0,1\}$ 
(i.e., a universally measurable function mapping 
a training set $S_n$ of size $n$ plus a test point $x$
to a prediction $f_n(S_n,x) \in \{0,1\}$).
For brevity of presentation, we largely omit the detailed 
discussion of this fact from the proofs of these theorems.
\end{remark}

In addition to the above functions $\algSOA$, $\algVCLSOA$, 
another key component of our learning algorithms achieving 
the claimed $e^{-o(n)}$ and $o(n^{-1/2})$ rates is 
the use of a \emph{universally Bayes-consistent} subroutine.
In the context of the proofs, such an algorithm is used to 
address the case where the distribution $\PXY$ is not 
\emph{Bayes-realizable}: that is, $\inf_{h \in \H} \er_{\PXY}(h) > \inf_{h} \er_{\PXY}(h)$, where $h$ ranges over all measurable functions $\X \to \{0,1\}$ in the latter infimum.
Such cases allow for an excess error rate 
$\er_{\PXY}(\hat{h}_n) - \inf_{h \in \H} \er_{\PXY}(h)$ 
which may even be \emph{negative}, 
and as such may be considered as rather \emph{easy} cases
in the context of universal rates for agnostic learning.
To address these distributions, 
we employ a secondary 
learning algorithm $\hat{h}_n$ as a subroutine, 
which guarantees, for all distributions $\PXY$,
that $\E\!\left[ \er_{\PXY}(\hat{h}_n) \right] \to \inf_{h} \er_{\PXY}(h)$, 
a property known as universal Bayes-consistency \citep*[see e.g.,][]{devroye:96}.
Such learning algorithms are known to exist 
\citep*{hanneke:21a,hanneke:21b}
under the conditions on $\X$ considered in the present work
(from footnote~\ref{footnote:measurable-class},
which in particular requires that $\X$ be a separable 
metric space),
as stated formally in the following lemma.
\stilltodo{TODO: Technically, those works consider the Borel $\sigma$-algebra.  We might add a footnote addressing their universal measurability and validity of the consistency claim to probability measures on the extended $\sigma$-algebra including the zero-measure sets.}

\begin{lemma}[\citealp*{hanneke:21b}] 
\label{lem:universally-bayes-consistent}
There exists a learning algorithm $\hat{h}_n$ such that, 
for all distributions $\PXY$ on $\X \times \{0,1\}$, 
\begin{equation*}
\lim_{n \to \infty} \E\!\left[ \er_{\PXY}\!\left(\hat{h}_n\right) \right] = \inf_{h} \er_{\PXY}(h).
\end{equation*}
\end{lemma}

In addition to supplying a key component of our learning algorithms below, 
Lemma~\ref{lem:universally-bayes-consistent} 
also substantiates the 
claim in Theorems~\ref{thm:agnostic-trichotomy} and \ref{thm:agnostic-char} 
that every concept class $\H$ at least admits 
arbitrarily slow rates: that is, 
existence of a learning algorithm 
with $\limsup_{n \to \infty} \E[\er_{\PXY}(\hat{h}_n)] - \inf_{h \in \H} \er_{\PXY}(h) \leq 0$
for all distributions $\PXY$.

Next, we present several lemmas which will be useful 
in the proofs of the main results below.
These may be of independent interest, 
as they concern general properties of concept classes.
The lemmas concern concept classes $\H$ 
satisfying the \emph{universal Glivenko-Cantelli} (UGC) property, defined as follows \citep*{vapnik:71,dudley:91,van-handel:13}.

\begin{definition}
\label{defn:UGC}
A concept class $\H$ is said to satisfy the 
universal Glivenko-Cantelli (UGC) property if, 
for every distribution $\Px$ on $\X$, 
for $X_1,X_2,\ldots$ independent $\Px$-distributed random variables, 
\begin{equation}
\label{eqn:P-GC-convergence}
\lim_{n \to \infty} \E\!\left[ \sup_{h \in \H} \left| \hat{P}_{X_{\leq n}}( x : h(x) = 1 ) - \Px( x : h(x) = 1 ) \right| \right] = 0.
\end{equation}
\end{definition}

The first lemma is based on an argument of \citet*{bousquet:21}, 
and establishes that the UGC property holds for 
any concept class $\H$ which does not shatter an infinite VCL tree.

\begin{lemma}
\label{lem:no-infinite-VCL-implies-UGC}
Any concept class $\H$ which does not shatter an infinite VCL tree satisfies the UGC property (Definition~\ref{defn:UGC}).
\end{lemma}
\begin{proof}
Since $\H$ does not shatter an infinite VCL tree, it follows that there does not exist any infinite subset $\X' = \{x_1,x_2,\ldots\} \subseteq \X$ that is 
  finitely shattered by $\H$ 
  (i.e., such that every $t,i_1,\ldots,i_t \in \nats$
  have $\forall y_1,\ldots,y_t \in \{0,1\}$, 
  $\{(x_{i_1},y_1),\ldots,(x_{i_t},y_t)\}$ is realizable with respect to $\H$), 
  noting that if such an $\X'$ were to exist then we could construct an infinite VCL tree shattered by $\H$ by defining 
  the nodes
  $x_{\mathbf{u}} \in (\X')^{k+1}$, $0 \leq k < \infty$, $\mathbf{u} \in \{0,1\}^1 \times \cdots \times \{0,1\}^k$, 
  so that every $x_{\mathbf{u}}^i$ is distinct over all $\mathbf{u}$ and $i$.
  It was shown by \citet*{dudley:91} that, if no such finitely-shattered infinite subset $\X'$ exists, 
  then $\H$ satisfies the UGC property.
\end{proof}


For a distribution $\Px$ on $\X$, 
we say a concept class $\H$ is \emph{totally bounded} in
the $L_1(\Px)$ pseudo-metric if, 
for every $\epsilon > 0$, 
there exists a \emph{finite} set $\tilde{\H}_{\epsilon} \subseteq \H$
that is an \emph{$\epsilon$-cover} of $\H$ in $L_1(\Px)$:
namely, 
\begin{equation} 
\label{eqn:epsilon-cover-definition}
\sup_{h \in \H} \min_{\tilde{h} \in \tilde{\H}_{\epsilon}} \Px( x : \tilde{h}(x) \neq h(x) ) \leq \epsilon.
\end{equation}
The following lemma establishes that classes $\H$ satisfying the UGC property are always totally bounded in $L_1(\Px)$ (the result is rather immediate from known characterizations of UGC, but we include a simple proof for completeness).

\begin{lemma}
\label{lem:UGC-totally-bounded}
For any concept class $\H$ satisfying the UGC property (Definition~\ref{defn:UGC}),
for any probability measure $\Px$ on $\X$,
$\H$ is totally bounded in $L_1(\Px)$. 
\end{lemma}
\begin{proof}
Fix any distribution $\Px$ on $\X$,
and let $X_1,X_2,\ldots$ be independent $\Px$-distributed random variables.
We first argue that, for any concept class $\H$, 
the Glivenko-Cantelli 
property stated in \eqref{eqn:P-GC-convergence}
implies the same condition for the class
$\H_{\Delta} = \{ x \mapsto \ind[ h(x) \neq h'(x) ] : h,h' \in \H \}$ of pairwise differences.
For any concept class $\H$, 
denote by 
$\H(X_{\leq n}) := \{ (h(X_1),\ldots,h(X_n)) : h \in \H \}$ the set of $\H$-realizable classifications of $X_{\leq n}$.
First, we recall a classic result of \citet*{vapnik:71},
which establishes that a concept class $\H$ satisfies 
\eqref{eqn:P-GC-convergence}
if and only if 
$\lim_{n \to \infty} \E\!\left[ \frac{1}{n}\log_2(|\H(X_{\leq n})|) \right] = 0$.
Second, since each element of 
$\H_{\Delta}(X_{\leq n})$ is specified as a 
pointwise difference of two elements of $\H(X_{\leq n})$, 
we can conclude that 
$|\H_{\Delta}(X_{\leq n})| \leq |\H(X_{\leq n})|^2$.
Therefore, 
if a concept class $\H$ satisfies \eqref{eqn:P-GC-convergence},
then 
\begin{equation*}
0 \leq \limsup_{n \to \infty} \E\!\left[ \frac{1}{n}\log_2(|\H_{\Delta}(X_{\leq n})|) \right]
\leq \limsup_{n \to \infty} \E\!\left[ \frac{2}{n}\log_2(|\H(X_{\leq n})|) \right] = 0,
\end{equation*}
so that \eqref{eqn:P-GC-convergence} also holds 
when we replace $\H$ with $\H_{\Delta}$.
It follows that, for any concept class $\H$ 
satisfying the UGC property (Definition~\ref{defn:UGC}), 
the class $\H_{\Delta}$ of pairwise differences also 
satisfies the UGC property.

  From this, we can argue that any such $\H$ must be totally bounded in $L_1(\Px)$, as follows. 
  For any $\epsilon > 0$, the UGC property for $\H_{\Delta}$ implies $\exists m \in \nats$ such that
  \begin{equation*}
    \E\!\left[ \sup_{h,h'\in \H} \left| \Px( x : h(x) \neq h'(x) ) - \hat{P}_{X_{\leq m}}( x : h(x) \neq h'(x) ) \right| \right] \leq \epsilon.
  \end{equation*}
  In particular, this implies that there exists $S \in \X^m$ such that
  \begin{equation*}
    \sup_{h,h'\in \H} \left| \Px( x : h(x) \neq h'(x) ) - \hat{P}_{S}( x : h(x) \neq h'(x) ) \right| \leq \epsilon,
  \end{equation*}
  which, as a special case, implies that any $h,h' \in \H$ with 
  the same classification $h(S) = h'(S)$ of $S$ 
  have $\Px( x : h(x) \neq h'(x) ) \leq \epsilon$.
  Let $\tilde{\H}_{\epsilon} \subseteq \H$ be a subset 
  containing exactly one classifier to witness each classification of $S$ realizable by $\H$:
  that is, $\tilde{\H}_{\epsilon}(S) = \H(S)$ and $|\tilde{\H}_{\epsilon}| = |\H(S)|$.
  It follows that $\tilde{\H}_{\epsilon}$ is an $\epsilon$-cover of $\H$ under the $L_1(\Px)$ pseudo-metric:
  that is, it satisfies \eqref{eqn:epsilon-cover-definition}.
  Since $|\tilde{\H}_{\epsilon}| = |\H(S)| \leq 2^m < \infty$, 
  and this argument holds for any $\epsilon > 0$, 
  we conclude that $\H$ is totally bounded in $L_1(\Px)$.
\end{proof}

Next we present a lemma showing that, for totally bounded classes, 
there exist functions \emph{achieving} the infimum error rate,
within the \emph{closure} of the class under $L_1(\Px)$.
This generalizes a similar result of \citet*{hanneke:12a} for VC classes.

\begin{lemma}
\label{lem:limit-function}
For any concept class $\H$ and any probability measure $\PXY$ on $\X \times \{0,1\}$, 
if $\H$ is totally bounded in $L_1(\Px)$, 
then here exists a measurable function $\target : \X \to \{0,1\}$ (not necessarily in $\H$)
such that
\begin{align*}
  & \inf_{h' \in \H} \Px( x : h'(x) \neq \target(x) ) = 0\\
  & \text{and }~~ \er_{\PXY}(\target) = \inf_{h' \in \H} \er_{\PXY}(h').
\end{align*}
In particular, it follows from Lemma~\ref{lem:UGC-totally-bounded} that this holds for any concept class $\H$ satisfying the universal Glivenko-Cantelli (UGC) property,
and, by Lemma~\ref{lem:no-infinite-VCL-implies-UGC}, 
this further implies that it holds for any $\H$ which does not shatter an infinite VCL tree.
\end{lemma}
\begin{proof}
  We prove this following an argument of \citet*{hanneke:12a}, which originally established this property for VC classes;
  noting that the argument there only relies on the total boundedness property, we can extend the result to any class that is totally bounded in $L_1(P_X)$.  For completeness, we include the complete argument here.

Fix any distribution $\PXY$ on $\X \times \{0,1\}$
and any concept class $\H$ that is totally bounded in $L_1(\Px)$.
Let $\mathrm{cl}(\H)$ denote 
the \emph{closure} of $\H$ in the $L_1(\Px)$ pseudo-metric:
namely, 
the set of all measurable $h : \X \to \{0,1\}$ satisfying 
$\inf_{h' \in \H} \Px( x : h(x) \neq h'(x) ) = 0$.
Note that any $\epsilon$-cover of $\H$ under $L_1(\Px)$
  is also an $\epsilon$-cover of $\mathrm{cl}(\H)$, so that $\mathrm{cl}(\H)$ is also totally bounded.
  Moreover, the entire space $\mathbf{L_1}(\Px)$ of integrable functions is \emph{complete} (Theorem 5.2.1 of \citealp*{dudley:02}),
  and since $\mathrm{cl}(\H)$ is a \emph{closed} subset of $\mathbf{L_1}(\Px)$ in the $L_1(\Px)$ pseudo-metric,
  $\mathrm{cl}(\H)$ is also complete \citep*{munkres:00}.
  Thus, $\mathrm{cl}(\H)$ is complete and totally bounded, hence compact.
  This gives us that, for any sequence $\{ h_t \}_{t \in \nats}$ in $\H$
  with $\er_{\PXY}(h_t) \to \inf_{h \in \H} \er_{\PXY}(h)$ (which must exist by definition of the infimum),
  there exists a subsequence $t_1 < t_2 < \cdots$ such that $\{ h_{t_s} \}_{s \in \nats}$ is a Cauchy sequence in $L_1(\Px)$,
  which, by compactness, must admit a limit point in $\mathrm{cl}(\H)$ under $L_1(\Px)$:
  that is, $\exists \target \in \mathrm{cl}(\H)$ with $\lim_{s \to \infty} \Px( x : h_{t_s}(x) \neq \target(x) ) = 0$.
  By the triangle inequality, we have for any $s \in \nats$, 
  \begin{equation*}
    \er_{\PXY}(\target) \geq \er_{\PXY}(h_{t_s}) - \Px( x : h_{t_s}(x) \neq \target(x) )
  \end{equation*}
  and taking the limit as $s \to \infty$ reveals $\er_{\PXY}(\target) \geq \inf_{h \in \H} \er_{\PXY}(h)$.
  Likewise, by the triangle inequality, 
  \begin{align*}
    \er_{\PXY}(\target) 
    & \leq \inf_{h \in \H} \er_{\PXY}(h) + \Px( x : h(x) \neq \target(x) ) 
    \\ & \leq \lim_{s \to \infty} \er_{\PXY}(h_{t_s}) + \Px( x : h_{t_s}(x) \neq \target(x) ) = \inf_{h \in \H} \er_{\PXY}(h).
  \end{align*}
  This completes the proof.
\end{proof}

\section{Near-Exponential Rate Lower Bound}
\label{sec:near-exponential-lower-bound}

This section begins our analysis leading to 
Theorem~\ref{thm:agnostic-char}, 
starting with the near-exponential lower bound.

\begin{theorem}
\label{thm:near-exponential-lower-agnostic}
For any infinite concept class $\H$, 
for any learning algorithm $\hat{h}_n$, 
there exists $\psi(n) = o(n)$ and a distribution $\PXY$ such that 
$\E\!\left[ \er_{\PXY}(\hat{h}_n) \right] - \inf_{h \in \H} \er_{\PXY}(h) \geq e^{-\psi(n)}$ for infinitely many $n$.
\end{theorem}

We will in fact show that this lower bound holds even for 
distributions $\PXY$ satisfying the \emph{Massart noise} 
condition \citep*{massart2007concentration,massart:06}.
To prove this theorem, we require a result establishing existence of certain structures within any infinite class, 
based on the recent work of \citet*{hanneke:24a}.
Specifically, consider the following definition, 
related to a quantity known as the \emph{eluder dimension} 
introduced by \citet*{russo:13,foster:21} in 
the literature on reinforcement learning and 
adversarial analysis of contextual bandits.

\begin{definition}
\label{defn:eluder}
A sequence $\{(x_i,y_i)\}_{i \in \nats}$ 
in $\X \times \{0,1\}$
is an infinite eluder sequence for a concept class $\H$
if there exist $h_1,h_2,\ldots$ in $\H$ 
such that, 
$\forall i \in \nats$,
$h_i(x_{< i}) = y_{< i}$
and $h_i(x_i) \neq y_i$.
We say $\H$ has an infinite eluder sequence if such a
sequence $(x_i,y_i)$ exists.
\end{definition}

Essentially, an infinite eluder sequence is a 
countable sequence of labeled examples for which, 
for every finite prefix, 
there is an $h \in \H$ correct on that 
prefix but incorrect on the next example.
Importantly for our purposes, 
recent work of \citet*{hanneke:24a} shows that 
any infinite concept class $\H$ has an infinite 
eluder sequence, as stated by the following lemma.

\begin{lemma}[\citealp*{hanneke:24a}]
\label{lem:infinite-classes-have-infinite-eluder}
Any infinite concept class $\H$ has an infinite eluder sequence.
\end{lemma}

The reasoning behind this lemma is a rather straightforward inductive construction of an eluder sequence: 
for each $i = 1,2,\ldots$, 
supposing we inductively maintain that the set 
$\H_{(x_{< i},y_{< i})} := \{ h \in \H : h(x_{< i}) = y_{< i} \}$ 
of prefix-consistent functions is \emph{infinite}, 
we can extend the eluder sequence
by choosing any $x_i \in \X$ for which 
the functions $h \in \H_{(x_{< i},y_{< i})}$ 
do not all agree on $h(x_i)$ (to satisfy the eluder property), 
and by the pigeonhole principle there is a choice of $y_i \in \{0,1\}$
for which $\H_{(x_{\leq i},y_{\leq i})}$ 
is still infinite (to maintain the inductive invariant).

We are now ready for the proof of Theorem~\ref{thm:near-exponential-lower-agnostic}, 
establishing that no algorithm can achieve better than a near-exponential rate for any infinite concept class.

\stilltodo{TODO: In the journal submission, extract the sketch for the 25 pages and move the proof to an appendix.  The proof already contains such a sketch, after the construction, so we can just use that.}

\begin{proof}[of Theorem~\ref{thm:near-exponential-lower-agnostic}]
By Lemma~\ref{lem:infinite-classes-have-infinite-eluder},
$\H$ has an infinite eluder sequence 
$(x_1,y_1),(x_2,y_2),\ldots$ in $\X \times \{0,1\}$.
Let $h_1,h_2,\ldots$ be the sequence in $\H$ witnessing this
(as in Definition~\ref{defn:eluder}):
that is, $\forall i \in \nats$, $h_i(x_{< i}) = y_{< i}$ and $h_i(x_i) \neq y_i$.
Noting that these $x_i$'s must all be distinct, 
let $h_0 : \X \to \{0,1\}$ be any measurable function
(not necessarily in $\H$)
with $h_0(x_i) = y_i$ for all $i \in \nats$.
Fix any $\beta \in (0,1/2)$,
and for every $j \in \nats$ define $p_j = 2^{-j}$.
Define a family of distributions $P_i$, $i \in \nats \cup \{0\}$, as follows.
Let $P_0$ denote a probability measure on $\X \times \{0,1\}$ such that,
for $(X,Y) \sim P_0$,
each $j \in \nats$ has $\P(X = x_j) = p_j$
and $\P(Y=y_j|X=x_j) = 1-\beta$.
For any $i \in \nats$,
define $P_i$ as the probability measure on $\X \times \{0,1\}$ such that,
for $(X,Y) \sim P_i$,
each $j \in \nats$ has $\P(X = x_j) = p_j$,
and
\begin{equation*}
\P(Y=h_i(X)|X=x_j) =
  \begin{cases}
    1-\beta & \text{ if } j < i\\
    1 & \text{ if } j \geq i
  \end{cases}.
\end{equation*}

Since $\sum_j p_j = 1$, this completely specifies the probability measures $P_i$, $i \in \nats \cup \{0\}$.
Define a sequence of independent $P_0$-distributed random variables: 
$(X_1,Y_1^{(0)}),(X_2,Y_2^{(0)}),\ldots$,
and for each $i,t \in \nats$, 
define 
\begin{equation*} 
Y_t^{(i)} = 
\begin{cases}
Y_t^{(0)}, & \text{ if } X_t \in \{x_1,\ldots,x_{i-1}\}\\
h_i(X_t), & \text{ otherwise }
\end{cases}.
\end{equation*}
Note that, for each $i \in \nats$, 
$(X_1,Y_1^{(i)}),(X_2,Y_2^{(i)}),\ldots$ 
is a sequence of independent $P_i$-distributed random variables.
%
%
Let $f_n : (\X \times \{0,1\})^n \times \X \to \{0,1\}$ ($n \in \nats$) be any learning algorithm,
and denote by $\hat{h}_n^{(i)} = f_n(X_1,Y^{(i)}_1,\ldots,X_n,Y^{(i)}_n)$ for each $n,i \in \nats \cup \{0\}$: 
that is $\hat{h}_n^{(i)}$ is the classifier returned by 
the algorithm trained under $P_i$.

\paragraph{Sketch:} Before presenting the formal proof, we first briefly outline the argument.
For 
$n,i \geq 1$, 
denote by $E_{n,i}$ the event that 
$(X_{\leq n},Y_{\leq n}^{(0)}) = (X_{\leq n},Y_{\leq n}^{(i)})$.
In particular, on $E_{n,i}$, 
$\hat{h}_n^{(0)} = \hat{h}_n^{(i)}$.
Moreover, 
since $h_0(x_i) = y_i \neq h_i(x_i)$,
the event $E_{n,i}$ forces the learner to \emph{guess}
the appropriate classification of $x_i$ 
\emph{without any distinguishing information}
to tell whether the true distribution is $P_0$ or $P_i$.
We also note that $E_{n,i}$ is equivalent to the event that every 
$t \leq n$ with $X_t \notin \{x_1,\ldots,x_{i-1}\}$
has $Y^{(0)}_t = h_i(X_t)$;
in particular, since this only concerns the variables $Y_t^{(0)}$ which are \emph{replaced} in the definition of $Y_t^{(i)}$, 
this implies the event $E_{n,i}$ is
\emph{conditionally independent} of $(X_{\leq n},Y_{\leq n}^{(i)})$
given $X_{\leq n}$.
Therefore, 
$\P\!\left( \hat{h}_n^{(0)}(x_i) \neq h_0(x_i) \middle| X_{\leq n} \right)
\geq \P\!\left( \hat{h}_n^{(i)}(x_i) = h_i(x_i) \land E_{n,i} \middle| X_{\leq n} \right)
= \P\!\left( \hat{h}_n^{(i)}(x_i) = h_i(x_i) \middle| X_{\leq n} \right) \P\!\left( E_{n,i} \middle| X_{\leq n} \right)$.
Denoting by $\hat{N}_{n,i}$ the number of
$t \leq n$ with 
$X_t \notin \{ x_1,\ldots, x_{i-1}\}$, 
we have 
$\P\!\left( E_{n,i} \middle| X_{\leq n} \right) \geq \beta^{\hat{N}_{n,i}}$.
A Chernoff bound 
reveals that with probability at least $1-e^{- 2 p_i n / 3}$, 
$\hat{N}_{n,i} \leq 4 p_i n$.
Altogether, 
\begin{align*} 
\E\!\left[ \er_{P_0}\!\left(\hat{h}_n^{(0)}\right)\right] - \er_{P_0}(h_0)
& \geq (1-2\beta) p_i \P\!\left( \hat{h}_n^{(0)}(x_i) \neq h_0(x_i) \right) 
\\ & \geq (1-2\beta) p_i \left( 1 - \P\!\left( \hat{h}_n^{(i)}(x_i) \neq h_i(x_i) \right) - e^{-2 p_i n / 3} \right) \beta^{4 p_i n}.
\end{align*}
In particular, since w.l.o.g.\ we have 
$\P\!\left( \hat{h}^{(i)}_n(x_i) \neq h_i(x_i) \right) \to 0$
(otherwise the claimed lower bound is trivially satisfied for $P_i$),
this implies there exists a sequence $i_n \to \infty$
for which
$\E\!\left[ \er_{P_0}\!\left(\hat{h}_n^{(0)}\right)\right] - \er_{P_0}(h_0) \geq (1-o(1)) (1-2\beta) p_{i_n} \beta^{4 p_{i_n} n} = e^{-o(n)}$.

\paragraph{Formal Proof:}
We now proceed with the formal proof, following the above outline.
For each $n,i \in \nats$, denote by
\begin{equation*}
  \epsilon^{(i)}_n = \E\!\left[ \er_{P_i}(\hat{h}^{(i)}_n) \right] - \inf_{h \in \H} \er_{P_i}(h).
\end{equation*}
There are two cases.
First, if $\exists i \in \nats$ such that
\begin{equation}
  \label{eqn:i-convergence-failure}
\limsup_{n \to \infty} \epsilon^{(i)}_n \neq 0,
\end{equation}
then the claimed lower bound holds for this algorithm trivially, by taking distribution $\PXY = P_i$
for the $i$ satisfying \eqref{eqn:i-convergence-failure}:
that is, in this case, there is a distribution 
for which the excess risk does not merely converge slowly, 
but in fact does not even converge to $0$.
Formally, for any such learning algorithm $\hat{h}_n$, 
taking $\PXY = P_i$ and 
$\psi(n) = \log\!\left( 2 / \limsup_{n \to \infty} \epsilon^{(i)}_n \right)$ (which is bounded, hence $o(n)$)
establishes the claimed lower bound 
by definition of the $\limsup$.

To address the remaining case,
suppose every $i \in \nats$ satisfies
\begin{equation}
  \label{eqn:i-convergence}
\limsup_{n \to \infty} \epsilon^{(i)}_n = 0.
\end{equation}
In particular, note that $\er_{P_i}(h)$ is minimized 
(over all measurable $h$) by $h = h_i$,
and any $h$ with $h(x_i) \neq h_i(x_i)$ satisfies
\begin{equation*}
\er_{P_i}(h) - \er_{P_i}(h_i) \geq p_i.
\end{equation*}
Thus, by Markov's inequality, $\forall n,i \in \nats$, 
\begin{equation}
\label{eqn:i-convergence-xi-bound}
\P\!\left( \hat{h}^{(i)}_n(x_i) \neq h_i(x_i) \right) 
\leq \P\!\left( \er_{P_i}\!\left(\hat{h}^{(i)}_n\right) - \er_{P_i}(h_i) \geq p_i \right)
\leq \frac{1}{p_i} \epsilon^{(i)}_n.
\end{equation}

On the other hand, noting that 
$\inf_{h \in \H} \er_{P_0}(h) = \er_{P_0}(h_0)$,
and any $h$ with $h(x_i) = h_i(x_i)$ 
must have $h(x_i) \neq h_0(x_i)$ 
(by definition of $h_i$), 
any such $h$ satisfies 
\begin{equation*}
\er_{P_0}(h) - \inf_{h' \in \H} \er_{P_0}(h')
= \er_{P_0}(h) - \er_{P_0}(h_0)
\geq (1-2\beta) p_i.
\end{equation*}
Thus, $\forall n,i \in \nats$, 
\begin{equation}
\label{eqn:near-exp-lb-P0-excess-risk-lower-bound}
\E\!\left[ \er_{P_0}\!\left( \hat{h}^{(0)}_n \right) \right] - \inf_{h \in \H} \er_{P_0}(h)
\geq (1-2\beta) p_i \P\!\left( \hat{h}^{(0)}_n(x_i) = h_i(x_i) \right).
\end{equation}

Next we establish a relation between the 
bounds \eqref{eqn:i-convergence-xi-bound} and \eqref{eqn:near-exp-lb-P0-excess-risk-lower-bound}
by arguing that $\hat{h}_n^{(i)} = \hat{h}_n^{(0)}$ 
on an event of probability $e^{-o(n)}$.
For any $n, i \in \nats$, denote by
\begin{equation*}
\hat{N}_{n,i} = |\{ t \leq n : X_t \in \{x_j : j \geq i\} \} |
\end{equation*}
and let $t_{i,1},t_{i,2},\ldots$ denote the subsequence of 
$t \in \{1,2,\ldots\}$ satisfying $X_t \in \{x_j : j \geq i\}$
(which may be infinite or finite, though the former occurs almost surely).
Let $E_{n,i}$ denote the event that
\begin{equation*}
\forall k \in \{1,\ldots,\hat{N}_{n,i}\}, Y^{(0)}_{t_{i,k}} = h_i(X_{t_{i,k}}).
\end{equation*}
Note that on $E_{n,i}$, we have $(X_{\leq n},Y_{\leq n}^{(0)}) = (X_{\leq n},Y_{\leq n}^{(i)})$ 
and hence $\hat{h}_n^{(0)} = \hat{h}_n^{(i)}$.
%
Recalling the definition of $P_0$, 
we note that by conditional independence of the $Y_t^{(0)}$ labels
given $X_{\leq n}$, 
\begin{equation}
  \label{eqn:Pi-vs-P0-event-conditional-prob}
  \P\!\left( E_{n,i} \middle| X_{\leq n} \right)
  \geq \beta^{\hat{N}_{n,i}}.
\end{equation}

Since $\hat{h}_n^{(0)} = \hat{h}_n^{(i)}$
on the event $E_{n,i}$, 
\begin{align}
\P\!\left( \hat{h}^{(0)}_n(x_i) = h_i(x_i) \right)
& \geq \P\!\left( \left\{ \hat{h}^{(i)}_n(x_i) = h_i(x_i) \right\} \cap E_{n,i} \right)
\notag \\ & = \E\!\left[ \P\!\left( \left\{ \hat{h}^{(i)}_n(x_i) = h_i(x_i) \right\} \cap E_{n,i} \middle| X_{\leq n} \right) \right]. \label{eqn:near-exp-lb-h0-xi-lb}
\end{align}
Recalling the definition of $Y^{(i)}_{t}$, 
we note that the labels $Y^{(0)}_{t_{i,k}}$, $k \in \{1,\ldots,\hat{N}_{n,i}\}$,
are conditionally independent of $(X_{\leq n},Y_{\leq n}^{(i)})$
given $X_{\leq n}$.
Since whether or not the event $E_{n,i}$ holds 
is determined entirely by $X_{\leq n}$ 
and these $Y^{(0)}_{t_{i,k}}$ variables, 
we have that 
$\hat{h}_n^{(i)}(x_i)$ is conditionally independent of $E_{n,i}$
given $X_{\leq n}$.
Thus, the expression in \eqref{eqn:near-exp-lb-h0-xi-lb} is equal to
\begin{align*}
& \E\!\left[ \P\!\left( \hat{h}^{(i)}_n(x_i) = h_i(x_i) \middle| X_{\leq n} \right) \P\!\left( E_{n,i} \middle| X_{\leq n} \right) \right],
\end{align*}
which, by \eqref{eqn:Pi-vs-P0-event-conditional-prob}, 
is no smaller than
\begin{align}
& \E\!\left[ \P\!\left( \hat{h}^{(i)}_n(x_i) = h_i(x_i) \middle| X_{\leq n} \right) \beta^{\hat{N}_{n,i}} \right]
\notag \\ & \geq \E\!\left[ \P\!\left( \hat{h}^{(i)}_n(x_i) = h_i(x_i) \middle| X_{\leq n} \right) \ind\!\left[ \hat{N}_{n,i} \leq 4 p_i n \right] \right] \beta^{4 p_i n}
\notag \\ & \geq \E\!\left[ 1 - \P\!\left( \hat{h}^{(i)}_n(x_i) \neq h_i(x_i) \middle| X_{\leq n} \right) - \ind\!\left[ \hat{N}_{n,i} > 4 p_i n \right] \right] \beta^{4 p_i n}
\notag \\ & = \left( 1 - \P\!\left( \hat{h}^{(i)}_n(x_i) \neq h_i(x_i) \right) - \P\!\left( \hat{N}_{n,i} > 4 p_i n \right) \right) \beta^{4 p_i n}. \label{eqn:near-exp-lb-almost-done}
\end{align}
Since each $X_t$ has 
$\P( X_t \in \{ x_j : j \geq i \} ) = \sum_{j \geq i} p_j = 2 p_i$, 
and $X_1,X_2,\ldots$ are independent,  
a Chernoff bound implies that
$\P\!\left( \hat{N}_{n,i} > 4 p_i n \right) \leq e^{-2 p_i n / 3}$.
Together with \eqref{eqn:i-convergence-xi-bound}, 
we have that \eqref{eqn:near-exp-lb-almost-done} is no smaller than
\begin{equation*}
\left( 1 - \frac{1}{p_i} \epsilon^{(i)}_n - e^{-2 p_i n / 3} \right) \beta^{4 p_i n}.
\end{equation*}
Together with \eqref{eqn:near-exp-lb-P0-excess-risk-lower-bound}, 
we have thus established that, for every $n,i \in \nats$, 
\begin{equation}
\label{eqn:near-exp-lb-single-i-lb}
\E\!\left[ \er_{P_0}\!\left( \hat{h}^{(0)}_n \right) \right] - \inf_{h \in \H} \er_{P_0}(h)
\geq (1-2\beta) p_i \left( 1 - \frac{1}{p_i} \epsilon^{(i)}_n - e^{-2 p_i n / 3} \right) \beta^{4 p_i n}.
\end{equation}

Note that, for any $i \in \nats$, 
since $\epsilon^{(i)}_n \to 0$ and $e^{-2 p_i n / 3} \to 0$ 
as $n \to \infty$, 
there exists $n_i \in \nats$ such that every $n \geq n_i$ 
satisfies 
\begin{equation*}
\frac{1}{p_i} \epsilon^{(i)}_n + e^{-2 p_i n / 3} \leq \frac{1}{2}.
\end{equation*}
For any $n \geq n_1$, 
define 
\begin{equation*} 
i_n = \max\{ i \in \nats : n \geq \max\{ n_i, i^2 \} \}.
\end{equation*}
Due to the $n \geq i^2$ constraint, this is a well-defined finite value, 
and moreover, since each $n_i$ is finite, 
we have $i_n \to \infty$ as $n \to \infty$.
Since \eqref{eqn:near-exp-lb-single-i-lb} holds for every $n,i \in \nats$, 
we have for any $n \geq n_1$
\begin{equation}
\label{eqn:near-exp-lb-i_n-lb}
\E\!\left[ \er_{P_0}\!\left( \hat{h}^{(0)}_n \right) \right] - \inf_{h \in \H} \er_{P_0}(h)
\geq \frac{1-2\beta}{2} p_{i_n} \beta^{4 p_{i_n} n}.
\end{equation}
Defining $\psi(n) = \ln\!\left(\frac{2}{1-2\beta}\right) + \ln\!\left(\frac{1}{p_{i_n}}\right) + 4 p_{i_n} n \ln\!\left(\frac{1}{\beta}\right)$, 
the right hand side of \eqref{eqn:near-exp-lb-i_n-lb} 
is equal $e^{-\psi(n)}$.
Finally, since the definition of $i_n$ implies $i_n \leq \sqrt{n}$, 
we have $\ln\!\left(\frac{1}{p_{i_n}}\right) = i_n \ln(2) = o(n)$
and since $i_n \to \infty$ implies $p_{i_n} \to 0$, 
we have $4 p_{i_n} n \ln\!\left(\frac{1}{\beta}\right) = o(n)$ as well.
Since the remaining term in $\psi(n)$ is an additive constant, 
we conclude that $\psi(n) = o(n)$.
This completes the proof.
\end{proof}

\section{Near-Exponential Rate Upper Bound}
\label{sec:near-exponential}

Continuing toward establishing Theorem~\ref{thm:agnostic-char}, 
we next establish the near-exponential rate upper bound 
for concept classes which do not shatter an infinite Littlestone tree.

\begin{theorem}
\label{thm:agnostic-near-exponential-upper-bound}
If $\H$ does not shatter an infinite Littlestone tree,
then 
for every $\psi(n) = o(n)$, 
$\H$ is agnostically learnable at rate $e^{-\psi(n)}$.
\end{theorem}

The formal proof of Theorem~\ref{thm:agnostic-near-exponential-upper-bound}
is presented in Section~\ref{sec:near-exponential-rate-upper-bound-proof} below.
Before getting into the details, 
we first briefly outline the strategy underlying the algorithm 
$\hat{h}_n$ achieving this $e^{-\psi(n)}$ rate, 
for any given $\psi(n) = o(n)$.
We divide the approach into two subroutines, 
intending to address separately the cases where 
$\PXY$ is or is not \emph{Bayes-realizable}, 
recalling that we say $\PXY$ is Bayes-realizable with respect to $\H$ if $\inf_{h \in \H} \er_{\PXY}(h) = \inf_{h} \er_{\PXY}(h)$ (where the latter infimum ranges over all measurable $h : \X \to \{0,1\}$).
The first subroutine, 
designed to address the case where $\PXY$ is not 
Bayes-realizable w.r.t.\ $\H$, 
simply 
trains a universally Bayes-consistent learning algorithm $\hat{h}^0_n$ (Lemma~\ref{lem:universally-bayes-consistent}).
A second subroutine $\hat{h}^1_n$, 
outlined below, is designed to address the case 
where $\PXY$ is Bayes-realizable w.r.t.\ $\H$.
We then select between $\hat{h}^0_n$ and $\hat{h}^1_n$,
via a hypothesis test based on held-out data,
to obtain the final predictor $\hat{h}_n$.
Thus, the main technical novelty in the approach is in the design of the $\hat{h}^1_n$ subroutine.

To understand the main ideas behind the $\hat{h}^1_n$ algorithm, 
it will be instructive to first recall the original realizable-case technique of \citet*{bousquet:21} establishing $e^{-n}$ rates, 
before discussing the modifications needed to extend the technique to the agnostic setting.

\paragraph{The Realizable Case:}
In the case of a distribution $\PXY$ realizable with respect to $\H$ (where $\H$ does not shatter an infinite Littlestone tree), 
the algorithm of \citet*{bousquet:21}
makes use of the procedure $\algSOA$
from Lemma~\ref{lem:SOA-zero-error-rate-bstar}
(which itself is based on a measurable winning strategy for an associated Gale-Stewart game).
Recall that Lemma~\ref{lem:SOA-zero-error-rate-bstar}
guarantees that $\exists b^* \in \nats$ such that,
for any $b \geq b^*$, for an i.i.d.\ data set $B^b \sim \PXY^b$, 
with probability at least $\frac{99}{100}$, 
the classifier defined by $\algSOA(B^b)$ has error rate \emph{zero}.
To convert this $\frac{99}{100}$ guarantee into an 
$e^{-n}$ rate, 
they use a fraction $m = n/2$ of the training examples $S^0_n$, 
break them into $m/b$ disjoint \emph{batches} $B^b_i$ ($i \in \{1,\ldots,m/b\}$), 
which induce predictors $h_{b,i} = \algSOA(B^b_i)$.
By a Chernoff bound, with probability $1-e^{-\Omega(m/b)}$, 
at least $\frac{9}{10}$ of these have $\er_{\PXY}(h_{b,i})=0$,
and thus the \emph{majority vote} predictor 
$x \mapsto \mathrm{Majority}(h_{b,1}(x),\ldots,h_{b,m/b}(x))$
would also have error rate zero.

All of this holds for any $b \geq b^*$.
However, since $b^*$ is $\PXY$-dependent, 
their method also needs to select an appropriate $b$
without knowledge of $b^*$.
For this, they construct the above predictors $h_{b,i}$
for \emph{every} $b \leq m$, and use the remaining 
$m = n/2$ training examples $S^1_n$ to \emph{select} which $b$ to use: 
namely, $\hat{b}$ is the minimum $b$ for which 
at least $\frac{9}{10}$ of the $h_{b,i}$ predictors 
have no mistakes on $S^1_n$.
From the above Chernoff bound, we know $b^*$ satisfies 
this, so $\hat{b} \leq b^*$.
Moreover, for any $b < b^*$ having 
$\P( \er_{\PXY}(h_{b,i}) > 0 ) > \frac{3}{10}$, 
there is some $\epsilon_b > 0$ such that 
$\P( \er_{\PXY}(h_{b,i}) > \epsilon_b ) > \frac{2.5}{10}$,
so that (by a Chernoff bound) with probability at least $1 - e^{-\Omega(m/b)}$
at least $\frac{2}{10}$ of these $h_{b,i}$ have 
$\er_{\PXY}(h_{b,i}) > \epsilon_b$, 
and hence (by a union bound) 
with probability at least $1 - \frac{m}{5b} e^{- \epsilon_b m}$
each of these will make at least one mistake on $S^1_n$.
Altogether, with probability $1 - \frac{m}{5b} e^{-\epsilon_b m} - e^{-\Omega(m/b)} = 1 - e^{-\Omega(n)}$, 
we conclude $\hat{b} \neq b$ for any such $b$.
By a union bound over all such $b < b^*$, 
we conclude that with probability at least $1-e^{-\Omega(n)}$, 
$\hat{b}$ is a value $b$ for which 
$\P( \er_{\PXY}(h_{b,i}) = 0 ) \geq \frac{7}{10}$, 
and the Chernoff argument regarding the majority 
vote yields the desired guarantee: 
i.e., for $\hat{h} = \mathrm{Majority}(h_{\hat{b},1},\ldots,h_{\hat{b},m/\hat{b}})$,
with probability at least $1 - e^{-\Omega(n)}$, 
$\er_{\PXY}(\hat{h}) = 0$, 
implying $\E\!\left[ \er_{\PXY}(\hat{h}) \right] = e^{-\Omega(n)}$.

\paragraph{The Agnostic Case:}
To extend the above technique to the agnostic case, 
as mentioned above, as the main ingredient,
the algorithm $\hat{h}^1_n$ only needs to address the case where $\PXY$ is 
Bayes-realizable with respect to $\H$: 
that is, $\inf_{h \in \H} \er_{\PXY}(h) = \inf_{h} \er_{\PXY}(h)$.
The two main challenges in extending the technique of 
\citet*{bousquet:21} to distributions $\PXY$ 
which are merely Bayes-realizable (rather than realizable)
are that the batches $B^b_i$ are no longer realizable, 
and when selecting an appropriate $b$, 
even the good predictors $h_{b,i}$ will still make 
\emph{some} mistakes on the held-out data set $S^1_n$.

The principle we adopt in approaching the first of these issues is to \emph{simulate} having batches $B^b_i$
sampled from an appropriate \emph{realizable} distribution.
Specifically, due to Lemmas~\ref{lem:no-infinite-VCL-implies-UGC}, \ref{lem:UGC-totally-bounded}, and \ref{lem:limit-function}, 
in the Bayes-realizable case there exists a measurable function 
$\target : \X \to \{0,1\}$ 
with $\er_{\PXY}(\target) = \inf_{h} \er_{\PXY}(h)$
and $\inf_{h \in \H} \Px( x : h(x) \neq \target(x) ) = 0$.
If we then denote by $P_{\target}$ the 
distribution of $(X,\target(X))$ (for $X \sim \Px$), 
note that $P_{\target}$ is \emph{realizable} with respect to $\H$,
and any $h'$ with $\er_{P_{\target}}(h') = 0$ has 
$\er_{\PXY}(h') = \er_{\PXY}(\target) = \inf_{h} \er_{\PXY}(h)$.
We are therefore interested in applying $\algSOA$ to 
batches $B^b_i$ sampled i.i.d.\ from this realizable distribution $P_{\target}$.

Since $\target$ is not known to the learner, 
we will in fact apply $\algSOA$ to an entire \emph{collection}
of data sets, one of which is indeed $P_{\target}$-distributed.
Namely, for each $b \leq m$ and $i \leq m/b$, 
we construct $2^b$ data sets: 
for each $\mathbf{y} = (y_1,\ldots,y_b) \in \{0,1\}^b$, define $B^b_i(\mathbf{y})$ as $B^b_i$ but with labels 
$Y$ \emph{replaced} by the hypothetical label sequence $y_t$.
We then define $h^{\mathbf{y}}_{b,i} = \algSOA(B^b_i(\mathbf{y}))$.
Since one of these $\mathbf{y}$ sequences \emph{agrees} 
with the $\target(X)$ values of all $X$ in $B^b_i$, 
for $b \geq b^*$ we can guarantee that 
with probability at least $\frac{99}{100}$
at least one $\mathbf{y}$ has 
$\er_{P_{\target}}(h^{\mathbf{y}}_{b,i}) = 0$, 
and hence $\er_{\PXY}(h^{\mathbf{y}}_{b,i}) = \inf_{h} \er_{\PXY}(h)$.
In particular, by a Chernoff bound, 
this occurs for at least $\frac{9}{10}$ of the batches,
with probability $1 - e^{-\Omega(m/b)}$.

Moreover, for each such $b,i$, we argue it is possible 
to use the held-out $m$ examples $S^1_n$
to \emph{select} an $h^{\mathbf{y}}_{b,i}$ for which $\er_{\PXY}(h^{\mathbf{y}}_{b,i}) = \inf_{h} \er_{\PXY}(h)$.
For this, we note that, for any $b$, 
there exists $\gamma_{b} > 0$ 
such that, with probability at least $\frac{99}{100}$, 
every $\mathbf{y}$ with 
$\er_{\PXY}(h^{\mathbf{y}}_{b,i}) > \inf_{h} \er_{\PXY}(h)$
in fact has 
$\er_{\PXY}(h^{\mathbf{y}}_{b,i}) > \gamma_b + \inf_{h} \er_{\PXY}(h)$: i.e., its suboptimality has at least some \emph{gap}.
By a Chernoff bound, this holds for at least $\frac{9}{10}$ of the batches with probability 
$1 - e^{-\Omega(m/b)}$.
In particular, combining these two parts, 
for any $b \geq b^*$, 
a standard Hoeffding and union bound argument implies that, for sufficiently large $n$, 
simply taking $h_{b,i} = \argmin_{h^{\mathbf{y}}_{b,i}} \hat{\er}_{S^1_n}(h^{\mathbf{y}}_{b,i})$
suffices to guarantee $\er_{\PXY}(h_{b,i}) = \inf_{h} \er_{\PXY}(h)$.
In particular, when this holds, returning 
$\hat{h}^1_n = \mathrm{Majority}(h_{b,1},\ldots,h_{b,m/b})$
also satisfies $\er_{\PXY}(\hat{h}^1_n) = \inf_{h} \er_{\PXY}(h)$.
Thus, the main remaining issue is selecting an 
appropriate $b$ without direct knowledge of 
the $\PXY$-dependent value $b^*$.
%

In general (both for the above argument about $h_{b,i}$, and for selecting $b$)
we make use of a concentration inequality for the error rates of the $h^{\mathbf{y}}_{b,i}$ classifiers (equation \ref{eqn:exponential-pf-concentration-hbi} in the proof).
Specifically, by Hoeffding's inequality and the union bound, 
with probability at least $1 - e^{-\psi(n)}$, 
for every $b \leq m$, $i \leq m/b$, and $\mathbf{y} \in \{0,1\}^b$, 
\begin{equation}
\label{eqn:sketch-exponential-concentration-hbi}
\left| \hat{\er}_{S^1_n}\!\left(h^{\mathbf{y}}_{b,i}\right) - \er_{\PXY}\!\left(h^{\mathbf{y}}_{b,i}\right) \right| \leq c \sqrt{\frac{\psi(n) + b + \ln(n)}{n}}
\end{equation}
for a universal constant $c$.
In particular, for $b = b^*$ and $n$ sufficiently large to make the right hand side $< \gamma_b/2$, 
this inequality supplies the stated 
guarantee 
$\er_{\PXY}(h_{b,i}) = \inf_{h} \er_{\PXY}(h)$
for the $\frac{9}{10}$ of batches satisfying 
both the 
$\gamma_b$ gap condition 
and the existence of a good $\mathbf{y}$, 
as described above.

We also use \eqref{eqn:sketch-exponential-concentration-hbi} to \emph{select} the value $b$.
Specifically, note that for $b = b^*$, 
on the above events, 
at least $\frac{9}{10}$ of the functions $h_{b,i}$
(being optimal functions) have the property that, 
for every $b' > b$ and $i' \leq m/b'$, 
$\hat{\er}_{S^1_n}(h_{b,i}) \leq \hat{\er}_{S^1_n}(h_{b',i'}) + 2 c \sqrt{\frac{\psi(n) + b' + \ln(n)}{n}}$.
We define $\hat{b}$ as the smallest $b$
for which this holds.
In particular, this implies $\hat{b} \leq b^*$.
Moreover, since any $b < b^*$ has 
$b^*$ as one of these $b'$ values we compare to, 
since at least $\frac{9}{10}$ of the $b$ batches 
satisfy the $\gamma_b$ gap condition discussed above, 
for $n$ sufficiently large (to make the concentration inequality smaller than $\gamma_{b}/4$) 
these comparisons 
will not select $\hat{b} = b$ for any $b < b^*$ 
for which the majority of $h_{b,i}$ functions 
have $\er_{\PXY}(h_{b,i}) > \inf_{h} \er_{\PXY}(h)$.
Altogether, we conclude that 
for $\hat{h}^1_n = \mathrm{Majority}(h_{\hat{b},1},\ldots,h_{\hat{b},m/\hat{b}})$, 
with probability at least $1 - e^{-\psi(n)} - e^{-\Omega(m/b^*)}$, 
$\er_{\PXY}(\hat{h}^1_n) = \inf_{h} \er_{\PXY}(h)$, 
which therefore implies the claimed rate $e^{-\psi(n)}$.

We present the formal details next, in Section~\ref{sec:near-exponential-rate-upper-bound-proof}.

\subsection{Proof of Theorem~\ref{thm:agnostic-near-exponential-upper-bound}: $e^{-o(n)}$ Rates}
\label{sec:near-exponential-rate-upper-bound-proof}

We now present the formal details of the proof of Theorem~\ref{thm:agnostic-near-exponential-upper-bound}.

\begin{proof}[of Theorem~\ref{thm:agnostic-near-exponential-upper-bound}] 
  Fix any $\psi : \nats \to \nats$ satisfying $\psi(n) = o(n)$.
  We will establish that $\H$ is agnostically learnable at rate $e^{-\psi(n)}$.
  Without loss of generality suppose $\psi(n) \to \infty$
  (establishing the claim for such functions also implies it 
  for functions $\psi(n) = o(n)$ with $\psi(n) \not\to \infty$
  by applying the claim to $\psi'(n) = \psi(n)\lor \sqrt{n}$).  
  We begin by describing the learning algorithm $f_n$.
  Let $\PXY$ be any distribution on $\X \times \{0,1\}$, let $n \in \nats$,
  and let $S_n = \{(X_1,Y_1),\ldots,(X_n,Y_n)\} \sim \PXY^n$ be the i.i.d.\ training set input to the learning algorithm.
  For simplicity, suppose $n$ is an integer multiple of $3$; otherwise, replace $n$ with $3 \lfloor n/3 \rfloor$ in the description and analysis below (supposing $n \geq 3$).
  We describe the construction of the function $\hat{h}_n = f_n(S_n)$
  produced by $f_n$ when given training set $S_n$.
  We first segment the data $S_n$ into chunks of size $n/3$: for $k \in \{0,1,2\}$, let
  \begin{equation*}
    S_n^k = \{(X_{1 + k n/3},Y_{1 + k n/3}),\ldots,(X_{(k+1)n/3},Y_{(k+1)n/3})\}.
  \end{equation*}

  The algorithm has two main components, to handle two types of distributions $\PXY$.
  Formally, recall that we say $\PXY$ is \emph{Bayes-realizable} 
  with respect to $\H$ 
  if $\inf_{h \in \H} \er_{\PXY}(h) = \inf_{h} \er_{\PXY}(h)$, 
  where $h$ on the right hand side ranges over all measurable 
  functions $\X \to \{0,1\}$.  
  The first component in the learning algorithm is designed to handle 
  the case that $\PXY$ is \emph{not} Bayes-realizable with respect to $\H$.
  Specifically, let $\hat{h}_n^0$ be the classifier returned by a universally Bayes-consistent learning algorithm with training set $S_n^0$ (as in Lemma~\ref{lem:universally-bayes-consistent}).
  Thus, we have that 
  $\E[\er_{\PXY}(\hat{h}_n^0)] \to \inf_{h} \er_{\PXY}(h)$ for all choices of the distribution $\PXY$.
  The intention behind this is that, in the arguments below,
  it will be convenient to focus on the case of $\PXY$ that is 
  Bayes-realizable with respect to $\H$.
  We will construct a predictor $\hat{h}_n^1$ below,
  which is guaranteed to have the required rate when $\PXY$ is Bayes-realizable with respect to $\H$.
  When this is not the case, the universally consistent 
  learner $\hat{h}_n^0$ will have excess risk 
  (relative to the infimum error among $h \in \H$) which is 
  in fact converging to a \emph{negative} value, 
  so that the non-Bayes-realizable case is handled 
  by $\hat{h}_n^0$.
  With both $\hat{h}_n^0$ and $\hat{h}_n^1$ defined,
  we can then merely select between these two functions 
  in the end (using additional independent samples) 
  to always guarantee the required rate for the excess risk
  (which may become negative in the case $\PXY$ is not Bayes-realizable with respect to $\H$).

  Next we describe the second component, defining the aforementioned predictor $\hat{h}_n^1$.
  For each $b \in \{1,\ldots,n/3\}$, let us further segment the sequence $S_n^0$ into \emph{batches} $B^b_i$ of size $b$:
  let $I_b = \{1,\ldots,\lfloor n/(3b) \rfloor\}$,
  and for each $i \in I_b$, define 
  \begin{equation*}
    B^b_i = \left\{(X_{1+(i-1)b},Y_{1+(i-1)b}),\ldots,(X_{i b},Y_{i b})\right\}. 
  \end{equation*}
  For each $i \in I_b$,
  and each $\mathbf{y} = (y_1,\ldots,y_b) \in \{0,1\}^b$, 
  define 
  \begin{equation*} 
  B^b_i(\mathbf{y}) = \{(X_{1+(i-1)b},y_1),\ldots,(X_{i b},y_b)\},
  \end{equation*}
  that is, $B^b_i$, but with the $Y$ labels replaced by $y$ labels.
  Let $\algSOA$ be the function defined in Lemma~\ref{lem:SOA-zero-error-rate-bstar} for the class $\H$. 
  For each $b \in \{1,\ldots,n/3\}$, $i \in I_b$, and $\mathbf{y} \in \{0,1\}^b$,  
  we define a function $h^{\mathbf{y}}_{b,i} : \X \to \{0,1\}$ 
  as follows: for every $x \in \X$,
  \begin{equation*} 
  h^{\mathbf{y}}_{b,i}(x) = \algSOA(B^b_i(\mathbf{y}), x).
  \end{equation*}

  We will next use the data subset $S_n^1$ to select one function 
  $h_{b,i}$ for each batch $B^b_i$, 
  from among these $h^{\mathbf{y}}_{b,i}$ functions.
  Specifically, for each $b \in \{1,\ldots,n/3\}$ 
  and $i \in I_b$,
  define 
  \begin{equation*}
   h_{b,i}  = \argmin_{h^{\mathbf{y}}_{b,i} : \mathbf{y} \in \{0,1\}^b} \hat{\er}_{S_n^1}(h^{\mathbf{y}}_{b,i}),
  \end{equation*}
  breaking ties in any way that preserves measurability 
  of the function mapping $B^b_i$ and any $x$
  to $h_{b,i}(x) \in \{0,1\}$
  (e.g., breaking ties lexicographically by $\mathbf{y}$ suffices).

  We will also use $S_n^1$ to select an appropriate batch size $\hat{b}_n$ as follows.
  For each $b \in \{1,\ldots,n/3\}$, define 
  $\GoodI(b)$ as the set of all $i \in I_b$ such that,
  $\forall b' \in \{1,\ldots,n/3\}$ with $b' > b$, $\forall i' \in I_{b'}$, 
  \begin{equation*}
    \hat{\er}_{S_n^1}(h_{b,i}) \leq \hat{\er}_{S_n^1}(h_{b',i'}) + 2 \sqrt{3\frac{\psi(n) + b' + \ln(n)}{n}}.
  \end{equation*}
  Define 
  \begin{equation*} 
  \hat{b}_n = \min\!\left\{ b \in \left\{1,\ldots,\frac{n}{3}\right\} : |\GoodI(b)| \geq \frac{9}{10} |I_b| \right\},
  \end{equation*}
  if the set on the right hand side is non-empty, 
  and otherwise (for completeness of the definition of $\hat{h}_n$) simply define $\hat{h}_n = \hat{h}_n^0$ and define $f_n(S_n) = \hat{h}_n$ (i.e., the algorithm returns this classifier $\hat{h}_n$), which completes the definition of the learning algorithm in this case. 

  Supposing $\hat{b}_n$ is defined, 
  the interpretation of the above choice of $\hat{b}_n$ is that, 
  for $b = \hat{b}_n$, at least $\frac{9}{10}$ of the classifiers $h_{b,i}$ are not verifiably worse
  than some other classifier $h_{b',i'}$ with $b' > b$,
  where this verification comes via a concentration inequality
  holding (uniformly over all such comparisons) with probability at least $1 - e^{-\psi(n)}$ (as we will argue below).

  Based on this choice of $\hat{b}_n$,
  for any $x \in \X$, we define
  \begin{equation*}
    \hat{h}_n^1(x) = \ind\!\left[ \sum_{i \in I_{\hat{b}_n}} h_{\hat{b}_n,i}(x) \geq \frac{1}{2} | I_{\hat{b}_n} | \right],
  \end{equation*}
  the majority vote of the $h_{\hat{b}_n,i}(x)$ predictions.

  To complete the definition of the algorithm, we define $\hat{h}_n$ by selecting one of $\hat{h}_n^0$ or $\hat{h}_n^1$ using the remaining data $S_n^2$, as follows.
  If
  \begin{equation*}
    \hat{\er}_{S_n^2}(\hat{h}_n^1) \leq \hat{\er}_{S_n^2}(\hat{h}_n^0) + \sqrt{\frac{6\psi(n)}{n}},
  \end{equation*}
  define $\hat{h}_n = \hat{h}_n^1$,
  and otherwise define $\hat{h}_n = \hat{h}_n^0$.
  In either case, define $f_n(S_n)$ as this $\hat{h}_n$ 
  (i.e., the algorithm returns the classifier $\hat{h}_n$).
  The algorithm is summarized in Figure~\ref{fig:exp-alg}.

  \begin{figure}
  \begin{bigboxit}
    1. Let $\hat{h}_n^0$ be returned by a universally Bayes-consistent learner trained on $S_n^0$
    \\2. Let $\hat{b}_n = \min\!\left\{ b \leq n/3 : |\GoodI(b)| \geq \frac{9}{10} \left\lfloor \frac{n}{3b} \right\rfloor \right\}$ {\small (if $\hat{b}_n$ does not exist, return $\hat{h}_n = \hat{h}_n^0$)}
    \\3. Let $\hat{h}_n^1 = \mathrm{Majority}\!\left(h_{\hat{b}_n,i} : i \leq n/(3\hat{b}_n)\right)$
    \\5. If $\hat{\er}_{S_n^2}(\hat{h}_n^1) \leq \hat{\er}_{S_n^2}(\hat{h}_n^0) + \sqrt{\frac{6\psi(n)}{n}}$, return $\hat{h}_n = \hat{h}_n^1$, else return $\hat{h}_n = \hat{h}_n^0$
  \end{bigboxit}
  \caption{Algorithm achieving $e^{-\psi(n)}$ rate for classes $\H$ with no infinite Littlestone tree.}
  \label{fig:exp-alg}
  \end{figure}

  With the definition of the classifier $\hat{h}_n$ now complete,
  we proceed to prove that it satisfies the required rate $e^{-\psi(n)}$.
  First, we remark that the claim of achieving rate $e^{-\psi(n)}$ 
  is \emph{asymptotic} in nature: that is, it suffices to argue that, 
  for some $n_0 \in \nats$ (possibly $\PXY$-dependent), 
  there exists a constant $C > 0$ (also possibly $\PXY$-dependent) such that, 
  for every $n \in \nats$ with $n \geq n_0$,
  $\E\!\left[ \er_{\PXY}(\hat{h}_n) \right] - \inf_{h \in \H} \er_{\PXY}(h) \leq C e^{- \psi(n)}$.
  This is due to the fact that, 
  since $\psi(n)$ is finite for every $n$,
  and $\E\!\left[ \er_{\PXY}(\hat{h}_n) \right] - \inf_{h \in \H} \er_{\PXY}(h) \leq 1$,
  we can simply define another (possibly $\PXY$-dependent) constant value $C' = \min\!\left\{ C e^{-\psi(n)} : n \in \{1,\ldots,n_0-1\} \right\} > 0$, 
  and it would then follow that \emph{every} $n \in \nats$ satisfies
  $\E\!\left[ \er_{\PXY}(\hat{h}_n) \right] - \inf_{h \in \H} \er_{\PXY}(h) \leq \frac{C}{\min\{1,C'\}} e^{- \psi(n)}$ (with this bound being no smaller than $1$ for $n < n_0$, hence holding vacuously for these small $n$ values).
  For this reason, in parts of the proof, we will 
  explicitly focus on the case of $n$ taken \emph{sufficiently large} 
  to satisfy certain constraints introduced in the analysis below.
  We now proceed with the proof.
  
  Let $E_1$ denote the event that, either $\hat{h}_n^1$ is undefined
  (due to the early termination event in the algorithm, when the criterion for $\hat{b}_n$ cannot be satisfied)
  or else 
  \begin{equation}
  \label{eqn:E1-Sn2-Hoeffding}
  \forall j \in \{0,1\}, \left| \hat{\er}_{S_n^2}(\hat{h}_n^j) - \er_{\PXY}(\hat{h}_n^j) \right| \leq \sqrt{\frac{3\psi(n)}{2n}}.
  \end{equation}
  Since $|S_n^2| = \frac{n}{3}$, 
  and $S_n^2$ is independent of $(S_n^0,S_n^1)$ 
  (from which $\hat{h}_n^0$ and $\hat{h}_n^1$ are derived), 
  Hoeffding's inequality 
  (applied under the conditional distribution given $S_n^0$ and $S_n^1$)
  and a union bound, with the law of total probability, imply that $E_1$ holds with probability at least $1 - 4 e^{-\psi(n)}$.

  We first consider the case where $\PXY$ is \emph{not} Bayes-realizable with respect to $\H$: that is, 
  $\inf_{h \in \H} \er_{\PXY}(h) > \inf_{h} \er_{\PXY}(h)$
  (where $h$ on the right hand side ranges over all measurable functions $\X \to \{0,1\}$).
  In this case, let $\epsilon > 0$ be such that $\inf_{h \in \H} \er_{\PXY}(h) > \inf_{h} \er_{\PXY}(h) + \epsilon$.
  As mentioned above, it suffices to establish a bound $Ce^{-\psi(n)}$ 
  for all \emph{sufficiently large} $n$.
  We may therefore focus on the case that $n$ is sufficiently large that 
  $\sqrt{\frac{3\psi(n)}{2n}} \leq \frac{\epsilon}{16}$ holds;
  since $\psi(n) = o(n)$, this inequality is satisfied for all sufficiently large $n$.
  On the event $E_1$ (and the event that $\hat{h}_n^1$ is defined), if additionally we have
  $\er_{\PXY}(\hat{h}_n^0) < \inf_{h \in \H} \er_{\PXY}(h) - \frac{\epsilon}{2}$
  and $\er_{\PXY}(\hat{h}_n^1) > \inf_{h \in \H} \er_{\PXY}(h) - \frac{\epsilon}{8}$,
  then
  \begin{align*}
    \hat{\er}_{S_n^2}(\hat{h}_n^1)
    & \geq \er_{\PXY}(\hat{h}_n^1) - \frac{\epsilon}{16}
    > \inf_{h \in \H} \er_{\PXY}(h) - \frac{3\epsilon}{16} 
    \\ & > \er_{\PXY}(\hat{h}_n^0) + \frac{5\epsilon}{16}
    \geq \hat{\er}_{S_n^2}(\hat{h}_n^0) + \frac{\epsilon}{4}
    \geq \hat{\er}_{S_n^2}(\hat{h}_n^0) + \sqrt{\frac{24\psi(n)}{n}},
  \end{align*}
  so that the algorithm chooses $\hat{h}_n = \hat{h}_n^0$.
  On the other hand, if
  $\er_{\PXY}(\hat{h}_n^0) < \inf_{h \in \H} \er_{\PXY}(h) - \frac{\epsilon}{2}$
  and $\er_{\PXY}(\hat{h}_n^1) \leq \inf_{h \in \H} \er_{\PXY}(h) - \frac{\epsilon}{8}$,
  then regardless of whether $\hat{h}_n = \hat{h}_n^0$ or $\hat{h}_n = \hat{h}_n^1$,
  we have $\er_{\PXY}(\hat{h}_n) \leq \inf_{h \in \H} \er_{\PXY}(h) - \frac{\epsilon}{8}$.
  Moreover, if $\hat{h}_n^1$ is not defined, then the algorithm chooses $\hat{h}_n = \hat{h}_n^0$.
  Thus, in any case, on the event $E_1$, if $\er_{\PXY}(\hat{h}_n^0) < \inf_{h \in \H} \er_{\PXY}(h) - \frac{\epsilon}{2}$,
  we always have $\er_{\PXY}(\hat{h}_n) \leq \inf_{h \in \H} \er_{\PXY}(h) - \frac{\epsilon}{8}$.
  In particular, since we always have $\er_{\PXY}(\hat{h}_n) \leq 1$, this implies 
  \begin{align*}
    \E\!\left[ \er_{\PXY}(\hat{h}_n) \right]
    & \leq \inf_{h \in \H} \er_{\PXY}(h) - \frac{\epsilon}{8} + \P\!\left( \er_{\PXY}(\hat{h}_n^0) \geq \inf_{h \in \H} \er_{\PXY}(h) - \frac{\epsilon}{2} \right) + \left( 1 - \P(E_1) \right)
    \\ & \leq \inf_{h \in \H} \er_{\PXY}(h) - \frac{\epsilon}{8} + \P\!\left( \er_{\PXY}(\hat{h}_n^0) > \inf_{h} \er_{\PXY}(h) + \frac{\epsilon}{2} \right) + 4 e^{-\psi(n)}
    \\ & \leq \inf_{h \in \H} \er_{\PXY}(h) - \frac{\epsilon}{8} + \frac{2}{\epsilon} \E\!\left[ \er_{\PXY}(\hat{h}_n^0) - \inf_{h} \er_{\PXY}(h) \right] + 4 e^{-\psi(n)},
  \end{align*}
  where the second inequality follows from the definition of $\epsilon$, 
  satisfying $\inf_{h \in \H} \er_{\PXY}(h) > \inf_{h} \er_{\PXY}(h) + \epsilon$, 
  and the bound on $\P(E_1)$ established above, 
  and the last line follows from Markov's inequality, 
  noting 
  $\er_{\PXY}(\hat{h}^0_n) - \inf_{h} \er_{\PXY}(h)$ is non-negative.
  Since $\lim_{n \to \infty} 4 e^{-\psi(n)} = 0$ and the Bayes consistency property of $\hat{h}_n^0$ guarantees
  $\lim_{n \to \infty} \E\!\left[ \er_{\PXY}(\hat{h}_n^0) - \inf_{h} \er_{\PXY}(h) \right] = 0$,
  we have that for all sufficiently large $n$ the right hand side of the last line is at most $\inf_{h \in \H} \er_{\PXY}(h) - \frac{\epsilon}{16}$.
  Thus, for all sufficiently large $n$,
  \begin{equation*}
    \E\!\left[ \er_{\PXY}(\hat{h}_n) \right] - \inf_{h \in \H} \er_{\PXY}(h) < 0 \leq e^{-\psi(n)}.
  \end{equation*}

  To complete the proof, we now turn to the remaining case, 
  in which $\PXY$ is Bayes-realizable with respect to $\H$: 
  that is, $\inf_{h \in \H} \er_{\PXY}(h) = \inf_{h} \er_{\PXY}(h)$.
  We will assume this to be the case for the remainder of the proof.
  Since $\H$ does not shatter an infinite Littlestone tree, 
  by Lemma~\ref{lem:no-infinite-VCL-implies-UGC}
  it satisfies the universal Glivenko-Cantelli (UGC) property (Definition~\ref{defn:UGC}),
 which further implies $\H$ is totally bounded in $L_1(\Px)$
 by Lemma~\ref{lem:UGC-totally-bounded}, 
 and therefore Lemma~\ref{lem:limit-function} implies there exists a measurable function $\target : \X \to \{0,1\}$ with
  $\er_{\PXY}(\target) = \inf_{h \in \H} \er_{\PXY}(h) = \inf_{h} \er_{\PXY}(h)$ and $\inf_{h \in \H} \Px( x : h(x) \neq \target(x) ) = 0$,
  where $\Px$ denotes the marginal of $\PXY$ on $\X$.
  Moreover, 
  letting $P_{\target}$ denote the distribution of $(X,\target(X))$, for $X \sim \Px$,
  note that the property $\inf_{h \in \H} \Px( x : h(x) \neq \target(x) ) = 0$ further implies
  that $P_{\target}$ is realizable with respect to $\H$.

  To argue the existence of $\hat{b}_n$, 
  and the quality of the functions $h_{\hat{b}_n,i}$ in the majority vote $\hat{h}_n^1$, 
  we will need the following
  argument about concentration of empirical error rates of all $h^{\mathbf{y}}_{b,i}$ functions.
  Let $E_2$ denote the event that, for every $b \in \{1,\ldots,n/3\}$, $i \in I_b$, and $\mathbf{y} \in \{0,1\}^b$,
  we have 
  \begin{equation}
    \label{eqn:exponential-pf-concentration-hbi}
    \left| \hat{\er}_{S_n^1}\!\left(h^{\mathbf{y}}_{b,i}\right) - \er_{\PXY}\!\left(h^{\mathbf{y}}_{b,i}\right) \right| \leq \sqrt{3\frac{\psi(n)+b+\ln(n)}{n}}.
   \end{equation}
  For each $b \in \{1,\ldots,n/3\}$, $i \in I_b$, and $\mathbf{y} \in \{0,1\}^b$,
  observing that
  \begin{equation*}
    \sqrt{\frac{\psi(n)+\ln(b(b+1) 2^{b+1} \lfloor n/(3b) \rfloor)}{2n/3}} \leq \sqrt{3\frac{\psi(n)+b+\ln(n)}{n}},
  \end{equation*}
  applying Hoeffding's inequality under the conditional 
  distribution given $S_n^0$ 
  implies that \eqref{eqn:exponential-pf-concentration-hbi} holds with conditional probability at least
  $1 - \frac{1}{b(b+1)2^{b}|I_b|}e^{-\psi(n)}$.
  Therefore, 
  by the union bound and the law of total probability, 
  \eqref{eqn:exponential-pf-concentration-hbi}
  holds simultaneously for all $b \in \{1,\ldots,n/3\}$, $i \in I_b$, and $\mathbf{y} \in \{0,1\}^b$
  with probability at least
  \begin{equation*}
    1 - \sum_{b,i,\mathbf{y}} \frac{1}{b (b+1) 2^b |I_b|} e^{-\psi(n)} 
    = 1 - \sum_{b} \frac{1}{b (b+1)} e^{-\psi(n)} \geq 1 - e^{-\psi(n)}.  
  \end{equation*}
  That is, $E_2$ holds with probability at least $1 - e^{-\psi(n)}$.

  Let $b^* = b^*_{99/100} \in \nats$ 
  be defined as in Lemma~\ref{lem:SOA-zero-error-rate-bstar} (with $\gamma = 99/100$) 
  for the class $\H$ and the distribution $P_{\target}$
  (recalling that $P_{\target}$ is realizable with respect to $\H$).
  Since, as mentioned above, it suffices to establish a bound 
  $C e^{-\psi(n)}$ for all \emph{sufficiently large} $n$, 
  for the remainder of the proof we suppose $n \geq 3 b^*$.
  For each $i \in I_{b^*}$,
  let $\mathbf{y}^*_{b^*,i} = (\target(X_{1+(i-1)b^*}),\ldots,\target(X_{i b^*}))$.
  Noting that the elements of 
  $B^{b^*}_i(\mathbf{y}^*_{b^*,i})$ 
  are i.i.d.\ $P_{\target}$-distributed, 
  Lemma~\ref{lem:SOA-zero-error-rate-bstar} implies that 
  $\forall i \in I_{b^*}$, 
  \begin{equation}
  \label{eqn:bstar-property}
  \P\!\left( \er_{P_{\target}}\!\left( h^{\mathbf{y}^*_{b^*,i}}_{b^*,i} \right) = 0 \right) > \frac{99}{100}.
  \end{equation}

  For each $b \in \{1,\ldots,n/3\}$ and $i \in I_b$,
  let
  \begin{equation*}
    \OPTY(b,i) = \left\{ \mathbf{y} \in \{0,1\}^b : \er_{\PXY}\!\left(h^{\mathbf{y}}_{b,i}\right) = \er_{\PXY}(\target) \right\}
  \end{equation*}
  and let 
  \begin{equation*} 
  \HasOPT(b) = \{ i \in \{1,\ldots,n/3\} : \OPTY(b,i) \neq \emptyset \}.
  \end{equation*}
  In particular, note that if 
  $\er_{P_{\target}}\!\left( h^{\mathbf{y}^*_{b^*,i}}_{b^*,i} \right) = 0$, then by the triangle inequality, 
  \begin{align}
      \er_{\PXY}(\target) \leq \er_{\PXY}\!\left( h^{\mathbf{y}^*_{b^*,i}}_{b^*,i} \right)
      & \leq \er_{\PXY}\!\left( \target \right) + \Px\!\left( x : h^{\mathbf{y}^*_{b^*,i}}_{b^*,i}(x) \neq \target(x) \right) 
      \notag \\ & = \er_{\PXY}(\target) + \er_{P_{\target}}\!\left( h^{\mathbf{y}^*_{b^*,i}}_{b^*,i} \right)
      = \er_{\PXY}(\target), \label{eqn:ystar-bstar-in-OPTy}
  \end{align}
  so that 
  $\mathbf{y}^*_{b^*,i} \in \OPTY(b^*,i)$, 
  and hence $i \in \HasOPT(b^*)$.  
  Together, by \eqref{eqn:bstar-property} and \eqref{eqn:ystar-bstar-in-OPTy}, 
  we have that $\P\!\left( 1 \in \HasOPT(b^*) \right) > \frac{99}{100}$.

  For each $b \in \{1,\ldots,n/3\}$, since $\{0,1\}^{b}$ has finite size, 
  there exists $\gamma_{b} > 0$ such that 
  \begin{equation*}
    \P\!\left( \OPTY(b,1) \!=\! \{0,1\}^{b} \text{ or } \min\!\left\{ \er_{\PXY}\!\left( h^{\mathbf{y}}_{b,1} \right) \!: \mathbf{y} \!\in\! \{0,1\}^{b} \!\setminus\! \OPTY(b,1) \right\} \!>\! \er_{\PXY}(\target) \!+\! \gamma_{b} \right)  > \frac{99}{100}.
  \end{equation*}
  In other words, with probability greater than $\frac{99}{100}$, 
  either \emph{every} $\mathbf{y} \in \{0,1\}^b$ has 
  optimal error $\er_{\PXY}\!\left(h^{\mathbf{y}}_{b,1}\right) = \inf_{h} \er_{\PXY}(h)$, 
  or else the smallest \emph{non-optimal} error rate among $h^{\mathbf{y}}_{b,1}$, $\mathbf{y} \in \{0,1\}^{b}$, is greater than $\inf_{h} \er_{\PXY}(h) + \gamma_{b}$.
  Let $\gamma^* = \min_{b \in \{1,\ldots,b^*\}} \gamma_b$, and note that $\gamma^* > 0$ since $\{1,\ldots,b^*\}$ is a finite set.
  For each $b \in \{1,\ldots,b^*\}$, define 
  \begin{align*}
    \Gap(b) = & \left\{ i \in I_b : \OPTY(b,i) = \{0,1\}^{b} \text{ or } \right.
    \\ & {\hskip 18mm}\left. \min\!\left\{ \er_{\PXY}\!\left( h^{\mathbf{y}}_{b,i} \right) : \mathbf{y} \in \{0,1\}^{b} \setminus \OPTY(b,i) \right\} > \er_{\PXY}(\target) + \gamma^* \right\}.
  \end{align*}
  Also define 
  \begin{equation*}
    \PrbGoodB = \left\{ b \in \{1,\ldots,b^*\} : \P\!\left( 1 \in \HasOPT(b) \right) \geq \frac{7}{10} \right\}.
  \end{equation*}
  In particular, since $\P\!\left( 1 \in \HasOPT(b^*) \right) > \frac{99}{100}$ (established above), 
  we have $b^* \in \PrbGoodB$.

  Let $E_3$ denote the event that $\forall b \in \PrbGoodB$, we have 
  \begin{equation*}
    |\HasOPT(b) \cap \Gap(b)| \geq \left( \P\!\left( 1 \in \HasOPT(b) \right) - \frac{1}{50} \right) |I_b|.
  \end{equation*}
  For any $b \in \PrbGoodB$,
  recall that the samples $B^b_i$, $i \in I_b$, 
  are identically distributed.
  Thus, 
  for any $i \in I_b$,
  by the union bound (and recalling that $\gamma^* \leq \gamma_b$), 
  \begin{equation*} 
  \P\!\left( i \in \HasOPT(b) \cap \Gap(b) \right) 
  > \P\!\left( 1 \in \HasOPT(b) \right) - \frac{1}{100}. 
  \end{equation*}
  Since $B^{b}_1,B^{b}_2,\ldots,B^{b}_{|I_b|}$ are independent,
  Hoeffding's inequality, and the union bound (over $b$ values in $\PrbGoodB$), 
  imply that $E_3$ has probability at least
  $1 - b^* e^{-C_1 n}$, 
  where $C_1 = \frac{1}{100^2} \frac{1}{3 b^*}$.

  As discussed above, since it suffices to establish a bound
  $C e^{-\psi(n)}$ for all \emph{sufficiently large} $n$,  
  for the remainder of the proof we suppose $n$ is sufficiently 
  large to satisfy the inequality 
  $\sqrt{3\frac{\psi(n)+b^*+\ln(n)}{n}} < \frac{\gamma^*}{2}$, which is indeed satisfied for all sufficiently large $n$ since $\psi(n) = o(n)$.
  Note that, on the event $E_2$, 
  for every $b \in \PrbGoodB$ and $i \in \HasOPT(b) \cap \Gap(b)$, 
  for every $\mathbf{y} \in \OPTY(b,i)$
  and every $\mathbf{y}' \in \{0,1\}^{b} \setminus \OPTY(b,i)$ 
  (if any such $\mathbf{y}'$ exists), we have 
  \begin{equation*}
    \hat{\er}_{S_n^1}(h^{\mathbf{y}}_{b,i})
    < \er_{\PXY}(\target) + \frac{\gamma^*}{2}
    < \er_{\PXY}(h^{\mathbf{y}'}_{b,i}) - \frac{\gamma^*}{2}
    < \hat{\er}_{S_n^1}(h^{\mathbf{y}'}_{b,i}),
  \end{equation*}
  where the first inequality is due to $\mathbf{y} \in \OPTY(b,i)$ and \eqref{eqn:exponential-pf-concentration-hbi}, 
  the second inequality is due to $i \in \Gap(b)$,
  and the third inequality is again due to \eqref{eqn:exponential-pf-concentration-hbi}.
  In particular, this implies $h_{b,i} \neq h^{\mathbf{y}'}_{b,i}$.
  Since $i \in \HasOPT(b)$ guarantees $\OPTY(b,i)$ is non-empty,
  we conclude that on $E_2$, for every $b \in \PrbGoodB$ and $i \in \HasOPT(b) \cap \Gap(b)$,
  the classifier $h_{b,i}$ equals some $h^{\mathbf{y}}_{b,i}$ for some $\mathbf{y} \in \OPTY(b,i)$, and hence 
  $\er_{\PXY}(h_{b,i}) = \er_{\PXY}(\target)$.
  Equivalently, defining 
  \begin{equation*}
    \OPTI(b) = \left\{ i \in I_b : \er_{\PXY}(h_{b,i}) = \er_{\PXY}(\target) \right\}, 
  \end{equation*}
  we have shown that on the event $E_2$, every $b \in \PrbGoodB$ has 
  \begin{equation*} 
  \OPTI(b) \supseteq \HasOPT(b) \cap \Gap(b).
  \end{equation*}
  Thus, by the definition of the event $E_3$, 
  we may also conclude that, 
  on the event $E_2 \cap E_3$, every $b \in \PrbGoodB$ has
  \begin{equation}
  \label{eqn:OPTI-b-size-lower-bound}
    |\OPTI(b)| \geq \left( \P\!\left( 1 \in \HasOPT(b) \right) - \frac{1}{50} \right) |I_b|.
  \end{equation}
  In particular, since $b^* \in \PrbGoodB$ 
  and $\P\!\left( 1 \in \HasOPT(b^*) \right) > \frac{99}{100}$
  (both established above), 
  on the event $E_2 \cap E_3$ 
  we have
  \begin{equation}
  \label{eqn:OPTI-bstar-size-lower-bound}
    |\OPTI(b^*)|
    > \frac{97}{100} |I_{b^*}|
    > \frac{9}{10} |I_{b^*}|.
  \end{equation}
  
  Additionally, note that on the event $E_2$, for any $i \in \OPTI(b^*)$, 
  for any $b' \in \{1,\ldots,n/3\}$ with $b' > b^*$, for any $i' \in I_{b'}$
  and any $\mathbf{y}' \in \{0,1\}^{b'}$, it holds that 
  \begin{align*}
    & \hat{\er}_{S_n^1}(h_{b^*,i})
    \leq \er_{\PXY}(\target) + \sqrt{3\frac{\psi(n)+b^*+\ln(n)}{n}}
    \\ & \leq \er_{\PXY}(h^{\mathbf{y}'}_{b',i'}) + \sqrt{3\frac{\psi(n)+b^*+\ln(n)}{n}}
    \leq \hat{\er}_{S_n^1}(h^{\mathbf{y}'}_{b',i'}) + 2\sqrt{3\frac{\psi(n)+b'+\ln(n)}{n}}.
  \end{align*}
  In particular, since $\exists \mathbf{y}' \in \{0,1\}^{b'}$ such that $h_{b',i'} = h^{\mathbf{y}'}_{b',i'}$,
 it follows that
  \begin{equation*}
    \hat{\er}_{S_n^1}(h_{b^*,i}) \leq \hat{\er}_{S_n^1}(h_{b',i'}) + 2\sqrt{3\frac{\psi(n) + b' + \ln(n)}{n}}.
  \end{equation*}
  Thus, recalling the definition of $\GoodI(b)$ from the learning algorithm, 
  we conclude that on the event $E_2$, 
  we have $\GoodI(b^*) \supseteq \OPTI(b^*)$.
  Together with \eqref{eqn:OPTI-bstar-size-lower-bound}, 
  we have that on the event $E_2 \cap E_3$, 
  $\hat{b}_n$ is defined (and hence so is $\hat{h}_n^1$)
  and satisfies $\hat{b}_n \leq b^*$.

  We must also argue that any ``bad'' $b < b^*$ will not be selected as $\hat{b}_n$.
  Specifically, we will establish that $\hat{b}_n \in \PrbGoodB$ with high probability.
  This is trivially satisfied if $\PrbGoodB = \{1,\ldots,b^*\}$,
  in which case we may skip the following argument and proceed from \eqref{eqn:hatb-in-PrbGoodB} below (defining $E_4$ as a vacuous event of probability one in that case); 
  to address the remaining case, let us suppose 
  $\PrbGoodB \neq \{1,\ldots,b^*\}$.
  For any $b \in \{1,\ldots,b^*\} \setminus \PrbGoodB$ 
  and $i \in I_b$, define 
  \begin{equation*}
    \varepsilon(b,i) = \min_{\mathbf{y} \in \{0,1\}^b} \er_{\PXY}(h^{\mathbf{y}}_{b,i}) - \er_{\PXY}(\target).
  \end{equation*}
  Since $b \in \{1,\ldots,b^*\} \setminus \PrbGoodB$, 
  by definition of $\PrbGoodB$ we have  
  $\P(1 \notin \HasOPT(b)) > \frac{3}{10}$.
  Also note that, since $\{0,1\}^b$ is finite, 
  if $1 \notin \HasOPT(b)$ then $\varepsilon(b,1) > 0$:
  that is, 
  $\P( \varepsilon(b,1) > 0 | 1 \notin \HasOPT(b) ) = 1$.
  This further implies that 
  there exists $\varepsilon^*(b) > 0$  
  such that 
  $\P( \varepsilon(b,1) > \varepsilon^*(b) | 1 \notin \HasOPT(b) ) > \frac{2}{3}$.
  Let 
  \begin{equation*} 
  \varepsilon^* = \min\{\varepsilon^*(b) : b \in \{1,\ldots,b^*\} \setminus \PrbGoodB\},
  \end{equation*}  
  and note that $\varepsilon^* > 0$ 
  since $\{1,\ldots,b^*\} \setminus \PrbGoodB$ is finite.
  Thus, 
  for any $b \in \{1,\ldots,b^*\} \setminus \PrbGoodB$, 
  since we also have $\P(1 \notin \HasOPT(b)) > \frac{3}{10}$, 
  altogether we have 
  \begin{equation*}
    \P\!\left( \varepsilon(b,1) > \varepsilon^* \right) 
    = \P\!\left( \varepsilon(b,1) > \varepsilon^* \middle| 1 \notin \HasOPT(b) \right) \P\!\left( 1 \notin \HasOPT(b) \right)
    > \frac{2}{3} \cdot \frac{3}{10}
    = \frac{1}{5}.
  \end{equation*}
  Moreover, since $B_i^b$, $i \in I_b$,
  are identically distributed, 
  for any $b \in \{1,\ldots,b^*\} \setminus \PrbGoodB$,
  every $i \in I_b$ 
  also satisfies $\P\!\left( \varepsilon(b,i) > \varepsilon^* \right) > \frac{1}{5}$.

  For each $b \in \{1,\ldots,b^*\} \setminus \PrbGoodB$, define
  \begin{equation*}
    \FarFromOPTI(b) = \left\{ i \in I_b : \varepsilon(b,i) > \varepsilon^* \right\}.
  \end{equation*}
  Let $E_4$ denote the event that every $b \in \{1,\ldots,b^*\} \setminus \PrbGoodB$ has 
  \begin{equation*}
    |\FarFromOPTI(b)| \geq \frac{3}{20} |I_b|.
  \end{equation*}
  Recalling that, for each $b \in \{1,\ldots,b^*\} \setminus \PrbGoodB$, 
  the $B^b_1,\ldots,B^b_{|I_b|}$ samples are independent,
  Hoeffding's inequality implies $|\FarFromOPTI(b)| \geq \frac{3}{20} |I_b|$ 
  with probability at least $1 - e^{-C_2 n}$
  where $C_2 = \frac{1}{400} \frac{1}{3 b^*}$.
  Applying this to every $b \in \{1,\ldots,b^*\} \setminus \PrbGoodB$, 
  the union bound implies 
  the event $E_4$ has probability at least $1 - (b^*-1) e^{-C_2 n}$.

  To complete this part of the argument, we show that the definition of $\hat{b}_n$ prevents $\hat{b}_n \in \{1,\ldots,b^*\} \setminus \PrbGoodB$.
  Again, as discussed above, since it suffices to establish a bound $C e^{-\psi(n)}$ for all \emph{sufficiently large} $n$, 
  for the remainder of the proof we suppose $n$ is sufficiently large to satisfy
  $\sqrt{3\frac{\psi(n)+b^*+\ln(n)}{n}} < \frac{\varepsilon^*}{4}$, which indeed holds for all sufficiently large $n$ since $\psi(n) = o(n)$.
  On the event $E_2 \cap E_3$, 
  we know from \eqref{eqn:OPTI-bstar-size-lower-bound} that 
  the set $\OPTI(b^*)$ is nonempty.
  Moreover, also on the event $E_2 \cap E_3$, 
  for any $i' \in \OPTI(b^*)$, 
  any $b \in \{1,\ldots,b^*\} \setminus \PrbGoodB$, 
  and any $i \in \FarFromOPTI(b)$ and $\mathbf{y} \in \{0,1\}^{b}$, 
  \begin{align*}
    & \hat{\er}_{S_n^1}\!\left(h^{\mathbf{y}}_{b,i}\right)
    > \er_{\PXY}\!\left(h^{\mathbf{y}}_{b,i}\right) - \frac{\varepsilon^*}{4} 
    \\ & > \er_{\PXY}(\target) + \frac{3\varepsilon^*}{4}
    = \er_{\PXY}(h_{b^*,i'}) + \frac{3\varepsilon^*}{4}
    \\ & > \hat{\er}_{S_n^1}(h_{b^*,i'}) + \frac{\varepsilon^*}{2}
    > \hat{\er}_{S_n^1}(h_{b^*,i'}) + 2\sqrt{3\frac{\psi(n)+b^*+\ln(n)}{n}},
  \end{align*}
  where the first inequality is due to \eqref{eqn:exponential-pf-concentration-hbi},
  the second inequality is due to $i \in \FarFromOPTI(b)$, 
  the equality following this is due to $i' \in \OPTI(b^*)$, 
  and the inequality following this is again due to 
  \eqref{eqn:exponential-pf-concentration-hbi}.
  In particular, this implies
  \begin{equation*}
    \hat{\er}_{S_n^1}\!\left(h_{b,i}\right) > \hat{\er}_{S_n^1}(h_{b^*,i'}) + 2\sqrt{3\frac{\psi(n)+b^*+\ln(n)}{n}},
  \end{equation*}
  so that $\GoodI(b) \subseteq I_b \setminus \FarFromOPTI(b)$.
  Thus, on the event $E_2 \cap E_3 \cap E_4$, 
  every $b \in \{1,\ldots,b^*\} \setminus \PrbGoodB$ 
  has 
  $|\GoodI(b)| \leq \frac{17}{20} |I_b| < \frac{9}{10} |I_b|$,
  so that $\hat{b}_n \neq b$.
  Hence, on the event $E_2 \cap E_3 \cap E_4$, we have $\hat{b}_n \notin \{1,\ldots,b^*\} \setminus \PrbGoodB$.
  Together with the above conclusion that $\hat{b}_n \leq b^*$,
  on the event $E_2 \cap E_3 \cap E_4$, 
  we have that 
  \begin{equation}
  \label{eqn:hatb-in-PrbGoodB}
  \hat{b}_n \in \PrbGoodB.
  \end{equation}
  We have established this under the assumption that $\PrbGoodB \neq \{1,\ldots,b^*\}$;
  however, as mentioned above, \eqref{eqn:hatb-in-PrbGoodB} is trivially
  satisfied on the event $E_2 \cap E_3$ 
  when $\PrbGoodB = \{1,\ldots,b^*\}$ 
  (due to the fact that $\hat{b}_n \leq b^*$, as established above).
  Thus, the conclusion that \eqref{eqn:hatb-in-PrbGoodB} holds on $E_2 \cap E_3 \cap E_4$ is valid in either case 
  (defining $E_4$ to be a vacuous event of probability one 
  when $\PrbGoodB = \{1,\ldots,b^*\}$).

  We now resume the proof for the general case 
  (regardless of whether $\PrbGoodB = \{1,\ldots,b^*\}$ or not).
  Since, on the above events, we have constrained the set of possible $\hat{b}_n$ values to $\PrbGoodB$,
  let us now argue that any of these possible $\hat{b}_n$ values would suffice to provide a guarantee on $\er_{\PXY}(\hat{h}_n^1)$.
  Consider any $b \in \PrbGoodB$.
  Recall that we have established above that, on the event $E_2 \cap E_3$, every $b \in \PrbGoodB$ satisfies \eqref{eqn:OPTI-b-size-lower-bound},
  which, together with the definition of $\PrbGoodB$, implies
  \begin{equation*}
    |\OPTI(b)|
    \geq \left( \P\!\left( 1 \in \HasOPT(b) \right) - \frac{1}{50} \right) |I_b|
    \geq \frac{34}{50} |I_b|
    > \frac{3}{5} |I_b|.
  \end{equation*}
  Thus, on the event $E_2 \cap E_3 \cap E_4$, 
  together with \eqref{eqn:hatb-in-PrbGoodB} we have
  \begin{equation}
  \label{eqn:opti-hatb-three-fifths}
    |\OPTI(\hat{b}_n)| > \frac{3}{5} | I_{\hat{b}_n} |.
  \end{equation}
  Additionally, as is well known, for any measurable function $g : \X \to \{0,1\}$,
  for $(X,Y) \sim \PXY$, 
  \begin{equation*}
    \er_{\PXY}(g) - \er_{\PXY}(\target) = \E\!\Big[ \left| 1 - 2 \P(Y=1|X) \right| \cdot \ind\!\left[ g(X) \neq \target(X) \right] \Big],  
  \end{equation*}
  which implies that if $\er_{\PXY}(g) = \er_{\PXY}(\target)$,
  then with probability one, 
  $\P(Y=1|X) \neq \frac{1}{2} \implies g(X) = \target(X)$.
  In particular, since every $i \in \OPTI(\hat{b}_n)$ has $\er_{\PXY}(h_{\hat{b}_n,i}) = \er_{\PXY}(\target)$,
  by the union bound we have that, 
  for $(X,Y) \sim \PXY$ (independent of $S_n$),
  with conditional probability one (given $S_n$),
  \begin{equation*}
    \P(Y=1|X) \neq \frac{1}{2} \implies \forall i \in \OPTI(\hat{b}_n), h_{\hat{b}_n,i}(X) = \target(X).
  \end{equation*}
  Since, on the event $E_2 \cap E_3 \cap E_4$, 
  \eqref{eqn:opti-hatb-three-fifths} implies 
  the set $\OPTI(\hat{b}_n)$ contains
  at least $\frac{3}{5}$ of the values in $I_{\hat{b}_n}$,
  and since $\hat{h}_n^1(x)$ is the majority vote of $\{h_{\hat{b}_n,i}(x) : i \in I_{\hat{b}_n}\}$,
  we have that on the event $E_2 \cap E_3 \cap E_4$,
  for any $x \in \X$ such that every $i \in \OPTI(\hat{b}_n)$ has $h_{\hat{b}_n,i}(x) = \target(x)$,
  we also have $\hat{h}_n^1(x) = \target(x)$.
  Altogether, we have that, on the event $E_2 \cap E_3 \cap E_4$,
  \begin{equation}
  \label{eqn:hathn1-zero-excess-error}
    \er_{\PXY}(\hat{h}_n^1) - \er_{\PXY}(\target) = \E\!\left[ \left| 1 - 2 \P(Y=1|X) \right| \cdot \ind\!\left[ \hat{h}_n^1(X) \neq \target(X) \right] \middle| S_n \right] = 0.
  \end{equation}

  It remains to bound $\er_{\PXY}(\hat{h}_n)$ in terms of $\er_{\PXY}(\hat{h}_n^1)$.
  On the event $E_1$, if $\hat{h}_n^1$ is defined and $\er_{\PXY}(\hat{h}_n^1) \leq \er_{\PXY}(\hat{h}_n^0)$, then
  \begin{equation*}
    \hat{\er}_{S_n^2}(\hat{h}_n^1)
    \leq \er_{\PXY}(\hat{h}_n^1) + \sqrt{\frac{3\psi(n)}{2n}}
    \leq \er_{\PXY}(\hat{h}_n^0) + \sqrt{\frac{3\psi(n)}{2n}}
    \leq \hat{\er}_{S_n^2}(\hat{h}_n^0) + \sqrt{\frac{6\psi(n)}{n}},
  \end{equation*}
  where the first inequality is due to \eqref{eqn:E1-Sn2-Hoeffding}, 
  the second inequality is due to $\er_{\PXY}(\hat{h}^1_n) \leq \er_{\PXY}(\hat{h}_n^0)$,
  and the third inequality is again due to \eqref{eqn:E1-Sn2-Hoeffding}.
  In particular, in this case, by definition the algorithm chooses $\hat{h}_n = \hat{h}_n^1$.
  On the other hand, if $\er_{\PXY}(\hat{h}_n^1) > \er_{\PXY}(\hat{h}_n^0)$,
  then regardless of whether $\hat{h}_n = \hat{h}_n^0$ or $\hat{h}_n = \hat{h}_n^1$,
  we have $\er_{\PXY}(\hat{h}_n) \leq \er_{\PXY}(\hat{h}_n^1)$.
  Thus, since we have argued that $\hat{h}_n^1$ is defined on the event $E_2 \cap E_3 \cap E_4$,
  and since $\er_{\PXY}(\hat{h}_n) - \inf_h \er_{\PXY}(h) \leq 1$, we have 
  \begin{align*}
    & \E\!\left[\er_{\PXY}(\hat{h}_n)\right] - \inf_{h \in \H} \er_{\PXY}(h) 
    = \E\!\left[\er_{\PXY}(\hat{h}_n)\right] - \er_{\PXY}(\target)
    \\ & \leq \E\!\left[\left( \er_{\PXY}(\hat{h}_n^1) - \er_{\PXY}(\target) \right) \ind_{E_1 \cap E_2 \cap E_3 \cap E_4} \right] + \left( 1 - \P\!\left(E_1 \cap E_2 \cap E_3 \cap E_4 \right) \right)
    \\ & = 1 - \P\!\left(E_1 \cap E_2 \cap E_3 \cap E_4 \right)
    \leq 4 e^{-\psi(n)} + e^{-\psi(n)} + b^* e^{-C_1 n} + (b^*-1) e^{-C_2 n},
  \end{align*}
  where the equality on the third line is due to \eqref{eqn:hathn1-zero-excess-error} and the last inequality is by the union bound
  and the respective bounds on the failure probabilities of each event 
  $E_i$ established above.
  Since $\psi(n) = o(n)$, for any sufficiently large $n$ we have $\psi(n) \leq \min\{C_1 n,C_2 n\}$,
  so that this last expression is at most $(4 + 2b^*) e^{-\psi(n)}$ for all sufficiently large $n$. 
  This completes the proof.
\end{proof}

\section{Slower Than Near-exponential is not Faster Than Super-Root}
\label{sec:super-root-lower-bound}

As our next component in the proof of Theorem~\ref{thm:agnostic-char}, 
we turn to establishing the super-root lower bounds 
for classes which shatter an infinite Littlestone tree.

\begin{theorem}
\label{thm:agnostic-super-root-lower-bound}
If $\H$ shatters an infinite Littlestone tree, then 
for every $R(n) = o(n^{-1/2})$,
$\H$ is not agnostically learnable at rate $R$.
\end{theorem}

\stilltodo{TODO: sketch of the proof for the main 25 pages.}

\begin{proof}
  Fix any function $\phi(n) = o(n^{-1/2})$, and without loss of generality, suppose $\phi(n) \geq \frac{1}{n}$.
  For any learning algorithm $f_n$ we will construct a distribution $\PXY$
  such that for $S_n \sim \PXY^n$ and $\hat{h}_n = f_n(S_n)$,  
  \begin{equation*}
    \E[ \er_{\PXY}(\hat{h}_n) ] - \inf_{h \in \H} \er_{\PXY}(h) \geq C \phi(n)
  \end{equation*}
  for infinitely many $n \in \nats$, for a constant $C = \frac{1}{30}$.
  The theorem follows immediately, by defining $\phi(n) = \max\!\left\{ 30 R(n), \frac{1}{n} \right\}$ for any given $R(n) = o(n^{-1/2})$.

  The construction is similar to a proof of \citet*{bousquet:21} for 
  the realizable case (establishing an $n^{-1}$ rate lower bound there),
  in that we construct a distribution supported along a single random branch of the Littlestone tree.
  The key difference here is that we add classification noise to the labels $Y$ 
  and adjust the marginal probability distribution on $\X$ to decrease more rapidly along the branch.
  The idea is that, unlike the realizable case, since there is classification 
  noise in this scenario, 
  the Bernoulli mean testing lower bound from Lemma~\ref{lem:coin-testing-lower-bound} implies 
  the data set must contain at least some number 
  of copies (growing as the mean approaches $1/2$) of a given point in the support of $\Px$ before the learner can identify 
  the Bayes optimal classification of that point.

  We will describe a family of distributions with marginals on $\X$ supported on the branches of a Littlestone tree.
  Specifically, consider an infinite Littlestone tree shattered by $\H$, 
  as in Definition~\ref{defn:littlestone-tree}:  
  namely, let  
  \begin{equation*} 
    \{x_{\mathbf{u}}: \mathbf{u}\in\{0,1\}^k, 0\leq k<\infty \} \subseteq \X
  \end{equation*}
  be such that for every $1 \leq d < \infty$ and 
  $\mathbf{y} = (y_1,\ldots,y_d) \in\{0,1\}^d$, 
  there exists
  $h_{\mathbf{y}} \in \H$ such that 
  $h_{\mathbf{y}}(x_{y_{< k}})=y_{k}$ for  
  $1\leq k\leq d$.

  We first define a sequence $\{ p_k \}_{k \in \nats}$ in $(0,1)$ 
  (where $p_k$ will be the marginal probability $\Px$ at depth $k$)
  and an increasing sequence $\{n_k\}_{k \in \nats}$ in $\nats \cup \{0\}$ (where $n_k$ will correspond to a sample size $n$ for which a lower bound may be obtained on average).
  These are defined inductively, as follows.
  Define $n_1 = 0$.
  Define $p_2 = \frac{1}{2}$. 
  Inductively, for each $k \in \nats \setminus \{1\}$, 
  supposing $p_{k}$ and $n_{k-1}$ are already defined,
  let $n_k > n_{k-1}$ be any sufficiently large value in $\nats$ satisfying
  \begin{equation}
  \label{eqn:superroot-nk-defn}
    \sqrt{n_k} \phi(n_k) < \sqrt{\frac{p_k}{16}}.
  \end{equation}
  Such a value must exist since $\phi(n) = o(n^{-1/2})$.
  Additionally, define $p_{k+1}$ as any value in $(0,1)$ satisfying
  \begin{equation}
  \label{eqn:superroot-pk-defn}
    p_{k+1} < \frac{1}{4 n_k}.
  \end{equation}
  This completes the inductive definition of $\{ p_k \}_{k \in \nats \setminus \{1\}}$ and $\{ n_k \}_{k \in \nats}$.
  Recalling that $\phi(n_k) \geq \frac{1}{n_k}$, 
  the above constraint \eqref{eqn:superroot-nk-defn} implies 
  $\frac{1}{n_k} < \frac{p_k}{16}$
  which, together with \eqref{eqn:superroot-pk-defn}, implies 
  $p_{k+1} < \frac{p_k}{64}$,
  so that every $k \geq 2$ satisfies $p_k \leq 2^{-6(k-2)-1}$,
  and hence $\sum_{k \geq 2} p_k < 1$.
  Thus, we may define $p_1 = 1 - \sum_{k \geq 2} p_k > 0$,
  so that altogether the sequence $\{ p_k \}_{k \in \nats}$ 
  satisfies $\sum_{k \in \nats} p_k = 1$: 
  that is, it defines a probability distribution over $\nats$.

  Note that, since $p_{k+1} < \frac{p_k}{64}$ for every $k \geq 2$,
  we have, for every $k \in \nats$, 
   \begin{equation}
   \label{eqn:superroot-pk-sum}
    \sum_{k' > k} p_{k'} 
    < \sum_{k' > k}  2^{-6 (k' - k-1)} p_{k+1} 
    = \frac{64}{63} p_{k+1}
    < \frac{16}{63 n_k},
  \end{equation}
  where the last inequality is due to \eqref{eqn:superroot-pk-defn}.
  Moreover, since $\phi(n_k) \geq \frac{1}{n_k}$ for any $k \geq 2$,
  the definition of $n_k$ in \eqref{eqn:superroot-nk-defn} implies
  \begin{equation}
    \label{eqn:superroot-phi-nk-bounded-by-pk}
    \frac{p_k}{16} > n_k \phi(n_k)^2 \geq \phi(n_k).
  \end{equation}

  For each $\mathbf{y} = \{y_k\}_{k \in \nats}$ 
  sequence of values in $\{0,1\}$
  (corresponding to a branch in the Littlestone tree), 
  we define a distribution $P_{\mathbf{y}}$ with 
  marginal distribution on $\X$ supported along this path.
  Specifically, define $P_{\mathbf{y}}$ by the property that,
  for $(X,Y) \sim P_{\mathbf{y}}$, 
  for every $k \in \nats$, 
  the point $x_{y_{< k}}$ (i.e., the point at depth $k$ along the branch indicated by $\mathbf{y}$)
  has $\P(X = x_{y_{< k}}) = p_k$.
  The conditional distribution of $Y$ given $X = x_{y_{< k}}$ is defined as
  \begin{equation}
  \label{eqn:superroot-conditional-definition}
    \P(Y = y_{k} | X = x_{y_{< k}} ) = \frac{1}{2} + \frac{1}{2} \frac{\phi(n_k)}{p_k},
  \end{equation}
  which is in the range $(\frac{1}{2},\frac{17}{32})$ due to \eqref{eqn:superroot-phi-nk-bounded-by-pk}.
  Since $\sum_{k \in \nats} p_k = 1$, 
  this completely specifies a probability measure $P_{\mathbf{y}}$.

  The values above are constructed to serve the following argument.
  Fix any learning algorithm $\{ f_n \}_{n \in \nats}$.
  Let $\bybranch = \{\ybranch_k\}_{k \in \nats}$ be i.i.d.\ $\mathrm{Bernoulli}(\frac{1}{2})$.
  Fix any $k \in \nats \setminus \{1\}$.
  Conditioned on $\bybranch$, 
  let $S_{n_k} = \{(X_1,Y_1),\ldots,(X_{n_k},Y_{n_k})\} \sim P_{\bybranch}^{n_k}$ 
  and let $\hat{h}_{n_k} = f_{n_k}(S_{n_k})$.
  Since the marginal distribution of $P_{\bybranch}$ on $\X$ 
  is supported entirely on the branch $\{ x_{\ybranch_{< k'}} : k' \in \nats \}$, 
  we may denote by $K_i$ the $X_i$-dependent $\nats$-valued random variable 
  with value satisfying $X_i = x_{\ybranch_{< K_i}}$ (almost surely): 
  that is, $K_i$ is the \emph{depth} of the sample $X_i$ along the branch $\bybranch$ 
  in the Littlestone tree.
  
  Let $N_k = |\{ i \leq n : K_i = k \}|$, 
  and let $i_1, \ldots, i_{N_k}$ be the subsequence of values $i \in \{1,\ldots,n_k\}$ 
  for which $K_i = k$.
  We aim to apply Lemma~\ref{lem:coin-testing-lower-bound}, 
  with $t^* = \ybranch_k$, $n = N_k$, 
  $\gamma = \frac{\phi(n_k)}{2 p_k}$,
  $\delta = \frac{1}{8e}$,
  $(B_1,\ldots,B_n) = (Y_{i_1},\ldots,Y_{N_k})$, 
  and $\hat{t}_n(B_1,\ldots,B_n) = \hat{h}_{n_k}(x_{\ybranch_{< k}})$.
  For this to be a valid use of the lemma, 
  we must apply the lemma under a \emph{conditional} distribution.
  Specifically, in the context of the lemma, 
  we intend to condition on the data $(X_i,Y_i)$ having $K_i \neq k$;
  we must be careful in this, since any elements $(X_i,Y_i)$ with $K_i > k$ 
  would cause a dependence with $\ybranch_k$ (in particular, 
  the value of such an $X_i$ would immediately reveal the value of $\ybranch_k$).
  To avoid this, we will restrict to the event 
  $\max\{ K_i : i \leq n_k \} \leq k$, 
  in which case the data subsequence 
  $S_{\setminus k} := \{ (X_i,Y_i) : i \in \{1,\ldots,n_k\} \setminus \{i_1,\ldots,i_{N_k}\} \}$
  is conditionally independent of $\ybranch_k$ given this event.
  In particular, on this event, 
  $\ybranch_k$ remains conditionally $\mathrm{Bernoulli}(\frac{1}{2})$ given
  $N_k$, $i_1,\ldots,i_{N_k}$, and $S_{\setminus k}$.
  Also note that the random variables  
  $Y_{i_1},\ldots,Y_{i_{N_k}}$ are conditionally 
  i.i.d.\ $\mathrm{Bernoulli}\!\left( \frac{1}{2} + (2 \ybranch_k - 1) \frac{1}{2} \frac{\phi(n_k)}{p_k} \right)$
  given $\ybranch_k$, $N_k$, $i_1,\ldots,i_{N_k}$, and $S_{\setminus k}$. 
  Thus, by Lemma~\ref{lem:coin-testing-lower-bound},
  on the events that 
  $\max\{ K_i : i \leq n_k \} \leq k$ and $N_k < \frac{p_k^2}{2 \phi(n_k)^2}$, 
  we have  
  \begin{equation}
  \label{eqn:superroot-lower-bound-testing-prob}
    \P\!\left( \hat{h}_{n_k}(x_{\ybranch_{< k}}) \neq \ybranch_k \middle| N_k, i_1,\ldots,i_{N_k}, S_{\setminus k} \right) > \frac{1}{8e}.
  \end{equation}

  We now turn to bounding the probabilities these events fail to hold.
  First, we have  
  \begin{align*}
  \P\!\left( \max\!\left\{ K_i : i \leq n_k \right\} > k \right)
  & = \P\!\left( \exists i \leq n_k : K_i > k \right)
  \leq \sum_{i \leq n_k} \P\!\left( K_i > k \right) 
  = n_k \sum_{k' > k} p_{k'}
  < \frac{16}{63},
  \end{align*}
  where the first inequality is by the union bound, 
  and the last inequality is due to \eqref{eqn:superroot-pk-sum}.
  Second, we lower bound the probability of $N_k < \frac{p_k^2}{2 \phi(n_k)^2}$.
  Note that $N_k$ is a $\mathrm{Binomial}(n_k,p_k)$
  random variable.
  By \eqref{eqn:superroot-nk-defn}, 
  we have $\E[ N_k ] = p_k n_k < \frac{p_k^2}{16 \phi(n_k)^2}$.
  Therefore, a Chernoff bound implies 
  \begin{equation*}
  \P\!\left( N_k \geq \frac{p_k^2}{2 \phi(n_k)^2} \right)
  \leq \exp\!\left\{ - \frac{49}{9} \frac{p_k^2}{16 \phi(n_k)^2} \right\}
  < \exp\!\left\{ - \frac{49 \cdot 16}{9} \right\}
  < \frac{1}{63},
  \end{equation*}
  where the second inequality is due to \eqref{eqn:superroot-phi-nk-bounded-by-pk} which implies $\frac{p_k}{\phi(n_k)} > 16$.

  Let us now combine the above elements into a single unconditional lower bound: namely, 
  \begin{align} 
  & \P\!\left( \hat{h}_{n_k}(x_{\ybranch_{< k}}) \neq \ybranch_k \right)
  \notag \\ & \geq \E\!\left[ \P\!\left( \hat{h}_{n_k}(x_{\ybranch_{< k}}) \neq \ybranch_k \middle| N_k, i_1,\ldots,i_{N_k}, S_{\setminus k} \right) \ind\!\left[ \max\{ K_i : i \leq n_k \} \leq k \text{ and } N_k < \frac{p_k^2}{2 \phi(n_k)^2} \right] \right]
  \notag \\ & \geq \frac{1}{8e} \P\!\left( \max\{ K_i : i \leq n_k \} \leq k \text{ and } N_k < \frac{p_k^2}{2 \phi(n_k)^2} \right)
  \notag \\ & \geq \frac{1}{8e} \left( 1 - \P\!\left( \max\{ K_i : i \leq n_k \} > k \right) - \P\!\left( N_k \geq \frac{p_k^2}{2 \phi(n_k)^2} \right) \right)
  \notag \\ & > \frac{1}{8e} \left( 1 - \frac{16}{63} - \frac{1}{63} \right)
  > \frac{1}{30}, \label{eqn:superroot-lower-bound-level-k-nk-constant-error-prob}
  \end{align}
  where the second inequality is due to \eqref{eqn:superroot-lower-bound-testing-prob},
  the third inequality is by the union bound, 
  and the fourth inequality is by the analysis of the 
  probabilities of these two events above.

  Our next step will be to remove the randomness 
  in the choice of branch $\bybranch$, via Fatou's lemma.
  Conditioned on $\bybranch$, 
  let $(X,Y) \sim P_{\bybranch}$ 
  be conditionally independent of $S_{n_k}$ given $\bybranch$.
  Denote by $K$ the $X$-dependent $\nats$-valued random variable 
  satisfying $X = x_{\ybranch_{< K}}$.
  Noting that $K$ is independent of $\bybranch$ 
  (since the marginal probability mass under $P_{\mathbf{y}}$ of the set of all nodes at any given depth $k$ is $p_k$, which is invariant to $\bybranch$), 
  we have (almost surely) 
  \begin{equation*}
  \frac{1}{p_k} \P\!\left( \hat{h}_{n_k}(x_{\ybranch_{< k}}) \neq \ybranch_k \text{ and } K = k \middle| \bybranch \right)
  \leq \frac{1}{p_k} \P( K = k ) 
  = 1.
  \end{equation*}
  In particular, this value is \emph{bounded} (almost surely), which will enable us to make use of a
  variant of Fatou's lemma for the $\limsup$ of bounded sequences of random variables.
  Specifically, by applying the above analysis to each $k \in \nats \setminus \{1\}$, 
  together with Fatou's lemma and the law of total probability, we have  
  \begin{align*} 
  & \E\!\left[ \limsup_{k \to \infty} \frac{1}{p_k} \P\!\left( \hat{h}_{n_k}(x_{\ybranch_{< k}}) \neq \ybranch_k \text{ and } K = k \middle| \bybranch \right) \right]
  \\ & \geq \limsup_{k \to \infty} \frac{1}{p_k} \P\!\left( \hat{h}_{n_k}(x_{\ybranch_{< k}}) \neq \ybranch_k \text{ and } K = k \right)
  \\ & = \limsup_{k \to \infty} \frac{1}{p_k} \P\!\left( \hat{h}_{n_k}(x_{\ybranch_{< k}}) \neq \ybranch_k \right) \P\!\left( K = k \right)
  \\ & = \limsup_{k \to \infty} \P\!\left( \hat{h}_{n_k}(x_{\ybranch_{< k}}) \neq \ybranch_k \right) \geq \frac{1}{30},
  \end{align*}
  where the equality on the third line follows from independence of $K$ 
  from $\hat{h}_{n_k}$ and $\bybranch$ (as discussed above), 
  the second inequality is due to the definition of $P_{\bybranch}$ (i.e., $\P( K = k ) = p_k$),
  and the final inequality is due to 
  \eqref{eqn:superroot-lower-bound-level-k-nk-constant-error-prob}.
  Therefore, there exists a realization of $\bybranch$ 
  such that 
  $\P\!\left( \hat{h}_{n_k}(x_{\ybranch_{< k}}) \neq \ybranch_k \text{ and } K = k \middle| \bybranch \right) \geq \frac{p_k}{30}$ 
  infinitely often.
  Let $\mathbf{y} = \bybranch$ for this realization:
  that is, $\mathbf{y}$ is a \emph{non-random} choice of 
  a sequence $\mathbf{y} = \{y_k\}_{k \in \nats}$ 
  of values in $\{0,1\}$ such that 
  for $\PXY = P_{\mathbf{y}}$, 
  $S_n \sim \PXY^n$ and $\hat{h}_n = f_n(S_n)$ 
  for each $n \in \nats$,
  for $(X,Y) \sim \PXY$ independent of every $S_n$,
  it holds that   
  \begin{equation*} 
  \P\!\left( \hat{h}_{n_k}(x_{y_{< k}}) \neq y_k \text{ and } X = x_{y_{< k}} \right) \geq \frac{p_k}{30}
  \end{equation*}
  for infinitely many $k$.

  Also note that, since the Littlestone tree above is 
  shattered by $\H$, every $k \in \nats$ has some 
  $h_{y_{\leq k}} \in \H$ with 
  $\forall k' \in \{1,\ldots,k\}$, $h_{y_{\leq k}}(x_{y_{< k'}}) = y_{k'}$.
  Since the marginal distribution of $\PXY$ is supported on the branch $\{ x_{y_{< k}} : k \in \nats \}$, 
  letting $\target_{\mathbf{y}}$ denote a measurable function 
  $\X \to \{0,1\}$ with $\target_{\mathbf{y}}(x_{y_{< k}}) = y_k$ for every $k \in \nats$, 
  we have that
   \begin{align*}
  & \inf_{h \in \H} \P\!\left( h(X) \neq \target_{\mathbf{y}}(X) \right) 
  \leq \lim_{k \to \infty} \P\!\left( h_{y_{\leq k}}(X) \neq \target_{\mathbf{y}}(X) \right) 
  = 0.
  \end{align*}
  Note that \eqref{eqn:superroot-conditional-definition}
  implies $\target_{\mathbf{y}}$ is a \emph{Bayes optimal} 
  function: 
  that is, for every measurable $h : \X \to \{0,1\}$, 
  $\P\!\left( \target_{\mathbf{y}}(X) \neq Y \middle| X \right) \leq \P\!\left( h(X) \neq Y \middle| X \right)$ almost surely,
  and therefore also $\er_{\PXY}\!\left(\target_{\mathbf{y}}\right) \leq \er_{\PXY}(h)$.
  It follows that
  \begin{equation*}
  \inf_{h \in \H} \er_{\PXY}(h) = \er_{\PXY}\!\left( \target_{\mathbf{y}} \right).
  \end{equation*}
  Moreover, for any measurable $h : \X \to \{0,1\}$ 
  and any $k \in \nats$, 
  letting $K$ be as above (that is, $K$ is defined to satisfy $X = x_{y_{< K}}$ almost surely)
  \begin{align*}
    & \er_{\PXY}(h) - \er_{\PXY}\!\left(\target_{\mathbf{y}}\right)
    \\ & = \E\!\left[ \P\!\left( h(X) \neq Y \middle| X \right) - \P\!\left( \target_{\mathbf{y}}(X) \neq Y \middle| X \right) \right]
    \\ & \geq \E\!\left[ \left( \P\!\left( h(X) \neq Y \middle| X \right) - \P\!\left( \target_{\mathbf{y}}(X) \neq Y \middle| X \right) \right) \ind\!\left[ K = k \right] \right]
    \\ & = \E\!\left[ \left( \left( \frac{1}{2} + \frac{1}{2}\frac{\phi(n_k)}{p_k} \right) - \left( \frac{1}{2} - \frac{1}{2}\frac{\phi(n_k)}{p_k} \right) \right) \ind\!\left[ h(X) \neq \target_{\mathbf{y}}(X) \right] \ind\!\left[ K = k \right] \right]
    \\ & = \frac{\phi(n_k)}{p_k} \P\!\left( h(x_{y_{< k}}) \neq y_k \text{ and } K = k \right),
  \end{align*}
  where the equality on the fourth line follows from \eqref{eqn:superroot-conditional-definition}
  and the last equality follows from the definitions of $K$ and $\target_{\mathbf{y}}$
  (i.e., $K=k$ implies $X = x_{y_{< k}}$ and $\target_{\mathbf{y}}(X) = \target_{\mathbf{y}}(x_{y_{< k}}) = y_k$).
  Altogether, we have that for infinitely many $k \in \nats$, 
  \begin{align*}
  \E\!\left[ \er_{\PXY}\!\left(\hat{h}_{n_k}\right) - \inf_{h \in \H} \er_{\PXY}(h) \right]
  & \geq \E\!\left[ \frac{\phi(n_k)}{p_k} \P\!\left( \hat{h}_{n_k}(x_{y_{< k}}) \neq y_k \text{ and } X = x_{y_{< k}} \middle| \hat{h}_{n_k} \right) \right] 
  \\ & \geq \frac{1}{30} \phi(n_k).
  \end{align*}
  Since $n_k \to \infty$ as $k \to \infty$, 
  we have established that 
  \begin{equation*}
    \E\!\left[ \er_{\PXY}\!\left(\hat{h}_{n}\right) \right] - \inf_{h \in \H} \er_{\PXY}(h)
    \geq \frac{1}{30} \phi(n)
  \end{equation*}
  for infinitely many $n$.
  Since this conclusion holds for any learning algorithm $f_n$, 
  this completes the proof.
  \end{proof}

\section{Super-Root Rates}
\label{sec:super-square-roof}

As the final component needed to establish Theorem~\ref{thm:agnostic-char}, 
we present the proof that classes which do not shatter an infinite VCL tree are learnable at super-root rates.
This represents by-far the most technically-challenging 
proof in this work.
In particular, a key piece of the proof 
requires a new analysis of universal rates for 
\emph{partial} concept classes of finite VC dimension, 
involving careful reasoning about relations of 
conditional error rates given data prefixes $X_{\leq m}$
of varying sizes.

Formally, this section is dedicated to establishing the following theorem.

\begin{theorem}
  \label{thm:agnostic-super-sqrt-upper-bound}
  If $\H$ does not shatter an infinite VCL tree, then $\H$ is agnostically learnable with $o\!\left(n^{-1/2}\right)$ rates: 
  that is, 
  there exists a learning algorithm $\hat{h}_n$ such that, 
  for every distribution $\PXY$, 
  $\E\!\left[ \er_{\PXY}(\hat{h}_n) \right] - \inf_{h \in \H} \er_{\PXY}(h) = o\!\left( n^{-1/2} \right)$.
\end{theorem}



The formal proof of Theorem~\ref{thm:agnostic-super-sqrt-upper-bound}
is developed in detail below.
Before getting into the details, 
we first briefly outline the strategy underlying the algorithm 
$\hat{h}_n$ achieving this $o\!\left(n^{-1/2}\right)$ rate.
Similarly to the strategy in Theorem~\ref{thm:agnostic-near-exponential-upper-bound}, 
we divide the approach into two subroutines:
one designed to address the case where $\PXY$ is not 
Bayes-realizable with respect to $\H$, which simply 
trains a universally Bayes-consistent learning algorithm $\hat{h}^0_n$ (Lemma~\ref{lem:universally-bayes-consistent}), 
and a second algorithm $\hat{h}^1_n$ 
designed to address the case 
where $\PXY$ is Bayes-realizable with respect to $\H$.
We then select between $\hat{h}^0_n$ and $\hat{h}^1_n$
using held-out data to obtain the final predictor $\hat{h}_n$.
Most of the technical contributions of the approach 
concern the $\hat{h}^1_n$ algorithm.

The strategy underlying this $\hat{h}^1_n$ algorithm 
is inspired by the realizable-case analysis of \citet*{bousquet:21}.
It will be instructive to first summarize that 
original strategy, before discussing the significant modifications required to extend to the agnostic setting.

\paragraph{The Realizable Case:}
In the case of a distribution $\PXY$ realizable with respect to $\H$, 
the key insight of \citet*{bousquet:21} is that, 
if $\H$ does not shatter an infinite VCL tree, 
then (by making use of a winning strategy for an associated Gale-Stewart game)
Lemma~\ref{lem:VCL-SOA-zero-error-rate-bstar} supplies a procedure 
for which, given a large enough i.i.d.\ data set $B^b \sim \PXY^b$ ($b \geq b^*_{\gamma}$), 
with some probability $\gamma > 0$ 
the procedure produces 
a value $\mathbf{k} \in \nats$ 
and a function $\alg^{b,\mathbf{k}}_{\VCL}(B^b) : \X^{\mathbf{k}} \to \{0,1\}^{\mathbf{k}}$
for which
\begin{equation} 
\label{eqn:realizable-pattern-avoidance}
\PXY^{\mathbf{k}}\!\left( (x_1,y_1),\ldots,(x_{\mathbf{k}},y_{\mathbf{k}}) : \alg^{b,\mathbf{k}}_{\VCL}(B^b,x_1,\ldots,x_{\mathbf{k}}) \neq (y_{1},\ldots,y_{\mathbf{k}}) \right) = 1.
\end{equation}
The next observation is that one can use this function 
$\alg^{b,\mathbf{k}}_{\VCL}(B^b)$ to induce a 
\emph{partial concept class} $\G$ of finite VC dimension
(i.e., a class of measurable functions $g : \X \to \{0,1,*\}$), 
for which $\PXY$ is realizable.
Specifically, consider the set $\G$ of \emph{all} partial concepts 
$g : \X \to \{0,1,*\}$ with finite \emph{support}
$\mathrm{supp}(g) := \{ x \in \X : g(x) \neq * \}$
such that every $\{x_1,\ldots,x_{\mathbf{k}}\} \subseteq \mathrm{supp}(g)$
satisfies $\alg^{b,\mathbf{k}}_{\VCL}(B^b,x_1,\ldots,x_{\mathbf{k}}) \neq (g(x_{1}),\ldots,g(x_{\mathbf{k}}))$.
Then note that, for any $m \in \nats$,
for $\{(X_1,Y_1),\ldots,(X_m,Y_m)\} \sim \PXY^m$, 
it follows from \eqref{eqn:realizable-pattern-avoidance},
and a union bound over the $\binom{m}{\mathbf{k}}$ subsets of size $\mathbf{k}$
that, with probability one, $\exists g \in \G$ 
with $(g(X_1),\ldots,g(X_m)) = (Y_1,\ldots,Y_m)$.
We refer to this scenario 
by saying $\PXY$ is \emph{realizable} with respect to $\G$.
Moreover, by its definition, 
the VC dimension of $\G$ is at most $\mathbf{k}-1$
(i.e., for any $x_1,\ldots,x_{\mathbf{k}} \in \X$,
no $g \in \G$ realizes the labels
$\alg^{b,\mathbf{k}}_{\VCL}(B^b,x_1,\ldots,x_{\mathbf{k}})$),
which we may coarsely bound by $b$ since $\mathbf{k} \leq b$
(by definition in Lemma~\ref{lem:VCL-SOA-zero-error-rate-bstar}).
\citet*{bousquet:21} then use a held-out training set of size $\Omega(n)$,
applying a general realizable-case learning algorithm 
for partial concept classes $\G$ of finite VC dimension (namely, the \emph{one-inclusion graph} predictor of \citealp*{haussler:94}), 
to obtain a conditional expected error guarantee $O\!\left(\frac{b}{n}\right)$
given $B^b$.

Some technicalities need be addressed for this strategy to succeed: 
namely, choosing a size $b$ sufficiently large 
($b \geq b^{*}_{\gamma}$) for Lemma~\ref{lem:VCL-SOA-zero-error-rate-bstar} to be applicable (recalling $b^*_{\gamma}$ is a $\PXY$-dependent value), 
and addressing the fact that the property \eqref{eqn:realizable-pattern-avoidance} (and hence realizability of $\PXY$ with respect to $\G$)
only holds with a constant probability $\gamma > 0$.
These are addressed via similar techniques as in 
the proof of Theorem~\ref{thm:agnostic-near-exponential-upper-bound}, 
by using multiple ($\propto n/b$) 
independent batches $B^b_i$ of size $b$
(with a $\gamma > 1/2$), 
and aggregating the resulting predictors via a majority vote,
and selecting the batch size $b$ using additional held-out data.


\paragraph{The Agnostic Case:}
As mentioned above, the algorithm $\hat{h}^1_n$ will 
only need to address the case where $\PXY$ is 
Bayes-realizable with respect to $\H$: 
that is, $\inf_{h \in \H} \er_{\PXY}(h) = \inf_{h} \er_{\PXY}(h)$.
In particular, due to Lemmas~\ref{lem:no-infinite-VCL-implies-UGC}, \ref{lem:UGC-totally-bounded}, and \ref{lem:limit-function}, 
in this case there exists a measurable function 
$\target : \X \to \{0,1\}$ 
with $\er_{\PXY}(\target) = \inf_{h} \er_{\PXY}(h)$
and $\inf_{h \in \H} \Px( x : h(x) \neq \target(x) ) = 0$.
This further implies that, if we denote by $P_{\target}$ the 
distribution of $(X,\target(X))$ (for $X \sim \Px$), 
$P_{\target}$ is realizable with respect to $\H$.

Several aspects of the above realizable-case strategy are 
made more challenging in the case when $\PXY$ 
is merely \emph{Bayes-realizable} with respect to $\H$
rather than realizable.
First, since $\PXY$ is not necessarily 
realizable with respect to $\H$, 
we cannot simply apply the procedure from Lemma~\ref{lem:VCL-SOA-zero-error-rate-bstar}
to induce an appropriate partial concept class $\G$.
We address this issue 
similarly to the use of Lemma~\ref{lem:SOA-zero-error-rate-bstar} in the proof of Theorem~\ref{thm:agnostic-near-exponential-upper-bound}.
Specifically, 
letting $B^b = \{(X_1,Y_1),\ldots,(X_b,Y_b)\} \sim \PXY^b$, 
we consider applying the procedure from Lemma~\ref{lem:VCL-SOA-zero-error-rate-bstar}
under all $2^b$ \emph{relabelings} of $B^b$:
that is, for each $\mathbf{y} = (y_1,\ldots,y_b) \in \{0,1\}^b$,
we let $B^b(\mathbf{y}) := \{ (X_1,y_1),\ldots,(X_b,y_b) \}$, 
and apply the procedure from Lemma~\ref{lem:VCL-SOA-zero-error-rate-bstar}
to obtain a $\mathbf{k}^\mathbf{y}_b \in \nats$
and function $\alg^{b,\mathbf{k}^{\mathbf{y}}_b}_{\VCL}(B^b(\mathbf{y})) : \X^{\mathbf{k}^{\mathbf{y}}_b} \to \{0,1\}^{\mathbf{k}^{\mathbf{y}}_b}$,
which induces a partial concept class $\G^{\mathbf{y}}_{b}$
of VC dimension $\mathbf{k}^{\mathbf{y}}_b-1$ as above.
We then define a partial concept class 
$\G_b := \bigcup_{\mathbf{y}} \G^{\mathbf{y}}_b$.
In particular, since each $\mathbf{k}^{\mathbf{y}}_b \leq b$
(by definition), 
we can argue that this partial concept class $\G$ will have VC dimension $O(b)$. 
Moreover, since one of these $\mathbf{y}$ labelings 
satisfies $\mathbf{y} = \target(X_{\leq b})$,
and the distribution $P_{\target}$ 
is realizable with respect to $\H$, 
the guarantee from Lemma~\ref{lem:VCL-SOA-zero-error-rate-bstar} provides that, 
with probability $\gamma > 0$, 
$P_{\target}$ is also realizable with respect to $\G_b$.

Given the above partial concept class $\G_b$, 
we arrive at what is perhaps the most challenging part of the proof: 
namely, we require a learning rule for partial concept classes $\G$ of finite VC dimension achieving $o(n^{-1/2})$ 
rate for all distributions Bayes-realizable with respect to $\G$.
Such a learning rule is not known in the literature (indeed, even $n^{-1/2}$ is not known; see \citealp*{alon:21}), 
and thus the design and analysis of such a learner is itself a novel contribution of this work.
Fortunately, the Bayes-realizability condition 
turns out to make this problem approachable by a natural 
learning strategy based on \emph{transductive} empirical risk minimization: that is, for a data set $S$ of size $n$, 
plus a test point $x$ we wish to predict for, 
we consider the set of all classifications of $S$ plus $x$
realizable by $\G$, and among these we choose the classification having minimum empirical error rate on the 
first $n/2$ examples in $S$, and predict the label of $x$ 
in this classification.
The analysis of this method, showing the $o(n^{-1/2})$ rate, 
relies on reasoning about the problem 
in the framework of \emph{transductive} learning, 
arguing that any reasonably-good function in the class 
has a near-Bayes function relatively close to it in empirical $L_1$ distance (specifically, at a distance vanishing as $n \to \infty$), so that a 
uniform Bernstein inequality implies an excess risk which decreases slightly faster than $n^{-1/2}$.
This reasoning focuses on concentration of 
differences of empirical error rates around their \emph{conditional} expectations given $X_{\leq n}$, 
together with concentration of such conditional expectations
(and related quantities) 
as a function of the sample size $n$.

With the $o(n^{-1/2})$ rate thereby established 
for any partial concept class of finite VC dimension, 
we can use this learning rule in combination with the 
partial concept class $\G_b$ induced by the first part above.
As with the realizable case, 
for this strategy to be successful, 
we need to address some additional technicalities:
most notably, the fact that the realizability of 
$P_{\target}$ with respect to $\G_b$ only holds with 
a constant probability $\gamma > 0$, 
and the fact that this guarantee only holds for 
choices of batch size $b \geq b^*_{\gamma}$
(a $P_{\target}$-dependent value).
We address these analogously to the realizable-case strategy, 
though the argument is significantly more nuanced
in the case of non-realizable distributions.
Specifically, rather than a single batch $B^b$, 
we use multiple ($\propto n/b$) 
independent batches $B^b_i$ of size $b$, 
each of which leads to a partial concept class $\G_{b,i}$
and corresponding predictor $\hat{h}_{n,i}$ 
via the above strategy (with a $\gamma > 1/2$),
and we aggregate the resulting predictors $\hat{h}_{n,i}$
by a majority vote
(which, we argue, retains the $o(n^{-1/2})$ rate guarantee,
using the Bayes-realizability of $\PXY$).
To select an appropriate batch size $\hat{b}_n$,
we use a held-out data set $S^1_n$ of size $\Omega(n)$
to compare the minimal empirical error rates 
$\hat{\er}_{S^1_n}(\G_{b,i})$
achievable by partial concepts in $\G_{b,i}$,
and select the smallest $b$ for which 
a significant fraction of indices $i$ 
have $\hat{\er}_{S^1_n}(\G_{b,i})$
not much larger than \emph{every} 
$\hat{\er}_{S^1_n}(\G_{b',i'})$, $b' > b$.
We argue that such a selection $\hat{b}_n$
will indeed (with high probability) 
result in $P_{\target}$ being Bayes-realizable 
with respect to most of the $\G_{b,i}$ classes,
so that the above guarantees on the 
resulting predictors $\hat{h}_{n,i}$ are valid.



We present the formal details in the subsections below,
beginning with the proof of $o(n^{-1/2})$ rates 
for partial concept classes of finite VC dimension
in the case of Bayes-realizable $\PXY$
(Section~\ref{sec:subsection-partial-concepts-super-root-upper-bound}).
We then present the main proof of Theorem~\ref{thm:agnostic-super-sqrt-upper-bound} 
in Section~\ref{sec:subsection-proof-of-super-root-upper-bound}.

\subsection{Super-root Rates for Partial Concept Classes of Finite VC Dimension}
\label{sec:subsection-partial-concepts-super-root-upper-bound}

As discussed above, as one of the key steps in the proof of Theorem~\ref{thm:agnostic-super-sqrt-upper-bound}, 
we will need a $o(n^{-1/2})$ rate guarantee for any \emph{partial} concept class of finite VC dimension, under the assumption that $\PXY$ is \emph{Bayes-realizable} (defined formally below).  Such partial concept classes arise in the proof of Theorem~\ref{thm:agnostic-super-sqrt-upper-bound} with data-dependent definitions, 
having bounded VC dimension with high probability.

We first introduce some notation for partial concept classes
based on the works of \citet*{alon:21,bartlett:98,long:01}.
A \emph{partial concept} is a function 
$g : \X \to \{0,1,*\}$,
where we interpret `$*$' as meaning the value is ``undefined''.
For any partial concept $g$, the \emph{support} of $g$
is defined as the set $\{ x \in \X : g(x) \in \{0,1\} \}$.
A \emph{partial concept class} $\G$ is a set of partial concepts.

For any $n \in \nats \cup \{0\}$ and 
$x_1,\ldots,x_n \in \X$, 
denote by 
\begin{equation*} 
\G(x_1,\ldots,x_n) = \{ (g(x_1),\ldots,g(x_n)) : g \in \G \} \cap \{0,1\}^n,
\end{equation*}
that is, 
the set of $\G$-realizable \emph{binary} classifications of $x_1,\ldots,x_n$.
For ease of notation, for a sequence 
$S = \{(x_1,y_1),\ldots,(x_n,y_n)\} \in (\X \times \{0,1\})^n$,
we also define $\G(S) = \G(x_1,\ldots,x_n)$.
We will say $\G$ is \emph{universally measurable} if the 
function $(x_1,\ldots,x_n) \mapsto \G(x_1,\ldots,x_n)$ 
is universally measurable for all $n \in \nats \cup \{0\}$.
In all contexts below, 
we restrict our focus to universally measurable
partial concept classes; for brevity, 
we omit this explicit qualification below, 
simply using the term ``partial concept class'' to mean 
``universally measurable partial concept class''.

The VC dimension of $\G$, denoted by $\VC(\G)$, 
is the maximum $n \in \nats \cup \{0\}$ such that 
$\exists S_X \in \X^n$ with $\G(S_X) = \{0,1\}^n$,
or else $\VC(\G) = \infty$ if no such maximum exists \citep*{bartlett:98,long:01,alon:21}.

For any partial concept class $\G$, 
we say a distribution $\PXY$ on $\X \times \{0,1\}$ is \emph{Bayes-realizable} with respect to $\G$ if, 
$\forall n \in \nats$, for $S = \{(X_1,Y_1),\ldots,(X_n,Y_n)\} \sim \PXY^n$, 
it holds that\footnote{It is a straightforward exercise to verify that this coincides with the definition of Bayes-realizable for total concept classes introduced above, in the special case that $\G$ is a total concept class which does not shatter an infinite VCL tree (follows from Lemmas~\ref{lem:no-infinite-VCL-implies-UGC}, \ref{lem:UGC-totally-bounded}, and \ref{lem:limit-function}), though this will not be important for our purposes.}
\begin{equation*}
  \P\!\left( \exists (y_1,\ldots,y_n) \in \G(S) : \forall t \leq n, \P(Y_t=y_t|X_t) \geq \frac{1}{2} \right) = 1.
\end{equation*}
We also denote by $\target_{\PXY}$ a \emph{Bayes optimal} classifier: namely, for $(X,Y) \sim \PXY$, for every $x \in \X$, define 
\begin{equation*}
\target_{\PXY}(x) := \ind\!\left[ \P(Y = 1 | X = x) \geq \frac{1}{2} \right].
\end{equation*}

\citet*{alon:21} study agnostic learning for partial concept classes, 
where they find that the appropriate notion 
``excess error rate'' for a predictor $\hat{h} : \X \to \{0,1\}$ relative to a partial concept class $\G$ is 
$\er_{\PXY}\!\left(\hat{h}\right) - \er_P(\G)$, 
where 
$\er_{\PXY}(\G) = \sup_{n}\E\!\left[\min_{g \in \G} \frac{1}{n} \sum_{i=1}^{n} \ind[ g(X_i) \neq Y_i ]  \right]$,
for $(X_1,Y_1),(X_2,Y_2),\ldots$ independent $\PXY$-distributed random variables.
They propose a learning rule guaranteeing 
a uniform bound of order $\sqrt{\frac{\VC(\G)\log^2(n)}{n}}$,
and it remains open whether this can be improved to 
$\sqrt{\frac{\VC(\G)}{n}}$.
Fortunately, for our purposes (in the proof of Theorem~\ref{thm:agnostic-super-sqrt-upper-bound}), 
we can focus on the case where $\PXY$ 
is \emph{Bayes-realizable} with respect to $\G$, 
which greatly simplifies the problem of designing learning rules with appropriate convergence rate guarantees.
In this context (and assuming $\VC(\G)<\infty$), one can show that 
the definition of $\er_{\PXY}(\G)$ 
simplifies to merely $\er_{\PXY}\!\left(\target_{\PXY}\right)$
(follows from Theorem 14 of \citealp*{alon:21} and 
Lemma~\ref{lem:transductive-uniform-bernstein} below),
and thus we can simply aim to 
bound $\er_{\PXY}\!\left(\hat{h}\right) - \er_{\PXY}\!\left(\target_{\PXY}\right)$ 
for an appropriate predictor $\hat{h}$.

Formally, we have the following result, which will be a central part 
of the proof of Theorem~\ref{thm:agnostic-super-sqrt-upper-bound}, 
and is also of independent interest.

\begin{lemma}
  \label{lem:partial-VC-super-root}
Let $\G$ be any partial concept class with $\VC(\G) < \infty$.
  There exists a sequence $\{\hat{f}_n^{\G}\}_{n \in \nats}$ of universally measurable functions 
  $\hat{f}_n^{\G} : (\X \times \{0,1\})^n \times \X \to \{0,1\}$ 
  (defined explicitly below)
  with the following property.
  For any distribution $\PXY$ on $\X \times \{0,1\}$ 
  that is Bayes-realizable with respect to $\G$,
  there exists a function $\phi(n;\G,\PXY) = o\!\left(n^{-1/2}\right)$ such that $\forall n \in \nats$, 
  $\phi(n;\G,\PXY) \leq c \sqrt{\frac{\VC(\G)}{n}}$ (for a universal constant $c$), and 
  for $(X_1,Y_1),\ldots,(X_n,Y_n)$ independent $\PXY$-distributed random variables, 
  \begin{equation*}
  \E\!\left[ \er_{\PXY}\!\left( \hat{f}_n^{\G}(X_{\leq n},Y_{\leq n}) \right) \right] - \er_{\PXY}\!\left(\target_{\PXY}\right) \leq \phi(n;\G,\PXY).
  \end{equation*}  
\end{lemma}



We construct the learning rule $\hat{f}_n^{\G}$ for Lemma~\ref{lem:partial-VC-super-root} explicitly, 
based on the classic method of \emph{transductive empirical risk minimization} \citep*{vapnik:74,vapnik:98} as follows.
Let $\G$ be a partial concept class 
with $\VC(\G) < \infty$.
For any $n \in \nats$, define a 
function $\TERM_n^{\G} : (\X \times \{0,1\})^{\lceil n/2 \rceil} \times \X^{\lfloor n/2 \rfloor + 1} \to \{0,1\}^{n+1}$
(called \emph{transductive empirical risk minimization}) with the property that, 
for any $(x_1,y_1),\ldots,(x_{\lceil n/2 \rceil},y_{\lceil n/2 \rceil}) \in \X \times \{0,1\}$ and 
$x_{\lceil n/2 \rceil+1},\ldots,x_{n+1} \in \X$, 
\begin{equation*} 
\TERM_n^{\G}(x_{\leq \lceil n/2 \rceil},y_{\leq \lceil n/2 \rceil},x_{\lceil n/2 \rceil + 1},\ldots,x_{n+1}) 
=\!\!  \argmin_{(g(x_1),\ldots,g(x_{n+1})) \in \G(x_1,\ldots,x_{n+1})} \!\sum_{i=1}^{\lceil n/2 \rceil}\! \ind[ g(x_i) \!\neq\! y_i ].
\end{equation*}
In other words, $\TERM_n^{\G}$ 
selects a classification of $x_1,\ldots,x_{n+1}$ among $\G(x_1,\ldots,x_{n+1})$ which minimizes the empirical error rate on the first $\lceil n/2 \rceil$ examples.
For a returned sequence 
$(\hat{g}(x_1),\ldots,\hat{g}(x_{n+1})) = \TERM_n^{\G}(x_{\leq \lceil n/2 \rceil},y_{\lceil n/2 \rceil},x_{\lceil n/2 \rceil+1},\ldots,x_{n+1})$,
for each $i \!\in\! \{1,\ldots,n\!+\!1\}$ 
we denote by\footnote{Note that if a value $x_i$ occurs multiple times in the sequence, $\hat{g}(x_i)$ is necessarily the same for each (since $\hat{g}$ is consistent with some $g \in \G$), so that there is no ambiguity expressing this as a function of $x_i$ rather than $i$.} 
\begin{equation*} 
\TERM_n^{\G}(x_{\leq \lceil n/2 \rceil},y_{\lceil n/2 \rceil},x_{\lceil n/2 \rceil+1},\ldots,x_{n+1})(x_i) = \hat{g}(x_i).
\end{equation*} 
Though it does not affect our analysis, 
for completeness we may define $\TERM_n^{\G}$ 
to return the 
classification $(0,\ldots,0)$ in the event that $\G(x_1,\ldots,x_{n+1}) = \emptyset$.
More importantly for our analysis, we require that the $\argmin$ in the definition of $\TERM_n^{\G}$ breaks ties 
in any way that ensures universal measurability of $\TERM_n^{\G}$
\emph{and} 
which has no dependence on the \emph{order} of $x_{\lceil n/2 \rceil+1},\ldots,x_{n+1}$,\footnote{For instance, this may be achieved 
by first constructing the order statistic 
$x_{(1)},\ldots,x_{(\lfloor n/2 \rfloor+1)}$
of the sequence 
$x_{\lceil n/2 \rceil+1},\ldots,x_{n+1}$
based on a measurable total ordering of $\X$ 
(to obscure the original order)
before applying any measurable tie-breaking strategy
for evaluating $\TERM_n^{\G}(x_{\leq \lceil n/2 \rceil},y_{\leq \lceil n/2 \rceil},x_{(1)},\ldots,x_{(\lfloor n/2 \rfloor+1)})$, 
and then reversing this re-ordering to produce the output sequence for  
$\TERM_n^{\G}(x_{\leq \lceil n/2 \rceil},y_{\leq \lceil n/2 \rceil},x_{\lceil n/2 \rceil+1},\ldots,x_{n+1})$.} 
meaning any bijection $\pi : \{ \lceil n/2 \rceil + 1, \ldots, n+1 \} \to \{ \lceil n/2 \rceil + 1, \ldots, n+1 \}$ satisfies that if 
$(y_1,\ldots,y_{n+1}) = \TERM_n^{\G}(x_{\leq \lceil n/2 \rceil},y_{\leq \lceil n/2 \rceil},x_{\lceil n/2 \rceil+1},\ldots,x_{n+1})$
then 
\begin{equation*} 
\TERM_n^{\G}(x_{\leq \lceil n/2 \rceil},y_{\leq \lceil n/2 \rceil},x_{\pi(\lceil n/2 \rceil + 1)},\ldots,x_{\pi(n+1)}) 
\!=\! (y_{1},\ldots,y_{\lceil n/2 \rceil},y_{\pi(\lceil n/2 \rceil + 1)},\ldots,y_{\pi(n+1)}).
\end{equation*} 
We then define the learning rule $\hat{f}_n^{\G}$ satisfying Lemma~\ref{lem:partial-VC-super-root} explicitly as follows.
For any $(x_1,y_1),\ldots,(x_n,y_n) \in \X \times \{0,1\}$ and $x_{n+1} \in \X$, 
\begin{equation}
\label{eqn:partial-VC-term-learning-rule}
\hat{f}_n^{\G}(x_{\leq n},y_{\leq n},x_{n+1}) = \TERM_n^{\G}(x_{\leq \lceil n/2 \rceil},y_{\leq \lceil n/2 \rceil},x_{\lceil n/2 \rceil +1},\ldots,x_{n+1})(x_{n+1}).
\end{equation}
In other words, the algorithm ignores the labels of the last 
$\lfloor n/2 \rfloor$ training examples, and runs transductive ERM.
The intuitive reasoning behind this approach is that, 
due to the 
above order-invariance property and exchangeability of the i.i.d.\ data sequence, 
letting $\hat{g} = \TERM_n^{\G}(X_{\leq \lceil n/2 \rceil}, Y_{\leq \lceil n/2 \rceil}, X_{\lceil n/2 \rceil+1},\ldots,X_{n+1})$,
we have 
\begin{align*} 
& \E\!\left[\er_{\PXY}\!\left(\hat{f}_n^{\G}(X_{\leq n},Y_{\leq n})\right) - \er_{\PXY}(\target_{\PXY}) \right] 
= \E\!\left[ \ind\!\left[ \hat{g}(X_{n+1}) \neq Y_{n+1} \right] - \ind\!\left[ \target_{\PXY}(X_{n+1}) \neq Y_{n+1} \right] \right]
\\ & = \frac{1}{\lfloor n/2 \rfloor + 1} \sum_{i=\lceil n/2 \rceil+1}^{n+1} \E\!\left[ \ind\!\left[ \hat{g}(X_i) \neq Y_i \right] - \ind\!\left[ \target_{\PXY}(X_i) \neq Y_i \right] \right]
\\ & = \E\!\left[ \frac{1}{\lfloor n/2 \rfloor + 1} \sum_{i=\lceil n/2 \rceil+1}^{n+1} \ind\!\left[ \hat{g}(X_i) \neq Y_i \right] - \ind\!\left[ \target_{\PXY}(X_i) \neq Y_i \right] \right],
\end{align*}
so that we can focus the analysis on the 
excess empirical loss on the entire half-sample $(X_{\lceil n/2 \rceil+1},Y_{\lceil n/2 \rceil+1}),\ldots,(X_{n+1},Y_{n+1})$.
Moreover, the learner minimizes the empirical loss on the first 
half-sample $(X_{\leq \lceil n/2 \rceil},Y_{\leq \lceil n/2 \rceil})$.
By exchangeability, this first half-sample can be viewed as a random sample (without replacement) from the full data set $(X_{\leq n+1},Y_{\leq n+1})$,
so that (by uniform concentration for sampling without replacement) 
it is unlikely that the excess empirical loss on the second half-sample 
$(X_{\lceil n/2 \rceil+1},Y_{\lceil n/2 \rceil+1}),\ldots,(X_{n+1},Y_{n+1})$
is significantly different from that on the first half-sample 
$(X_{\leq \lceil n/2 \rceil},Y_{\leq \lceil n/2 \rceil})$.
This latter type of analysis is common to the literature on 
\emph{transductive} learning, a setting in which the data sequence 
$(X_{\leq n+1},Y_{\leq n+1})$ is considered fixed, and the learner 
observes all of $X_{\leq n+1}$ along with the labels $Y_i$ 
for a uniform random (without replacement) 
sample of $m < n+1$ of the indices $i \leq n+1$, 
and performance is measured by the empirical error on the remaining $n+1-m$ samples \citep*{vapnik:74,vapnik:98}.
In the case of i.i.d.\ samples, by exchangeability the $m$ observable labels can be $Y_1,\ldots,Y_m$ without loss of generality.

\paragraph{Outline of the proof of Lemma~\ref{lem:partial-VC-super-root}:}
It remains to prove that the properties claimed in the lemma are indeed satisfied by this learning rule $\hat{f}_n^{\G}$.
We briefly outline the argument, before proceeding with the formal details.
As mentioned above, the core idea is to proceed by analysis of the \emph{transductive} 
learning problem: namely, it suffices to argue that the 
excess empirical error rate on $(X_{\lceil n/2 \rceil+1},Y_{\lceil n/2 \rceil +1}),\ldots,(X_{n+1},Y_{n+1})$ is $O\!\left(\sqrt{\frac{\vc(\G)}{n}}\right)$ and $o(n^{-1/2})$ in expectation, 
in which case the same is true of
the expected excess error rate $\E\!\left[ \er_{\PXY}\!\left(\hat{f}_n^{\G}(X_{\leq n},Y_{\leq n})\right) \right] - \er_{\PXY}(\target_{\PXY})$
by the above argument. 

The $\sqrt{\frac{\vc(\G)}{n}}$ bound follows from an 
extension of the classic \emph{chaining} analysis to transductive 
learning, and is known \citep*[e.g.,][]{el-yaniv:09}.
However, the $o(n^{-1/2})$ bound is significantly more involved.
In essence, this part of the proof follows a \emph{localization}-based analysis for empirical risk minimization (see e.g., \citealp*{bartlett:04,bartlett:05,koltchinskii:06}).
However, the details are made more challenging 
due to the strictly empirical nature of the transductive setting.
While extensions of localization-based analysis are known for transductive 
learning \citep*{tolstikhin:14}, the key component enabling the $o(n^{-1/2})$ rate 
does not follow from prior works.

\paragraph{Total concept classes:}
To outline the basic idea of the analysis, consider first how this analysis would proceed 
in the much simpler case of a \emph{total} concept class $\H$ of finite VC dimension,
for the standard (inductive) \emph{empirical risk minimization} algorithm: 
namely, 
$\hat{h}_n = \argmin_{h \in \H} \hat{\er}_{S_n}(h)$, 
for $S_n = \{(X_1,Y_1),\ldots,(X_n,Y_n)\}$ i.i.d.\ $\PXY$.
A rate of $o(n^{-1/2})$ for this method is established in the work of \citet*{hanneke:25a}, by the following argument.
The first key component is the well-known \emph{uniform Bernstein inequality} 
for excess error rates \citep*[based on][]{bousquet:02,van-der-Vaart:96},
which states that, 
with probability at least $1-\delta$, 
every $f,g \in \H$ satisfy that 
\begin{align}
& \left| \left( \hat{\er}_{S_n}(f) - \hat{\er}_{S_n}(g) \right) - \left( \er_{\PXY}(f) - \er_{\PXY}(g) \right) \right|
\notag \\ & \lesssim \sqrt{ \Px( x : f(x) \neq g(x) ) \frac{1}{n}\left( \vc(\H) \log\!\left(\frac{1}{\Px( x : f(x) \neq g(x) )}\right) + \log\!\left(\frac{1}{\delta}\right) \right) }
\notag \\ & + \frac{1}{n}\left( \vc(\H) \log\!\left(\frac{n}{\vc(\H)}\right) + \log\!\left(\frac{1}{\delta}\right) \right).
\label{eqn:inductive-uniform-bernstein}
\end{align}
Together with the fact that $\hat{h}_n$ has minimal $\hat{\er}_{S_n}(\hat{h}_n)$, 
it follows that with probability at least $1-\frac{1}{n}$, 
$\er_{\PXY}(\hat{h}_n) - \inf_{h \in \H} \er_{\PXY}(h) \leq \epsilon'_n$ for a value $\epsilon'_n = O\!\left(\sqrt{\frac{\vc(\H)+\log(n)}{n}}\right)$.
Denote by $\H(\epsilon) := \left\{ h \in \H : \er_{\PXY}(h) - \inf_{h' \in \H} \er_{\PXY}(h') \leq \epsilon \right\}$ for all $\epsilon > 0$, 
and, for a value $\tilde{\epsilon}_n = O\!\left( \frac{\vc(\H)}{n}\log\!\left(\frac{n}{\vc(\H)}\right) \right)$,
define
\begin{equation*}
\sigma^2_n := \tilde{\epsilon}_n \lor \sup_{h \in \H(\epsilon'_n)} \inf_{h' \in \H(\tilde{\epsilon}_n)} \Px( x : h(x) \neq h'(x) ).
\end{equation*}
Then, again based on the inequality \eqref{eqn:inductive-uniform-bernstein}, this further implies that 
\begin{equation}
\label{eqn:inductive-localized-bound}
\E\!\left[ \er_{\PXY}(\hat{h}_n) \right] - \inf_{h \in \H} \er_{\PXY}(h)
\lesssim \sqrt{ \sigma^2_n \frac{1}{n} \left( \vc(\H) \log\!\left(\frac{1}{\sigma^2_n}\right) \right) },
\end{equation}

As a second key component, 
by a simple topological argument, 
using the fact that total concept classes $\H$ of finite VC dimension are \emph{totally bounded} in the pseudo-metric $(f,g) \mapsto \Px( x : f(x) \neq g(x) )$,
\citet*{hanneke:25a} argue that $\sigma^2_n \to 0$ as $n \to \infty$, 
establishing the rate $o(n^{-1/2})$.

\paragraph{Partial concept classes:} 
There are several challenges in extending this analysis to \emph{partial} concept classes $\G$.
As a first issue, 
abbreviating by $\hat{\er}_m(g) := \frac{1}{m} \sum_{i=1}^{m} \ind[ g(X_i) \neq Y_i ]$ for any $m \in \nats$, 
we \emph{cannot} rely on concentration of the empirical error rates 
$\hat{\er}_{\lceil n/2 \rceil}(g)$ to the population error rate $\er_{\PXY}(g)$ for $g \in \G$: for instance, there may be partial concepts $g \in \G$ 
whose support is precisely restricted to the data $X_{\leq n+1}$, 
so that even if they have low empirical error $\hat{\er}_{\lceil n/2 \rceil}(g)$, 
they may have $g(X) = *$ with probability one for an independent $X \sim \Px$.
Thus, we need a purely \emph{empirical} analysis, 
which only argues about performance \emph{on the data}.
Naturally, we may consider replacing $\er_{\PXY}(g)$ 
with the empirical loss $\hat{\er}_{n+1}(g)$ 
on the data set $(X_{\leq n+1},Y_{\leq n+1})$, 
which at least avoids the above issue of the supports.
Likewise, we may replace $\Px( x : f(x) \neq g(x) )$ 
with the empirical distances $P_{n+1}( f \neq g ) := \frac{1}{n+1} \sum_{i=1}^{n+1} \ind[ f(X_i) \neq g(X_i) ]$.
However, the second issue, which is not addressed by this, 
is the argument that $\sigma^2_n \to 0$.  This is inherently a 
property of a \emph{population} $\PXY$.
Extending this aspect of the argument is significantly more challenging,
and most of the proof will be dedicated toward this issue.

Our approach to handling this is in two parts,
separating the analysis of the random $Y_i$ 
labels from the analysis of the random $X_i$ sequence.  
We outline these parts briefly here, 
and explain each part more explicitly after,
and in full detail in Section~\ref{sec:proof-of-lemma-partial-VC-super-root}.
As the first part, 
we provide a concentration inequality analogous to the 
uniform Bernstein inequality in \eqref{eqn:inductive-uniform-bernstein},
but rather than concentration of differences of empirical errors 
$\hat{\er}_{m}(f) - \hat{\er}_{m}(g)$ around
differences of \emph{population} errors
for $f,g \in \G(X_{\leq m})$, 
we instead argue concentration around
differences $\ber_{m}(f) - \ber_{m}(g)$ 
of the respective \emph{conditional} error rates given 
the $X_i$'s: 
\begin{equation*} 
\ber_{m}(g) := \E\!\left[ \hat{\er}_{m}(g) \middle| X_{\leq m} \right] = \frac{1}{m} \sum_{i=1}^{m} \P(Y_i \neq g(X_i) | X_i),
\end{equation*}
a quantity we refer to as the \emph{semi-empirical} error rate.
From there, as the second part of the analysis, 
we focus on the randomness in the $X_i$ sequence.
In this context, it turns out the relevant quantity to replace $\sigma^2_n$
is the maximum \emph{empirical} distance $P_{n+1}( g \neq g' )$ 
from any $g \in \G(X_{\leq n+1})$
with $\ber_{n+1}(g) - \ber_{n+1}(\target_{\PXY}) \lesssim \epsilon_n = \tilde{O}\!\left(n^{-1/2}\right)$
to the closest $g' \in \G(X_{\leq n+1})$ 
having $\ber_{n+1}(g') - \ber_{n+1}(\target_{\PXY}) \lesssim \epsilon_n^2$
(see $\bar{\sigma}^2_{n+1}$ defined below).
To show this distance is small, we again must proceed by arguments 
purely based on the data sequence. 
Toward this end, we imagine the $(X_{\leq n+1},Y_{\leq n+1})$ data as a prefix of 
an infinite i.i.d.\ data sequence $(X_{< \infty},Y_{< \infty})$, 
and reason that, for some diverging sequence $M_n \leq n$ and some $n' = \tilde{\Omega}(n^2)$,
for each classification $g$ in $\G(X_{\leq n+1})$ 
with $\ber_{n+1}(g) - \ber_{n+1}(\target_{\PXY}) \leq \epsilon_n$,
there exists a classification $g' \in \G(X_{\leq n'})$ 
with $\ber_{n'}(g) - \ber_{n'}(\target_{\PXY}) \leq \epsilon_{n'}$
such that $g$ and $g'$ \emph{agree} on a \emph{prefix}
$X_{\leq M_n}$.
Together with reasoning about 
uniform concentration for sampling without replacement, 
this leads to the desired result.

More explicitly,  
for the first part, 
%
by an extension of the proof of \eqref{eqn:inductive-uniform-bernstein}
to independent but \emph{non-identically distributed} samples (Proposition~\ref{prop:uniform-bernstein} below), 
we are able to prove 
(in Lemma~\ref{lem:transductive-uniform-bernstein})
a variant of the uniform Bernstein inequality
\eqref{eqn:inductive-uniform-bernstein}, 
establishing that, 
with conditional probability at least $1-\delta$ (given
$X_{\leq m}$), 
every $f,g \in \G(X_{\leq m})$ satisfy
\begin{align}
& \left| \left( \hat{\er}_{m}(f) - \hat{\er}_{m}(g) \right) - \left( \ber_{m}(f) - \ber_{m}(g) \right) \right|
\notag \\ & \lesssim \sqrt{ P_m( f \neq g ) \frac{1}{m}\left( \vc(\H) \log\!\left(\frac{1}{P_m( f \neq g )}\right) + \log\!\left(\frac{1}{\delta}\right) \right) }
\notag \\ & + \frac{1}{m}\left( \vc(\H) \log\!\left(\frac{m}{\vc(\H)}\right) + \log\!\left(\frac{1}{\delta}\right) \right),
\label{eqn:transductive-uniform-bernstein}
\end{align}
where 
\begin{equation*} 
P_m( f \neq g ) := \frac{1}{m} \sum_{i=1}^{m} \ind[ f(X_i) \neq g(X_i) ].
\end{equation*}

Letting $m = \lceil n/2 \rceil$, 
it follows from \eqref{eqn:transductive-uniform-bernstein} 
that
\begin{equation*}
\E\!\left[ \ber_{m}(\hat{f}_n^{\G}(X_{\leq n},Y_{\leq n})) - \ber_{m}(\target_{\PXY}) \middle| X_{\leq m} \right] = O\!\left( \sqrt{\frac{\vc(\G)}{m}} \right).
\end{equation*}
Moreover, since  
$\P( Y_i \neq g(X_i) | X_i ) - \P( Y_i \neq \target_{\PXY}(X_i) | X_i ) \in [0,1]$ for every $g \in \G(X_{\leq n+1})$ and $i \leq n+1$,  
we can employ a uniform \emph{multiplicative} Chernoff bound 
for sampling without replacement (Lemma~\ref{lem:transductive-ber-multiplicative-chernoff}) 
to find that, 
with probability at least $1-\delta$, 
every $g \in \G(X_{\leq n+1})$ satisfies 
\begin{equation}
\label{eqn:multiplicative-bar-er-m-to-n+1}
\ber_{n+1}(g) - \ber_{n+1}(\target_{\PXY})
\leq 2 \left( \ber_{m}(g) - \ber_{m}(\target_{\PXY}) \right) + O\!\left( \frac{1}{m} \left( \vc(\G) \log\!\left(\frac{m}{\vc(\G)}\right) + \log\!\left(\frac{1}{\delta}\right) \right) \right).
\end{equation}
A combination of these two facts already recovers 
(by exchangeability of $X_{m+1},\ldots,X_{n+1}$) 
the claimed guarantee
\begin{equation*}
\E\!\left[ \er_{\PXY}\!\left(\hat{f}_n^{\G}(X_{\leq n},Y_{\leq n})\right) \right] - \er_{\PXY}(\target_{\PXY}) = O\!\left( \sqrt{\frac{\vc(\G)}{n}} \right).
\end{equation*}

Turning to the issue of extending the $o(n^{-1/2})$ argument to partial concept classes,
for any $t \in \nats$ and $\epsilon > 0$, 
let $\bar{\G}_{t}(\epsilon) := \{ g \in \G(X_{\leq t}) : \ber_{t}(g) - \ber_{t}(\target_{\PXY}) \leq \epsilon \}$.
Letting $\epsilon_{t} = \sqrt{\frac{\VC(\G)}{t}\log\!\left(\frac{t}{\vc(\G)}\right)}$,
we define (for appropriate universal constants $c,c'$)
\begin{equation*}
\bar{\sigma}^2_{t} := \epsilon_{t}^2 \lor \max_{g \in \bar{\G}_{t}(c\epsilon_{t})} \min_{f \in \bar{\G}_{t}(c' \epsilon_{t}^2)} P_{t}( f \neq g ).
\end{equation*}
Then, analogously to \eqref{eqn:inductive-localized-bound}, 
a localization argument based on 
\eqref{eqn:transductive-uniform-bernstein}, 
together with \eqref{eqn:multiplicative-bar-er-m-to-n+1}
and an analogous bound relating $P_m( g \neq g' )$ to $P_{n+1}( g \neq g' )$, 
implies 
(Lemma~\ref{lem:transductive-localized-expected-ber-term-expectation-bound})
that
\begin{equation}
\label{eqn:outline-excess-risk-bound}
\E\!\left[ \ber_{n+1}\!\left(\hat{f}_n^{\G}(X_{\leq n},Y_{\leq n})\right) - \ber_{n+1}(\target_{\PXY}) \right] 
\lesssim 
\sqrt{\E\!\left[\bar{\sigma}^2_{n+1}\right] \frac{1}{n} \left(\VC(\G)\log\!\left(\frac{1}{\E\!\left[\bar{\sigma}^2_{n+1}\right]}\right) \right)}.
\end{equation}

Since \eqref{eqn:transductive-uniform-bernstein} and \eqref{eqn:multiplicative-bar-er-m-to-n+1}
imply $\ber_{n+1}\!\left(\hat{f}_n^{\G}(X_{\leq n},Y_{\leq n})\right) - \ber_{n+1}(\target_{\PXY}) \leq c\epsilon_{n+1}$ with probability $1-O\!\left(\frac{1}{n}\right)$, 
the above inequality will suffice for the $o(n^{-1/2})$ guarantee if we can 
argue that $\E\!\left[ \bar{\sigma}^2_{n}\right] \to 0$ as $n \to \infty$.
The latter is perhaps the most challenging 
aspect of this analysis.
The argument for total concept classes, 
showing $\sigma^2_n \to 0$ \citep*{hanneke:25a}, 
relies on the topology of the total concept class $\H$ 
induced by the 
the $L_1(\Px)$ pseudo-metric: 
namely, total boundedness of $\H$.
In the case of partial concept classes, these arguments 
do not apply, and we must formulate an argument 
based purely on the \emph{empirical} behaviors of the 
patterns in $\G(X_{\leq n})$.

Our argument establishing this employs several novel techniques, 
which may be of independent interest.
Specifically, we first provide a uniform concentration 
inequality for sampling without replacement 
(via a technique we call \emph{sample size chaining}), 
guaranteeing that for any $m \in \nats$ with 
$m \leq n$, 
with probability at least $1-\frac{1}{m}$, 
every $f,g \in \G(X_{\leq n})$ with $f(X_{\leq m}) = g(X_{\leq m})$
have $P_n(f \neq g) \lesssim \epsilon_m$ (Lemma~\ref{lem:transductive-ber-uniform-convergence}).
Thus, denoting by $\bar{\G}_{n}(\epsilon) \big|_{m} := \{ g(X_{\leq m}) : g \in \bar{\G}_n(\epsilon) \}$ (the projection of $\bar{\G}_n(\epsilon)$ to $X_{\leq m}$), 
if we can argue that there exists a sequence 
$M_n \to \infty$ such that, with probability $1-O\big(\frac{1}{M_n}\big)$, we have 
$\bar{\G}_n(c\epsilon_n) \big|_{M_n} \subseteq \bar{\G}_n(c'\epsilon_n^2) \big|_{M_n}$, 
then together with the above we will have that
$\bar{\sigma}^2_n \leq \epsilon_{{\scriptscriptstyle M_n}}$
with probability $1-O\big(\frac{1}{M_n}\big)$.

Toward establishing this, as mentioned above, 
we first imagine the sequence 
$(X_{\leq n},Y_{\leq n})$ to be a \emph{prefix} of a (hypothetical) 
\emph{infinite} i.i.d.\ sequence 
$(X_{< \infty},Y_{< \infty}) \sim \PXY^{\infty}$.
Then note that, 
for any outcome of the data sequence $X_{< \infty}$, 
there must exist an $\hat{N}_m \geq m$ such that every 
$n \geq \hat{N}_m$ has infinitely many $t > n$ with 
$\bar{\G}_t(c \epsilon_t) \big|_{m} = \bar{\G}_n(c \epsilon_n) \big|_{m}$: to see this, simply take $\hat{N}_m$ 
one larger than the largest $n \geq m$
for which $\bar{\G}_n(c \epsilon_n) \big|_{m}$ appears finitely many 
times in the sequence of sets $\bar{\G}_t(c \epsilon_t) \big|_{m}$
(noting that $\hat{N}_m$ is well-defined since there are only finitely many possible subsets of $\{0,1\}^m$).
We may then define $N_m$ as a $(\G,\PXY)$-dependent 
sequence such that, for any $m \in \nats$, 
with probability at least $1-\frac{1}{m}$, 
$\hat{N}_m \leq N_m$.
This implies that for any $m,n \in \nats$ 
with $n \geq N_m$, with probability at least $1-\frac{1}{m}$, 
there exist arbitrarily large values $n'$
with $\bar{\G}_{n'}(c \epsilon_{n'})\big|_{m} = \bar{\G}_n(c\epsilon_n)\big|_{m}$
(Lemma~\ref{lem:transductive-equal-equivalence-classes}).
In particular, 
taking such a value $n'$ with $n' = \tilde{\Omega}(n^2)$, 
we have $\bar{\G}_{n'}(c\epsilon_{n'}) \subseteq \bar{\G}_{n'}(c\epsilon_n^2)$.
Moreover, by additional uniform concentration arguments 
for sampling without replacement 
(Lemmas~\ref{lem:transductive-ber-multiplicative-chernoff} and \ref{lem:transductive-ber-uniform-convergence}, 
the latter again based on sample size chaining),
we show that with probability at least $1 - O\!\left(\frac{1}{m}\right)$, 
any such $n' = \tilde{\Omega}(n^2)$
has $\bar{\G}_{n'}(c\epsilon_n^2) \big|_{n} \subseteq \bar{\G}_{n}(c'\epsilon_n^2)$.
Altogether, for any $n \geq N_m$, 
with probability at least 
$1 - O\!\left(\frac{1}{m}\right)$, we have 
$\bar{\G}_n(c\epsilon_n) \big|_{m} \subseteq \bar{\G}_{n}(c'\epsilon_{n}^2) \big|_{m}$
(Lemma~\ref{lem:transductive-projected-prefix-patterns-agree-with-good-concepts}).
Thus, letting $M_n = \max\{ m : n \geq N_m \}$, 
combining the above arguments yields that 
$\bar{\sigma}^2_n \lesssim \epsilon_{{\scriptscriptstyle M_n}}$
with probability at least $1 - O\big(\frac{1}{M_n}\big)$, 
so that $\E\!\left[\bar{\sigma}^2_n\right] \lesssim \epsilon_{{\scriptscriptstyle M_n}} + \frac{1}{M_n}$ (Lemma~\ref{lem:transductive-expected-bar-sigma-converging-to-zero}), 
which indeed establishes $\E\!\left[\bar{\sigma}^2_n\right] \to 0$
by observing that $M_n \to \infty$.
Combining this with \eqref{eqn:outline-excess-risk-bound}
completes the proof of Lemma~\ref{lem:partial-VC-super-root}.

We present the details of this argument in a series of lemmas below.

\subsubsection{A Uniform Bernstein Inequality for Non-identically Distributed Data}
\label{sec:non-identical-uniform-bernstein}

Before presenting the proof of Lemma~\ref{lem:partial-VC-super-root}, 
we first provide a general concentration inequality, 
which extends the \emph{uniform Bernstein inequality} --- stated in \eqref{eqn:inductive-uniform-bernstein} above --- 
to the case of independent non-identically distributed random variables.
In addition to being a crucial part of the proof of Lemma~\ref{lem:partial-VC-super-root}, 
this result is also useful in other contexts, 
such as multi-task learning \citep*[e.g.,][]{hanneke:22a}
or learning with distribution shifts \citep*[e.g.,][]{barve:97,mohri:12,hanneke:19a}.
In particular, it refines a $\log$ factor in Lemma 1 of \citealp*{hanneke:22a} (stated in Lemma~\ref{lem:uniform-bernstein-with-worse-log-factor} below).
As such, we state the result in a slightly stronger form than 
needed for our purposes in the present work.
Its proof follows a familiar line from the literature on 
localized uniform concentration inequalities, adapted to 
account for non-identically distributed samples.

\begin{proposition}
    \label{prop:uniform-bernstein}
    Let $\H$ be any concept class with $\VC(\H) < \infty$.
    Fix any $n \in \nats$ and $\delta \in (0,1)$.
    Let $P_1,\ldots,P_n$ be any probability distributions on $\X \times \{0,1\}$ 
    and denote by $\bar{P} = \frac{1}{n} \sum_{i=1}^n P_i$ (the uniform mixture distribution).
    Let $S = \{(X_1,Y_1),\ldots,(X_n,Y_n)\} \sim P_1 \times \cdots \times P_n$.
    For every $f,g \in \H$ define $\sigma^2(f,g) = \bar{P}((x,y) : f(x) \neq g(x))$,
    $\hat{\sigma}^2(f,g) = \frac{1}{n} \sum_{i=1}^{n} \ind[ f(X_i) \neq g(X_i) ]$,
    and $\tilde{\sigma}^2(f,g) = \sigma^2(f,g) \land \hat{\sigma}^2(f,g)$.
    Also define 
    $\epsilon^2(n,\delta) := \frac{1}{n}\left( \vc(\H)\log\!\left(\frac{n}{\vc(\H)}\right) + \log\!\left(\frac{1}{\delta}\right) \right)$.
    Then with probability at least $1-\delta$, 
    every $f,g \in \H$ satisfy
\begin{align*}
    & \left| \left( \er_{\bar{P}}(f) - \er_{\bar{P}}(g) \right) - \left( \hat{\er}_{S}(f) - \hat{\er}_{S}(g) \right) \right| 
    \\ & \leq \sqrt{\tilde{\sigma}^2(f,g) \frac{\Czero}{n}\!\left( \!\VC(\H) \log\!\left(\frac{1}{\tilde{\sigma}^2(f,g)} \!\land\! \frac{n}{\VC(\H)} \right) \!+\! \log\!\left(\frac{1}{\delta}\right) \right)}
    + \Czero \epsilon^2(n,\delta)
\end{align*}
and {\hskip 22mm} $\dfrac{1}{2}\sigma^2(f,g) - \Czero \epsilon^2(n,\delta)  
\leq \hat{\sigma}^2(f,g) 
\leq 2 \sigma^2(f,g) + \Czero \epsilon^2(n,\delta)$, 

{\vskip 2mm}\noindent where $\Czero \geq 1$ is a universal constant.
\end{proposition}

As mentioned, a slightly weaker form of the above proposition 
(insufficient for our purposes in the present work) was given in the 
work of \citet*[][Lemma 1]{hanneke:22a}: namely, 
they show the following lemma.

\begin{lemma}[\citealp*{hanneke:22a}]
\label{lem:uniform-bernstein-with-worse-log-factor}
    Consider the same setup as in Proposition~\ref{prop:uniform-bernstein}.
    With probability at least $1-\delta$, 
    every $f,g \in \H$ satisfy
\begin{align*}
    & \left| \left( \er_{\bar{P}}(f) - \er_{\bar{P}}(g) \right) - \left( \hat{\er}_{S}(f) - \hat{\er}_{S}(g) \right) \right| 
    \leq \sqrt{\tilde{\sigma}^2(f,g) \HKCzero \epsilon^2(n,\delta)} + \HKCzero \epsilon^2(n,\delta)
\end{align*}
and {\hskip 22mm} $\dfrac{1}{2}\sigma^2(f,g) - \HKCzero \epsilon^2(n,\delta)  
\leq \hat{\sigma}^2(f,g) 
\leq 2 \sigma^2(f,g) + \HKCzero \epsilon^2(n,\delta)$, 

{\vskip 2mm}\noindent where $\HKCzero \geq 1$ is a universal constant.
\end{lemma}

Our proof of Proposition~\ref{prop:uniform-bernstein} 
will directly rely on the portion of this lemma relating $\hat{\sigma}^2(f,g)$ and $\sigma^2(f,g)$, 
rather than retracing the proof of this part 
(we include this relation in the statement of Proposition~\ref{prop:uniform-bernstein} merely for the convenience of unifying the events).

We additionally rely on another basic result, 
stated by \citet*{hanneke:22a} 
(as equation 30 therein, specialized to our present context by taking
$\mathcal{W} = \X \times \{0,1\}$, 
$\F = \{ (x,y) \mapsto \ind[f(x) \neq y] - \ind[g(x) \neq y] : f,g \in \H \}$,
and all $\alpha_i=1$, in their notation), 
which follows immediately from a result of 
\citet*{klein:05} (see also Section 12.5 of \citealp*{boucheron:13}),
a variant of a concentration inequality of 
\citet*{bousquet:02} generalized to 
handle independent but non-identically distributed samples.

\begin{lemma}[\citealp*{klein:05,hanneke:22a}]
\label{lem:non-ident-bousqet-inequality}
Consider the same setup as in Proposition~\ref{prop:uniform-bernstein}.
For any $\beta > 0$, 
letting $\F_{\beta} = \{ (f,g) \in \H^2 : n\sigma^2(f,g) \leq \beta^2 \}$
and
$L_{\beta} = n \cdot \sup_{(f,g) \in \F_{\beta}} \left( \left( \hat{\er}_{S}(f) - \hat{\er}_{S}(g) \right) - \left( \er_{\bar{P}}(f) - \er_{\bar{P}}(g) \right) \right)$,
with probability at least $1-\delta$, 
\begin{equation*}
L_{\beta} \leq \E[L_{\beta}] + 24 \max\!\left\{ \sqrt{ \left( \E[L_{\beta}] + \beta^2 \right) \ln\!\left(\frac{1}{\delta}\right) }, \ln\!\left(\frac{1}{\delta}\right) \right\}.
\end{equation*}
\end{lemma}

We now proceed with the proof of Proposition~\ref{prop:uniform-bernstein}.

\begin{proof}[of Proposition~\ref{prop:uniform-bernstein}]
We adapt a well-known technique for obtaining similar inequalities 
in the case of i.i.d.\ samples (moreover, the essential difference compared to 
the proof of \citealp*{hanneke:22a} is merely the use of 
a sharper inequality, based on chaining, to refine the log factor).
The general idea applies a combination of localization \citep*{bartlett:04,bartlett:05,koltchinskii:06}, 
the concentration inequality from Lemma~\ref{lem:non-ident-bousqet-inequality}, 
and an entropy integral bound on the rate of uniform convergence which also accounts for variances of loss differences \citep*{van-der-Vaart:96,gine:06,van-der-Vaart:11}, 
together with well-known bounds on the covering numbers of VC classes \citep*{haussler:95}.

The claimed inequality is vacuously true for $n < e^2 \VC(\H)$ (for $\Czero \geq e^2$)
or $\vc(\H) = 0$ (i.e., $|\H|=1$);
to address the remaining case, for the remainder of the proof 
we suppose $n \geq e^2\VC(\H) > 0$.

Fix any $\delta \in (0,1)$ and any $\beta > 0$, 
let $\delta_{\beta} = \frac{\delta \beta^2}{4 n}$,
and let $L_{\beta}$ be as in Lemma~\ref{lem:non-ident-bousqet-inequality}.
By Lemma~\ref{lem:non-ident-bousqet-inequality}, 
on an event $E_{\beta}$ of probability at least 
$1 - \delta_{\beta}$,
\begin{equation}
\label{eqn:non-ident-bousquet-inequality}
L_{\beta} 
\leq \E[L_{\beta}] + 24 \max\!\left\{ \sqrt{ \left( \E[L_{\beta}] + \beta^2 \right) \ln\!\left(\frac{1}{\delta_{\beta}}\right) }, \ln\!\left(\frac{1}{\delta_{\beta}}\right) \right\}.
\end{equation}

Next we turn to bounding $\E[L_{\beta}]$.
Let $\xi_1,\ldots,\xi_n$ be independent 
$\mathrm{Uniform}(\{-1,1\})$ random variables 
(also independent of $S$).
By the classic \emph{symmetrization} argument (see e.g., Lemma 11.4 of \citealp*{boucheron:13}),
\begin{align}
\label{eqn:symmetrization-pre-chaining}
\E\!\left[ L_{\beta} \right]
\leq 2 \E\!\left[ \sup_{(f,g) \in \F_{\beta}} \sum_{i=1}^{n} \xi_i \left( \ind[ f(X_i) \neq Y_i ] - \ind[ g(X_i) \neq Y_i ] \right) \right]. 
\end{align}

Next we apply an inequality based on the 
classic \emph{chaining} argument: 
specifically, Corollary 2.2.8 of \citet*{van-der-Vaart:96}.
For each $h \in \H$, 
define a random variable 
$Z_h = \sum_{i=1}^{n} \xi_i \ind[ h(X_i) \neq Y_i ]$.
Note that, for any $f,g \in \H$, 
for any $\gamma > 0$, 
since each $\left| \xi_i \left( \ind[ h(X_i) \neq Y_i ] - \ind[ g(X_i) \neq Y_i ] \right) \right| \leq \ind[ f(X_i) \neq g(X_i) ]$, 
Hoeffding's inequality implies that  
\begin{equation*}
\P\!\left( |Z_f - Z_g| > \gamma \Big| S \right)
\leq 2 e^{- \frac{1}{2} \gamma^2 / n \hat{\sigma}^2(f,g)}.
\end{equation*}
Therefore, letting $\hat{D}(\gamma)$ denote the 
\emph{$\gamma$-packing number} of $\H$
(for any $\gamma > 0$) 
under the pseudo-metric $(f,g) \mapsto \sqrt{n \hat{\sigma}^2(f,g)}$, 
Corollary 2.2.8 of \citet*{van-der-Vaart:96} 
implies that, 
letting $\hat{\beta}^2 = \sup_{(f,g) \in \F_{\beta}} n \hat{\sigma}^2(f,g)$, 
\begin{align*}
& \E\!\left[ \sup_{(f,g) \in \F_{\beta}} \sum_{i=1}^{n} \xi_i \left( \ind[ f(X_i) \neq Y_i ] - \ind[ g(X_i) \neq Y_i ] \right) \middle| S \right]
\\ & \leq \E\!\left[ \sup_{(f,g) \in \F_{\beta}} \left| Z_f - Z_g \right| \middle| S \right]
\leq C_1 \int_{0}^{\hat{\beta}} \sqrt{\log \hat{D}(\gamma)} \mathrm{d}\gamma
\end{align*}
for a universal constant $C_1$.
Moreover, \citet*{haussler:95} has shown that,
for any $\gamma > 0$, 
$\hat{D}(\gamma) \leq \left( \frac{n C_2}{\gamma^2} \right)^{\vc(\H)}$
for a universal constant $C_2$,
so that the last expression above is at most
\begin{align*}
C_1 \int_{0}^{\hat{\beta}} \sqrt{\vc(\H) \log\!\left(\frac{n C_2}{\gamma^2}\right)} \mathrm{d}\gamma
\leq C_3 \sqrt{ \hat{\beta}^2 \cdot \vc(\H) \log\!\left( \frac{n}{\hat{\beta}^2} \right) }
\end{align*}
for a universal constant $C_3$.

Together with \eqref{eqn:symmetrization-pre-chaining} and the law of total expectation, 
we have that
\begin{equation}
\label{eqn:Lbeta-bound-1}
\E\!\left[ L_{\beta} \right] 
\leq 2 C_3 \E\!\left[ \sqrt{ \hat{\beta}^2 \cdot \vc(\H) \log\!\left( \frac{n}{\hat{\beta}^2} \right) } \right]
\leq 2C_3 \sqrt{ \E\!\left[\hat{\beta}^2\right] \cdot \vc(\H) \log\!\left( \frac{n}{\E\!\left[\hat{\beta}^2\right]} \right) },
\end{equation}
where the second inequality is due to 
Jensen's inequality
(noting that $x \mapsto \sqrt{x \log(1/x)}$ is concave).

Our next step is to bound $\E\!\left[\hat{\beta}^2\right]$.
A bound sufficient for our purposes 
follows directly from Lemma~\ref{lem:uniform-bernstein-with-worse-log-factor}.
Specifically, since any $(f,g) \in \F_{\beta}$ 
have $n\sigma^2(f,g) \leq \beta^2$, 
Lemma~\ref{lem:uniform-bernstein-with-worse-log-factor}
implies that,
for any $\delta' \in (0,e^{-1})$, 
with probability at least
$1-\delta'$, 
\begin{equation*}
\frac{\hat{\beta}^2}{n}
= \sup_{(f,g) \in \F_{\beta}} \hat{\sigma}^2(f,g)
\leq \sup_{(f,g) \in \F_{\beta}} 2 \sigma^2(f,g) + \HKCzero \epsilon^2(n,\delta')
\leq \frac{2 \beta^2}{n} + \HKCzero \epsilon^2(n,\delta').
\end{equation*}
Since this holds for any choice of $\delta' \in (0,e^{-1})$, 
it can be stated equivalently as a tail bound:
namely, for any $\gamma > n\epsilon^2(n,e^{-1}) = \vc(\H) \log\!\left(\frac{n}{\vc(\H)}\right)+1$, 
\begin{equation*}
\P\!\left( \hat{\beta}^2 > 2 \beta^2 + \gamma \right) 
\leq \exp\!\left\{\vc(\H) \log\!\left(\frac{n}{\vc(\H)}\right) - \gamma \right\}
= \left(\frac{n}{\vc(\H)}\right)^{\vc(\H)} e^{ - \gamma }.
\end{equation*}
Thus, we have
\begin{align*}
\E\!\left[ \hat{\beta}^2 \right]
& = \int_{0}^{\infty} \P\!\left( \hat{\beta^2} > \tau \right) \mathrm{d}\tau
\\ & \leq 2\beta^2 + n\epsilon^2(n,e^{-1}) + \int_{n\epsilon^2(n,e^{-1})}^{\infty} \P\!\left( \hat{\beta}^2 > 2\beta^2 + \gamma \right) \mathrm{d}\gamma
\\ & \leq 2\beta^2 + n\epsilon^2(n,e^{-1}) + \int_{n\epsilon^2(n,e^{-1})}^{\infty} \left(\frac{n}{\vc(\H)}\right)^{\vc(\H)} e^{-\gamma} \mathrm{d}\gamma
\\ & = 2\beta^2 + n\epsilon^2(n,e^{-1}) + e^{-1}.
\end{align*}
Denoting by $\epsilon^2_n = \frac{\vc(\H)}{n}\log\!\left(\frac{n}{\vc(\H)}\right)$, 
we may further note that for $\beta^2 \geq (1/2) n \epsilon^2_n$, 
$2\beta^2 + n\epsilon^2(n,e^{-1}) + e^{-1} \leq 6\beta^2$.

Plugging this back into \eqref{eqn:Lbeta-bound-1}, 
together with monotonicity of $x \mapsto x \log(1/x)$, 
yields that for $\beta^2 \geq (1/2) n \epsilon^2_n$, 
\begin{equation}
\label{eqn:ELbeta-bound}
\E\!\left[ L_{\beta} \right] 
\leq 2C_3 \sqrt{ 6 \beta^2 \cdot \vc(\H) \log\!\left( \frac{n}{6 \beta^2} \right) }.
\end{equation}
This bound will be used to bound the additive $\E\!\left[ L_{\beta} \right]$
term in \eqref{eqn:non-ident-bousquet-inequality}.
For the $\E\!\left[ L_{\beta} \right]$ term inside the square root
in \eqref{eqn:non-ident-bousquet-inequality}, 
a simpler coarse bound will suffice: 
namely, for $\beta^2 \geq (1/2) n \epsilon^2_n$, 
the expression on the right hand side of \eqref{eqn:ELbeta-bound} 
is at most
\begin{equation*}
2C_3 \sqrt{ 6 \beta^2 \cdot n \epsilon^2_n }
\leq C_3 \sqrt{48} \beta^2.
\end{equation*}
Plugging these into \eqref{eqn:non-ident-bousquet-inequality} yields that,
on the event $E_{\beta}$, 
\begin{align*}
L_{\beta} 
& \leq 2C_3 \sqrt{ 6 \beta^2 \cdot \vc(\H) \log\!\left( \frac{n}{6 \beta^2} \right) }
+ 24 \max\!\left\{ \sqrt{ \left( C_3 \sqrt{48} + 1 \right) \beta^2 \cdot \ln\!\left(\frac{1}{\delta_{\beta}}\right) }, \ln\!\left(\frac{1}{\delta_{\beta}}\right) \right\}
\\ & \leq C_4 \sqrt{ \beta^2 \cdot \left( \vc(\H) \log\!\left(\frac{n}{\beta^2}\right) + \ln\!\left(\frac{1}{\delta_{\beta}}\right) \right) } + C_4 \ln\!\left(\frac{1}{\delta_{\beta}}\right)
\end{align*}
for a universal constant $C_4$.
Recalling the definition 
$\delta_{\beta} = \frac{\delta \beta^2}{4 n}$,
this further implies that for 
$\beta^2 \geq (1/2) n \epsilon^2_n$, 
on the event $E_{\beta}$, 
\begin{equation}
\label{eqn:final-Lbeta-bound}
L_{\beta} 
\leq C_5 \sqrt{ \beta^2 \cdot \left( \vc(\H) \log\!\left(\frac{n}{\beta^2}\right) + \log\!\left(\frac{1}{\delta}\right) \right) } + C_5 n \epsilon^2(n,\delta)
\end{equation}
for a universal constant $C_5$.

Let $B_n = [(1/2) n \epsilon^2_n, n] \cap \{ n \cdot 2^{1-k} : k \in \nats \}$,
and note that $B_n \neq \emptyset$ by the assumption $n \geq e^2 \vc(\H)$.
Define an event $E = \bigcap_{\beta^2 \in B_n} E_{\beta}$.
By the union bound, the event $E$ has probability at least 
$1 - \sum_{\beta^2 \in B_n} \delta_{\beta}
\geq 1 - \sum_{k \in \nats} \frac{\delta}{4} 2^{1-k} = 1 - \frac{\delta}{2}$.

Consider any $f,g \in \H$,
and let $\beta^2(f,g) = \min\!\left\{ \beta^2 \in B_n : n \sigma^2(f,g) \leq \beta^2 \right\}$ (which is well-defined since $n\sigma^2(f,g) \leq n \in B_n$).
In particular, we have $(f,g),(g,f) \in \F_{\beta(f,g)}$,
so that 
\begin{equation} 
\label{eqn:diff-of-diff-bounded-by-Lbeta}
\left| \left( \hat{\er}_{S}(f) - \hat{\er}_{S}(g) \right) - \left( \er_{\bar{P}}(f) - \er_{\bar{P}}(g) \right) \right|
\leq \frac{1}{n} L_{\beta(f,g)}.
\end{equation}
Thus, on the event $E$, 
\eqref{eqn:final-Lbeta-bound} and \eqref{eqn:diff-of-diff-bounded-by-Lbeta} imply
\begin{align} 
& \left| \left( \hat{\er}_{S}(f) - \hat{\er}_{S}(g) \right) - \left( \er_{\bar{P}}(f) - \er_{\bar{P}}(g) \right) \right|
\notag \\ & \leq C_5 \sqrt{ \beta^2(f,g) \cdot \frac{1}{n^2} \left( \vc(\H) \log\!\left(\frac{n}{\beta^2(f,g)}\right) + \log\!\left(\frac{1}{\delta}\right) \right) } + C_5 \cdot \epsilon^2(n,\delta). \label{eqn:er-diff-bound-with-beta-f-g}
\end{align}

We consider two cases.
First, if $\beta^2(f,g) = \min B_n$, 
then $n \sigma^2(f,g) \leq \beta^2(f,g) \leq n \epsilon^2_n$,
so that (by monotonicity of $x \mapsto x \log(1/x)$ for $x > 0$) 
the expression in \eqref{eqn:er-diff-bound-with-beta-f-g} is at most
\begin{equation*}
C_5 \sqrt{ \epsilon^2_n \cdot \frac{1}{n} \left( \vc(\H) \log\!\left(\frac{1}{\epsilon^2_n}\right) + \log\!\left(\frac{1}{\delta}\right) \right) } + C_5 \cdot \epsilon^2(n,\delta)
\leq 2 C_5 \cdot \epsilon^2(n,\delta).
\end{equation*}
For the remaining case, if $\beta^2(f,g) > \min B_n$,
then $n \sigma^2(f,g) \leq \beta^2(f,g) \leq 2 n \sigma^2(f,g)$,
so that the expression in \eqref{eqn:er-diff-bound-with-beta-f-g} is at most
\begin{align*}
& C_5 \sqrt{ 2 \sigma^2(f,g) \cdot \frac{1}{n} \left( \vc(\H) \log\!\left(\frac{1}{2\sigma^2(f,g)}\right) + \log\!\left(\frac{1}{\delta}\right) \right) } + C_5 \cdot \epsilon^2(n,\delta).
\end{align*}
Since this second case, where $\beta^2(f,g) > \min B_n$, necessarily has 
$\sigma^2(f,g) > (1/2) \epsilon^2_n > \frac{\vc(\H)}{2n}$,
we have in both cases that, on the event $E$, 
the expression in \eqref{eqn:er-diff-bound-with-beta-f-g} is at most
\begin{equation}
\label{eqn:sigma-er-diff-bound-in-proof}
C_5 \sqrt{ 2 \sigma^2(f,g) \cdot \frac{1}{n} \left( \vc(\H) \log\!\left(\frac{1}{\sigma^2(f,g)} \land \frac{n}{\vc(\H)} \right) + \log\!\left(\frac{1}{\delta}\right) \right) } + 2 C_5 \cdot \epsilon^2(n,\delta).
\end{equation}

It remains only to replace $\sigma^2(f,g)$ with $\tilde{\sigma}^2(f,g)$
in the above inequality.
For this, we again rely on Lemma~\ref{lem:uniform-bernstein-with-worse-log-factor}, which, noting that 
$\epsilon^2\!\left(n,\frac{\delta}{2}\right) \leq 2 \epsilon^2(n,\delta)$,
guarantees that, 
on an event $E'$ of probability at least $1 - \frac{\delta}{2}$, 
every $f,g \in \H$ satisfy 
\begin{equation}
\label{eqn:diff-chernoff-in-proof}
\frac{1}{2} \sigma^2(f,g) - 2\HKCzero \epsilon^2(n,\delta)
\leq \hat{\sigma}^2(f,g) 
\leq 2 \sigma^2(f,g) + 2\HKCzero \epsilon^2(n,\delta).
\end{equation}
In particular, this also implies 
\begin{equation*}
\sigma^2(f,g) 
\leq 2 \hat{\sigma}^2(f,g) + 4 \HKCzero \epsilon^2(n,\delta)
\leq \max\!\left\{ 4 \hat{\sigma}^2(f,g), 8 \HKCzero \epsilon^2(n,\delta)\right\}.
\end{equation*}
On this event, by monotonicity of $x \mapsto x \log(1/x)$, 
the expression in \eqref{eqn:sigma-er-diff-bound-in-proof}
is at most
\begin{align}
& C_5 \sqrt{ \left( 8 \hat{\sigma}^2(f,g) \!\lor\! 16 \HKCzero \epsilon^2(n,\delta) \right) \!\frac{1}{n}\! \left( \vc(\H) \log\!\left(\frac{1}{4 \hat{\sigma}^2(f,g) \!\lor\! 8 \HKCzero \epsilon^2(n,\delta)} \right) \!\!+\! \log\!\left(\frac{1}{\delta}\right) \!\right) } \!+\! 2 C_5 \epsilon^2(n,\delta)
\notag \\ & \leq C_5 \sqrt{ 8 \hat{\sigma}^2(f,g) \cdot \frac{1}{n} \left( \vc(\H) \log\!\left(\frac{1}{\hat{\sigma}^2(f,g)} \land \frac{n}{\vc(\H)} \right) + \log\!\left(\frac{1}{\delta}\right) \right) } 
+ \left( 8 \sqrt{ \HKCzero } \!+\! 2 \right) C_5 \epsilon^2(n,\delta).
\label{eqn:hat-sigma-er-diff-bound-in-proof}
\end{align}

By the union bound, the event $E \cap E'$ has probability at least $1-\delta$.
Since we have established that, 
on the event $E \cap E'$, 
every $f,g \in \H$ satisfy 
\eqref{eqn:diff-chernoff-in-proof}
and also satisfy that 
$\left| \left( \hat{\er}_{S}(f) - \hat{\er}_{S}(g) \right) - \left( \er_{\bar{P}}(f) - \er_{\bar{P}}(g) \right) \right|$
is upper bounded by both 
\eqref{eqn:sigma-er-diff-bound-in-proof}
and \eqref{eqn:hat-sigma-er-diff-bound-in-proof},
the result follows by defining 
$\Czero = \max\!\left\{ \left( 8\sqrt{\HKCzero}+2 \right) C_5, 2\HKCzero \right\}$.
\end{proof}

\subsubsection{Detailed Proof of Lemma~\ref{lem:partial-VC-super-root}}
\label{sec:proof-of-lemma-partial-VC-super-root}

Throughout this subsection, we let $\G$ be any partial concept class with $\vc(\G) < \infty$, 
let $\PXY$ be any distribution on $\X \times \{0,1\}$ 
which is Bayes-realizable with respect to $\G$, 
and let $(X_1,Y_1),(X_2,Y_2),\ldots$ be i.i.d.\ $\PXY$-distributed random variables.
For any $t \in \nats$, 
and for $g = (g(X_1),\ldots,g(X_t)) \in \{0,1\}^{t'}$ 
and $g' = (g'(X_1),\ldots,g'(X_{t'})) \in \{0,1\}^{t''}$ (for $t',t'' \geq t$), 
define 
$\hat{\er}_{t}(g) := \frac{1}{t} \sum_{i=1}^{t} \ind[ Y_i \neq g(X_i) ]$
and 
$\ber_t(g) = \frac{1}{t} \sum_{i=1}^{t} \P(Y_i \neq g(X_i) | X_i)$ (the \emph{semi-empirical} error rate) 
and 
$P_t( g \neq g' ) = \frac{1}{t} \sum_{i=1}^{t} \ind[ g(X_i) \neq g'(X_i) ]$.
To be clear, in the context of a $
(X_{\leq t},Y_{\leq t})$-dependent random $\{0,1\}^t$-valued $\hat{g}$,
we define $\ber_t(\hat{g}) = \frac{1}{t} \sum_{i=1}^{t} \P(Y'_i \neq \hat{g}(X_i) | X_i, \hat{g})$
where $Y'_i$ is conditionally independent of $(X_{\leq t},Y_{\leq t})$
given $X_i$, with conditional distribution given $X_i$ identical to that of $Y_i$.
Also define the \emph{Bayes classifier}
$x \mapsto \target(x) := \ind\!\left[ \P(Y=1|X=x) \geq \frac{1}{2} \right]$ 
(for $(X,Y) \sim \PXY$ independent of the $(X_i,Y_i)$ sequence), 
and define 
$\ber_t(\target_{\PXY}) = \frac{1}{t} \sum_{i=1}^{t} \P(Y_i \neq \target_{\PXY}(X_i) | X_i) = \frac{1}{t} \sum_{i=1}^{t} \min_{y \in \{0,1\}} \P(Y_i = y | X_i)$.

Many of the results below will rely on the basic event 
implied by the Bayes-realizability assumption.
Formally, for any $t \in \nats$, we let $E_t^\star$ 
denote the event of probability one (by Bayes-realizability) that 
\begin{equation}
\label{eqn:Bayes-realizable-g-star}
\exists g_t^\star \in \G(X_{\leq t}) \text{ s.t. } \forall i \leq t, \P( Y_i = g_n^\star(X_i) | X_i ) = \P( Y_i = \target_{\PXY}(X_i) | X_i ).
\end{equation}
Note that this event depends only on the 
\emph{set} of values $\{ X_i : i \leq t \}$ (not their order in $X_{\leq t}$).

We remark that Lemma~\ref{lem:partial-VC-super-root} holds rather trivially 
in the case of $\vc(\G) = 0$.
For simplicity, in all remaining lemmas in this subsection,  
we will suppose $\vc(\G) \geq 1$.
The proof of Lemma~\ref{lem:partial-VC-super-root} will address the 
general case of any $\vc(\G) \in \nats \cup \{0\}$.

The proof of Lemma~\ref{lem:partial-VC-super-root} 
relies on several lemmas concerning concentration of empirical and semi-empirical errors (Lemma~\ref{lem:transductive-uniform-bernstein}),
and concentration of semi-empirical errors for different sample sizes
(Lemmas~\ref{lem:transductive-ber-multiplicative-chernoff} and \ref{lem:transductive-ber-uniform-convergence}).
We remark that, for these lemmas and many others below, 
it in fact suffices for the $X_{i}$ sequence to 
merely be \emph{exchangeable}, 
with the $Y_i$ variables conditionally independent 
given the respective $X_i$ variables (and supposing 
an appropriate analogue of Bayes-realizability).
For instance, for the lemmas only involving finite 
prefixes of the data sequence, 
it would suffice for the $X_i$'s to be a 
uniform random \emph{permutation} of some fixed sequence 
(as is commonly studied in the literature on \emph{transductive}
learning).
For simplicity, we do not discuss the formal details of this 
generalization of the results.

We begin with the following lemma, 
representing a \emph{uniform Bernstein inequality} 
for the empirical and semi-empirical error rates.

\begin{lemma}
\label{lem:transductive-uniform-bernstein}
Fix any $n \in \nats$ and $\delta \in (0,1)$.
With conditional probability at least $1-\delta$
given $X_{\leq n}$, 
every $f,g \in \G(X_{\leq n})$ satisfy 
\begin{align*}
    & \left| \left( \ber_n(f) - \ber_n(g) \right) - \left( \hat{\er}_{n}(f) - \hat{\er}_{n}(g) \right) \right| 
    \\ & \leq \sqrt{P_n( f \neq g ) \frac{\Czero}{n}\!\left( \VC(\G) \log\!\left(\frac{1}{P_n( f \neq g )} \land \frac{n}{\VC(\G)} \right) + \log\!\left(\frac{1}{\delta}\right) \right)}
    + \Czero \epsilon^2(n,\delta), 
\end{align*}
    where $\Czero$ is the universal constant from Proposition~\ref{prop:uniform-bernstein}.
\end{lemma}
\begin{proof}
For each $i \in \{1,\ldots,n\}$, 
let $P_i$ denote the conditional distribution of $(X_i,Y_i)$ given $X_i$.
Note that, given $X_{\leq n}$,  
$(X_{\leq n},Y_{\leq n})$ has conditional distribution 
$P_1 \times \cdots \times P_n$.
Applying Proposition~\ref{prop:uniform-bernstein}
yields the claimed inequality, holding simultaneously 
for all $f,g \in \G(X_{\leq n})$, 
with conditional probability at least $1-\delta$ 
given $X_{\leq n}$.
\end{proof}

For any $n \in \nats$, define
\begin{equation*}
\epsilon_n := \sqrt{\frac{\vc(\G)}{n} \log\!\left(\frac{n}{\vc(\G)}\right)}.
\end{equation*}

The following lemma is an immediate consequence of 
Lemma~\ref{lem:transductive-uniform-bernstein}.

\begin{lemma}
\label{lem:transductive-erm-ber-excess-bound}
For 
$n \in \nats$, 
letting $m = \lceil n/2 \rceil$
and $\hat{g}_n = \TERM^{\G}_n(X_{\leq m},Y_{\leq m},X_{m+1},\ldots,X_{n+1})$,
on the event $E_{n+1}^\star$, 
with conditional probability at least $1-\frac{1}{m}$ given $X_{\leq n+1}$,
\begin{equation*}
\ber_{m}(\hat{g}_n) - \ber_{m}(\target_{\PXY}) \leq \Cone \epsilon_{m},    
\end{equation*}
where 
$\Cone = 2\sqrt{2 \Czero}$
is a universal constant.
\end{lemma}
\begin{proof}
Supposing the event $E_{n+1}^\star$, 
let $g^\star = g_{n+1}^\star$ for $g_{n+1}^\star$ as in \eqref{eqn:Bayes-realizable-g-star}.
In particular, this implies $\G(X_{\leq n+1}) \neq \emptyset$
and moreover that 
$g^\star(X_i) = \target_{\PXY}(X_i)$ for each $i \leq n+1$ with $\P(Y_i = 1 | X_i) \neq \frac{1}{2}$, 
whereas for those $i \leq n+1$ with $\P(Y_i = 1 | X_i)$, we have $\P(Y_i \neq g^\star(X_i) | X_i) = \frac{1}{2} = \P(Y_i \neq \target_{\PXY}(X_i) | X_i)$.
In particular, these observations imply 
$\ber_{m}(g^\star) = \ber_{m}(\target_{\PXY})$.
Noting that 
$\hat{g}_n(X_{\leq m}), g^\star(X_{\leq m}) \in \G(X_{\leq m})$, 
Lemma~\ref{lem:transductive-uniform-bernstein} (with $\delta = \frac{1}{m}$)
implies that with conditional probability at least $1 - \frac{1}{m}$ given $X_{\leq n+1}$, 
\begin{align*}
& \ber_m(\hat{g}_n) - \ber_m(\target_{\PXY})
= \ber_m(\hat{g}_n) - \ber_m(g^\star)
\\ & \leq \hat{\er}_{m}(\hat{g}_n) - \hat{\er}_{m}(g^\star) + 
\sqrt{P_m( \hat{g}_n \neq g^\star ) \frac{\Czero}{m} \left( \vc(\G) \log\!\left( \frac{1}{P_m( \hat{g}_n \neq g^\star )} \land \frac{m}{\vc(\G)} \right) + \log(m) \right) }
\\ & {\hskip 12mm}+ \frac{\Czero}{m} \left( \vc(\G) \log\!\left( \frac{m}{\vc(\G)} \right) + \log(m) \right).
\end{align*}
Since both $\hat{g}_n, g^\star \in \G(X_{\leq n+1})$ 
and (by definition) $\hat{g}_n = \argmin_{g \in \G(X_{\leq n+1})} \hat{\er}_{m}(g)$, 
we have that 
$\hat{\er}_{m}(\hat{g}_n) - \hat{\er}_{m}(g^\star) \leq 0$.
We can also coarsely bound $P_m( \hat{g}_n \neq g^\star ) \leq 1$, which moreover implies 
$P_m( \hat{g}_n \neq g^\star ) \log\!\left( \frac{1}{P_m( \hat{g}_n \neq g^\star )} \right) \leq 1$
(noting that $x \mapsto x\log(1/x) = x \ln((1/x) \lor e)$ is increasing in $x > 0$).
Altogether, on the above event, we have
\begin{align*}
& \ber_m(\hat{g}_n) - \ber_m(\target_{\PXY})
\leq \sqrt{\frac{\Czero}{m} \left( \vc(\G) + \log(m) \right) }
+ \frac{\Czero}{m} \left( \vc(\G) \log\!\left( \frac{m}{\vc(\G)} \right) + \log(m) \right)
\\ & \leq \sqrt{\frac{2\Czero \vc(\G)}{m} \log\!\left(\frac{m}{\vc(\G)}\right) }
+ \frac{2\Czero \vc(\G)}{m} \log\!\left( \frac{m}{\vc(\G)} \right)
= \sqrt{2\Czero} \epsilon_m + 2\Czero \epsilon_m^2.
\end{align*}
If $2\Czero \epsilon_m^2 \leq 1$, 
taking square root of the last term above yields  
\begin{equation*}
\ber_m(\hat{g}_n) - \ber_m(\target_{\PXY})
\leq 2 \sqrt{2 \Czero} \epsilon_m,
\end{equation*}
whereas if $2 \Czero \epsilon_m^2 > 1$, 
then the above inequality trivially holds
since $\ber_m(\hat{g}_n) - \ber_m(\target_{\PXY}) \leq 1$.
The result thus follows, defining $\Cone = 2\sqrt{2 \Czero}$.
\end{proof}

We next state several lemmas based on the observation that, 
for any $n,m \in \nats$ with $m \leq n$, 
by exchangeability
the prefix $X_{\leq m}$ may be viewed as a uniform random 
sample (without replacement) from the $n$-element 
\emph{multiset} of values $X_i$, $i \in \{1,\ldots,n\}$.
We denote by $\lbag X_{\leq n} \rbag$ this multiset of values.
Moreover, we may note that the set $\G(X_{\leq n})$ is 
(up to permutations of indices) only dependent on 
this multiset $\lbag X_{\leq n} \rbag$, not the particular order in which these elements occur in the sequence $X_{\leq n}$.
To formalize the arguments based on these observations,
it will be convenient to introduce notation 
to study random \emph{permutations} of 
the sequence $X_{\leq n}$.

Toward this end, for any $n \in \nats$ denote by $[n] = \{1,\ldots,n\}$
and let $\Sym(n)$ denote the set of all bijections $\pi : [n] \to [n]$ (i.e., the set of all permutations of $\{1,\ldots,n\}$).
For each $\pi \in \Sym(n)$,
for any $m \in [n]$, 
denote by $\pi([m]) = (\pi(1),\ldots,\pi(m))$, 
and similarly 
$X_{\pi([m])} = (X_{\pi(1)},\ldots,X_{\pi(m)})$.
As mentioned, we will rely on the observation 
that $\G(X_{\pi([n])}) = \{ g(X_{\pi([n])}) : g \in \G(X_{\leq n}) \}$, 
where $g(X_{\pi([n])}) := (g(X_{\pi(1)}),\ldots,g(X_{\pi(n)}))$: 
that is, the sets of 
$n$-element \emph{multisets} $\lbag (X_{\leq n},g(X_{\leq n})) \rbag$ of classifications 
are the same for $\G(X_{\leq n})$ and $\G(X_{\pi([n])})$.


We now state the lemmas based on these ideas.

\begin{lemma}
\label{lem:transductive-ber-multiplicative-chernoff}
For any $n,m \in \nats$ with $n \geq m$,
for any $\delta \in (0,1)$,
with conditional probability at least $1-\delta$ given $\lbag X_{\leq n} \rbag$, 
every $g \in \G(X_{\leq n})$ satisfies 
\begin{align*}
& \ber_{n}(g) - \ber_{n}(\target_{\PXY})
\leq 2 \left( \ber_{m}(g) - \ber_{m}(\target_{\PXY}) \right)
+ \frac{\Ctwo}{m} \left( \vc(\G) \log\!\left(\frac{n}{\vc(\G)}\right) + \log\!\left(\frac{1}{\delta}\right) \right)
\\ \text{and } & \ber_{m}(g) - \ber_{m}(\target_{\PXY})
\leq 2 \left( \ber_{n}(g) - \ber_{n}(\target_{\PXY}) \right)
+ \frac{\Ctwo}{m} \left( \vc(\G) \log\!\left(\frac{n}{\vc(\G)}\right) + \log\!\left(\frac{1}{\delta}\right) \right),
\end{align*}
and every $f,g \in \G(X_{\leq n})$ satisfy 
\begin{align*}
& P_n( f \neq g ) \leq 2 P_m( f \neq g ) + \frac{\Ctwo}{m}\left( \vc(\G) \log\!\left(\frac{n}{\vc(\G)}\right) + \log\!\left(\frac{1}{\delta}\right) \right)
\\ \text{and } & P_m( f \neq g ) \leq 2 P_n( f \neq g ) + \frac{\Ctwo}{m}\left( \vc(\G) \log\!\left(\frac{n}{\vc(\G)}\right) + \log\!\left(\frac{1}{\delta}\right) \right),
\end{align*}
where $\Ctwo = 24$ is a universal constant.
\end{lemma}
%

\begin{proof}[of Lemma~\ref{lem:transductive-ber-multiplicative-chernoff}]
If $m < \vc(\G)$, the bound trivially holds (the right hand sides being greater than $1$); to focus on the remaining case, suppose $m \geq \vc(\G)$.
For any $g \in \G(X_{\leq n})$ and $i \leq n$, 
denote by 
\begin{equation*} 
\ell(g;i) := \P(Y_i \neq g(X_i) | X_i) - \P(Y_i \neq \target_{\PXY}(X_i) | X_i),
\end{equation*}
and for any $k \leq n$ and any sequence 
$(i_1,\ldots,i_k) \in \{1,\ldots,n\}^k$, 
denote by $\bar{L}(g; i_1,\ldots,i_k) = \frac{1}{k}\sum_{j=1}^k \ell(g;i_j)$.
In particular, note that 
$\ber_n(g) - \ber_n(\target_{\PXY}) = \bar{L}(g; [n])$ 
and 
$\ber_m(g) - \ber_m(\target_{\PXY}) = \bar{L}(g; [m])$.

Fix a value $\epsilon = \frac{8}{m} \left( 2 \vc(\G) \log\!\left(\frac{e n}{\vc(\G)}\right)+\log\!\left(\frac{4}{\delta}\right) \right)$.
Let $\bpi$ be a $\mathrm{Uniform}(\Sym(n))$ random variable 
independent of $X_{\leq n}$.
Note that, by exchangeability,
given $\lbag X_{\leq n} \rbag$
the sequence $X_{\leq n}$ is a uniform random ordering of the 
multiset $\lbag X_{\leq n} \rbag$.
Moreover, recall that for any $\pi \in \Sym(n)$ we always have 
$\G(X_{\pi([n])}) = \{ (g(X_{\pi([n])}) : g \in \G(X_{\leq n}) \}$,
and that $\bar{L}(g; \pi([n])) = \bar{L}(g; [n])$.
We therefore have that 
\begin{align*} 
& \P\!\left( \exists g \in \G(X_{\leq n}) : \bar{L}(g; [m]) \notin \left[ \frac{1}{2} \bar{L}(g;[n]) - \frac{\epsilon}{2}, 2 \bar{L}(g; [n]) + \epsilon \right] \middle| \lbag X_{\leq n} \rbag \right)
\\ & = \P\!\left( \exists g \in \G(X_{\leq n}) : \bar{L}(g; \bpi([m])) \notin \left[ \frac{1}{2} \bar{L}(g;[n]) - \frac{\epsilon}{2}, 2 \bar{L}(g; [n]) + \epsilon \right] \middle| X_{\leq n} \right).
\end{align*}
By the union bound, this last quantity is at most 
\begin{equation}
\label{eqn:transductive-multiplicative-chernoff-sum-intermediate}
\sum_{g \in \G(X_{\leq n})} \P\!\left( \bar{L}(g; \bpi([m])) \notin \left[ \frac{1}{2} \bar{L}(g;[n]) - \frac{\epsilon}{2}, 2 \bar{L}(g; [n]) + \epsilon \right] \middle| X_{\leq n} \right).
\end{equation}

We bound each term in this sum  
by a multiplicative Chernoff bound for sampling without replacement.
Specifically, 
for each $g \in \G(X_{\leq n})$, by the union bound we have 
\begin{align*}
& \P\!\left( \bar{L}(g; \bpi([m])) \notin \left[ \frac{1}{2} \bar{L}(g;[n]) - \frac{\epsilon}{2}, 2 \bar{L}(g; [n]) + \epsilon \right] \middle| X_{\leq n} \right)
\\ & \leq \P\!\left( \bar{L}(g; \bpi([m])) + \frac{\epsilon}{2} < \frac{1}{2} \bar{L}(g;[n]) \middle| X_{\leq n} \right) 
+ \P\!\left( \bar{L}(g; \bpi([m])) > 2 \bar{L}(g; [n]) + \epsilon \phantom{\frac{1}{2}}\!\!\!\middle| X_{\leq n} \right).
\end{align*}
Then note that, conditioned on $X_{\leq n}$,
for any $g \in \G(X_{\leq n})$, 
the value 
$\bar{L}(g;\bpi([m])) = \frac{1}{m} \sum_{i=1}^m \ell(g;\bpi(i))$
is an average of $m$ values $\ell(g;i)$
sampled uniformly without replacement from 
the multiset $\lbag \ell(g;i) : i \leq n \rbag$.
Further noting that each $\ell(g;i) \in [0,1]$,
applying a multiplicative Chernoff bound 
for sampling without replacement \citep*{hoeffding:53} 
under the conditional distribution given $X_{\leq n}$ 
yields that 
\begin{align*}
& \P\!\left( \bar{L}(g; \bpi([m])) \!+\! \frac{\epsilon}{2} \!<\! \frac{1}{2} \bar{L}(g;[n]) \middle| X_{\leq n} \right)
\leq \ind\!\left[ \bar{L}(g;[n]) \!>\! \epsilon \right] \P\!\left( \bar{L}(g; \bpi([m])) \!<\! \frac{1}{2} \bar{L}(g;[n]) \middle| X_{\leq n} \right)
\\ & \leq \ind\!\left[ \bar{L}(g;[n]) \!>\! \epsilon \right] e^{- \bar{L}(g;[n]) m / 8}
\leq e^{- \epsilon m / 8}.
\end{align*}
Likewise, if $\bar{L}(g;[n]) > \epsilon$, 
\begin{align*}
& \P\!\left( \bar{L}(g; \bpi([m])) > 2 \bar{L}(g; [n]) + \epsilon \middle| X_{\leq n} \right)
\\ & \leq \P\!\left( \bar{L}(g; \bpi([m])) > 2 \bar{L}(g; [n]) \middle| X_{\leq n} \right)
\leq e^{- \bar{L}(g;[n]) m / 3} 
\leq e^{-\epsilon m / 3},
\end{align*}
whereas if $0 < \bar{L}(g;[n]) \leq \epsilon$, 
\begin{align*}
& \P\!\left( \bar{L}(g; \bpi([m])) > 2 \bar{L}(g; [n]) + \epsilon \middle| X_{\leq n} \right)
\\ & \leq \P\!\left( \bar{L}(g; \bpi([m])) > \left(1 + \frac{\epsilon}{\bar{L}(g;[n])} \right) \bar{L}(g; [n]) \middle| X_{\leq n} \right)
\\ & \leq \exp\!\left\{- \bar{L}(g;[n]) m \left(\frac{\epsilon}{\bar{L}(g;[n])} \right)^2 \frac{1}{3} \right\}
= \exp\!\left\{- m \frac{\epsilon^2}{\bar{L}(g;[n])} \frac{1}{3} \right\}
\leq e^{- \epsilon m / 3}.
\end{align*}
Moreover, if $\bar{L}(g;[n]) = 0$, we clearly have 
\begin{equation*}
\P\!\left( \bar{L}(g; \bpi([m])) > 2 \bar{L}(g; [n]) + \epsilon \middle| X_{\leq n} \right) 
\leq \P\!\left( \bar{L}(g; \bpi([m])) > 0 \middle| X_{\leq n} \right) 
= 0.
\end{equation*}

Plugging these back into \eqref{eqn:transductive-multiplicative-chernoff-sum-intermediate} yields
\begin{align*}
& \P\!\left( \exists g \in \G(X_{\leq n}) : \bar{L}(g; [m]) \notin \left[ \frac{1}{2} \bar{L}(g;[n]) - \frac{\epsilon}{2}, 2 \bar{L}(g; [n]) + \epsilon \right] \middle| \lbag X_{\leq n} \rbag \right)
\\ & \leq \E\!\left[ |\G(X_{\leq n})| \right] 2 e^{- \epsilon m / 8}
\leq \left(\frac{e n}{\vc(\G)}\right)^{\vc(\G)} 2 e^{-\epsilon m / 8} \leq \frac{\delta}{2},
\end{align*}
where the second inequality is due to 
Sauer's lemma \citep*{vapnik:71,sauer:72}, 
which states that 
$|\G(X_{\leq n})| \leq \sum_{i=0}^{\vc(\G)} \binom{n}{i} \leq \left(\frac{e n}{\vc(\G)}\right)^{\vc(\G)}$,
and the final inequality is due to our choice of $\epsilon$.
Choosing $\Ctwo = 24$, we have that 
$\epsilon \leq \frac{\Ctwo}{m} \left( \vc(\G) \log\!\left(\frac{n}{\vc(\G)}\right) + \log\!\left(\frac{1}{\delta}\right) \right)$, 
which establishes the first two inequalities on an event of conditional probability at least $1-\frac{\delta}{2}$ given $\lbag X_{\leq n} \rbag$.

The remaining inequalities are established similarly.
Specifically, denote by 
$P_{\pi([m])}( f \neq g ) = \frac{1}{m} \sum_{i=1}^m \ind[ f(X_{\pi(i)}) \neq g(X_{\pi(i)}) ]$ 
and $P_{\pi([n])}( f \neq g ) = \frac{1}{n} \sum_{i=1}^n \ind[ f(X_{\pi(i)}) \neq g(X_{\pi(i)}) ]$ 
for any $\pi \in \Sym(n)$ and $f,g \in \G(X_{\leq n})$.
Noting that $P_{\pi([n])}( f \neq g ) = P_n( f \neq g )$, 
together with the above observations relating $\G(X_{\pi([n])})$ and $\G(X_{\leq n})$, 
by exchangeability of $X_{\leq n}$ we have 
\begin{align*}
& \P\!\left( \exists f,g \in \G(X_{\leq n}) : P_m( f \neq g ) \notin \left[ \frac{1}{2} P_n( f \neq g ) - \frac{\epsilon}{2}, 2 P_n( f \neq g ) + \epsilon \right] \middle| \lbag X_{\leq n} \rbag \right)
\\ & = \P\!\left( \exists f,g \in \G(X_{\leq n}) : P_{\bpi([m])}( f \neq g ) \notin \left[ \frac{1}{2} P_n( f \neq g ) - \frac{\epsilon}{2}, 2 P_n( f \neq g ) + \epsilon \right] \middle| X_{\leq n} \right).
\end{align*}
By the union bound, this last quantity is at most
\begin{equation*}
\sum_{f,g \in \G(X_{\leq n})} \P\!\left( P_{\bpi([m])}( f \neq g ) \notin \left[ \frac{1}{2} P_n( f \neq g ) - \frac{\epsilon}{2}, 2 P_n( f \neq g ) + \epsilon \right] \middle| X_{\leq n} \right).
\end{equation*}

We now proceed to bound each term in the sum.
For each $f,g \in \G(X_{\leq n})$, by the union bound, 
\begin{align*}
& \P\!\left( P_{\bpi([m])}( f \neq g ) \notin \left[ \frac{1}{2} P_n( f \neq g ) - \frac{\epsilon}{2}, 2 P_n( f \neq g ) + \epsilon \right] \middle| X_{\leq n} \right)
\\ & \leq \P\!\left( P_{\bpi([m])}( f \neq g ) \!+\! \frac{\epsilon}{2} < \frac{1}{2} P_n( f \neq g ) \middle| X_{\leq n} \right) 
\!+ \P\!\left( P_{\bpi([m])}( f \neq g ) > 2 P_n( f \neq g ) \!+\! \epsilon \middle| X_{\leq n} \right). 
\end{align*}
Noting that $P_{\bpi([m])}( f \neq g ) = \frac{1}{m} \sum_{i=1}^m \ind[ f(X_{\bpi(i)}) \neq g(X_{\bpi(i)}) ]$
is an average of $m$ binary values $\ind[ f(X_i) \neq g(X_i) ]$
sampled uniformly without replacement from the multiset 
$\lbag \ind[ f(X_i) \neq g(X_i) ] : i \leq n \rbag$, 
applying a multiplicative Chernoff bound for sampling without replacement 
\citep*{hoeffding:53} under the conditional distribution given $X_{\leq n}$
yields that 
\begin{align*}
& \P\!\left( P_{\bpi([m])}( f \neq g ) + \frac{\epsilon}{2} < \frac{1}{2} P_n( f \neq g ) \middle| X_{\leq n} \right)
\leq \ind\!\left[ P_n( f \neq g ) > \epsilon \right] e^{- P_n( f \neq g ) m / 8}
\leq e^{- \epsilon m / 8 }.
\end{align*}
Similarly for the other term, if $P_n( f \neq g ) > \epsilon$, 
\begin{align*}
& \P\!\left( P_{\bpi([m])}( f \neq g ) > 2 P_n( f \neq g ) + \epsilon \middle| X_{\leq n} \right) 
\leq e^{- P_n( f \neq g ) m / 3} \leq e^{- \epsilon m / 3},
\end{align*}
whereas if $0 < P_n( f \neq g ) \leq \epsilon$, 
\begin{align*}
& \P\!\left( P_{\bpi([m])}( f \neq g ) > 2 P_n( f \neq g ) + \epsilon \middle| X_{\leq n} \right)
\\ & \leq \P\!\left( P_{\bpi([m])}( f \neq g ) > \left( 1 + \frac{\epsilon}{P_n( f \neq g )} \right) P_n( f \neq g ) \middle| X_{\leq n} \right)
\\ & \leq \exp\!\left\{ P_n( f \neq g ) m \left( \frac{\epsilon}{P_n( f \neq g )} \right)^2 \frac{1}{3} \right\}
= \exp\!\left\{ - m \frac{\epsilon^2}{P_n( f \neq g )} \frac{1}{3} \right\}
\leq e^{- \epsilon m / 3}.
\end{align*}
If $P_n( f \neq g) = 0$, we trivially have 
\begin{equation*}
\P\!\left( P_{\bpi([m])}( f \neq g ) > 2 P_n( f \neq g ) + \epsilon \middle| X_{\leq n} \right) 
\leq \P\!\left( P_{\bpi([m])}( f \neq g ) > 0 \middle| X_{\leq n} \right) 
= 0.
\end{equation*}

Altogether, we have that 
\begin{align*}
& \P\!\left( \exists f,g \in \G(X_{\leq n}) : P_m( f \neq g ) \notin \left[ \frac{1}{2} P_n( f \neq g ) - \frac{\epsilon}{2}, 2 P_n( f \neq g ) + \epsilon \right] \middle| \lbag X_{\leq n} \rbag \right)
\\ & \leq |\G(X_{\leq n})|^2 2 e^{- \epsilon m / 8}
\leq \left( \frac{e n}{\vc(\G)} \right)^{2 \vc(\G)} 2 e^{- \epsilon m / 8}
\leq \frac{\delta}{2},
\end{align*}
where the second inequality is again due to Sauer's lemma \citep*{vapnik:71,sauer:72}, 
and the final inequality is due to our choice of $\epsilon$.
Thus, the third and fourth claimed inequalities in the lemma are established, 
on an event of conditional probability at least $1-\frac{\delta}{2}$ given $\lbag X_{\leq n} \rbag$.
By the union bound, both of these events occur simultaneously
with conditional probability at least $1-\delta$,
which completes the proof.
\end{proof}

Next we combine the above lemmas to obtain a bound on 
the excess error rate of $\hat{f}^{\G}_n$, 
which is our starting point for the proof of the $o(n^{-1/2})$ rate.
For any $n \in \nats$ and any $\epsilon > 0$, define
\begin{equation*}
\bar{\G}_n(\epsilon) := \left\{ g \in \G(X_{\leq n}) : \ber_n(g) - \ber_n(\target_{\PXY}) \leq \epsilon \right\},
\end{equation*}
and define 
\begin{equation} 
\label{eqn:bar-sigma-definition}
\bar{\sigma}^2_{n} := \epsilon_{n}^2 \lor \max_{g \in \bar{\G}_{n}( \Cthree \epsilon_n )} \min_{f \in \bar{\G}_{n}( \Cfour \epsilon_n^2 )} P_n( f \neq g ),
\end{equation}
where $\Cthree$ and $\Cfour$ are universal constants (informed by the analysis below):
namely, $\Cthree = 4 \Cone + 8 \Ctwo$ 
and $\Cfour = \sqrt{8}(\Cthree+\Cseven) + 4\Ctwo$
for $\Cseven = \sqrt{2e} \left( 5 + \sqrt{12}  \right)$.
For completeness, define $\bar{\sigma}^2_n = 1$ 
in the event that either of 
$\bar{\G}_n(\Cthree \epsilon_n)$ or $\bar{\G}(\Cfour \epsilon_n^2)$ 
is empty (though this is unimportant for our analysis, 
since this event has probability zero under the Bayes-realizability assumption).

\begin{lemma}
\label{lem:transductive-localized-ber-term-tail-bound}
For 
$n \in \nats$, letting $m = \lceil n/2 \rceil$,
and $\hat{g}_n \!=\! \TERM^{\G}_n(X_{\leq m},Y_{\leq m},X_{m+1},\ldots,X_{n+1})$, 
for any $\delta \in (0,1)$,
on the event $E_{n+1}^\star$, 
with conditional probability at least $1-\frac{4}{n} - \delta$ given $\lbag X_{\leq n+1} \rbag$, 
\begin{equation*}
\ber_{n+1}(\hat{g}_n) - \ber_{n+1}(\target_{\PXY})
\leq \Cfive \sqrt{ \bar{\sigma}^2_{n+1} \frac{1}{n}\!\left( \vc(\G) \log\!\left(\frac{1}{\bar{\sigma}^2_{n+1}}\right) + \log\!\left( \frac{1}{\delta}\right) \right)} + \Cfive\! \left( \epsilon_n^2 + \frac{1}{n}\log\!\left(\frac{1}{\delta}\right) \right)\!, 
\end{equation*}
where 
$\Cfive = 4 \Czero + 4 \Cfour + 12 \Ctwo + 2\sqrt{8 \Ctwo \Czero}$
is a universal constant.
\end{lemma}
\begin{proof}
The claim is trivially satisfied if $n \leq \vc(\G) \lor 2e$
(as the right hand side is greater than $1$);
to address the remaining case, suppose $n > \vc(\G) \lor 2e$.
Suppose the event $E_{n+1}^\star$ occurs.
In particular, this implies $\G(X_{\leq n+1}) \neq \emptyset$, 
so that (by definition) $\hat{g}_n \in \G(X_{\leq n+1})$ 
and is of minimal $\hat{\er}_{m}(g)$ among $g \in \G(X_{\leq n+1})$.
The event $E_{n+1}^\star$ 
also implies $\bar{\G}_{n+1}(\Cfour \epsilon_{n+1}^2) \neq \emptyset$.
Let $\bar{g} = \argmin_{g \in \bar{\G}_{n+1}(\Cfour \epsilon_{n+1}^2)} P_{n+1}( g \neq \hat{g}_n )$, breaking ties arbitrarily.

By Lemma~\ref{lem:transductive-erm-ber-excess-bound} and the law of total probability, 
on an event $E$ of conditional probability at least $1-\frac{1}{m}$ given $\lbag X_{\leq n+1} \rbag$, 
we have  
$\hat{g}_n(X_{\leq m}) \in \bar{\G}_m(\Cone \epsilon_m)$.
%
%
Moreover, Lemma~\ref{lem:transductive-uniform-bernstein} and the law of total probability imply that, 
on an event $E'$ of conditional probability at least $1-\delta$ given $\lbag X_{\leq n+1} \rbag$,
every $f,g \in \G(X_{\leq m})$ satisfy
\begin{align*}
\ber_m(f) & - \ber_m(g) \leq \hat{\er}_{m}(f) - \hat{\er}_{m}(g)
\\ & + \sqrt{P_m(f \neq g) \frac{\Czero}{m}\left(\vc(\G)\log\!\left(\frac{1}{P_m( f \neq g )} \right) + \log\!\left(\frac{1}{\delta}\right) \right)} + \Czero \left( \epsilon_m^2 + \frac{1}{m}\log\!\left(\frac{1}{\delta}\right) \right).
\end{align*}

In particular, on the event $E_{n+1}^\star \cap E'$,
since (by definition) $\forall g \in \G(X_{\leq n+1})$, 
$\hat{\er}_{m}(\hat{g}_n) \leq \hat{\er}_{m}(g)$, we have 
\begin{align}
& \ber_m(\hat{g}_n) - \ber_m(\bar{g})
\notag \\ & \leq \sqrt{P_m( \hat{g}_n \neq \bar{g} ) \frac{\Czero}{m}\left(\vc(\G)\log\!\left(\frac{1}{P_m( \hat{g}_n \neq \bar{g} )} \right) + \log\!\left(\frac{1}{\delta}\right) \right)} + \Czero \left( \epsilon_m^2 + \frac{1}{m}\log\!\left(\frac{1}{\delta}\right) \right).
\label{eqn:ber-m-hat-g-vs-bar-g}
\end{align}

Lemma~\ref{lem:transductive-ber-multiplicative-chernoff}
implies that, on an event $E''$ of conditional probability at least $1-\frac{2}{n}$ given $\lbag X_{\leq n+1} \rbag$, 
every $g \in \G(X_{\leq n+1})$ satisfies
\begin{align}
\ber_{n+1}(g) - \ber_{n+1}(\target_{\PXY})
& \leq 2 \left( \ber_m(g) - \ber_m(\target_{\PXY}) \right) + 4\Ctwo \epsilon_{n}^2 \label{eqn:chernoff-n-to-m-size}
\\ \text{ and } \ber_m(g) - \ber_m(\target_{\PXY})
& \leq 2 \left( \ber_{n+1}(g) - \ber_{n+1}(\target_{\PXY}) \right) + 4\Ctwo \epsilon_{n}^2 \label{eqn:chernoff-m-to-n-size}
\end{align}
and $\forall f,g \in \G(X_{\leq n+1})$, 
\begin{equation}
\label{eqn:chernoff-m-to-n-size-diffs}
P_{m}( f \neq g ) \leq 2 P_{n+1}( f \neq g ) + 4\Ctwo \epsilon_{n}^2.
\end{equation}

In particular, 
since $\bar{g} \in \bar{\G}_{n+1}(\Cfour \epsilon_{n+1}^2)$, 
on $E_{n+1}^\star \cap E''$
\eqref{eqn:chernoff-m-to-n-size} implies 
\begin{equation*} 
\ber_m(\bar{g}) - \ber_m(\target_{\PXY}) \leq (2 \Cfour + 4 \Ctwo) \epsilon_{n}^2.
\end{equation*}
Together with 
\eqref{eqn:ber-m-hat-g-vs-bar-g}, we have 
that on the event $E_{n+1}^\star \cap E' \cap E''$, 
letting $C = 2 \Czero + 2 \Cfour + 4 \Ctwo$, 
\begin{align*}
& \ber_m(\hat{g}_n) - \ber_m(\target_{\PXY}) 
= \ber_m(\hat{g}_n) - \ber_m(\bar{g}) + \ber_m(\bar{g}) - \ber_m(\target_{\PXY})
\\ & \leq \sqrt{P_m( \hat{g}_n \neq \bar{g} ) \frac{2\Czero}{n}\left(\vc(\G)\log\!\left(\frac{1}{P_m( \hat{g}_n \neq \bar{g} )} \right) + \log\!\left(\frac{1}{\delta}\right) \right)} 
+ C \left( \epsilon_{n}^2 + \frac{1}{n}\log\!\left(\frac{1}{\delta}\right) \right).
\end{align*}
Combined with \eqref{eqn:chernoff-n-to-m-size}, 
this further implies that on $E_{n+1}^\star \cap E' \cap E''$, 
letting $C' = 2 C + 4 \Ctwo$, 
\begin{align*}
& \ber_{n+1}(\hat{g}_n) - \ber_{n+1}(\target_{\PXY})
\\ & \leq 2\sqrt{P_m( \hat{g}_n \neq \bar{g} ) \frac{2\Czero}{n}\left(\vc(\G)\log\!\left(\frac{1}{P_m( \hat{g}_n \neq \bar{g} )} \right) + \log\!\left(\frac{1}{\delta}\right) \right)} + C' \left( \epsilon_{n}^2 + \frac{1}{n}\log\!\left(\frac{1}{\delta}\right) \right).
\end{align*}
Additionally, \eqref{eqn:chernoff-m-to-n-size-diffs} 
implies that on this same event, 
$P_m( \hat{g}_n \neq \bar{g} ) \leq 2 P_{n+1}( \hat{g}_n \neq \bar{g} ) + 4\Ctwo \epsilon_{n}^2$.
Since $x \mapsto x \log(1/x)$ is increasing in $x \geq 0$, 
altogether (after some basic relaxations to simplify the expression) 
we have that on $E_{n+1}^\star \cap E' \cap E''$, 
letting 
$\Cfive = C' + 2\sqrt{8 \Ctwo \Czero} = 4 \Czero + 4 \Cfour + 12 \Ctwo + 2\sqrt{8 \Ctwo \Czero}$,
\begin{align*}
& \ber_{n+1}(\hat{g}_n) - \ber_{n+1}(\target_{\PXY})
\\ & \leq 4\sqrt{P_{n+1}( \hat{g}_n \neq \bar{g} ) \frac{\Czero}{n}\left(\vc(\G)\log\!\left(\frac{1}{P_{n+1}( \hat{g}_n \neq \bar{g} )} \right) + \log\!\left(\frac{1}{\delta}\right) \right)} + \Cfive \left( \epsilon_{n}^2 + \frac{1}{n}\log\!\left(\frac{1}{\delta}\right) \right).
\end{align*}

Recalling that, on the event $E_{n+1}^\star \cap E$, 
$\hat{g}_n \in \bar{\G}_m(\Cone \epsilon_m)$,
\eqref{eqn:chernoff-n-to-m-size} implies that 
on the event $E_{n+1}^\star \cap E \cap E''$, 
\begin{equation*}
\ber_{n+1}(\hat{g}_n) - \ber_{n+1}(\target_{\PXY})
\leq 2 \Cone \epsilon_m +  4 \Ctwo \epsilon_{n}^2
\leq (4 \Cone + 8 \Ctwo) \epsilon_{n+1}
= \Cthree \epsilon_{n+1},
\end{equation*}
so that $\hat{g}_n \in \bar{\G}_{n+1}(\Cthree \epsilon_{n+1})$.
In particular, on this event, 
since by definition 
$\bar{g}$ has minimal $P_{n+1}( g \neq \hat{g}_n )$
among all $g \in \bar{\G}_{n+1}(\Cfour \epsilon_{n+1}^2)$, 
by definition of $\bar{\sigma}^2_{n+1}$ 
we have 
\begin{equation*}
P_{n+1}( \hat{g}_n \neq \bar{g} ) \leq \bar{\sigma}^2_{n+1}.
\end{equation*}
Altogether, recalling that $x \mapsto x \log(1/x)$ is increasing in $x \geq 0$, 
on the event 
$E_{n+1}^\star \cap E \cap E' \cap E''$, 
\begin{align*}
& \ber_{n+1}(\hat{g}_n) - \ber_{n+1}(\target_{\PXY})
\\ & \leq 4\sqrt{\bar{\sigma}^2_{n+1} \frac{\Czero}{n}\left(\vc(\G)\log\!\left(\frac{1}{\bar{\sigma}^2_{n+1}} \right) + \log\!\left(\frac{1}{\delta}\right) \right)} + \Cfive \left( \epsilon_{n}^2 + \frac{1}{n}\log\!\left(\frac{1}{\delta}\right) \right).
\end{align*}
To complete the proof, 
we note that $\Cfive \geq 4\sqrt{\Czero}$ 
and that, by the union bound, 
on $E_{n+1}^\star$ 
the event $E \cap E' \cap E''$
has conditional probability at least 
$1 - \frac{1}{m} - \frac{2}{n} - \delta \geq 1 - \frac{4}{n} - \delta$.
\end{proof}

In particular, the following lemma is an immediate 
implication of Lemma~\ref{lem:transductive-localized-ber-term-tail-bound}.

\begin{lemma}
\label{lem:transductive-localized-expected-ber-term-expectation-bound}
For 
$n \in \nats$, 
letting $m = \lceil n/2 \rceil$
and $\hat{g}_n = \TERM^{\G}_n(X_{\leq m},Y_{\leq m},X_{m+1},\ldots,X_{n+1})$, 
\begin{equation*}
\E\!\left[ \ber_{n+1}(\hat{g}_n) - \ber_{n+1}(\target_{\PXY}) \right] 
\leq \Csix \sqrt{ \E\!\left[\bar{\sigma}^2_{n+1}\right] \frac{\vc(\G)}{n} \log\!\left(\frac{1}{\E\!\left[\bar{\sigma}^2_{n+1}\right]}\right)}, 
\end{equation*}
where 
$\Csix = \sqrt{18}\left( 2 + e^{-1} \right) \Cfive + \sqrt{32}$
is a universal constant.
\end{lemma}
\begin{proof}
By Lemma~\ref{lem:transductive-localized-ber-term-tail-bound} and a relaxation of $a + b \leq 2 \max\{a,b\}$ (for $a,b \geq 0$), 
for any $\delta \in (0,e^{-1})$, 
on the event $E_{n+1}^\star$ (of probability one), 
with conditional probability at least $1-\frac{4}{n} - \delta$
given $\lbag X_{\leq n+1} \rbag$, 
\begin{align*}
& \ber_{n+1}(\hat{g}_n) - \ber_{n+1}(\target_{\PXY})
\\ & \leq \max\!\left\{ 2 \Cfive \sqrt{\bar{\sigma}^2_{n+1} \frac{1}{n}\left( \vc(\G) \log\!\left(\frac{1}{\bar{\sigma}^2_{n+1}}\right) + \ln\!\left(\frac{1}{\delta}\right) \right)}, 2 \Cfive \left( \epsilon_n^2 + \frac{1}{n}\ln\!\left(\frac{1}{\delta}\right) \right) \right\}.
\end{align*}
This implies the following tail bound: 
namely, letting 
\begin{equation*} 
\tau_n = \max\!\left\{ 2 \Cfive \sqrt{\bar{\sigma}^2_{n+1} \frac{1}{n}\left( \vc(\G) \log\!\left(\frac{1}{\bar{\sigma}^2_{n+1}}\right) + 1 \right)}, 2 \Cfive \left( \epsilon_n^2 + \frac{1}{n} \right) \right\},
\end{equation*}
for any $\tau > \tau_n$, 
\begin{align*}
& \P\!\Big( \ber_{n+1}(\hat{g}_n) - \ber_{n+1}(\target_{\PXY}) > \tau ~\Big|~ \lbag X_{\leq n+1} \rbag \Big)
\\ & \leq \frac{4}{n} + \exp\!\left\{ \vc(\G) \log\!\left(\frac{1}{\bar{\sigma}^2_{n+1}}\right) - \frac{\tau^2 n}{4 \Cfive^2 \bar{\sigma}^2_{n+1}} \right\} + \exp\!\left\{ \epsilon_n^2 n - \frac{\tau n}{2 \Cfive} \right\}.
\end{align*}
We can therefore bound the conditional expectation as follows:
\begin{align*}
& \E\!\left[ \ber_{n+1}(\hat{g}_n) - \ber_{n+1}(\target_{\PXY}) ~\Big| \lbag X_{\leq n+1} \rbag \right]
= \int_{0}^{1} \P\!\Big( \ber_{n+1}(\hat{g}_n) - \ber_{n+1}(\target_{\PXY}) > \tau ~\Big|~ \lbag X_{\leq n+1} \rbag \Big) \mathrm{d}\tau
\\ & \leq \tau_n + \frac{4}{n} + \int_{\tau_n}^{1} \exp\!\left\{ \vc(\G) \log\!\left(\frac{1}{\bar{\sigma}^2_{n+1}}\right) - \frac{\tau^2 n}{4 \Cfive^2 \bar{\sigma}^2_{n+1}} \right\} \mathrm{d}\tau + \int_{\tau_n}^{1} \exp\!\left\{ \epsilon_n^2 n - \frac{\tau n}{2 \Cfive} \right\} \mathrm{d}\tau.
\end{align*}
Then note that
\begin{align*}
& \int_{\tau_n}^{1} \exp\!\left\{ \epsilon_n^2 n - \frac{\tau n}{2 \Cfive} \right\} \mathrm{d}\tau
\leq \frac{2 \Cfive}{n} \exp\!\left\{\epsilon_n^2 n - \frac{\tau_n n}{2 \Cfive} \right\}
\leq \frac{2 \Cfive}{n} \exp\!\left\{\epsilon_n^2 n - \left( \epsilon_n^2 + \frac{1}{n} \right) n \right\}
= \frac{2 \Cfive e^{-1}}{n}. 
\end{align*}
Also, 
\begin{align*}
& \int_{\tau_n}^{1} \exp\!\left\{ \vc(\G) \log\!\left(\frac{1}{\bar{\sigma}^2_{n+1}}\right) - \frac{\tau^2 n}{4 \Cfive^2 \bar{\sigma}^2_{n+1}} \right\} \mathrm{d}\tau
\\ & \leq \exp\!\left\{ \vc(\G) \log\!\left(\frac{1}{\bar{\sigma}^2_{n+1}}\right) \right\} \int_{\tau_n}^{\infty} \frac{\tau}{\tau_n} \exp\!\left\{ - \frac{\tau^2 n}{4 \Cfive^2 \bar{\sigma}^2_{n+1}} \right\} \mathrm{d}\tau
\\ & = \exp\!\left\{ \vc(\G) \log\!\left(\frac{1}{\bar{\sigma}^2_{n+1}}\right) \right\} \frac{2 \Cfive^2 \bar{\sigma}^2_{n+1}}{\tau_n n} \int_{\tau_n^2 n / (4 \Cfive^2 \bar{\sigma}^2_{n+1})}^{\infty} e^{-u} \mathrm{d}u
\\ & = \frac{2 \Cfive^2 \bar{\sigma}^2_{n+1}}{\tau_n n} \exp\!\left\{ \vc(\G) \log\!\left(\frac{1}{\bar{\sigma}^2_{n+1}}\right) - \frac{\tau_n^2 n}{4 \Cfive^2 \bar{\sigma}^2_{n+1}} \right\}.
\end{align*}
Since $\tau_n \geq 2 \Cfive \sqrt{\bar{\sigma}^2_{n+1} \frac{1}{n}\left( \vc(\G) \log\!\left(\frac{1}{\bar{\sigma}^2_{n+1}}\right) + 1 \right)}$, 
this last expression is at most
\begin{align*}
& \frac{2 \Cfive^2 \bar{\sigma}^2_{n+1} e^{-1}}{\tau_n n} 
\leq \Cfive e^{-1} \sqrt{\bar{\sigma}^2_{n+1} \frac{1}{n}}.
\end{align*}
Altogether, we have that
\begin{align*}
& \E\!\left[ \ber_{n+1}(\hat{g}_n) - \ber_{n+1}(\target_{\PXY}) ~\Big|~ \lbag X_{\leq n+1} \rbag \right]
\leq \tau_n + \frac{4}{n} + \Cfive e^{-1} \sqrt{\bar{\sigma}^2_{n+1} \frac{1}{n}} + \frac{2 \Cfive e^{-1}}{n}
\\ & \leq \left( 2 + e^{-1} \right) \sqrt{2} \Cfive \sqrt{\bar{\sigma}^2_{n+1} \frac{\vc(\G)}{n} \log\!\left(\frac{1}{\bar{\sigma}^2_{n+1}}\right)}
+ \left( \left( 4 + 2 e^{-1} \right) \Cfive + 4 \right) \epsilon_n^2.
\end{align*}
Recalling that (by definition) $\bar{\sigma}^2_{n+1} \geq \epsilon_{n+1}^2 \geq \frac{1}{2} \epsilon_n^2$,
and that $x \mapsto x \log(1/x)$ is increasing in $x > 0$, 
we have that 
\begin{equation*}
\sqrt{\bar{\sigma}^2_{n+1} \frac{\vc(\G)}{n} \log\!\left(\frac{1}{\bar{\sigma}^2_{n+1}}\right)}
\geq \sqrt{ \frac{1}{2} \epsilon_n^2 \frac{\vc(\G)}{n} \log\!\left( \frac{n}{\vc(\G)} \right)}
= \frac{1}{\sqrt{2}} \epsilon_n^2.
\end{equation*}
Thus, we conclude that
\begin{equation*}
\E\!\left[ \ber_{n+1}(\hat{g}_n) - \ber_{n+1}(\target_{\PXY}) ~\Big|~ \lbag X_{\leq n+1} \rbag \right]
\leq \Csix \sqrt{\bar{\sigma}^2_{n+1} \frac{\vc(\G)}{n} \log\!\left(\frac{1}{\bar{\sigma}^2_{n+1}}\right)}
\end{equation*}
for a universal constant 
$\Csix = \sqrt{18}\left( 2 + e^{-1} \right) \Cfive + \sqrt{32}$.
The result now follows by the law of total expectation and Jensen's inequality,
noting that $x \mapsto \sqrt{x \log(1/x)}$ is concave 
in $x > 0$.
\end{proof}

Together with exchangeability of the samples 
$(X_{m+1},Y_{m+1}),\ldots,(X_{n+1},Y_{n+1})$, 
we arrive at the following result.

\begin{lemma}
\label{lem:partial-VC-localized-expected-error-term-expectation-bound}
For any $n \in \nats$, letting $\hat{f}_n^{\G}$ be as in \eqref{eqn:partial-VC-term-learning-rule}, 
\begin{equation*}
\E\!\left[ \er_{\PXY}(\hat{f}^{\G}_n(X_{\leq n},Y_{\leq n})) \right] - \er_{\PXY}(\target_{\PXY}) 
\leq 2 \Csix \sqrt{ \E\!\left[\bar{\sigma}^2_{n+1}\right] \frac{\vc(\G)}{n} \log\!\left(\frac{1}{\E\!\left[\bar{\sigma}^2_{n+1}\right]}\right)}.
\end{equation*}
\end{lemma}
\begin{proof}
Let $m = \lceil n/2 \rceil$
and let $\hat{g}_n = \TERM(X_{\leq m},Y_{\leq m},X_{m+1},\ldots,X_{n+1})$.
Our main goal in this proof is relating 
$\E\!\left[ \er_{\PXY}(\hat{f}_n^{\G}(X_{\leq n},Y_{\leq n})) \right] - \er_{\PXY}(\target_{\PXY})$ 
to $\E\!\left[ \ber_{n+1}(\hat{g}_n) - \ber_{n+1}(\target_{\PXY}) \right]$.
First note that, by the law of total expectation, 
\begin{equation}
\label{eqn:error-to-probability}
\E\!\left[ \er_{\PXY}(\hat{f}_n^{\G}(X_{\leq n},Y_{\leq n})) \right]
= \P\!\left( Y_{n+1} \neq \hat{g}_n(X_{n+1}) \right).
\end{equation}

For each $i \in \{m+1,\ldots,n+1\}$, 
let $\pi_i : \{m+1,\ldots,n+1\} \to \{m+1,\ldots,n+1\}$ 
denote the bijection which has $\pi_i(n+1)=i$, $\pi_i(i)=n+1$, 
and every $j \in \{m+1,\ldots,n+1\} \setminus \{i,n+1\}$ 
has $\pi_i(j) = j$ (i.e., $\pi_i$ merely swaps $i$ and $n+1$).
Let $\hat{g}_n^i = \TERM(X_{\leq m}, Y_{\leq m}, X_{\pi_i(m+1)},\ldots,X_{\pi_i(n+1)})$.
By exchangeability of the random variables 
$(X_{m+1},Y_{m+1}),\ldots,(X_{n+1},Y_{n+1})$ 
(and independence from $(X_{\leq m}, Y_{\leq m})$), 
we have that \eqref{eqn:error-to-probability} equals
\begin{equation}
\label{eqn:leave-one-out-expectation}
\frac{1}{n+1-m} \sum_{i=m+1}^{n+1} \P\!\left( Y_{\pi_i(n+1)} \neq \hat{g}_n^{i}(X_{\pi_i(n+1)}) \right)
= \frac{1}{n+1-m} \sum_{i=m+1}^{n+1} \P\!\left( Y_{i} \neq \hat{g}_n^{i}(X_i) \right).
\end{equation}

Recall that the definition of $\TERM$ stipulates that,
for any bijection $\pi : \{m+1,\ldots,n+1\} \to \{m+1,\ldots,n+1\}$, 
\begin{equation*} 
\TERM_n^{\G}(X_{\leq m}, Y_{\leq m}, X_{\pi(m+1)},\ldots,X_{\pi(n+1)})
= (\hat{g}_n(X_{\leq m}),\hat{g}_n(X_{\pi(m+1)}),\ldots,\hat{g}_n(X_{\pi(n+1)})).
\end{equation*}
We therefore have that \eqref{eqn:leave-one-out-expectation} equals
\begin{align*}
\frac{1}{n+1-m} \sum_{i=m+1}^{n+1} \P\!\left( Y_{i} \neq \hat{g}_n(X_i) \right)
& = \frac{1}{n+1-m} \sum_{i=m+1}^{n+1} \E\!\left[ \P\!\left( Y_{i} \neq \hat{g}_n(X_i) \middle| X_i, \hat{g}_n \right) \right]
\\ & = \E\!\left[ \frac{1}{n+1-m} \sum_{i=m+1}^{n+1} \P\!\left( Y_{i} \neq \hat{g}_n(X_i) \middle| X_i, \hat{g}_n \right) \right],
\end{align*}
where the first equality is by the law of total probability and the second is by linearity of the expectation.
Similarly, by exchangeability, the law of total probability, and linearity of the expectation, we have that
\begin{equation*}
\er_{\PXY}(\target_{\PXY})
= \frac{1}{n+1-m} \sum_{i=m+1}^{n+1} \P( Y_i \neq \target_{\PXY}(X_i) )
= \E\!\left[ \frac{1}{n+1-m} \sum_{i=m+1}^{n+1} \P\!\left( Y_i \neq \target_{\PXY}(X_i) \middle| X_{i} \right) \right].
\end{equation*}
Altogether, we have that
\begin{align*}
& \E\!\left[ \er_{\PXY}(\hat{f}_n^{\G}(X_{\leq n},Y_{\leq n})) \right] - \er_{\PXY}(\target_{\PXY})
\\ & = \E\!\left[ \frac{1}{n+1-m} \sum_{i=m+1}^{n+1} \Big( \P\!\left( Y_{i} \neq \hat{g}_n(X_i) \middle| X_{i}, \hat{g}_n \right) - \P\!\left( Y_i \neq \target_{\PXY}(X_i) \middle| X_{i} \right) \Big) \right].
\end{align*}
Since $\ber_{m}(\hat{g}_n) - \ber_{m}(\target_{\PXY}) \geq 0$, 
the last expression above is upper bounded by 
\begin{align*}
& \frac{1}{n+1-m} \E\!\left[ m \left( \ber_{m}(\hat{g}_n) - \ber_{m}(\target_{\PXY}) \right) +\!\! \sum_{i=m+1}^{n+1} \!\!\left( \P\!\left( Y_{i} \neq \hat{g}_n\!(X_i) \middle| X_{i}, \hat{g}_n \right) - \P\!\left( Y_i \neq \target_{\PXY}(X_i) \middle| X_{i} \right) \right) \right]
\\ & = \frac{1}{n+1-m} \E\!\left[ (n+1) \left( \ber_{n+1}(\hat{g}_n) - \ber_{n+1}(\target_{\PXY}) \right) \right]
\leq 2 \E\!\left[ \ber_{n+1}(\hat{g}_n) - \ber_{n+1}(\target_{\PXY}) \right].
\end{align*}
The claimed result now follows from the above, 
in combination with Lemma~\ref{lem:transductive-localized-expected-ber-term-expectation-bound}.
\end{proof}

The remainder of the lemmas in this subsection 
build toward establishing that 
$\E[\bar{\sigma}^2_n] \to 0$ as $n \to \infty$, 
to obtain the claimed $o(n^{-1/2})$ rate.


\begin{lemma}
\label{lem:transductive-ber-uniform-convergence}
For any $m \in \nats$, 
with probability at least $1-\frac{1}{m}$,
for every $n \geq m$, 
every $g \in \G(X_{\leq n})$ has 
\begin{align*}
\left| \left( \ber_m(g) - \ber_m(\target_{\PXY}) \right) - \left( \ber_n(g) - \ber_n(\target_{\PXY}) \right) \right| \leq \Cseven \epsilon_m,
\end{align*}
and every $f,g \in \G(X_{\leq n})$ have 
\begin{align*}
\left| P_m( f \neq g ) - P_n( f \neq g ) \right| \leq \Cseven \epsilon_m,
\end{align*}
where $\Cseven = \sqrt{2e} \left( 5 + \sqrt{12}  \right)$ is a universal constant (as defined above, following \eqref{eqn:bar-sigma-definition}).
\end{lemma}
\begin{proof}
Since $\Cseven \geq \sqrt{2e}$, 
the result is trivially satisfied if 
$m \leq 2e \vc(\G)$, 
since then the term $\Cseven \epsilon_m$ on the right hand side is greater than one (whereas the quantities on the left hand sides of the two inequalities are both at most one).
To focus the remainder of the proof on the remaining case 
we suppose 
$m > 2 e \vc(\G)$.

We continue to use the notation introduced in the proof of Lemma~\ref{lem:transductive-ber-multiplicative-chernoff}:
namely, for any $n,i \in \nats$ with $i \leq n$
and any $g \in \G(X_{\leq n})$, 
$\ell(g;i) := \P(Y_i \neq g(X_i)|X_i) - \P(Y_i \neq \target_{\PXY}(X_i)|X_i)$, 
and for any $k \leq n$ and distinct $i_1,\ldots,i_k \in [n]$,
$\bar{L}(g;i_1,\ldots,i_k) := \frac{1}{k}\sum_{j=1}^k \ell(g;i_j)$.
In particular, we have $\bar{L}(g;[n]) = \ber_n(g)-\ber_n(\target_{\PXY})$ and for $m \leq n$ we have $\bar{L}(g;[m]) = \ber_m(g) - \ber_m(\target_{\PXY})$.
Also recall that we always have $\ell(g;i) \in [0,1]$.
Likewise, for any distinct $i_1,\ldots,i_m \in [n]$, 
denote by $P_{\{i_1,\ldots,i_m\}}( f \neq g ) := \frac{1}{m} \sum_{j=1}^m \ind[ f(X_{i_j}) \neq g(X_{i_j}) ]$.
Also, for any $n \in \nats$, define $\lfloor n \rfloor_2 = 2^{\lfloor \log_2(n) \rfloor}$: 
that is, $\lfloor n \rfloor_2$ rounds $n$ down to a power of $2$.

The essential strategy is to first establish the claimed relation for any fixed $n,m$ with $m > 2 e \vc(\G)$ 
and $n \geq m \geq \lfloor n-1 \rfloor_2$, 
with probability at least $1-\frac{1}{3 m^2}$.
For fixed $m$, we then extend this (by a union bound) to hold simultaneously for all $n$ satisfying this relation to $m$, 
with probability at least $1 - \frac{1}{3m}$. 
For the case of general $n \geq m > 2 e \vc(\G)$, 
we then chain together these relations (in a telescoping series) among consecutive $m',n'$ powers of $2$ in the interval $[m,n]$, 
which establishes the claim since the sum of these $\epsilon_{n'}$ values is $O(\epsilon_m)$
and the sum of their failure probabilities is at most 
$\frac{1}{m}$.
We refer to this idea as \emph{sample size chaining}.

To start, consider any $n,m \in \nats$ 
with $m > 2 e \vc(\G)$ and 
$n \geq m \geq \lfloor n-1 \rfloor_2$.
As discussed above, the set $\G(X_{\leq n})$ is 
determined (up to permutations) by the multiset 
$\lbag X_{\leq n} \rbag$ of values in $X_{\leq n}$, 
and each $f,g \in \G(X_{\leq n})$
have $P_n( f \neq g )$ and $\ber_n(g) - \ber_n(\target_{\PXY})$ invariant to permutations of this sequence (and similarly $P_m( f \neq g )$ and $\ber_m(g)-\ber_m(\target_{\PXY})$ are invariant to permutations of $X_{\leq m}$).
Moreover, by exchangeability, 
given the multiset $\lbag X_{\leq n} \rbag$ of values in $X_{\leq n}$, 
the sequence $X_{\leq n}$ is conditionally a uniform random ordering of the multiset $\lbag X_{\leq n} \rbag$.
Let $\bpi$ be a $\mathrm{Uniform}(\Sym([n]))$ 
random variable independent of $X_{\leq n}$.
By the above observations, for any $\epsilon > 0$, 
we have that 
\begin{align*}
& \P\Big( \exists g \in \G(X_{\leq n}) : \left| \bar{L}(g;[m]) - \bar{L}(g;[n]) \right| > \epsilon ~\Big|~ \lbag X_{\leq n} \rbag \Big)
\\ & = \P\Big( \exists f,g \in \G(X_{\leq n}) : \left| \bar{L}(g;\bpi([m])) - \bar{L}(g;[n]) \right| > \epsilon ~\Big|~ X_{\leq n} \Big).
\end{align*}
By the union bound, this last quantity is at most
\begin{equation*}
\sum_{g \in \G(X_{\leq n})} \P\Big( \left| \bar{L}(g;\bpi([m])) - \bar{L}(g;[n]) \right| > \epsilon ~\Big|~ X_{\leq n} \Big).
\end{equation*}
We bound each term in this sum by Hoeffding's inequality
for sampling without replacement \citep*{hoeffding:53}.
Specifically, for each $g \in \G(X_{\leq n})$, we have
\begin{equation*}
\P\Big( \left| \bar{L}(g;\bpi([m])) - \bar{L}(g;[n]) \right| > \epsilon ~\Big|~ X_{\leq n} \Big)
\leq 2e^{-2 \epsilon^2 m}.
\end{equation*}
Sauer's lemma \citep*{vapnik:71,sauer:72} 
implies that $|\G(X_{\leq n})| \leq \left( \frac{e n}{\vc(\G)} \right)^{\vc(\G)}$,
so that altogether we have that 
\begin{equation}
\label{eqn:transductive-uniform-convergence-tail-bound}
\P\Big( \exists g \in \G(X_{\leq n}) : \left| \bar{L}(g;[m]) - \bar{L}(g;[n]) \right| > \epsilon ~\Big|~ \lbag X_{\leq n} \rbag \Big) \leq 2 \left( \frac{e n}{\vc(\G)} \right)^{\vc(\G)} e^{-2 \epsilon^2 m}.
\end{equation}

Similarly, for the claim about pairs $f,g \in \G(X_{\leq n})$, 
for any $\epsilon > 0$, we have that
\begin{align*}
& \P\Big( \exists f,g \in \G(X_{\leq n}) : \left| P_m( f \neq g ) - P_n( f \neq g ) \right| > \epsilon ~\Big|~ \lbag X_{\leq n} \rbag \Big)
\\ & = \P\Big( \exists f,g \in \G(X_{\leq n}) : \left| P_{\bpi([m])}( f \neq g ) - P_n( f \neq g ) \right| > \epsilon ~\Big|~ X_{\leq n} \Big).
\end{align*}
By the union bound, this last quantity is at most
\begin{equation*}
\sum_{f,g \in \G(X_{\leq n})} \P\Big( \left| P_{\bpi([m])}( f \neq g ) - P_n( f \neq g ) \right| > \epsilon ~\Big|~ X_{\leq n} \Big).
\end{equation*}
By Hoeffding's inequality for sampling without replacement \citep*{hoeffding:53},
for each $f,g \in \G(X_{\leq n})$, we have
\begin{equation*}
\P\Big( \left| P_{\bpi([m])}( f \neq g ) - P_n( f \neq g ) \right| > \epsilon ~\Big|~ X_{\leq n} \Big)
\leq 2e^{-2 \epsilon^2 m}.
\end{equation*}
Again, Sauer's lemma 
implies that $|\G(X_{\leq n})| \leq \left( \frac{e n}{\vc(\G)} \right)^{\vc(\G)}$, 
so that altogether we have that 
\begin{equation}
\label{eqn:transductive-uniform-convergence-differences-tail-bound}
\P\Big( \exists f,g \in \G(X_{\leq n}) : \left| P_m( f \neq g ) - P_n( f \neq g ) \right| > \epsilon ~\Big|~ \lbag X_{\leq n} \rbag \Big) \leq 2 \left( \frac{e n}{\vc(\G)} \right)^{2\vc(\G)} e^{-2 \epsilon^2 m}.
\end{equation}

Letting $\epsilon = \sqrt{2 e} \epsilon_m$, 
and noting that (since $m \geq n/2$ and $m > 2 e \vc(\G)$)
\begin{align*} 
& \left( \frac{e n}{\vc(\G)} \right)^{2\vc(\G)} \!\!e^{-4e \epsilon_m^2 m} 
= \left( \frac{e n}{\vc(\G)} \right)^{2\vc(\G)} \!\!\left( \frac{m}{\vc(\G)} \right)^{\!- 4e\vc(\G)}
\!\!\!\!\leq (2e)^{2\vc(\G)} \left( \frac{m}{\vc(\G)} \right)^{\!- 8\vc(\G)}
\\ & = \frac{1}{12m^2} \cdot 12 \vc(\G)^2 (2e)^{2\vc(\G)} \left( \frac{m}{\vc(\G)} \right)^{2 - 8\vc(\G)}
\!\!\leq \frac{1}{12m^2} \cdot 12 \vc(\G)^2 ( 2e )^{2 - 6\vc(\G)}
\leq \frac{1}{12m^2},
\end{align*}
we conclude that, by the union bound (over these two failure events from \eqref{eqn:transductive-uniform-convergence-tail-bound} and \eqref{eqn:transductive-uniform-convergence-differences-tail-bound}),
on an event $E'_{m,n}$ of probability at least $1 - \frac{1}{3m^2}$, 
\begin{equation}
\label{eqn:transductive-uniform-convergence-power-of-two-bound}
\forall g \in \G(X_{\leq n}),~ \left| \bar{L}(g;[m]) - \bar{L}(g;[n]) \right| \leq \sqrt{2 e}\epsilon_m
\end{equation}
and 
\begin{equation}
\label{eqn:transductive-uniform-convergence-differences-power-of-two-bound}
\forall f,g \in \G(X_{\leq n}),~ \left| P_m( f \neq g ) - P_n( f \neq g ) \right| \leq \sqrt{2 e}\epsilon_m.
\end{equation}

Before proceeding to general $n,m$, we first slightly strengthen the above guarantee, 
to hold simultaneously for all $n$ satisfying this relation to a given $m$.
For any $m \in \nats$ with $m > 2 e \vc(\G)$, 
define an event $E'_{m} = \bigcap_{n = m+1}^{2 \lfloor m \rfloor_2} E'_{m,n}$.
The values $n \in \{m+1,\ldots,2 \lfloor m \rfloor_2 \}$ 
are merely those $n > m$ ranging up to the next power of $2$ above $m$;
equivalently, these are all values $n \in \nats$ satisfying 
$n > m \geq \lfloor n - 1 \rfloor_2$.
Therefore, by the union bound and the above analysis, 
the event $E'_m$ has probability at least 
$1 - \frac{1}{3 m^2} \left( 2 \lfloor m \rfloor_2 - m \right) \geq 1 - \frac{1}{3m}$.
We then have that, for any $m > 2 e \vc(\G)$, 
on the event $E'_m$, 
\eqref{eqn:transductive-uniform-convergence-power-of-two-bound} and 
\eqref{eqn:transductive-uniform-convergence-differences-power-of-two-bound}
hold for every $n \in \{m,\ldots,2 \lfloor m \rfloor_2 \}$
(noting that they hold vacuously for $n=m$).

Next, for any fixed $m > 2 e \vc(\G)$,
we extend \eqref{eqn:transductive-uniform-convergence-power-of-two-bound} and 
\eqref{eqn:transductive-uniform-convergence-differences-power-of-two-bound}
(with a slight increase in the constant factor) 
to hold for \emph{all} $n \geq m$
by a \emph{sample size chaining} technique: 
chaining together each of the above inequalities
in a telescoping series, for consecutive \emph{powers of two} $m',n'$ in $[m,n]$.
Let $E_m = E'_m \cap \bigcap_{k= 1+\lfloor \log_2(m) \rfloor}^{\infty} E'_{2^k}$,
and note that, by the union bound and the fact that $\lfloor m \rfloor_2 > m/2$, 
the event $E_m$ has probability at least 
\begin{equation*} 
1 - \frac{1}{3m} - \sum_{k=1+\lfloor \log_2(m) \rfloor}^{\infty} \frac{1}{3 \cdot 2^k} = 1 - \frac{1}{3m} - \frac{1}{3 \lfloor m \rfloor_2} \geq 1 - \frac{1}{m}.
\end{equation*}

Consider any $n \in \nats$ with $n \geq m$.
Let $K = \lfloor \log_2(n-1) \rfloor$.
As we have already established above that 
\eqref{eqn:transductive-uniform-convergence-power-of-two-bound} and 
\eqref{eqn:transductive-uniform-convergence-differences-power-of-two-bound}
hold in the case $m \geq 2^{K}$ on the event $E_m$,
let us suppose $2^{K} > m$.
Let $K' = \lfloor \log_2(m) \rfloor$.
Define $m_{K'} = m$ and $m_{K+1} = n$, 
and for values $k \in \{K'+1,\ldots,K \}$, 
define $m_k = 2^k$.
Note that for each $k \in \{K',\ldots,K\}$, 
we have $m_{k} \leq m_{k+1} \leq 2 \lfloor m_{k} \rfloor_2$. 
Therefore, since $E_m \subseteq \bigcap_{k=K'}^{K} E'_{m_k}$,
on the event $E_m$ it holds that, 
for every $k \in \{K',\ldots,K\}$, 
we have 
\begin{equation*}
\forall g \in \G(X_{\leq m_{k+1}}),~ \left| \bar{L}(g;[m_k]) - \bar{L}(g;[m_{k+1}]) \right| \leq \sqrt{2 e}\epsilon_{m_k}
\end{equation*}
and 
\begin{equation*}
\forall f,g \in \G(X_{\leq m_{k+1}}),~ \left| P_{m_k}( f \neq g ) - P_{m_{k+1}}( f \neq g ) \right| \leq \sqrt{2 e}\epsilon_{m_k}.
\end{equation*}
Moreover, we may express $\bar{L}(g;[m]) - \bar{L}(g;[n])$
and $P_{m}(f \neq g) - P_{n}(f \neq g)$
as telescoping series: namely,
$\forall g \in \G(X_{\leq n})$, 
\begin{equation*}
\bar{L}(g;[m]) - \bar{L}(g;[n])
= \sum_{k = K'}^{K} \left( \bar{L}(g;[m_{k}]) - \bar{L}(g;[m_{k+1}]) \right) 
\end{equation*}
and $\forall f,g \in \G(X_{\leq n})$, 
\begin{equation*}
P_{m}(f \neq g) - P_{n}(f \neq g)
= \sum_{k = K'}^{K} \left( P_{m_{k}}( f \neq g ) - P_{m_{k+1}}( f \neq g ) \right).
\end{equation*}
Thus, on the event $E_m$, 
$\forall g \in \G(X_{\leq n})$, 
\begin{align*}
\left| \bar{L}(g;[m]) - \bar{L}(g;[n]) \right|
& = \left| \sum_{k = K'}^{K} \left( \bar{L}(g;[m_{k}]) - \bar{L}(g;[m_{k+1}]) \right) \right|
\\ & \leq \sum_{k = K'}^{K} \left| \bar{L}(g;[m_{k}]) - \bar{L}(g;[m_{k+1}]) \right|
\leq \sum_{k=K'}^{K} \sqrt{2 e} \epsilon_{m_{k}} 
\end{align*}
and similarly, 
$\forall f,g \in \G(X_{\leq n})$, 
\begin{align*}
\left| P_{m}(f \neq g) - P_{n}(f \neq g) \right| 
& = \left| \sum_{k = K'}^{K} \left( P_{m_{k}}( f \neq g ) - P_{m_{k+1}}( f \neq g ) \right) \right|
\\ & \leq \sum_{k = K'}^{K} \left| P_{m_{k}}( f \neq g ) - P_{m_{k+1}}( f \neq g ) \right|
\leq \sum_{k=K'}^{K} \sqrt{2 e} \epsilon_{m_{k}}.
\end{align*}
Also note that, for any $k \in \{K'+2,\ldots,K\}$, 
we have $\epsilon_{m_{k}} \leq \sqrt{\frac{3}{4}}\epsilon_{m_{k-1}}$ (using the fact that $m_{k-1} \geq m > 2 e \vc(\G)$), 
whereas $\epsilon_{m_{K'+1}} \leq \epsilon_{m}$ 
(by monotonicity of $x \mapsto x \log(1/x)$ for $x > 0$).
We therefore have that 
\begin{align*}
\sum_{k=K'}^{K} \sqrt{2 e} \epsilon_{m_{k}}
\leq \sqrt{2e} \left( \epsilon_m + \sum_{k=K'+1}^{K} \!\left( \epsilon_m \cdot \left( \frac{3}{4} \right)^{\!(k-K'-1)/2} \right) \right)
\leq \sqrt{2e} \left( 5 + \sqrt{12}  \right) \epsilon_m
= \Cseven \epsilon_m,
\end{align*}
which completes the proof.
\end{proof}

For any $n,m \in \nats$ with $n \geq m$, 
for any $\epsilon > 0$, 
define the \emph{projection} of $\bar{\G}_{n}(\epsilon)$ to $X_{\leq m}$: 
\begin{equation*}
\bar{\G}_{n}(\epsilon)\big|_{m} := \left\{ g(X_{\leq m}) : g \in \bar{\G}_{n}(\epsilon) \right\} \subseteq \G(X_{\leq m}).
\end{equation*}

\begin{lemma}
\label{lem:transductive-equal-equivalence-classes}
For any sequence $\{\epsilon'_n\}_{n \in \nats}$ in $(0,1]$, 
for any $m \in \nats$, 
there exists $N_m \in \nats$ with $N_m \geq m$ such that,
with probability at least $1 - \frac{1}{m}$, 
for every $n \geq N_m$, 
there exists an infinite sequence 
$n < n_1 < n_2 < \cdots$ such that 
$\forall i \in \nats$, 
\begin{equation*} 
\bar{\G}_{n_i}(\epsilon'_{n_i})\big|_{m} = \bar{\G}_{n}(\epsilon'_n)\big|_{m}~.
\end{equation*}
\end{lemma}
\begin{proof}
Fix any $m \in \nats$.
Since each $\bar{\G}_n(\epsilon'_n) \big|_m \in 2^{\{0,1\}^m}$ (the set of all subsets of $\{0,1\}^m$), 
and $2^{\{0,1\}^m}$ has finite size, 
we observe that there must exist $\hat{N}_m \in \nats$ with $\hat{N}_m \geq m$ 
such that, $\forall n \geq \hat{N}_m$, 
the set $\bar{\G}_n(\epsilon'_n) \big|_m$
occurs infinitely often in the sequence 
$\{\bar{\G}_{n'}(\epsilon'_{n'})\big|_m \}_{n' = m}^{\infty}$.
To see this, note that for each $G \in 2^{\{0,1\}^m}$ 
which occurs in the sequence
$\{\bar{\G}_{n'}(\epsilon'_{n'})\big|_m \}_{n' = m}^{\infty}$
but only finitely often, 
there must be a largest value $\hat{n}(G)$ for which 
$\bar{\G}_{\hat{n}(G)}(\epsilon'_{\hat{n}(G)}) \big|_m = G$;
it therefore suffices to take $\hat{N}_m$ as one larger than the 
maximum $\hat{n}(G)$ among all $G \in 2^{\{0,1\}^m}$ 
which occur finitely often in this sequence 
(or else $\hat{N}_m = m$ if every $G$ occurs either zero or infinitely many times).
Equivalently, the above guarantees that, for every $n \geq \hat{N}_m$, 
there exists an infinite sequence $n < n_1 < n_2 < \cdots$
for which $\bar{\G}_{n_i}(\epsilon'_{n_i}) \big|_m = \bar{\G}_n(\epsilon'_n)\big|_m$.
Since $\hat{N}_m$ is a finite $\{X_t\}_{t \in \nats}$-dependent random variable, 
we have $\lim_{n \to \infty} \P( \hat{N}_m > n ) = 0$, 
so that there exists $N_m \in \nats$ such that 
$\P(\hat{N}_m > N_m) \leq \frac{1}{m}$.
Thus, with probability at least $1-\frac{1}{m}$, 
$\hat{N}_m \leq N_m$, 
so that every $n \geq N_m$ has $n \geq \hat{N}_m$, 
and the claim in the lemma follows.
\end{proof}

We are now ready to start combining the above lemmas to 
obtain that $\E[\bar{\sigma}^2_n] \to 0$
via the following two lemmas.

\begin{lemma}
\label{lem:transductive-projected-prefix-patterns-agree-with-good-concepts}
Fix any $m \in \nats$ 
and let $N_{m}$ be as in Lemma~\ref{lem:transductive-equal-equivalence-classes} for the sequence 
$\epsilon'_{n} := \Cthree \epsilon_{n}$, $n \in \nats$.
Then for every $n \in \nats$ with $n \geq N_{m}$ and $n \geq \vc(\G)$,
with probability at least $1 - \frac{3}{m}$,
\begin{equation*}
\bar{\G}_n( \Cthree \epsilon_n ) \big|_{m} \subseteq \bar{\G}_{n}( \Cfour \epsilon_{n}^2 ) \big|_{m}~,
\end{equation*}
where $\Cfour = \sqrt{8}(\Cthree+\Cseven) + 4\Ctwo$
is a universal constant (as defined above, following \eqref{eqn:bar-sigma-definition}).
\end{lemma}
\begin{proof}
Fix any $n \in \nats$ with $n \geq N_m$ and $n \geq \vc(\G)$, 
and let $n' = \left\lceil \frac{n^2}{\vc(\G) \log(n/\vc(\G))} \right\rceil$ 
(and note that $n' \geq n$ due to $n \geq \vc(\G)$).
Lemma~\ref{lem:transductive-ber-uniform-convergence} 
implies that, with probability at least 
$1-\frac{1}{n'} \geq 1 - \frac{1}{m}$, 
every $n'' \geq n'$ satisfies that, 
$\forall g \in \G(X_{\leq n''})$, 
\begin{equation*}
\ber_{n'}(g) - \ber_{n'}(\target_{\PXY}) 
\leq \ber_{n''}(g) - \ber_{n''}(\target_{\PXY}) + \Cseven \epsilon_{n'}.
\end{equation*}
In particular, this implies that 
\begin{equation}
\label{eqn:npp-inclusion-np}
\bar{\G}_{n''}( \Cthree \epsilon_{n''} ) \big|_{n'} \subseteq \bar{\G}_{n'}( \Cthree \epsilon_{n''} + \Cseven \epsilon_{n'} ) \subseteq \bar{\G}_{n'}( (\Cthree+\Cseven) \epsilon_{n'} ). 
\end{equation}

By Lemma~\ref{lem:transductive-ber-multiplicative-chernoff} (applied to $n'$ and $n$, and with $\delta = \frac{1}{m}$), 
with probability at least $1-\frac{1}{m}$, 
every $g \in \G(X_{\leq n'})$ has
\begin{align*}
\ber_n(g) - \ber_n(\target_{\PXY}) 
& \leq 2 \left( \ber_{n'}(g) - \ber_{n'}(\target_{\PXY}) \right) + \frac{\Ctwo}{n} \left( \vc(\G) \log\!\left(\frac{n'}{\vc(\G)}\right) + \log(m) \right)
\\ & \leq 2 \left( \ber_{n'}(g) - \ber_{n'}(\target_{\PXY}) \right) + 4 \Ctwo \epsilon_n^2.
\end{align*}
In particular, 
noting that 
\begin{equation*} 
\epsilon_{n'} \leq \sqrt{\frac{\vc(\G)^2}{n^2} \log\!\left(\frac{n}{\vc(\G)}\right) \log\!\left(\frac{n^2}{\vc(\G)^2}\right)} \leq \sqrt{2} \epsilon_n^2,
\end{equation*}
this implies that
\begin{equation}
\label{eqn:np-proj-n-inclusion}
\bar{\G}_{n'}((\Cthree + \Cseven)\epsilon_{n'}) \big|_n 
\subseteq \bar{\G}_{n}( 2(\Cthree+\Cseven)\epsilon_{n'} + 4\Ctwo\epsilon_n^2)
\subseteq \bar{\G}_{n}\!\left( \left( \sqrt{8}(\Cthree+\Cseven) + 4\Ctwo \right)\epsilon_{n}^2 \right) 
= \bar{\G}_{n}( \Cfour \epsilon_n^2 ).
\end{equation}

Finally, by Lemma~\ref{lem:transductive-equal-equivalence-classes}, 
with probability at least $1 - \frac{1}{m}$, 
there exists $n'' \geq n'$ with 
\begin{equation} 
\label{eqn:npp-proj-m-n-proj-m}
\bar{\G}_{n''}(\Cthree\epsilon_{n''})\big|_{m} = \bar{\G}_{n}(\Cthree\epsilon_n)\big|_{m}.
\end{equation}
By the union bound, all three of the above events occur simultaneously
with probability at least $1-\frac{3}{m}$.
Supposing this occurs, then for the $n''$ satisfying \eqref{eqn:npp-proj-m-n-proj-m}, we have that
\begin{equation*}
\bar{\G}_{n}(\Cthree\epsilon_n)\big|_{m} 
= \bar{\G}_{n''}(\Cthree \epsilon_{n''})\big|_{m}
\subseteq \bar{\G}_{n'}((\Cthree + \Cseven)\epsilon_{n'})\big|_{m}
\subseteq \bar{\G}_{n}( \Cfour\epsilon_{n}^2 )\big|_{m}, 
\end{equation*}
where the first equality is due to \eqref{eqn:npp-proj-m-n-proj-m} 
and the subsequent subset inclusions are due to 
\eqref{eqn:npp-inclusion-np} and \eqref{eqn:np-proj-n-inclusion}, 
respectively.
\end{proof}

\begin{lemma}
\label{lem:transductive-expected-bar-sigma-converging-to-zero}
For every $m \in \nats$, let $N_{m}$
be defined as in Lemma~\ref{lem:transductive-equal-equivalence-classes} 
for the sequence $\epsilon'_{n} := \Cthree \epsilon_{n}$, $n \in \nats$.
For every $n \in \nats$ with $n \geq N_1$ and $n \geq \vc(\G)$, 
letting $M_n = \max\{ m \in \nats : n \geq N_m \}$,
it holds that 
\begin{equation*}
\E\!\left[ \bar{\sigma}^2_n \right] \leq \Cseven \epsilon_{{\scriptscriptstyle M_n}} + \frac{4}{M_n}.
\end{equation*}
In particular, since each $N_m < \infty$, 
we have $M_n \to \infty$ as $n \to \infty$, 
so that 
\begin{equation*} 
\lim_{n \to \infty} \E\!\left[ \bar{\sigma}^2_n \right] = 0.
\end{equation*}
\end{lemma}
\begin{proof}
We first note that $M_n$ is well-defined since 
$N_m \geq m$ implies $\{ m \in \nats : n \geq N_m \}$ is a finite set 
(and is non-empty since $n \geq N_1$).
Moreover, by definition of $M_n$, 
it holds that $n \geq N_{M_n}$.
Since, by assumption, we also have $n \geq \vc(\G)$, 
Lemma~\ref{lem:transductive-projected-prefix-patterns-agree-with-good-concepts}
implies that, with probability at least $1 - \frac{3}{M_n}$, 
we have
$\bar{\G}_n(\Cthree \epsilon_n)\big|_{M_n} \subseteq \bar{\G}_n(\Cfour \epsilon_n^2)\big|_{M_n}$.
In other words, 
for every $g \in \bar{\G}_n(\Cthree \epsilon_n)$, 
there exists $f \in \bar{\G}_n(\Cfour \epsilon_n^2)$
with $f(X_{\leq M_n}) = g(X_{\leq M_n})$.

Also, since $n \geq N_{M_n} \geq M_n$, 
Lemma~\ref{lem:transductive-ber-uniform-convergence} (applied to $M_n$)
implies that, with probability at least $1 - \frac{1}{M_n}$, 
every $f,g \in \G(X_{\leq n})$ 
with $f(X_{\leq M_n}) = g(X_{\leq M_n})$ 
(hence $P_{{\scriptscriptstyle M_n}}(f \neq g)=0$)
have
$P_n(f \neq g) \leq \Cseven \epsilon_{{\scriptscriptstyle M_n}}$.

By the union bound, the above two events occur simultaneously
with probability at least $1 - \frac{4}{M_n}$.
Supposing this occurs, 
we have that for every $g \in \bar{\G}_n(\Cthree \epsilon_n)$, 
there exists $f \in \bar{\G}_n(\Cfour \epsilon_n^2)$
with $P_n(f \neq g) \leq \Cseven \epsilon_{{\scriptscriptstyle M_n}}$.
Since $n \geq \vc(\G)$ (hence $\epsilon_n \leq 1$)
and $n \geq N_{M_n} \geq M_n$, 
and since $\Cseven \geq 1$, 
we have $\epsilon_n^2 \leq \epsilon_n \leq \Cseven \epsilon_{{\scriptscriptstyle M_n}}$.
Thus, on the above event, 
$\bar{\sigma}^2_n \leq \Cseven \epsilon_{{\scriptscriptstyle M_n}}$.

To conclude, we note that if either of the above events fails, 
we still have $\bar{\sigma}_n^2 \leq 1$ (again due to $\epsilon_n \leq 1$),
and therefore 
$\E\!\left[ \bar{\sigma}^2_n \right]
\leq \Cseven \epsilon_{{\scriptscriptstyle M_n}} + \frac{4}{M_n}$.
\end{proof}

We are now ready to present the proof of Lemma~\ref{lem:partial-VC-super-root}, establishing that the learning rule $\hat{f}^{\G}_n$ based on $\TERM$ 
achieves excess error that is 
$o(n^{-1/2})$.

{\vskip 2mm}\begin{proof}[of Lemma~\ref{lem:partial-VC-super-root}]
We consider the learning rule 
$\hat{f}^{\G}_n$ as defined in \eqref{eqn:partial-VC-term-learning-rule}.
Let $\PXY$ be any distribution that is Bayes-realizable with respect to $\G$.
We first address the trivial case of $\vc(\G) = 0$.
In this case, note that for any $n \in \nats$, 
the set $\G(X_{\leq n+1})$ has size at most one.
Moreover, the assumption that $\PXY$ is Bayes-realizable with respect to $\G$ implies that, 
with probability one, 
$\exists g^{\star} \in \G(X_{\leq n+1})$
such that the value $y_{n+1} = g^{\star}(X_{n+1})$
satisfies $\P( y_{n+1} \neq Y_{n+1} | X_{n+1} ) = \P( \target_{\PXY}(X_{n+1}) \neq Y_{n+1} | X_{n+1} )$.
In particular, these two observations together imply that, 
with probability one, 
\begin{equation*} 
\P\!\left( \hat{f}^{\G}_n(X_{\leq n},Y_{\leq n},X_{n+1}) \neq Y_{n+1} \middle| X_{\leq n+1}, Y_{\leq n} \right) = \P\!\left( \target_{\PXY}(X_{n+1}) \neq Y_{n+1} \middle| X_{n+1} \right).
\end{equation*}
Therefore, 
\begin{align*} 
& \E\!\left[ \er_{\PXY}\!\left( \hat{f}^{\G}_n(X_{\leq n},Y_{\leq n}) \right) \right] - \er_{\PXY}(\target_{\PXY})
\\ & = \E\!\left[ \P\!\left(\hat{f}^{\G}_n(X_{\leq n},Y_{\leq n},X_{n+1}) \neq Y_{n+1} \middle| X_{\leq n+1},Y_{\leq n}\right) - \P\!\left( \target_{\PXY}(X_{n+1}) \neq Y_{n+1} \middle| X_{n+1} \right) \right]
= 0.
\end{align*}
Thus, defining $\phi(n;\G,\PXY) = 0$ for every $n \in \nats$ 
completes the proof for this case.

To address the remaining case, suppose $1 \leq \vc(\G) < \infty$.
For every $n \in \nats$, define
\begin{equation*}
\phi(n;\G,\PXY) = 2 \Csix \sqrt{ \E\!\left[\bar{\sigma}^2_{n+1}\right] \frac{\vc(\G)}{n} \log\!\left(\frac{1}{\E\!\left[\bar{\sigma}^2_{n+1}\right]}\right)}.
\end{equation*}
By Lemma~\ref{lem:partial-VC-localized-expected-error-term-expectation-bound}, 
for every $n \in \nats$, 
\begin{equation*}
\E\!\left[ \er_{\PXY}(\hat{f}^{\G}_n(X_{\leq n},Y_{\leq n})) \right] - \er_{\PXY}(\target_{\PXY}) 
\leq \phi(n;\G,\PXY).
\end{equation*}
It remains only to argue that $\phi(n;\G,\PXY) \leq c \sqrt{\frac{\vc(\G)}{n}}$ (for a universal constant $c$)
and $\phi(n;\G,\PXY) = o(n^{-1/2})$.
Since $0 \leq \E\!\left[ \bar{\sigma}^2_{n+1} \right] \leq 1$, 
monotonicity of $x \mapsto x \log(1/x)$ for $x \geq 0$ implies 
$\phi(n;\G,\PXY) \leq 2 \Csix \sqrt{\frac{\vc(\G)}{n}}$ (recalling that $\log(x) = \ln(\max\{x,e\})$),
so that $c = 2 \Csix$ suffices for the first claim.
Moreover, 
by Lemma~\ref{lem:transductive-expected-bar-sigma-converging-to-zero},
we have $\lim_{n \to \infty} \E\!\left[\bar{\sigma}^2_{n+1}\right] = 0$,
and since $\sqrt{x \log(1/x)} \to 0$ as $x \to 0$, 
we have 
\begin{equation*}
\lim_{n \to \infty} n^{1/2} \cdot \phi(n;\G,\PXY) 
= 2 \Csix \sqrt{\vc(\G)} \lim_{n \to \infty} \sqrt{\E\!\left[ \bar{\sigma}^2_{n+1}\right] \log\!\left(\frac{1}{\E\!\left[ \bar{\sigma}^2_{n+1}\right]}\right)} 
= 0,
\end{equation*}
so that we indeed have $\phi(n;\G,\PXY) = o(n^{-1/2})$.
This completes the proof of Lemma~\ref{lem:partial-VC-super-root}.
\end{proof}

\subsection{Proof of Theorem~\ref{thm:agnostic-super-sqrt-upper-bound}: $o(n^{-1/2})$ Rates}
\label{sec:subsection-proof-of-super-root-upper-bound}

We are now ready for the proof of Theorem~\ref{thm:agnostic-super-sqrt-upper-bound},
which combines the above analysis of partial concept classes 
with the strategy outlined in the paragraphs 
following the statement of the theorem above.

\begin{proof}[of Theorem~\ref{thm:agnostic-super-sqrt-upper-bound}]
  %
  We begin by describing the learning algorithm.
  Let $\PXY$ be any distribution on $\X \times \{0,1\}$.
  Let $n \in \nats$,
  and let $S_n = \{(X_1,Y_1),\ldots,(X_n,Y_n)\} \sim \PXY^n$ be the i.i.d.\ training data input to the learning algorithm.
  For simplicity, suppose $n$ is an integer multiple of $4$; otherwise, replace $n$ with $4 \lfloor n/4 \rfloor$ in the description and analysis below (supposing $n \geq 4$).
  We first segment the data into chunks of size $n/4$: 
  for $k \in \{0,1,2,3\}$, let
  \begin{equation*}
    S_n^k = \{(X_{1 + k n/4},Y_{1 + k n/4}),\ldots,(X_{(k+1)n/4},Y_{(k+1)n/4})\}.
  \end{equation*}

  Recall that we say a distribution $\PXY$ on $\X \times \{0,1\}$ 
  is \emph{Bayes-realizable} with respect to $\H$ if 
  $\inf_{h \in \H} \er_{\PXY}(h) = \inf_{h} \er_{\PXY}(h)$, 
  where $h$ on the right hand side ranges over all measurable functions 
  $\X \to \{0,1\}$.  
  The algorithm has two main components.
  The first component is designed to handle the case that $\PXY$ is \emph{not} Bayes-realizable with respect to $\H$.
  In this case, since $\inf_{h} \er_{\PXY}(h) < \inf_{h \in \H} \er_{\PXY}(h)$, 
  it is possible to converge to a strictly \emph{negative} excess error rate
  $\er_{\PXY}(\hat{h}) - \inf_{h \in \H} \er_{\PXY}(h)$
  if we employ a learning algorithm guaranteed to converge to the 
  \emph{Bayes error} $\inf_{h} \er_{\PXY}(h)$.
  Specifically, let $\hat{h}_n^0$ be the classifier 
  returned by a universally Bayes-consistent learning algorithm 
  with training set $S_n^0$ (as in Lemma~\ref{lem:universally-bayes-consistent}).
  Thus, we have the property that 
  $\E[\er_{\PXY}(\hat{h}_n^0)] \to \inf_{h} \er_{\PXY}(h)$ 
  for all distributions $\PXY$ (where the $h$ in the infimum ranges over all measurable $h : \X \to \{0,1\}$).
  The intention is that, once we also specify a predictor $\hat{h}_n^1$ below,
  which is guaranteed to have the required rate when $\PXY$ is Bayes-realizable with respect to $\H$,
  we can merely select between these two functions $\hat{h}_n^0$, $\hat{h}_n^1$ in the end to always guarantee the required rate for the excess error rate 
  (indeed, the rate we actually achieve may become negative in the case $\PXY$ is not Bayes-realizable with respect to $\H$).

  Next we describe the second component, defining the aforementioned predictor $\hat{h}_n^1$ suited to the case that $\PXY$ is Bayes-realizable w.r.t.\ $\H$.
  For each $b \in \{1,\ldots,n/4\}$, 
  let us further segment the sequence $S_n^0$ into \emph{batches} $B^b_i$ of size $b$:
  let $I_b := \{1,\ldots, \lfloor n/(4b) \rfloor \}$, 
  and for each $i \in I_b$ define 
  \begin{equation*}
    B^b_i = \left\{(X_{1+(i-1)b},Y_{1+(i-1)b}),\ldots,(X_{i b},Y_{i b})\right\}.
  \end{equation*}
  For each $i \in I_b$,
  for each $\mathbf{y} = (y_1,\ldots,y_b) \in \{0,1\}^b$, 
  define 
  \begin{equation*} 
  B^b_{i}(\mathbf{y}) = \{(X_{1+(i-1)b},y_1),\ldots,(X_{i b},y_b)\},
  \end{equation*}
  that is, $B^b_i$, but with the $Y$ labels replaced by $y$ labels.
  Let $k^b_{\VCL}$ and $\algVCLSOA^{b,k}$ be as in Lemma~\ref{lem:VCL-SOA-zero-error-rate-bstar} (for the class $\H$),
  and define $k^{\mathbf{y}}_{b,i} = k^b_{\VCL}(B^b_i(\mathbf{y}))$
  and $f^{\mathbf{y}}_{b,i} = \algVCLSOA^{b,k^{\mathbf{y}}_{b,i}}(B^b_i(\mathbf{y}))$.
  Recall (from Lemma~\ref{lem:VCL-SOA-zero-error-rate-bstar})
  that $k^{\mathbf{y}}_{b,i}$ is a ($B^b_i$-dependent) value in $\nats$
  and 
  $f^{\mathbf{y}}_{b,i}$ is a ($B^b_i$-dependent) 
  function $\X^{k^{\mathbf{y}}_{b,i}} \to \{0,1\}^{k^{\mathbf{y}}_{b,i}}$.
  Additionally define a function 
  $G^{\mathbf{y}}_{b,i} : \bigcup_{t=1}^{\infty} (\X \times \{0,1\})^t \to \{0,1\}$ 
  defined by the property that, 
  $\forall t \in \nats$, 
  $\forall (x_1,y_1),\ldots,(x_t,y_t) \in \X \times \{0,1\}$,
  $G^{\mathbf{y}}_{b,i}(x_1,y_1,\ldots,x_t,y_t) = 1$ iff both of the following hold:
  \begin{align*}
  1.~~ & \forall i,j \leq t \text{ with } i \neq j, \text{ if } x_i = x_j \text{ then } y_i = y_j, 
   \\ 2.~~ & \forall \text{ distinct } x_{i_1},\ldots,x_{i_{k^{\mathbf{y}}_{b,i}}} \in \{x_1,\ldots,x_t\}, f^{\mathbf{y}}_{b,i}(x_{i_1},\ldots,x_{i_{k^{\mathbf{y}}_{b,i}}}) \neq (y_{i_1},\ldots,y_{i_{k^{\mathbf{y}}_{b,i}}}).
  \end{align*}
  In particular, note that $G^{\mathbf{y}}_{b,i}(x_1,y_1,\ldots,x_t,y_t) = 1$ iff every distinct $i_1,\ldots, i_{k^{\mathbf{y}}_{b,i}} \in \{1,\ldots,t\}$ have $G^{\mathbf{y}}_{b,i}(x_{i_1},y_{i_1},\ldots,x_{i_1},y_{i_{k^{\mathbf{y}}_{b,i}}})=1$.
  The interpretation is that each $f^{\mathbf{y}}_{b,i}$ defines a set of \emph{forbidden patterns},
  and hence defines a partial concept class $\G^{\mathbf{y}}_{b,i}$
  of all (finite-support) partial concepts which avoid the forbidden patterns. 
  Formally, $\G^{\mathbf{y}}_{b,i}$ is comprised of all partial concepts $g : \X \to \{0,1,*\}$ with finite support
  $\mathrm{supp}(g) := \{ x \in \X : g(x) \neq * \}$ 
  such that $\forall t$, $\forall \{x_1,\ldots,x_t\} \subseteq \mathrm{supp}(g)$,
  $G^{\mathbf{y}}_{b,i}(x_1,g(x_1),\ldots,x_t,g(x_t)) = 1$.
  
  For each $b \in \{1,\ldots,n/4\}$ and $i \in I_b$, 
  define a function $G_{b,i} : \bigcup_{t=1}^{\infty} (\X \times \{0,1\})^t \to \{0,1\}$ 
  defined by the property that, 
  $\forall t \in \nats$, 
  $\forall (x_1,y_1),\ldots,(x_t,y_t) \in \X \times \{0,1\}$,
  \begin{equation*} 
  G_{b,i}(x_1,y_1,\ldots,x_t,y_t) = 1 \text{ iff } \exists \mathbf{y} \in \{0,1\}^b \text{ s.t. } G^{\mathbf{y}}_{b,i}(x_1,y_1,\ldots,x_t,y_t) = 1.
  \end{equation*}
  The function $G_{b,i}$ thus specifies a partial concept class 
  $\G_{b,i} := \bigcup_{\mathbf{y} \in \{0,1\}^b} \G^{\mathbf{y}}_{b,i}$: 
  the set of all finite-support partial concepts $g$ for which 
  at least one of the forbidden-pattern functions $f^{\mathbf{y}}_{b,i}$
  does not forbid any of the patterns in its support.
  Moreover, since each $k^{\mathbf{y}}_{b,i} \leq b$ (by its definition in Lemma~\ref{lem:VCL-SOA-zero-error-rate-bstar}), and $f^{\mathbf{y}}_{b,i}$ forbids at least one pattern on any
  distinct $k^{\mathbf{y}}_{b,i}$ points $x_{i_1},\ldots,x_{i_{k^{\mathbf{y}}_{b,i}}}$,
  each $\G^{\mathbf{y}}_{b,i}$ has $\VC(\G^{\mathbf{y}}_{b,i}) \leq b-1$.
  To bound $\VC(\G_{b,i})$, suppose there is a set
  $\{x_1,\ldots,x_t\}$ of size 
  $t \geq b-1$ shattered by $\G_{b,i}$.
  By Sauer's lemma \citep*{sauer:72,vapnik:71}, 
  each $\G^{\mathbf{y}}_{b,i}$ has at most 
  $\left(\frac{e t}{b-1}\right)^{b-1}$
  realizable classifications of $\{x_1,\ldots,x_t\}$.
  Since each of the $2^t$ classifications of $\{x_1,\ldots,x_t\}$ 
  realizable by $\G_{b,i}$ is realizable by at least 
  one $\G^{\mathbf{y}}_{b,i}$, $\mathbf{y} \in \{0,1\}^b$, 
  we have $2^t \leq 2^{b} \left(\frac{e t}{b-1}\right)^{b-1}$,
  which, by Lemma 4.6 of \citet*{vidyasagar:03}, 
  implies 
  $\VC(\G_{b,i}) \leq 2 b + 2 (b-1)\log_2(e) \leq 5 b$.

  Next we use $S_n^1$ to select an appropriate batch size $\hat{b}_n$ as follows.
  For each $b \in \{1,\ldots,n/4\}$ and $i \in I_b$,
  recall that we define 
  \begin{equation*}
    \G_{b,i}(S_n^1) = \left\{ (g(X_{1+n/4}),\ldots,g(X_{n/2})) : g \in \G_{b,i} \right\} \cap \{0,1\}^{n/4},
  \end{equation*}
  the set of \emph{total} classifications of $X_{1+n/4},\ldots,X_{n/2}$ realizable by $\G_{b,i}$.
%
  For each $b \!\in\! \{1,\ldots,n/4\}$, define $\GoodI(b)$ as the set of all $i \in I_b$ such that,
  $\G_{b,i}(S_n^1)$ is non-empty, and 
  $\forall b' > b$ with $b' \in \{1,\ldots,n/4\}$, $\forall i' \in I_{b'}$ for which $\G_{b',i'}(S_n^1)$ is non-empty,
  \begin{equation*}
    \min_{\substack{(y_1,\ldots,y_{n/4}) \\ \in \G_{b,i}(S_n^1)}} \frac{4}{n} \sum_{t=1}^{n/4} \ind[ y_t \neq Y_{t+n/4} ]
    \leq \min_{\substack{(y_1,\ldots,y_{n/4}) \\ \in \G_{b',i'}(S_n^1)}} \frac{4}{n} \sum_{t=1}^{n/4} \ind[ y_t \neq Y_{t+n/4} ]
    + 8 \sqrt{\frac{b'\log(n)}{n}}.
  \end{equation*}
  Define $\hat{b}_n$ to be the minimum $b \in \{1,\ldots,n/4\}$ such that $|\GoodI(b)| \geq \frac{9}{10} |I_b|$,
  if such a $b$ exists,
  and otherwise simply define $\hat{h}_n = \hat{h}_n^0$ and in this case the algorithm returns this classifier $\hat{h}_n$.

  Supposing $\hat{b}_n$ exists, 
  the interpretation is that, for $b = \hat{b}_n$, at least $\frac{9}{10}$ of the partial concept classes $\G_{b,i}$, $i \in I_b$, are not verifiably worse
  (in their ability to fit the data $S_n^1$) than some other partial concept class $\G_{b',i'}$ with $b' > b$,
  where this verification comes via a uniform convergence bound
  we establish below 
  (using the fact that $\VC(\G_{b,i})$ and $\VC(\G_{b',i'})$ are both at most $5 b'$)
  holding uniformly over all such comparisons 
  with probability at least $1 - \frac{1}{n}$.

  Based on this choice of $\hat{b}_n$,
  for each $i \in I_{\hat{b}_n}$,
  let $\hat{f}_{n/4}^{i}$ denote the function $\hat{f}_{n/4}^{\G_{\hat{b}_n,i}}$ 
  from Lemma~\ref{lem:partial-VC-super-root}, defined for the partial concept class $\G_{\hat{b}_n,i}$. 
  To define $\hat{h}_n^1$, 
  we will use $S_n^2$ as the training set for the learning 
  algorithms $\hat{f}_{n/4}^{i}, i \in I_{\hat{b}_n}$, 
  and aggregate the resulting predictors by a majority vote.
  Formally, for any $x \in \X$, we define
  \begin{equation*}
    \hat{h}_n^1(x) = \ind\!\left[ \sum_{i \in I_{\hat{b}_n}} \hat{f}_{n/4}^{i}(S_n^2,x) \geq \frac{1}{2} |I_{\hat{b}_n}| \right],  
  \end{equation*}
  the majority vote of the $\hat{f}_{n/4}^{i}(S_n^2,x)$ predictions.

  To complete the definition of the algorithm, we define $\hat{h}_n$ by selecting one of $\hat{h}_n^0$ or $\hat{h}_n^1$ using the remaining data $S_n^3$, as follows.
  If
  \begin{equation*}
    \hat{\er}_{S_n^3}\!\left(\hat{h}_n^1\right) \leq \hat{\er}_{S_n^3}\!\left(\hat{h}_n^0\right) + \sqrt{\frac{32\ln(n)}{n}},
  \end{equation*}
  define $\hat{h}_n = \hat{h}_n^1$,
  and otherwise define $\hat{h}_n = \hat{h}_n^0$.
  In either case, the algorithm returns the classifier $\hat{h}_n$.
  The algorithm is summarized in Figure~\ref{fig:VCL-alg}.
  In particular, we note that the algorithm may be expressed 
  as a universally measurable function $f_n : (\X \times \{0,1\})^n \times \X \to \{0,1\}$ 
  mapping $S_n$ and a test point $X$ to a prediction $\{0,1\}$, 
  since evaluation of $\hat{h}_n^1(x)$ ultimately 
  requires only the use of 
  the universally measurable functions
  $B \mapsto k^b_{\VCL}(B)$ and $(B,x_1,\ldots,x_{k^b_{\VCL}(B)}) \mapsto \algVCLSOA^{b,k^b_{\VCL}(B)}(B,x_1,\ldots,x_{k^b_{\VCL}(B)})$,
  and $\hat{h}_n^0(x)$ is a universally measurable function of 
  $(S_n^0,x)$.

  \begin{figure}
  \begin{bigboxit}
    1. Let $\hat{h}_n^0$ be returned by a universally Bayes-consistent learner trained on $S_n^0$
    \\2. Let $\hat{b}_n = \min\!\left\{ b \leq n/4 : |\GoodI(b)| \geq \frac{9}{10} |I_b| \right\}$ {\small (if $\hat{b}_n$ does not exist, return $\hat{h}_n = \hat{h}_n^0$)}
    \\3. For each $i \leq n/(4\hat{b}_n)$, let $\hat{f}^i_{n/4}(S_n^2) = \hat{f}_{n/4}^{\G_{\hat{b}_n,i}}(S_n^2)$ (from Lemma~\ref{lem:partial-VC-super-root})
    \\4. Let $\hat{h}_n^1 = \mathrm{Majority}\!\left(\hat{f}^i_{n/4}(S_n^2) : i \leq n/(4\hat{b}_n)\right)$ 
    \\5. If $\hat{\er}_{S_n^3}(\hat{h}_n^1) \leq \hat{\er}_{S_n^3}(\hat{h}_n^0) + \sqrt{\frac{32\ln(n)}{n}}$, return $\hat{h}_n = \hat{h}_n^1$, else return $\hat{h}_n = \hat{h}_n^0$
  \end{bigboxit}
  \caption{Algorithm achieving $o(n^{-1/2})$ rate for classes $\H$ with no infinite VCL tree.}
  \label{fig:VCL-alg}
  \end{figure}

  With the definition of the returned classifier $\hat{h}_n$ now complete,
  we proceed to prove that it satisfies the required rate $o(n^{-1/2})$.
  Let $E_1$ denote the event that, either $\hat{h}_n^1$ is undefined
  (due to the early termination event in the algorithm, when the criterion for $\hat{b}_n$ cannot be satisfied)
  or else 
  \begin{equation*}
  \forall j \in \{0,1\}, \left| \hat{\er}_{S_n^3}(\hat{h}_n^j) - \er_{\PXY}(\hat{h}_n^j) \right| \leq \sqrt{\frac{2\ln(n)}{n}}.
  \end{equation*}
  Since $S_n^3$ is independent of $(S_n^0,S_n^1,S_n^2)$ (from which $\hat{h}_n^0$ and $\hat{h}_n^1$ are derived), 
  Hoeffding's inequality (applied under the conditional distribution given $\hat{h}_n^0$ and $\hat{h}_n^1$)
  and a union bound, with the law of total probability, imply that $E_1$ holds with probability at least $1 - \frac{4}{n}$.
    
  We first consider the case where $\PXY$ is \emph{not} Bayes-realizable 
  with respect to $\H$: that is, 
  $\inf_{h \in \H} \er_{\PXY}(h) > \inf_{h} \er_{\PXY}(h)$.
  In this case, let $\epsilon > 0$ be such that $\inf_{h \in \H} \er_{\PXY}(h) > \inf_{h} \er_{\PXY}(h) + \epsilon$.
  Since the $o(n^{-1/2})$ rate is asymptotic in nature,
  we may focus on the case that $n \geq \frac{2^{10}}{\epsilon^2}\ln\!\left(\frac{2^{9}}{\epsilon^2}\right)$,
  so that $\sqrt{\frac{2\ln(n)}{n}} \leq \frac{\epsilon}{16}$.
  On the event $E_1$ (and the event that $\hat{h}_n^1$ is defined), if additionally,
  $\er_{\PXY}(\hat{h}_n^0) < \inf_{h \in \H} \er_{\PXY}(h) - \frac{\epsilon}{2}$
  and $\er_{\PXY}(\hat{h}_n^1) > \inf_{h \in \H} \er_{\PXY}(h) - \frac{\epsilon}{8}$,
  then
  \begin{align*}
    \hat{\er}_{S_n^3}\!\left(\hat{h}_n^1\right)
    & \geq \er_{\PXY}\!\left(\hat{h}_n^1\right) - \frac{\epsilon}{16}
    > \inf_{h \in \H} \er_{\PXY}(h) - \frac{3\epsilon}{16} 
    \\ & > \er_{\PXY}\!\left(\hat{h}_n^0\right) + \frac{5\epsilon}{16}
    \geq \hat{\er}_{S_n^3}\!\left(\hat{h}_n^0\right) + \frac{\epsilon}{4}
    \geq \hat{\er}_{S_n^3}\!\left(\hat{h}_n^0\right) + \sqrt{\frac{32\ln(n)}{n}},
  \end{align*}
  so that the algorithm chooses $\hat{h}_n = \hat{h}_n^0$.
  On the other hand, if
  $\er_{\PXY}(\hat{h}_n^0) < \inf_{h \in \H} \er_{\PXY}(h) - \frac{\epsilon}{2}$
  and $\er_{\PXY}(\hat{h}_n^1) \leq \inf_{h \in \H} \er_{\PXY}(h) - \frac{\epsilon}{8}$,
  then regardless of whether $\hat{h}_n = \hat{h}_n^0$ or $\hat{h}_n = \hat{h}_n^1$,
  we have $\er_{\PXY}(\hat{h}_n) \leq \inf_{h \in \H} \er_{\PXY}(h) - \frac{\epsilon}{8}$.
  Moreover, if $\hat{h}_n^1$ is not defined, then the algorithm chooses $\hat{h}_n = \hat{h}_n^0$.
  Thus, in any case, on the event $E_1$, if $\er_{\PXY}(\hat{h}_n^0) < \inf_{h \in \H} \er_{\PXY}(h) - \frac{\epsilon}{2}$,
  we always have $\er_{\PXY}(\hat{h}_n) \leq \inf_{h \in \H} \er_{\PXY}(h) - \frac{\epsilon}{8}$.
  Since we always have $\er_{\PXY}(\hat{h}_n) \leq 1$, this implies 
  \begin{equation}
  \label{eqn:super-root-h0-case}
    \E\!\left[ \er_{\PXY}(\hat{h}_n) \right]
    \leq \inf_{h \in \H} \er_{\PXY}(h) - \frac{\epsilon}{8} + \P\!\left( \er_{\PXY}(\hat{h}_n^0) \geq \inf_{h \in \H} \er_{\PXY}(h) - \frac{\epsilon}{2} \right) + \left( 1 - \P(E_1) \right).
  \end{equation}
  Recall we have $1 - \P(E_1) \leq \frac{4}{n}$.
  Moreover, noting that $\er_{\PXY}(\hat{h}_n^0) - \inf_{h} \er_{\PXY}(h) \geq 0$, 
  \begin{equation*} 
  \P\!\left( \er_{\PXY}(\hat{h}_n^0) \!\geq\! \inf_{h \in \H} \er_{\PXY}(h) - \frac{\epsilon}{2} \right)
  \!\leq \P\!\left( \er_{\PXY}(\hat{h}_n^0) - \inf_{h} \er_{\PXY}(h) \!>\! \frac{\epsilon}{2} \right)
  \!\leq \frac{2}{\epsilon} \E\!\left[ \er_{\PXY}(\hat{h}_n^0)  - \inf_{h} \er_{\PXY}(h) \right]
  \end{equation*}
  where the last inequality follows from Markov's inequality, 
  Since the Bayes consistency property of $\hat{h}_n^0$ guarantees
  $\lim_{n \to \infty} \E\!\left[ \er_{\PXY}(\hat{h}_n^0) - \inf_{h} \er_{\PXY}(h) \right] = 0$, 
  and clearly $\lim_{n \to \infty} \frac{4}{n} = 0$, 
  we have for all sufficiently large $n$ 
  the right hand side of \eqref{eqn:super-root-h0-case} is at most 
  $\inf_{h \in \H} \er_{\PXY}(h) - \frac{\epsilon}{16}$.
  Thus, for all sufficiently large $n$,
  \begin{equation*}
    \E\!\left[ \er_{\PXY}(\hat{h}_n) \right] - \inf_{h \in \H} \er_{\PXY}(h) < 0 = o\!\left(n^{-1/2}\right).
  \end{equation*}

  To complete the proof, we now turn to the remaining case, in which 
  $\PXY$ is Bayes-realizable with respect to $\H$: 
  that is, 
  $\inf_{h \in \H} \er_{\PXY}(h) = \inf_{h} \er_{\PXY}(h)$, 
  and henceforth assume this to be the case.
  Since $\H$ does not shatter an infinite VCL tree, 
  Lemmas~\ref{lem:no-infinite-VCL-implies-UGC}, \ref{lem:UGC-totally-bounded}, and \ref{lem:limit-function} 
  imply that
  there exists a measurable function $\target : \X \to \{0,1\}$ with
  $\er_{\PXY}(\target) = \inf_{h \in \H} \er_{\PXY}(h)$ 
  and $\inf_{h \in \H} \Px( x : h(x) \neq \target(x) ) = 0$,
  where $\Px$ denotes the marginal of $\PXY$ on $\X$.
  In particular, let $E_2$ denote the event that 
  $\forall t \leq n/2$, 
  $\P(Y_t=\target(X_t)|X_t) \geq \frac{1}{2}$.
  As is well known, the fact that $\er_{\PXY}(\target) = \inf_h \er_{\PXY}(h)$ implies $E_2$ is satisfied with probability one.
  Moreover, note that because $\inf_{h \in \H} P_X( x : h(x) \neq \target(x) ) = 0$,
  letting $P_{\target}$ denote the distribution of $(X,\target(X))$, for $X \sim P_X$,
  we have that $P_{\target}$ is realizable with respect to $\H$.

  Let $b^* := b^*_{99/100} \in \nats$ be the value defined in 
  Lemma~\ref{lem:VCL-SOA-zero-error-rate-bstar} 
  under the realizable distribution $P_{\target}$ (and $\gamma = \frac{99}{100}$).
  For the remainder of the proof, suppose $n \geq 4 b^*$ (w.l.o.g., since the $o(n^{-1/2})$ claim is asymptotic in nature).
  We first argue that $\hat{b}_n \leq b^*$ with high probability.
  For each $i \in I_{b^*}$,
  let $\mathbf{y}^*_{i} = (\target(X_{1+(i-1)b^*}),\ldots,\target(X_{i b^*}))$.
  Then note that, for each such $i$, 
  the distribution $P_{\target}$ will be Bayes-realizable 
  with respect to the partial concept class $\G^{\mathbf{y}^*_{i}}_{b^*,i}$
  if, for $k = k^{\mathbf{y}^*_i}_{b^*,i}$,  
  $\forall t \in \nats$, 
  \begin{align*}
  0 & = \Px^{t}\!\left( (x_1,\ldots,x_t) \in \X^t : G^{\mathbf{y}^*_i}_{b^*,i}(x_1,\target(x_1),\ldots,x_t,\target(x_t)) = 0 \right)
  \\ & = \Px^{t}\Big( (x_1,\ldots,x_t) \in \X^t : \exists \text{ distinct } x_{i_1},\ldots,x_{i_{k}} \in \{x_1,\ldots,x_t\} \text{ s.t. }
  \\ & {\hskip 44mm}f^{\mathbf{y}^*_i}_{b^*,i}(x_{i_1},\ldots,x_{i_k}) = (\target(x_{i_1}),\ldots,\target(x_{i_k})) \Big),
  \end{align*}
  where the last equality is by definition of $G^{\mathbf{y}^*_i}_{b^*,i}$.
  Since there are only finitely many choices of distinct 
  $i_1,\ldots,i_k \in \{1,\ldots,t\}$, 
  by the union bound the quantity above is $0$ if 
  \begin{equation}
  \label{eqn:zero-prob-pattern}
    \Px^{k}\!\left( (x_1,\ldots,x_k) \in \X^k :  f^{\mathbf{y}^*_i}_{b^*,i}(x_{1},\ldots,x_{k}) = (\target(x_{1}),\ldots,\target(x_{k})) \right) = 0.
  \end{equation}
  Recalling that $k = k^{\mathbf{y}^*_i}_{b^*,i} = k^{b^*}_{\VCL}(B^{b^*}_{i}(\mathbf{y}^*_i))$ 
  and $f^{\mathbf{y}^*_i}_{b^*,i} = \algVCLSOA^{b^*,k}(B^{b^*}_{i}(\mathbf{y}^*_i))$,
  and noting that $B^{b^*}_{i}(\mathbf{y}^*_i) \sim P_{\target}^{b^*}$,
  and that $(X,Y) \sim P_{\target}$ has $Y=\target(X)$ (by definition),
  the defining property of $b^*$ from Lemma~\ref{lem:VCL-SOA-zero-error-rate-bstar} 
  implies \eqref{eqn:zero-prob-pattern} holds with probability 
  greater than $\frac{99}{100}$.
  Altogether, we have 
  \begin{equation*}
    \P\!\left( P_{\target} \text{ is Bayes-realizable with respect to } \G^{\mathbf{y}^*_{i}}_{b^*,i} \right) > \frac{99}{100}.
  \end{equation*}

  Based on this fact, 
  we will argue that $\hat{b}_n \leq b^*$ 
  with high probability. 
  Let $E'_3$ denote the event that the set 
  \begin{equation*}
    I^{\target}_{b^*} = \left\{ i \in I_{b^*} : P_{\target} \text{ is Bayes-realizable with respect to } \G^{\mathbf{y}^*_i}_{b^*,i} \right\}
  \end{equation*}
  satisfies $|I^{\target}_{b^*}| \geq \frac{9}{10} |I_{b^*}|$.
  Noting that the data sets $B^{b^*}_{i}(\mathbf{y}^*_i)$, $i \in I_{b^*}$, are independent,
  a Chernoff bound implies that
  the event $E'_3$ has probability at least $1 - e^{-C_1 n}$, 
  where $C_1 = \frac{99}{100} \frac{1}{11^2} \frac{1}{2} \frac{1}{8b^*}$.
  Thus, to show $\hat{b}_n \leq b^*$, on this event $E'_3$ 
  it suffices to argue that $\GoodI(b^*) \supseteq I^{\target}_{b^*}$.
  Additionally, denote by $E''_3$ the event that 
  every $i \in I^{\target}_{b^*}$ has 
  $(\target(X_{1+n/4}),\ldots,\target(X_{n/2}))$ $\in \G_{b^*,i}^{\mathbf{y}^*_i}(S_n^1)$.
  Note that, conditional on $S_n^0$, for $i \in I^{\target}_{b^*}$, 
  the definitions of $P_{\target}$ and Bayes-realizability of 
  $P_{\target}$ with respect to $\G_{b^*,i}^{\mathbf{y}^*_i}$
  imply that 
  with conditional probability one given $S_n^0$
  we have $(\target(X_{1+n/4}),\ldots,\target(X_{n/2})) \in \G_{b^*,i}^{\mathbf{y}^*_i}(S_n^1)$.
  Therefore, by the union bound and law of total probability, 
  $E''_3$ has probability one.
  Denote by $E_3 = E'_3 \cap E''_3$,
  and note that, by the union bound, 
  $E_3$ has probability at least $1 - e^{-C_1 n}$.
  Also note that one implication of $E_3$ is that 
  every $i \in I^{\target}_{b^*}$ has $\G_{b^*,i}(S_n^1) \neq \emptyset$
  (one of the properties necessary for establishing $\GoodI(b^*) \supseteq I^{\target}_{b^*}$).
  
  As in Section~\ref{sec:proof-of-lemma-partial-VC-super-root}, 
  it will simplify the discussion to introduce convenient notation in discussing the partial concept classes $\G_{b,i}$, 
  specifically in the context of the projections 
  $\G_{b,i}(S_n^1)$ to the data set $S_n^1$.
  For any $g \in \{0,1\}^{n/4}$, 
  to make explicit the correspondence of entries of $g$ 
  to the variables $X_{t+n/4}$ in $S_n^1$, 
  we denote the $t^{\mathrm{th}}$ entry of $g$ by $g(X_{t+n/4})$ (for $t \in \{1,\ldots,n/4\}$): 
  that is, $g =: (g(X_{1+n/4}),\ldots,g(X_{n/2}))$.
  Additionally, 
  denote by $\hat{\er}_{S_n^1}(g) = \frac{4}{n} \sum_{t=1}^{n/4} \ind[ g(X_{t+n/4}) \neq Y_{t+n/4} ]$
  (the empirical error rate of $g$ on $S_n^1$), 
  and denote by $\ber_{n}^{1}(g) = \frac{4}{n} \sum_{t=1}^{n/4} \P( Y_{t+n/4} \neq g(X_{t+n/4}) | X_{t+n/4} )$
  (the semi-empirical error rate of $g$ on $S_n^1$).
  Also extend this notation to $\target$, 
  defining $\ber_{n}^{1}(\target) = \frac{4}{n} \sum_{t=1}^{n/4} \P( Y_{t+n/4} \neq \target(X_{t+n/4}) | X_{t+n/4} )$.
  In particular, in this notation, 
  for each $b \in \{1,\ldots,n/4\}$, 
  the set $\GoodI(b)$ is comprised of all $i \in I_b$ 
  such that $\G_{b,i}(S_n^1)$ is non-empty,
  and $\forall b' > b$ with $b' \in \{1,\ldots,n/4\}$, 
  $\forall i' \in I_{b'}$, 
  for which $\G_{b',i'}(S_n^1)$ is non-empty, 
  \begin{equation*}
    \min_{g \in \G_{b,i}(S_n^1)} \hat{\er}_{S_n^1}(g)
    \leq \min_{g \in \G_{b',i'}(S_n^1)} \hat{\er}_{S_n^1}(g) + 8\sqrt{\frac{b' \log(n)}{n}}.
  \end{equation*}

  For every $b \in \{1,\ldots,n/4\}$ and $i \in I_b$,
  applying Hoeffding's inequality under the conditional distribution 
  given $\{X_1,\ldots,X_{n/2}\}$ implies that,
  for each $g \in \G_{b,i}(S_n^1)$, 
  with conditional probability at least $1 - 2 n^{-8b}$,
  \begin{equation}
  \label{eqn:conditional-uniform-convergence}
    \left| \hat{\er}_{S_n^1}(g) - \ber_n^1(g) \right| \leq 
    4\sqrt{\frac{b \log(n)}{n}}.
  \end{equation}
  Since $\VC(\G_{b,i}) \leq 5b$, 
  Sauer's lemma \citep*{sauer:72,vapnik:71} implies that 
  $|\G_{b,i}(S_n^1)| \leq n^{5b}$, 
  so that the union bound, 
  together with the law of total probability,  
  imply that with probability at least $1 - 2 n^{-3b}$,
  the inequality \eqref{eqn:conditional-uniform-convergence} holds
  simultaneously for every $g \in \G_{b,i}(S_n^1)$.
  Let $E_4$ denote the event 
  that \eqref{eqn:conditional-uniform-convergence} holds simultaneously for all 
  $b \in \{1,\ldots,n/4\}$, 
  $i \in I_b$,
  and $g \in \G_{b,i}(S_n^1)$.
  Noting that 
  $\sum_{b \leq n/4} \sum_{i \leq n/(4b)} 2 n^{-3b} \leq \frac{1}{n}$, 
  the union bound implies $E_4$ has probability at least 
  $1 - \frac{1}{n}$.

  On the event $E_2 \cap E_3 \cap E_4$, for all $i \in I^{\target}_{b^*}$, it holds that
  \begin{equation}
  \label{eqn:b-star-vs-target}
    \min_{g \in \G_{b^*,i}(S_n^1)} \hat{\er}_{S_n^1}(g) - 4 \sqrt{\frac{b^* \log(n)}{n}}
    \leq \min_{g \in \G_{b^*,i}(S_n^1)} \ber_n^1(g) 
    \leq \ber_{n}^{1}(\target)
    = \min_{g \in \{0,1\}^{n/4}} \ber_n^1(g),
  \end{equation}
  where the first inequality is due to $E_4$, 
  the second inequality is due to $i \in I^{\target}_{b^*}$ 
  and the event $E_3$, 
  and the last equality is due to the event $E_2$.
  Moreover, on the event $E_4$,
  any $b' \in \{1,\ldots,n/4\}$ and $i' \in I_{b'}$ with $\G_{b',i'}(S_n^1) \neq \emptyset$ have 
  \begin{equation*}
    \min_{g \in \G_{b',i'}(S_n^1)} \hat{\er}_{S_n^1}(g) + 4 \sqrt{\frac{b' \log(n)}{n}}
    \geq \min_{g \in \G_{b',i'}(S_n^1)} \ber_n^1(g)
    \geq \min_{g \in \{0,1\}^{n/4}} \ber_n^1(g).
  \end{equation*}
  Together, these facts imply that, on the event $E_2 \cap E_3 \cap E_4$,
  every $i \in I^{\target}_{b^*}$, every $b' \in \{1,\ldots,n/4\}$ with $b' > b^*$, and every $i' \in I_{b'}$ 
  with $\G_{b',i'}(S_n^1) \neq \emptyset$ satisfy
  \begin{equation*}
    \min_{g \in \G_{b^*,i}(S_n^1)} \hat{\er}_{S_n^1}(g) 
    \leq \min_{g \in \G_{b',i'}(S_n^1)} \hat{\er}_{S_n^1}(g)
    + 8 \sqrt{\frac{b' \log(n)}{n}}.
  \end{equation*}
  Thus, on the event $E_2 \cap E_3 \cap E_4$, 
  we have $\GoodI(b^*) \supseteq I^{\target}_{b^*}$
  and $|\GoodI(b^*)| \geq |I^{\target}_{b^*}| \geq \frac{9}{10} |I_{b^*}|$,
  so that $\hat{b}_n \leq b^*$.  
  In particular, on this event, the function $\hat{h}_n^1$ is defined.

  Having established that $\hat{b}_n \leq b^*$, 
  we next need to argue that any ``bad'' $b < b^*$ 
  will \emph{not} be selected as $\hat{b}_n$.
  Recall our definition of \emph{Bayes-realizable} 
  for partial concept classes, introduced in Section~\ref{sec:subsection-partial-concepts-super-root-upper-bound}.
Let $\PrbBadB$ denote the set of all $b < b^*$ for which 
  \begin{equation*}
    \P\!\left( \PXY \text{ is Bayes-realizable with respect to } \G_{b,1} \right) < \frac{7}{10}. 
  \end{equation*}
  We aim to show that $\hat{b}_n \notin \PrbBadB$.
  If $\PrbBadB = \emptyset$, this trivially holds;
  we focus the next portion of the proof on the nontrivial 
  case that $\PrbBadB \neq \emptyset$.
   
  Let $S'_{< \infty} = \{(X'_1,Y'_1),(X'_2,Y'_2),\ldots\}$ be an infinite sequence of independent $\PXY$-distributed random variables 
  (with $S'_{< \infty}$ independent of $S_n^0$), 
  and for any $m \in \nats$ 
  denote by $S'_m = \{(X'_1,Y'_1),\ldots,(X'_m,Y'_m)\}$.
  For any $b \in \PrbBadB$, by definition, 
  if $\PXY$ is \emph{not} Bayes-realizable with respect to $\G_{b,1}$, 
  then $\exists m(\G_{b,1}) \in \nats$
  and $\Delta(\G_{b,1}) > 0$ (both independent of $S'_{< \infty}$)
  such that,
  denoting by $S'_{b,1} = S'_{m(\G_{b,1})}$ (which is conditionally distributed as $\PXY^{m(\G_{b,1})}$ given $\G_{b,1}$), 
  \begin{equation*}
    \P\!\left( \forall (y_1,\ldots,y_{m(\G_{b,1})}) \in \G_{b,1}(S'_{b,1}), \exists t \leq m(\G_{b,1}) \text{ s.t. } \P(Y'_t=y_t|X'_t) < \frac{1}{2} \middle| \G_{b,1} \right) > \Delta(\G_{b,1}).
  \end{equation*}
  In particular, if we define (for any $m \in \nats$ and $i \in I_b$)
  \begin{equation*}
      \epsilon(\G_{b,i},S'_{m}) = \min_{(y_1,\ldots,y_{m}) \in \G_{b,i}(S'_{m})} \max_{1 \leq t \leq m} \P(Y'_t \neq y_t | X'_t) - \frac{1}{2},
  \end{equation*}
  or $\epsilon(\G_{b,i},S'_{m}) = 1$ if $\G_{b,i}(S'_{m}) = \emptyset$, 
  then since $\G_{b,1}(S'_{b,1})$ has finite size, 
  the above can equivalently be stated as 
  \begin{equation*}
      \P\Big( \epsilon(\G_{b,1},S'_{b,1}) > 0 \Big| \G_{b,1} \Big) > \Delta(\G_{b,1}).
  \end{equation*}
  Due to the strict inequality in the event $\epsilon(\G_{b,1},S'_{b,1}) > 0$, 
  this further implies $\exists \epsilon(\G_{b,1}) > 0$
  such that 
  \begin{equation*}
      \P\Big( \epsilon(\G_{b,1},S'_{b,1}) > \epsilon(\G_{b,1}) \Big| \G_{b,1} \Big) > \frac{\Delta(\G_{b,1})}{2}.
  \end{equation*}

  For each $b \in \PrbBadB$, 
  let $m^*_b \in \nats$, $\epsilon^*_b \in (0,1]$, and $\Delta^*_b \in (0,1]$
  be such that
  \begin{align*}
      & \P\Big( m^*_b \geq m(\G_{b,1})  \Big| \PXY \text{ not Bayes-realizable wrt } \G_{b,1} \Big) > \frac{17}{18}
      \\ & \P\Big( \epsilon(\G_{b,1}) \geq \epsilon^*_b \Big| \PXY \text{ not Bayes-realizable wrt } \G_{b,1} \Big) > \frac{17}{18}
      \\ \text{and } & \P\Big( \Delta(\G_{b,1}) \geq \Delta^*_b \Big| \PXY \text{ not Bayes-realizable wrt } \G_{b,1} \Big) > \frac{17}{18},
  \end{align*}
  so that, by the union bound, 
  \begin{equation*}
      \P\Big( m^*_b \!\geq\! m(\G_{b,1}) \text{ and } \epsilon(\G_{b,1}) \!\geq\! \epsilon^*_b \text{ and } \Delta(\G_{b,1}) \!\geq\! \Delta^*_b \Big| \PXY \text{ not Bayes-realizable wrt } \G_{b,1} \Big) > \frac{5}{6}.
  \end{equation*}
  Let 
  \begin{equation*} 
  m^* = \max\limits_{b \in \PrbBadB} m^*_b,~~~~~ 
  \epsilon^* = \min\limits_{b \in \PrbBadB} \epsilon^*_b,~~~~
  \text{ and }~ \Delta^* = \min\limits_{b \in \PrbBadB} \Delta^*_b,
  \end{equation*}
  and note that since $\PrbBadB$ is finite (being a subset of $\{1,\ldots,b^*-1\}$), we have 
  $m^* < \infty$, $\epsilon^* > 0$, and $\Delta^* > 0$.
  By the above properties, for any $b \in \PrbBadB$, 
  \begin{equation*}
      \P\!\left( m^* \!\geq\! m(\G_{b,1}) \text{ and } 
      \P\Big( \epsilon(\G_{b,1},S'_{b,1}) \!>\! \epsilon^* \Big| \G_{b,1} \Big) \!>\! \frac{\Delta^*}{2} \middle| \PXY \text{ not Bayes-realizable wrt } \G_{b,1} \!\right) \!>\! \frac{5}{6}.
  \end{equation*}
  Also note that $\epsilon(\G_{b,1},S'_{m})$ is non-decreasing in $m$, 
  so that the above implies 
  \begin{align*}
      \frac{5}{6} & < \P\!\left( m^* \!\geq\! m(\G_{b,1}) \text{ and } 
      \P\Big( \epsilon(\G_{b,1},S'_{m^*}) \!>\! \epsilon^* \Big| \G_{b,1} \Big) \!>\! \frac{\Delta^*}{2} \middle| \PXY \text{ not Bayes-realizable wrt } \G_{b,1} \!\right)
      \\ & \leq \P\!\left( \P\Big( \epsilon(\G_{b,1},S'_{m^*}) > \epsilon^* \Big| \G_{b,1} \Big) > \frac{\Delta^*}{2} \middle| \PXY \text{ not Bayes-realizable wrt } \G_{b,1} \!\right).
  \end{align*}
  Since (by definition) any $b \in \PrbBadB$ has probability greater than $\frac{3}{10}$ that $\PXY$ is not Bayes-realizable with respect to $\G_{b,1}$,
  and since the data sets $B^b_i$, $i \in I_b$, 
  are identically distributed,
  we can further conclude that 
  $\forall b \in \PrbBadB$, $\forall i \in I_b$, 
  \begin{equation*}
    \P\!\left( \P\Big( \epsilon(\G_{b,i},S'_{m^*}) > \epsilon^* \Big| \G_{b,i} \Big) > \frac{\Delta^*}{2} \right) > \frac{1}{4}.
  \end{equation*}

  For each $b \in \PrbBadB$, letting
  \begin{equation*}
    J_b = \left\{ i \in I_b : \P\Big( \epsilon(\G_{b,i},S'_{m^*}) > \epsilon^* \Big| \G_{b,i} \Big) > \frac{\Delta^*}{2} \right\},
  \end{equation*}
  since the data sets $B^b_i$, $i \in I_b$, 
  are independent,  
  a Chernoff bound implies that, 
  with probability at least $1 - e^{-C_2 n}$, 
  where $C_2=\frac{1}{64 \cdot 25 \cdot b^*}$,
  it holds that $|J_b| \geq \frac{1}{5} |I_b|$.
  Let $E_5$ denote the event that $|J_b| \geq \frac{1}{5} |I_b|$ holds simultaneously for all $b \in \PrbBadB$.
  By the union bound, $E_5$ has probability at least $1 - (b^*-1)e^{-C_2 n}$.

  Since the $o\!\left(n^{-1/2}\right)$ rate guarantee is asymptotic in nature,
  we may focus on the case $n \geq 4 m^*$; we suppose this to be the case for the remainder of the proof.
  Let $t_0 = n/4$, 
  and for each $j \in \{1,\ldots, \lfloor n/(4 m^*) \rfloor\}$, 
  let $t_j = j m^* + n/4$,
  and for each such $j$, let 
  $S_{m^*,j} = \left\{(X_{1+t_{j-1}},Y_{1+t_{j-1}}),\ldots,(X_{t_j},Y_{t_j})\right\}$.
  Note that $S_{m^*,1},\ldots,S_{m^*,\lfloor n/(4m^*) \rfloor}$
  are independent (and independent of $S_n^0$), 
  each with distribution the same as that of $S'_{m^*}$ defined above.
  Moreover, defining $\epsilon(\G_{b,i},S_{m^*,j})$ as above (i.e., by taking $S'_{m^*} = S_{m^*,j}$ above), 
  for any $b \in \PrbBadB$ and $i \in I_b$, 
  on the event $E_2$, 
  if $\G_{b,i}(S_n^1) \neq \emptyset$, we have  
  \begin{align*}
    & \min_{g \in \G_{b,i}(S_n^1)} \ber_n^1(g) - \ber_n^1(\target) 
    \\ & = \min_{g \in \G_{b,i}(S_n^1)} \frac{4}{n} \sum_{t=1+n/4}^{n/2} \left( \P\!\left( Y_t \neq g(X_t) \middle| X_t \right) - \P\!\left( Y_t \neq \target(X_t) \middle| X_t \right) \right)
    \\ & = \min_{g \in \G_{b,i}(S_n^1)} \frac{4}{n} \sum_{t=1+n/4}^{n/2} \ind\!\left[ \P\!\left( Y_t \neq g(X_t) \middle| X_t \right) > \frac{1}{2} \right] 2 \left( \P\!\left( Y_t \neq g(X_t) \middle| X_t \right) - \frac{1}{2} \right)
    \\ & \geq \frac{4}{n} \sum_{j=1}^{\lfloor n/(4m^*) \rfloor}\! \min_{(y_{1+t_{j-1}},\ldots,y_{t_{j}}) \in \G_{b,i}(S_{m^*,j})} \sum_{t=1+t_{j-1}}^{t_{j}} \!\ind\!\left[ \P\!\left( Y_t \neq y_t \middle| X_t \right) > \frac{1}{2} \right] 2 \left( \P\!\left( Y_t \neq y_t \middle| X_t \right) - \frac{1}{2} \right)
    \\ & \geq \frac{4}{n} \sum_{j=1}^{\lfloor n/(4m^*) \rfloor} 2 \epsilon^* \ind\!\left[ \epsilon(\G_{b,i},S'_{m^*,j}) > \epsilon^*\right].
  \end{align*}
  Given $S_n^0$, 
  for any $b \in \PrbBadB$ and $i \in J_b$,
  applying a Chernoff bound under the conditional distribution given $S_n^0$ (from which $\G_{b,i}$ is derived, independent from $S_n^1$),  
  we have that with conditional probability at least $1 - e^{-C_3 n}$, where $C_3 = \frac{\Delta^*}{128 m^*}$,
  \begin{equation*}
    \sum_{j=1}^{\lfloor n/(4m^*) \rfloor} \ind\!\left[ \epsilon(\G_{b,i},S'_{m^*,j}) > \epsilon^*\right] \geq \frac{\Delta^*}{4} \left\lfloor \frac{n}{4m^*} \right\rfloor.
  \end{equation*}
  Letting $E_6$ denote the event that this inequality is satisfied for every $b \in \PrbBadB$ and $i \in J_b$,
  the union bound and the law of total probability together imply that $E_6$ holds with probability at least
  $1 - (b^*-1) \frac{n}{4} e^{-C_3 n}$.
  Altogether, on the event $E_2 \cap E_6$, 
  each $b \in \PrbBadB$ and $i \in J_b$ 
  with $\G_{b,i}(S_n^1) \neq \emptyset$
  satisfy 
  \begin{equation*}
     \min_{g \in \G_{b,i}(S_n^1)} \ber_n^1(g) - \ber_n^1(\target)
     \geq \frac{8 \epsilon^*}{n} \frac{\Delta^*}{4} \left\lfloor \frac{n}{4m^*} \right\rfloor
     \geq \frac{\epsilon^* \Delta^*}{4m^*}.
  \end{equation*}

  To complete this argument (toward the conclusion that $\hat{b}_n \notin \PrbBadB$ on the above events), 
  we combine this with the above guarantees for $b^*$ from \eqref{eqn:b-star-vs-target}.
  Specifically, on the event $E_2 \cap E_3 \cap E_4 \cap E_6$,
  for any $b \in \PrbBadB$ and $i \in J_b$ with $\G_{b,i}(S_n^1) \neq \emptyset$,
  for any $i^* \in I^{\target}_{b^*}$, it holds that $\G_{b^*,i^*}(S_n^1) \neq \emptyset$ (established above) and 
  \begin{align*}
    & \min_{g \in \G_{b,i}(S_n^1)} \hat{\er}_{S_n^1}(g) + 4\sqrt{\frac{b\log(n)}{n}}
    \geq \min_{g \in \G_{b,i}(S_n^1)} \ber_n^1(g) 
    \\ & \geq \frac{\epsilon^* \Delta^*}{4 m^*} + \ber_n^1(\target)
    \geq \frac{\epsilon^* \Delta^*}{4 m^*} + \min_{g \in \G_{b^*,i^*}(S_n^1)} \hat{\er}_{S_n^1}(g) - 4\sqrt{\frac{b^*\log(n)}{n}},
   \end{align*}
  where the first inequality is due to $E_4$, 
  the second inequality is due to $E_6$, 
  and the last inequality is from \eqref{eqn:b-star-vs-target}.
  Since the $o(n^{-1/2})$ rate guarantee is asymptotic in nature, 
  we may suppose for the remainder of the proof that 
  $n > 2 b^* \left(\frac{64 m^*}{\epsilon^* \Delta^*}\right)^2 \log\!\left( b^* \left(\frac{64 m^*}{\epsilon^* \Delta^*}\right)^2 \right)$, 
  so that $\frac{\epsilon^* \Delta^*}{4 m^*} > 16\sqrt{\frac{b^*\log(n)}{n}}$.
  In this case, on $E_2 \cap E_3 \cap E_4 \cap E_6$, 
  the above inequality implies that, 
  for any $b \in \PrbBadB$ and $i \in J_b$ with $\G_{b,i}(S_n^1) \neq \emptyset$, for any $i^* \in I^{\target}_{b^*}$, 
  \begin{equation*}
     \min_{g \in \G_{b,i}(S_n^1)} \hat{\er}_{S_n^1}(g)
     > \min_{g \in \G_{b^*,i^*}(S_n^1)} \hat{\er}_{S_n^1}(g) + 8\sqrt{\frac{b^*\log(n)}{n}}.
  \end{equation*}
  Since we have shown above that 
  $|I^{\target}_{b^*}| \geq \frac{9}{10} |I_{b^*}| > 0$ on $E_3$, 
  the above implies that on the event $E_2 \cap E_3 \cap E_4 \cap E_6$,
  every $b \in \PrbBadB$ and $i \in J_b$ 
  have $i \notin \GoodI(b)$.
  Since the event $E_5$ guarantees 
  every $b \in \PrbBadB$ has 
  $|J_b| \geq \frac{1}{5} |I_b|$,
  we conclude that, on the event $\bigcap_{j=2}^{6} E_j$, 
  every $b \in \PrbBadB$ has $|\GoodI(b)| \leq \frac{4}{5} |I_b| < \frac{9}{10} |I_b|$,
  and hence $\hat{b}_n \notin \PrbBadB$.
  This concludes the portion of the proof concerned with the 
  case $\PrbBadB \neq \emptyset$;
  for the remainder of the proof, we return to the general case, 
  allowing $\PrbBadB$ either non-empty or empty (in the latter case letting $E_5$ and $E_6$ denote vacuous events of probability one).

  Since, on the above events, $\hat{b}_n \leq b^*$ and $\hat{b}_n \notin \PrbBadB$,
  let us consider the complementary set.  Consider any $b \in \{1,\ldots,b^*\} \setminus \PrbBadB$.
  By definition, we have that any $i \in I_b$ has 
    $\P\!\left( P \text{ is Bayes-realizable with respect to } \G_{b,i} \right) \geq \frac{7}{10}$. 
  Let
  \begin{equation*}
    I^*_b = \left\{ i \in I_b : P \text{ is Bayes-realizable with respect to } \G_{b,i} \right\}.
  \end{equation*}
  By Hoeffding's inequality, 
  it holds that $|I^*_b| \geq \frac{3}{5} |I_b|$
  with probability at least $1 - e^{-C_4 n}$, 
  where $C_4 = \frac{1}{400 \cdot b^*}$.
  Let $E_7$ denote the event that this occurs simultaneously for all $b \in \{1,\ldots,b^*\} \setminus \PrbBadB$.
  By the union bound, $E_7$ holds with probability at least $1 - b^* \cdot e^{-C_4 n}$.

  Now note that, for each $b \in \{1,\ldots,b^*\} \setminus \PrbBadB$, 
  the set $I^*_b$ and the associated set $\{\G_{b,i} : i \in I^*_b\}$ are determined by $S_n^0$, and hence are independent of $S_n^2$.
  Thus, given $S_n^0$,  
  letting $\hat{f}^{b,i}_{n/4} = \hat{f}^{\G_{b,i}}_{n/4}$ 
  defined in Lemma~\ref{lem:partial-VC-super-root},
  by the guarantee from Lemma~\ref{lem:partial-VC-super-root}, 
  each $i \in I^*_b$ has
  \begin{equation}
    \label{eqn:fhat-b-i-conditional-expectation-bounded-by-phi}
    \E\!\left[ \er_{\PXY}(\hat{f}_{n/4}^{b,i}(S_n^2)) - \er_{\PXY}(\target) \Bigg| S_n^0 \right] \leq \phi(n/4;\G_{b,i},\PXY),
  \end{equation}
  where $\phi(n/4;\G_{b,i},\PXY) \leq c \sqrt{\frac{20 b}{n}}$
  (for a universal constant $c$, recalling that $\VC(\G_{b,i}) \leq 5 b$) 
  and $\phi(n/4;\G_{b,i},\PXY) = o(n^{-1/2})$
  (noting that $\G_{b,i}$ is invariant to $n$, subject to $n \geq 4 b i$).

  Note that the values $\phi(n/4;\G_{b,i},\PXY)$, $i \in I^*_b$,  
  are conditionally i.i.d.\ given $I^*_b$.
  Denote their common conditional expectation by $\phi(n/4;b,\PXY)$:
  namely, for any $i \in I_b$, 
  \begin{equation*}
     \phi(n/4;b,\PXY) := \E\!\left[ \phi(n/4;\G_{b,i},\PXY) \middle| i \in I^*_b \right].
  \end{equation*}
  We will next argue that $\phi(n/4;b,\PXY) = o(n^{-1/2})$.
  Since the above value is identical for every $i$ (which is why the notation $\phi(n/4;b,\PXY)$ omits a dependence on $i$), 
  and since $b \notin \PrbBadB$, 
  we have (taking $i=1$ for simplicity) 
  \begin{equation*}
     \phi(n/4;b,\PXY) 
     \leq \frac{10}{7} \E\!\left[ \phi(n/4;\G_{b,1},\PXY) \ind\!\left[ 1 \in I^*_{b} \right] \right].
  \end{equation*}
  Thus, since $n^{1/2} \cdot \phi(n/4;\G_{b,1},\PXY) \ind\!\left[ 1 \in I^*_{b} \right] \leq c \sqrt{20 b}$ (i.e., it is \emph{bounded} independent of $n$) 
  for all $n \geq 4b$,
  Fatou's lemma implies that  
  \begin{align*}
    \limsup_{n \to \infty}~ n^{1/2} \cdot \phi(n/4;b,\PXY)
    & ~\leq~ \limsup_{n \to \infty}~ \frac{10}{7} \E\!\left[ n^{1/2} \cdot \phi(n/4;\G_{b,1},\PXY) \ind\!\left[ 1 \in I^*_{b} \right] \right]
    \\ & ~\leq~ \frac{10}{7} \E\!\left[ \limsup_{n \to \infty}~ n^{1/2} \cdot \phi(n/4;\G_{b,1},\PXY) \ind\!\left[ 1 \in I^*_{b} \right] \right]
    = 0,
  \end{align*}
  so that we indeed have $\phi(n/4;b,\PXY) = o(n^{-1/2})$.
  Moreover, defining 
  \begin{equation*} 
  \phi(n/4;\PXY) = \sum_{b \in \{1,\ldots,b^*\} \setminus \PrbBadB} \phi(n/4;b,\PXY),
  \end{equation*}
  this immediately implies $\phi(n/4;\PXY) = o(n^{-1/2})$ as well.

  Also note that for any $x \in \X$, we have 
  \begin{equation*} 
  \hat{h}_n^1(x) \neq \target(x) \implies \frac{1}{|I_{\hat{b}_n}|} \sum_{i \in I_{\hat{b}_n}}  \ind\!\left[ \hat{f}_{n/4}^{i}(S_n^2,x) \neq \target(x) \right]  \geq \frac{1}{2}.
  \end{equation*}
  Additionally, 
  letting $\alpha = \frac{|I^*_{\hat{b}_n}|}{|I_{\hat{b}_n}|}$, 
  as argued above, on the event $\bigcap_{j=2}^{7} E_j$,
  we have $\alpha \geq \frac{3}{5}$, 
  which has the further implication that  
  \begin{align*}
      & \frac{1}{|I_{\hat{b}_n}|} \sum_{i \in I_{\hat{b}_n}}  \ind\!\left[ \hat{f}_{n/4}^{i}(S_n^2,x) \neq \target(x) \right]  \geq \frac{1}{2}
      \\ & \implies 
      \frac{1}{|I_{\hat{b}_n}|} \sum_{i \in I^*_{\hat{b}_n}} \ind\!\left[ \hat{f}_{n/4}^{i}(S_n^2,x) \neq \target(x) \right] \geq \frac{1}{2} - (1-\alpha)
      \\ & \implies \frac{1}{|I^*_{\hat{b}_n}|} \sum_{i \in I^*_{\hat{b}_n}} \ind\!\left[ \hat{f}_{n/4}^{i}(S_n^2,x) \neq \target(x) \right] \geq 1 - \frac{1}{2\alpha} \geq \frac{1}{6}.
  \end{align*}
  Thus, on the event $\bigcap_{j=2}^{7} E_j$, 
  for $(X,Y) \sim \PXY$ (independent of $S_n$), 
  \begin{align*}
    & \er_{\PXY}(\hat{h}_n^1) - \er_{\PXY}(\target)
    \\ & = \E\!\left[ \ind\!\left[ \hat{h}_n^1(X) \neq \target(X) \right] \left( 1 - 2 \P(Y \neq \target(X) | X) \right) \middle| S_n \right]
    \\ & \leq \E\!\left[ \ind\!\left[ \frac{1}{|I_{\hat{b}_n}|} \sum_{i \in I_{\hat{b}_n}}  \ind\!\left[ \hat{f}_{n/4}^{i}(S_n^2,X) \neq \target(X) \right]  \geq \frac{1}{2} \right] \left( 1 - 2 \P(Y \neq \target(X) | X) \right) \middle| S_n \right]
   \\ & \leq \E\!\left[ \ind\!\left[ \frac{1}{|I^*_{\hat{b}_n}|} \sum_{i \in I^*_{\hat{b}_n}}  \ind\!\left[ \hat{f}_{n/4}^{i}(S_n^2,X) \neq \target(X) \right] \geq \frac{1}{6} \right] \left( 1 - 2 \P(Y \neq \target(X) | X) \right)\middle| S_n \right]
   \\ & \leq 6 \frac{1}{|I^*_{\hat{b}_n}|} \sum_{i \in I^*_{\hat{b}_n}}  \E\!\left[ \ind\!\left[ \hat{f}_{n/4}^{i}(S_n^2,X) \neq \target(X) \right] \left( 1 - 2 \P(Y \neq \target(X) | X) \right)\middle| S_n \right]
   \\ & = 6 \frac{1}{|I^*_{\hat{b}_n}|} \sum_{i \in I^*_{\hat{b}_n}}  \left( \er_{\PXY}(\hat{f}^i_{n/4}(S_n^2)) - \er_{\PXY}(\target) \right).
   \end{align*}
  Altogether, since the event $\bigcap_{j=2}^{7} E_j$ 
  depends only on $S_n^0$ and $S_n^1$ (i.e., it is $\sigma(S_n^0,S_n^1)$-measurable), 
  we have 
  \begin{align*}
    & \E\!\left[ \left( \er_{\PXY}(\hat{h}_n^1) - \er_{\PXY}(\target) \right) \ind_{\bigcap_{j=2}^{7} E_j} \right]
    \\ & \leq 6 \E\!\left[ \frac{1}{|I^*_{\hat{b}_n}|} \sum_{i \in I^*_{\hat{b}_n}} \E\!\left[ \er_{\PXY}(\hat{f}_{n/4}^{i}(S_n^2)) - \er_{\PXY}(\target) \Bigg| (S_n^0,S_n^1) \right] \ind_{\bigcap_{j=2}^{7} E_j} \right].
  \end{align*}
  Since, on $\bigcap_{j=2}^{7} E_j$, we have 
  $\hat{b}_n \in \{1,\ldots,b^*\} \setminus \PrbBadB$, 
  this last expression is at most 
  \begin{align}  
    & 6 \E\!\left[ \sum_{b \in \{1,\ldots,b^*\} \setminus \PrbBadB} \frac{1}{|I^*_{b}|} \sum_{i \in I^*_{b}} \E\!\left[ \er_{\PXY}(\hat{f}_{n/4}^{b,i}(S_n^2)) - \er_{\PXY}(\target) \Bigg| (S_n^0,S_n^1) \right] \ind_{\bigcap_{j=2}^{7} E_j} \right]
    \notag \\ & = 6 \sum_{b \in \{1,\ldots,b^*\} \setminus \PrbBadB} \E\!\left[ \frac{1}{|I^*_{b}|} \sum_{i \in I^*_{b}} \E\!\left[ \er_{\PXY}(\hat{f}_{n/4}^{b,i}(S_n^2)) - \er_{\PXY}(\target) \Bigg| S_n^0 \right] \ind_{\bigcap_{j=2}^{7} E_j} \right]
    \notag \\ & \leq 6 \sum_{b \in \{1,\ldots,b^*\} \setminus \PrbBadB} \E\!\left[  \frac{1}{|I^*_{b}|} \sum_{i \in I^*_{b}} \phi(n/4;\G_{b,i},\PXY) \ind_{\bigcap_{j=2}^{7} E_j} \right], \label{eqn:sum-over-prb-bad-b}
  \end{align}
  where the second line is based on independence of $\{\hat{f}^{b,i}_{n/4}(S_n^2) : i \in I^*_b\}$ from $S_n^1$, 
  and linearity of expectations, 
  and the final inequality is due to \eqref{eqn:fhat-b-i-conditional-expectation-bounded-by-phi}.
  Since the event $\bigcap_{j=2}^{7} E_j$ 
  implies $I^*_b \neq \emptyset$ for each $b \in \{1,\ldots,b^*\} \setminus \PrbBadB$, 
  each such $b$ satisfies 
  \begin{align*}
     & \E\!\left[  \frac{1}{|I^*_{b}|} \sum_{i \in I^*_{b}} \phi(n/4;\G_{b,i},\PXY) \ind_{\bigcap_{j=2}^{7} E_j} \right]
     \leq \E\!\left[ \ind\!\left[ I^*_b \neq \emptyset \right] \frac{1}{|I^*_{b}|} \sum_{i \in I^*_{b}} \phi(n/4;\G_{b,i},\PXY) \right] = 
     \\ & \E\!\left[  \ind\!\left[ I^*_b \!\neq\! \emptyset \right] \E\!\left[ \frac{1}{|I^*_{b}|} \sum_{i \in I^*_{b}} \phi(n/4;\G_{b,i},\PXY) \middle| I^*_b \right] \right]
     = \E\!\left[  \ind\!\left[ I^*_b \!\neq\! \emptyset \right] \frac{1}{|I^*_{b}|} \sum_{i \in I^*_{b}} \E\!\left[ \phi(n/4;\G_{b,i},\PXY) \middle| I^*_b \right] \right]\!. 
  \end{align*}
  Due to the independence of the $\G_{b,i}$ sets
  ($i \in I_b$), 
  each $i \in I^*_b$ has 
  $\E[\phi(n/4;\G_{b,i},\PXY)| I^*_b] = \E[\phi(n/4;\G_{b,i},\PXY)| i \in I^*_b] = \phi(n/4;b,\PXY)$.
  Thus, for each $b \in \{1,\ldots,b^*\} \setminus \PrbBadB$, 
  the rightmost expression above equals
  \begin{equation*}
     \E\!\left[ \ind\!\left[ I^*_b \neq \emptyset \right] \frac{1}{|I^*_{b}|} \sum_{i \in I^*_{b}} \phi(n/4;b,\PXY) \right]
     \leq \phi(n/4;b,\PXY).
  \end{equation*}
  Plugging back into \eqref{eqn:sum-over-prb-bad-b}, 
  we conclude that 
  \begin{equation*}
  \E\!\left[ \left( \er_{\PXY}(\hat{h}_n^1) - \er_{\PXY}(\target) \right) \ind_{\bigcap_{j=2}^{7} E_j} \right]
      \leq 6 \sum_{b \in \{1,\ldots,b^*\} \setminus \PrbBadB} \phi(n/4;b,\PXY) = 6 \cdot \phi(n/4;\PXY).
  \end{equation*}
  
  It remains to bound $\E\!\left[\er_{\PXY}(\hat{h}_n) - \er_{\PXY}(\target) \right]$ in terms of $\E\!\left[\left( \er_{\PXY}(\hat{h}_n^1) - \er_{\PXY}(\target) \right) \ind_{\bigcap_{j=2}^{7} E_j} \right]$.
  On the event $E_1$, if $\hat{h}_n^1$ is defined and $\er_{\PXY}(\hat{h}_n^1) \leq \er_{\PXY}(\hat{h}_n^0)$, then
  \begin{equation*}
    \hat{\er}_{S_n^3}(\hat{h}_n^1)
    \leq \er_{\PXY}(\hat{h}_n^1) + \sqrt{\frac{2\ln(n)}{n}}
    \leq \er_{\PXY}(\hat{h}_n^0) + \sqrt{\frac{2\ln(n)}{n}}
    \leq \hat{\er}_{S_n^3}(\hat{h}_n^0) + \sqrt{\frac{8\ln(n)}{n}},
  \end{equation*}
  so that the algorithm chooses $\hat{h}_n = \hat{h}_n^1$.
  On the other hand, if $\er_{\PXY}(\hat{h}_n^1) > \er_{\PXY}(\hat{h}_n^0)$,
  then regardless of whether $\hat{h}_n = \hat{h}_n^0$ or $\hat{h}_n = \hat{h}_n^1$,
  we have $\er_{\PXY}(\hat{h}_n) \leq \er_{\PXY}(\hat{h}_n^1)$.
  Thus, since we have argued above that, on the event $E_2 \cap E_3 \cap E_4$, 
  $\hat{h}_n^1$ is always defined, 
  we conclude that on the event $\bigcap_{j=1}^{4} E_j$, 
  we have 
  $\er_{\PXY}(\hat{h}_n) \leq \er_{\PXY}(\hat{h}_n^1)$.
  Since $\er_{\PXY}(\hat{h}_n) - \er_{\PXY}(\target) \leq 1$,
  this implies 
  \begin{align*}
    & \E\!\left[\er_{\PXY}(\hat{h}_n)\right] - \inf_{h \in \H} \er_{\PXY}(h) 
    = \E\!\left[\er_{\PXY}(\hat{h}_n)\right] - \er_{\PXY}(\target)
    \\ & \leq \E\!\left[\left( \er_{\PXY}(\hat{h}_n^1) - \er_{\PXY}(\target) \right) \ind_{\bigcap_{j=1}^{7} E_j} \right] + \left( 1 - \P\!\left( \bigcap_{j=1}^{7} E_j \right) \right)
    \\ & \leq 6 \cdot \phi(n/4;\PXY) + \frac{4}{n} + e^{-C_1 n}\! + \frac{1}{n} + (b^*-1)e^{-C_2 n}\! + (b^*-1) \frac{n}{4} e^{-C_3 n}\! + b^* e^{-C_4 n}
   = o\!\left(n^{-1/2}\right),
  \end{align*}
  where the second inequality is by the union bound.
  This completes the proof.
\end{proof}

\section{Conclusions and Future Directions}
\label{sec:conclusions}

We have established a complete theory of optimal universal rates for binary classification in the agnostic setting.
The main result (Theorem~\ref{thm:agnostic-char})
establishes a fundamental trichotomy (or tetrachotomy, if we count finite classes).
Finite concept classes $\H$ exhibit an optimal rate $e^{-n}$.
Any infinite concept class which does not shatter an infinite Littlestone tree exhibits an optimal rate $e^{-o(n)}$,
while any concept class which shatters an infinite concept class but does not shatter an infinite VCL tree exhibits an optimal rate $o(n^{-1/2})$,
and otherwise the class requires arbitrarily slow rates.
This extends the realizable-case theory of \citet*{bousquet:21}
by removing the realizability assumption on the distribution.

As the realizable-case theory of \citet*{bousquet:21}
has to-date been extended to numerous other learning settings beyond supervised binary classification
\citep*{kalavasis:22,hanneke:22b,hanneke:23b,bousquet:23,attias:24,hanneke:24b,hanneke:25b},
a natural future direction for the study of universal rates is to pursue analogous extensions of the theory of agnostic universal rates to these settings.

\subsection*{Acknowledgments}
Shay Moran is a Robert J.\ Shillman Fellow; he acknowledges support by ISF grant 1225/20, by BSF grant 2018385, by Israel PBC-VATAT, by the Technion Center for Machine Learning and Intelligent Systems (MLIS), and by the the European Union (ERC, GENERALIZATION, 101039692). Views and opinions expressed are however those of the author(s) only and do not necessarily reflect those of the European Union or the European Research Council Executive Agency. Neither the European Union nor the granting authority can be held responsible for them.

\bibliography{learning}

\begin{thebibliography}{66}
\providecommand{\natexlab}[1]{#1}
\providecommand{\url}[1]{\texttt{#1}}
\expandafter\ifx\csname urlstyle\endcsname\relax
  \providecommand{\doi}[1]{doi: #1}\else
  \providecommand{\doi}{doi: \begingroup \urlstyle{rm}\Url}\fi

\bibitem[Alon et~al.(2021)Alon, Hanneke, Holzman, and Moran]{alon:21}
N.~Alon, S.~Hanneke, R.~Holzman, and S.~Moran.
\newblock A theory of {PAC} learnability of partial concept classes.
\newblock In \emph{Proceedings of the $62^{\mathrm{nd}}$ Annual Symposium on
  Foundations of Computer Science}, 2021.

\bibitem[Anthony and Bartlett(1999)]{anthony:99}
M.~Anthony and P.~L. Bartlett.
\newblock \emph{Neural Network Learning: Theoretical Foundations}.
\newblock Cambridge University Press, 1999.

\bibitem[Antos and Lugosi(1998)]{antos:98}
A.~Antos and G.~Lugosi.
\newblock Strong minimax lower bounds for learning.
\newblock \emph{Machine Learning}, 30:\penalty0 31--56, 1998.

\bibitem[Attias et~al.(2024)Attias, Hanneke, Kalavasis, Karbasi, and
  Velegkas]{attias:24}
I.~Attias, S.~Hanneke, A.~Kalavasis, A.~Karbasi, and G.~Velegkas.
\newblock Universal rates for regression: {S}eparations between cut-off and
  absolute loss.
\newblock In \emph{Proceedings of $37^{\mathrm{th}}$ Annual Conference on
  Learning Theory}, 2024.

\bibitem[Balcan et~al.(2010)Balcan, Hanneke, and Vaughan]{hanneke:10a}
M.-F. Balcan, S.~Hanneke, and J.~Wortman Vaughan.
\newblock The true sample complexity of active learning.
\newblock \emph{Machine Learning}, 80\penalty0 (2--3):\penalty0 111--139, 2010.

\bibitem[Bartlett and Long(1998)]{bartlett:98}
P.~L. Bartlett and P.~M. Long.
\newblock Prediction, learning, uniform convergence, and scale-sensitive
  dimensions.
\newblock \emph{Journal of Computer and System Sciences}, 56\penalty0
  (2):\penalty0 174--190, 1998.

\bibitem[Bartlett et~al.(2004)Bartlett, Mendelson, and Philips]{bartlett:04}
P.~L. Bartlett, S.~Mendelson, and P.~Philips.
\newblock Local complexities for empirical risk minimization.
\newblock In \emph{Proceedings of the $17^{{\rm th}}$ Conference on Learning
  Theory}, 2004.

\bibitem[Bartlett et~al.(2005)Bartlett, Bousquet, and Mendelson]{bartlett:05}
P.~L. Bartlett, O.~Bousquet, and S.~Mendelson.
\newblock Local {R}ademacher complexities.
\newblock \emph{The Annals of Statistics}, 33\penalty0 (4):\penalty0
  1497--1537, 2005.

\bibitem[Barve and Long()]{barve:97}
R.~D. Barve and P.~M. Long.
\newblock On the complexity of learning from drifting distributions.
\newblock \emph{Information and Computation}, 138\penalty0 (2):\penalty0
  170–--193.

\bibitem[Ben-David et~al.(1995)Ben-David, Benedek, and Mansour]{bendavid:95}
S.~Ben-David, G.~M. Benedek, and Y.~Mansour.
\newblock A parameterization scheme for classifying models of {PAC}
  learnability.
\newblock \emph{Information and Computation}, 120\penalty0 (1):\penalty0
  11--21, 1995.

\bibitem[Benedek and Itai(1994)]{benedek:94}
G.~M. Benedek and A.~Itai.
\newblock Nonuniform learnability.
\newblock \emph{Journal of Computer and System Sciences}, 48:\penalty0
  311--323, 1994.

\bibitem[Boucheron et~al.(2013)Boucheron, Lugosi, and Massart]{boucheron:13}
S.~Boucheron, G.~Lugosi, and P.~Massart.
\newblock \emph{Concentration Inequalities: {A} Nonasymptotic Theory of
  Independence}.
\newblock Oxford University Press, 2013.

\bibitem[Bousquet(2002)]{bousquet:02}
O.~Bousquet.
\newblock A {B}ennett concentration inequality and its application to suprema
  of empirical processes.
\newblock \emph{Comptes Rendus Mathematique}, 334\penalty0 (6):\penalty0
  495--500, 2002.

\bibitem[Bousquet et~al.(2021)Bousquet, Hanneke, Moran, {van Handel}, and
  Yehudayoff]{bousquet:21}
O.~Bousquet, S.~Hanneke, S.~Moran, R.~{van Handel}, and A.~Yehudayoff.
\newblock A theory of universal learning.
\newblock In \emph{Proceedings of the $53^{\mathrm{rd}}$ Annual {ACM} Symposium
  on Theory of Computing}, 2021.
\newblock URL \url{https://arxiv.org/abs/2011.04483}.

\bibitem[Bousquet et~al.(2023)Bousquet, Hanneke, Moran, Shafer, and
  Tolstikhin]{bousquet:23}
O.~Bousquet, S.~Hanneke, S.~Moran, J.~Shafer, and I.~Tolstikhin.
\newblock Fine-grained distribution-dependent learning curves.
\newblock In \emph{Proceedings of the $36^{\mathrm{th}}$ Annual Conference on
  Learning Theory}, 2023.

\bibitem[Castelli and Cover(1996)]{castelli:96}
V.~Castelli and T.~M. Cover.
\newblock The relative value of labeled and unlabeled samples in pattern
  recognition with an unknown mixing parameter.
\newblock \emph{{IEEE} Transactions on Information Theory}, 42\penalty0 (6 Part
  2):\penalty0 2102--2117, 1996.

\bibitem[Devroye et~al.(1996)Devroye, Gy\"{o}rfi, and Lugosi]{devroye:96}
L.~Devroye, L.~Gy\"{o}rfi, and G.~Lugosi.
\newblock \emph{A Probabilistic Theory of Pattern Recognition}.
\newblock Springer-Verlag New York, Inc., 1996.

\bibitem[Dudley(2002)]{dudley:02}
R.~M. Dudley.
\newblock \emph{Real Analysis and Probability}.
\newblock Cambridge University Press, 2002.

\bibitem[Dudley(2014)]{Dud14}
R.~M. Dudley.
\newblock \emph{Uniform central limit theorems}, volume 142 of \emph{Cambridge
  Studies in Advanced Mathematics}.
\newblock Cambridge University Press, New York, second edition, 2014.
\newblock ISBN 978-0-521-73841-5; 978-0-521-49884-5.

\bibitem[Dudley et~al.(1991)Dudley, Gin\'{e}, and Zinn]{dudley:91}
R.~M. Dudley, E.~Gin\'{e}, and J.~Zinn.
\newblock Uniform and universal {G}livenko-{C}antelli classes.
\newblock \emph{Journal of Theoretical Probability}, 4\penalty0 (3), 1991.

\bibitem[El-Yaniv and Pechyony(2009)]{el-yaniv:09}
R.~El-Yaniv and D.~Pechyony.
\newblock Transductive {R}ademacher complexity and its applications.
\newblock \emph{Journal of Artificial Intelligence Research}, 35:\penalty0
  193--234, 2009.

\bibitem[Fisher(1924)]{fisher1924distribution}
R.~A. Fisher.
\newblock On a distribution yielding the error functions of several well known
  statistics.
\newblock In \emph{Proceedings International Mathematical Congress, Toronto},
  volume~2, pages 805--813, 1924.

\bibitem[Foster et~al.(2021)Foster, Rakhlin, {Simchi-Levi}, and Xu]{foster:21}
D.~J. Foster, A.~Rakhlin, D.~{Simchi-Levi}, and Y.~Xu.
\newblock Instance-dependent complexity of contextual bandits and reinforcement
  learning: {A} disagreement-based perspective.
\newblock In \emph{Proceedings of the $34^{\mathrm{th}}$ Conference on Learning
  Theory}, 2021.

\bibitem[Gin\'{e} and Koltchinskii(2006)]{gine:06}
E.~Gin\'{e} and V.~Koltchinskii.
\newblock Concentration inequalities and asymptotic results for ratio type
  empirical processes.
\newblock \emph{The Annals of Probability}, 34\penalty0 (3):\penalty0
  1143--1216, 2006.

\bibitem[Hanneke(2012)]{hanneke:12a}
S.~Hanneke.
\newblock Activized learning: Transforming passive to active with improved
  label complexity.
\newblock \emph{Journal of Machine Learning Research}, 13\penalty0
  (5):\penalty0 1469--1587, 2012.

\bibitem[Hanneke(2021)]{hanneke:21a}
S.~Hanneke.
\newblock Learning whenever learning is possible: {U}niversal learning under
  general stochastic processes.
\newblock \emph{Journal of Machine Learning Research}, 22\penalty0 (130), 2021.

\bibitem[Hanneke(2024)]{hanneke:24a}
S.~Hanneke.
\newblock The eluder dimension and star number: {E}lementary observations about
  the dimensions of disagreement.
\newblock In \emph{Proceedings of $37^{\mathrm{th}}$ Annual Conference on
  Learning Theory}, 2024.

\bibitem[Hanneke and Kpotufe(2022)]{hanneke:22a}
S.~Hanneke and S.~Kpotufe.
\newblock A no-free-lunch theorem for multitask learning.
\newblock \emph{The Annals of Statistics}, 50\penalty0 (6):\penalty0
  3119--3143, 2022.

\bibitem[Hanneke and Xu(2024)]{hanneke:24c}
S.~Hanneke and M.~Xu.
\newblock Universal rates of empirical risk minimization.
\newblock In \emph{Advances in Neural Information Processing Systems 37}, 2024.

\bibitem[Hanneke and Xu(2025)]{hanneke:25a}
S.~Hanneke and M.~Xu.
\newblock Universal rates of empirical risk minimization in the agnostic case.
\newblock In \emph{Proceedings of the $38^{\mathrm{th}}$ Annual Conference on
  Learning Theory}, 2025.

\bibitem[Hanneke and Yang(2019)]{hanneke:19a}
S.~Hanneke and L.~Yang.
\newblock Statistical learning under nonstationary mixing processes.
\newblock In \emph{Proceedings of the $22^{\mathrm{nd}}$ International
  Conference on Artificial Intelligence and Statistics}, 2019.

\bibitem[Hanneke et~al.(2021)Hanneke, Kontorovich, Sabato, and
  Weiss]{hanneke:21b}
S.~Hanneke, A.~Kontorovich, S.~Sabato, and R.~Weiss.
\newblock Universal {B}ayes consistency in metric spaces.
\newblock \emph{The Annals of Statistics}, 49\penalty0 (4):\penalty0
  2129--2150, 2021.

\bibitem[Hanneke et~al.(2022)Hanneke, Karbasi, Moran, and
  Velegkas]{hanneke:22b}
S.~Hanneke, A.~Karbasi, S.~Moran, and G.~Velegkas.
\newblock Universal rates for interactive learning.
\newblock In \emph{Advances in Neural Information Processing Systems 35}, 2022.

\bibitem[Hanneke et~al.(2023)Hanneke, Moran, and Zhang]{hanneke:23b}
S.~Hanneke, S.~Moran, and Q.~Zhang.
\newblock Universal rates for multiclass learning.
\newblock In \emph{Proceedings of the $36^{\mathrm{th}}$ Annual Conference on
  Learning Theory}, 2023.

\bibitem[Hanneke et~al.(2024)Hanneke, Karbasi, Moran, and
  Velegkas]{hanneke:24b}
S.~Hanneke, A.~Karbasi, S.~Moran, and G.~Velegkas.
\newblock Universal rates for active learning.
\newblock In \emph{Advances in Neural Information Processing Systems 37}, 2024.

\bibitem[Hanneke et~al.(2025)Hanneke, Shaeiri, and Zhang]{hanneke:25b}
S.~Hanneke, A.~Shaeiri, and Q.~Zhang.
\newblock Universal rates for multiclass learning with bandit feedback.
\newblock In \emph{Proceedings of the $38^{\mathrm{th}}$ Annual Conference on
  Learning Theory}, 2025.

\bibitem[Haussler(1995)]{haussler:95}
D.~Haussler.
\newblock Sphere packing numbers for subsets of the {B}oolean n-cube with
  bounded {V}apnik-{C}hervonenkis dimension.
\newblock \emph{Journal of Combinatorial Theory A}, 69\penalty0 (2):\penalty0
  217--232, 1995.

\bibitem[Haussler et~al.(1994)Haussler, Littlestone, and Warmuth]{haussler:94}
D.~Haussler, N.~Littlestone, and M.~Warmuth.
\newblock Predicting $\{0,1\}$-functions on randomly drawn points.
\newblock \emph{Information and Computation}, 115\penalty0 (2):\penalty0
  248--292, 1994.

\bibitem[Hoeffding(1953)]{hoeffding:53}
W.~Hoeffding.
\newblock A lower bound for the average sample number of a sequential test.
\newblock \emph{The Annals of Mathematical Statistics}, 24\penalty0
  (1):\penalty0 127--130, 1953.

\bibitem[Hopkins et~al.(2022)Hopkins, Kane, Lovett, and Mahajan]{hopkins:22}
M.~Hopkins, D.~M. Kane, S.~Lovett, and G.~Mahajan.
\newblock Realizable learning is all you need.
\newblock In \emph{Proceedings of $35^{\mathrm{th}}$ Annual Conference on
  Learning Theory}, 2022.

\bibitem[Kalavasis et~al.(2022)Kalavasis, Velegkas, and Karbasi]{kalavasis:22}
A.~Kalavasis, G.~Velegkas, and A.~Karbasi.
\newblock Multiclass learnability beyond the {PAC} framework: {U}niversal rates
  and partial concept classes.
\newblock In \emph{Advances in Neural Information Processing Systems 35}, 2022.

\bibitem[Kearns et~al.(1994)Kearns, Schapire, and Sellie]{kearns:94a}
M.~J. Kearns, R.~E. Schapire, and L.~M. Sellie.
\newblock Toward efficient agnostic learning.
\newblock \emph{Machine Learning}, 17:\penalty0 115--141, 1994.

\bibitem[Klein and Rio(2005)]{klein:05}
T.~Klein and E.~Rio.
\newblock Concentration around the mean for maxima of empirical processes.
\newblock \emph{The Annals of Probability}, 33\penalty0 (3):\penalty0
  1060--1077, 2005.

\bibitem[Koltchinskii(2006)]{koltchinskii:06}
V.~Koltchinskii.
\newblock Local {R}ademacher complexities and oracle inequalities in risk
  minimization.
\newblock \emph{The Annals of Statistics}, 34\penalty0 (6):\penalty0
  2593--2656, 2006.

\bibitem[Kpotufe and Orabona(2013)]{kpotufe:13}
S.~Kpotufe and F.~Orabona.
\newblock Regression-tree tuning in a streaming setting.
\newblock In \emph{Advances in Neural Information Processing Sytems 26}, 2013.

\bibitem[{Le Cam}(1952)]{lecam:52}
L.~M. {Le Cam}.
\newblock \emph{On Some Asymptotic Properties of Maximum Likelihood Estimates
  and Related {B}ayes' Estimates}.
\newblock PhD thesis, University of California at Berkeley, 1952.

\bibitem[Littlestone(1988)]{littlestone:88}
N.~Littlestone.
\newblock Learning quickly when irrelevant attributes abound: A new
  linear-threshold algorithm.
\newblock \emph{Machine Learning}, 2:\penalty0 285--318, 1988.

\bibitem[Long(2001)]{long:01}
P.~M. Long.
\newblock On agnostic learning with $\{0,*,1\}$-valued and real-valued
  hypotheses.
\newblock In \emph{Proceedings of the $14^{\mathrm{th}}$ Annual Conference on
  Computational Learning Theory and $5^{\mathrm{th}}$ European Conference on
  Computational Learning Theory}, 2001.

\bibitem[Massart(2007)]{massart2007concentration}
P.~Massart.
\newblock \emph{Concentration Inequalities and Model Selection: {E}cole
  d'Et{\'e} de Probabilit{\'e}s de {S}aint-{F}lour {XXXIII}-2003}.
\newblock Springer, 2007.

\bibitem[Massart and N\'{e}d\'{e}lec(2006)]{massart:06}
P.~Massart and E.~N\'{e}d\'{e}lec.
\newblock Risk bounds for statistical learning.
\newblock \emph{The Annals of Statistics}, 34\penalty0 (5):\penalty0
  2326--2366, 2006.

\bibitem[Mohri and {n}oz Medina(2012)]{mohri:12}
M.~Mohri and A.~Mu\ {n}oz Medina.
\newblock New analysis and algorithm for learning with drifting distributions.
\newblock In \emph{Proceedings of The $23^{\mathrm{rd}}$ International
  Conference on Algorithmic Learning Theory}, 2012.

\bibitem[Munkres(2000)]{munkres:00}
J.~R. Munkres.
\newblock \emph{Topology}.
\newblock Prentice Hall, Inc., 2nd edition, 2000.

\bibitem[Russo and {Van Roy}(2013)]{russo:13}
D.~Russo and B.~{Van Roy}.
\newblock Eluder dimension and the sample complexity of optimistic exploration.
\newblock In \emph{Advances in Neural Information Processing Systems 26}, 2013.

\bibitem[Sauer(1972)]{sauer:72}
N.~Sauer.
\newblock On the density of families of sets.
\newblock \emph{Journal of Combinatorial Theory (A)}, 13\penalty0 (1):\penalty0
  145--147, 1972.

\bibitem[Schuurmans(1997)]{schuurmans:97}
D.~Schuurmans.
\newblock Characterizing rational versus exponential learning curves.
\newblock \emph{Journal of Computer and System Sciences}, 55\penalty0
  (1):\penalty0 140--160, 1997.

\bibitem[Talagrand(1994)]{talagrand:94}
M.~Talagrand.
\newblock Sharper bounds for {G}aussian and empirical processes.
\newblock \emph{The Annals of Probability}, 22\penalty0 (1):\penalty0 28--76,
  1994.

\bibitem[Tolstikhin et~al.(2014)Tolstikhin, Blanchard, and
  Kloft]{tolstikhin:14}
I.~Tolstikhin, G.~Blanchard, and M.~Kloft.
\newblock Localized complexities for transductive learning.
\newblock In \emph{Proceedings of The $27^{\mathrm{th}}$ Conference on Learning
  Theory}, 2014.

\bibitem[van~der Vaart and Wellner(2011)]{van-der-Vaart:11}
A.~van~der Vaart and J.~A. Wellner.
\newblock A local maximal inequality under uniform entropy.
\newblock \emph{Electronic Journal of Statistics}, 5:\penalty0 192--203, 2011.

\bibitem[van~der Vaart and Wellner(1996)]{van-der-Vaart:96}
A.~W. van~der Vaart and J.~A. Wellner.
\newblock \emph{Weak Convergence and Empirical Processes}.
\newblock Springer, 1996.

\bibitem[van Handel(2013)]{van-handel:13}
R.~van Handel.
\newblock The universal {G}livenko-{C}antelli property.
\newblock \emph{Probability and Related Fields}, 155:\penalty0 911--934, 2013.

\bibitem[Vapnik(1998)]{vapnik:98}
V.~Vapnik.
\newblock \emph{Statistical Learning Theory}.
\newblock John Wiley $\&$ Sons, Inc., 1998.

\bibitem[Vapnik and Chervonenkis(1971)]{vapnik:71}
V.~Vapnik and A.~Chervonenkis.
\newblock On the uniform convergence of relative frequencies of events to their
  probabilities.
\newblock \emph{Theory of Probability and its Applications}, 16\penalty0
  (2):\penalty0 264--280, 1971.

\bibitem[Vapnik and Chervonenkis(1974)]{vapnik:74}
V.~Vapnik and A.~Chervonenkis.
\newblock \emph{Theory of Pattern Recognition}.
\newblock Nauka, Moscow, 1974.

\bibitem[Vidyasagar(2003)]{vidyasagar:03}
M.~Vidyasagar.
\newblock \emph{Learning and Generalization with Applications to Neural
  Networks}.
\newblock Springer-Verlag, $2^{{\rm nd}}$ edition, 2003.

\bibitem[Wilks(1938)]{wilks:38}
S.~S. Wilks.
\newblock The large-sample distribution of the likelihood ratio for testing
  composite hypotheses.
\newblock \emph{The Annals of Mathematical Statistics}, 9\penalty0
  (1):\penalty0 60–--62, 1938.

\bibitem[Yang and Hanneke(2013)]{hanneke:13}
L.~Yang and S.~Hanneke.
\newblock Activized learning with uniform classification noise.
\newblock In \emph{Proceedings of the $30^{{\rm th}}$ International Conference
  on Machine Learning}, 2013.

\end{thebibliography}

\end{document}